%% file: main.tex
\documentclass[10pt,twocolumn,letterpaper]{article}

\usepackage{cvpr}              
\input{preamble}
\definecolor{cvprblue}{rgb}{0.21,0.49,0.74}
\usepackage[pagebackref,breaklinks,colorlinks,allcolors=cvprblue]{hyperref}

\usepackage{tikz}
\usepackage{graphicx}
\usepackage{amsmath}
\usepackage{adjustbox}
\usepackage{color, colortbl, xcolor}
\usepackage{subcaption}
\usepackage{multicol}
\usepackage{multirow}
\usepackage{arydshln}
\usepackage{url}

\def\paperID{9337} 
\def\confName{CVPR}
\def\confYear{2026}

\definecolor{LightGray}{HTML}{F0F1F2}
\newenvironment{proof}{{\noindent\it Proof.}\quad}{\hfill $\square$\par}
\newtheorem{proposition}{Proposition}

\newcommand*\circled[1]{\tikz[baseline=(char.base)]{\node[shape=circle,draw,inner sep=1pt,scale=1.0] (char) {#1};}}
\newcommand{\ours}{ASCOOD}

\input{math_commands}

\title{Image-based Outlier Synthesis With Training Data}

\author{Sudarshan Regmi\\
Department of Computer Science\\
Dartmouth College\\
{\tt\small sudarshan.regmi.gr@dartmouth.edu}
}

\begin{document}
\maketitle
\input{sec/0_abstract}    
\input{sec/1_intro}
\input{sec/2_prelim}
\input{sec/3_method}
\input{sec/4_experiments}
\input{sec/5_results}
\input{sec/6_related}
\input{sec/7_conclusion}
{
    \small
    \bibliographystyle{ieeenat_fullname}
    \bibliography{main}
}

\newpage
\input{sec/X_suppl}
\end{document}

%% file: math_commands.tex
\usepackage{amsmath,amsfonts,bm}

\def\eqref#1{equation~\ref{#1}}

\def\1{\bm{1}}

\def\vp{{\bm{p}}}

\def\vy{{\bm{y}}}

\DeclareMathAlphabet{\mathsfit}{\encodingdefault}{\sfdefault}{m}{sl}
\SetMathAlphabet{\mathsfit}{bold}{\encodingdefault}{\sfdefault}{bx}{n}



%% file: sec/0_abstract.tex
\begin{abstract}
Out-of-distribution (OOD) detection is critical to ensure the safe deployment of deep learning models in critical applications. Deep learning models can often misidentify OOD samples as in-distribution (ID) samples. This vulnerability worsens in the presence of spurious correlation in the training set. Likewise, in fine-grained classification settings, detection of fine-grained OOD samples becomes inherently challenging due to their high similarity to ID samples. However, current research on OOD detection has focused instead largely on relatively easier (conventional) cases. Even the few recent works addressing these challenging cases rely on carefully curated or synthesized outliers, ultimately requiring external data. This motivates our central research question: ``Can we innovate OOD detection training framework for fine-grained and spurious settings \textbf{without requiring any external data at all?}" In this work, we present a unified \textbf{A}pproach to \textbf{S}purious, fine-grained, and \textbf{C}onventional \textbf{OOD D}etection (\textbf{\ours}) that eliminates the reliance on external data. First, we synthesize virtual outliers from ID data by approximating the destruction of invariant features. Specifically, we propose to add gradient attribution values to ID inputs to disrupt invariant features while amplifying true-class logit, thereby synthesizing challenging near-manifold virtual outliers. Then, we simultaneously incentivize ID classification and predictive uncertainty towards virtual outliers. For this, we further propose to leverage standardized features with z-score normalization. \ours~effectively mitigates impact of spurious correlations and encourages capturing fine-grained attributes. Extensive experiments across \textbf{7} datasets and and comparisons with \textbf{30+} methods demonstrate merit of \ours~in spurious, fine-grained and conventional settings.
\end{abstract}

%% file: sec/1_intro.tex
\section{Introduction}
\label{sec:intro}
\begin{figure}[!t]
  \centering
  \includegraphics[width=0.85\linewidth]{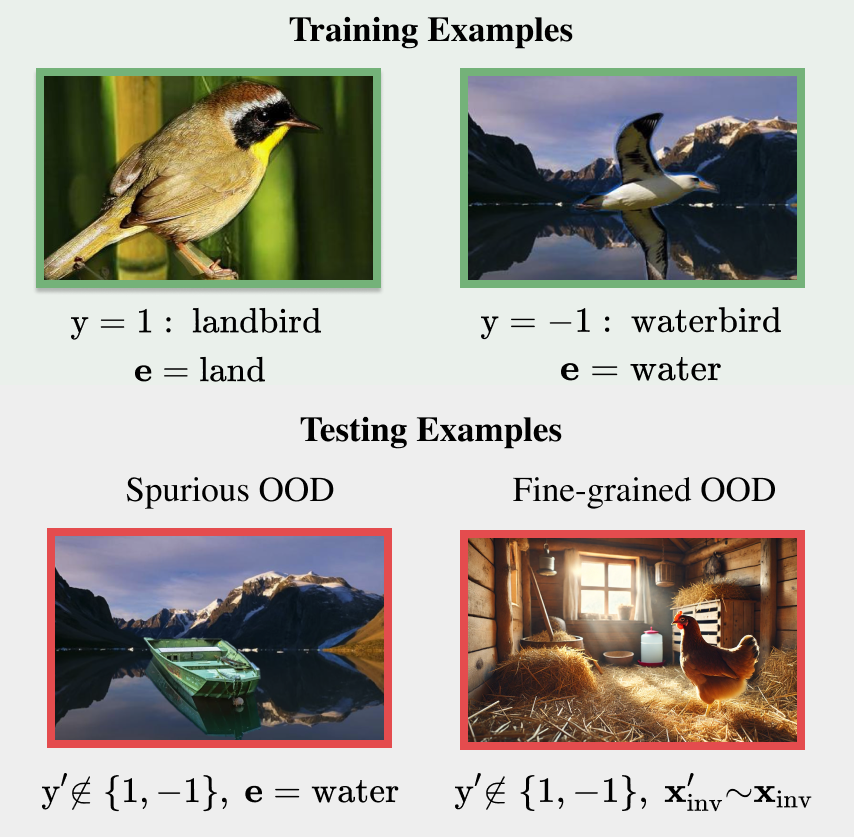}
  \caption{In Waterbirds dataset~\cite{Sagawa2020Distributionally}, label \( y \in \{\text{waterbird, landbird}\} \) is correlated with environmental feature \( \mathbf{e} \in \{\text{water, land}\} \). Spurious OOD retains environmental feature \(\mathbf{e}\) (water) while fine-grained OOD has its invariant feature similar to ID invariant feature \((\mathbf{x}'_\text{inv} \sim \mathbf{x}_\text{inv})\). Both present significant challenges for OOD detection.}
  \label{fig:teaser}
\end{figure}
Deploying deep learning models, trained under the \textit{closed-world} assumption $\left( \mathbb{D}_\text{train} = \mathbb{D}_\text{test}\right)$, often becomes challenging in real-world scenarios as they frequently encounter OOD inputs. OOD inputs should be accurately flagged as they lie beyond the training distribution. Such identification of OOD inputs becomes challenging if models rely on spurious features that do not generalize beyond the training distribution~\citep{ming2021impact}. For instance, a medical diagnosis model might erroneously rely on spurious features, such as image artifacts, leading it to incorrectly classify any image containing such artifacts as ID sample. Similarly, in fine-grained scenarios~\cite{yang2021re,hsu2019fine,zheng2017learning,chen2025openinsect} like species classification, novel species visually similar to known ones can easily be misidentified as ID species. Proper consideration of these scenarios while ensuring high ID accuracy is essential for safe deployment of deep learning models.
\begin{figure*}[!t]
    \centering
    \includegraphics[width=0.9\textwidth]{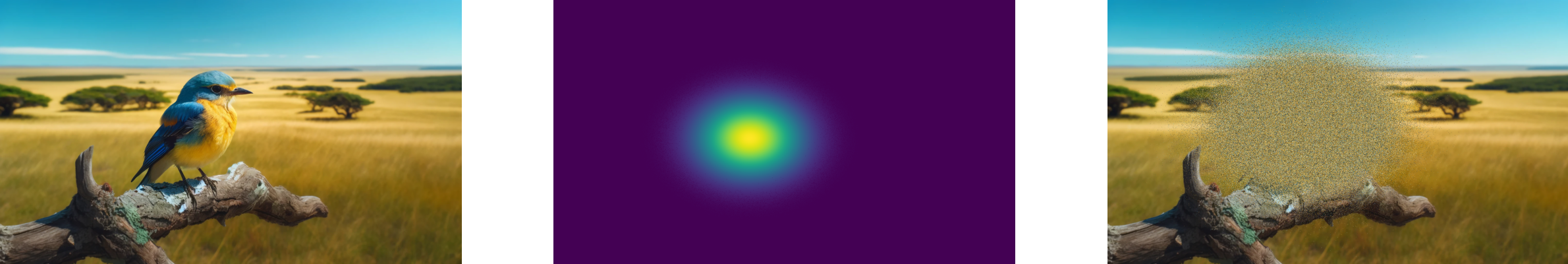}
    \caption{Motivating example of the outlier synthesis pipeline. \textbf{Left:} An image $\mathbf{x} = \psi(\mathbf{x}_{\text{inv}}, \mathbf{e}) \in \mathcal{X}$ from the in-distribution dataset $\mathbb{D}_\text{in}$ is shown. \textbf{Middle:} A 2D distribution $\mathcal{G}_\text{oracle}$ is shown which signifies the presence of invariant feature in a smaller region of the image $\mathbf{x}$. \textbf{Right:} Corresponding outlier $\mathbf{x}^{\prime}$ is shown, which is formed by destroying the invariant feature $\mathbf{x}_\text{inv}$ of $\mathbf{x}$ through a perturbation function $\mathcal{P}_F$, having access to $\mathcal{G}_\text{oracle}$. \textit{Can we synthesize similar virtual outlier $\mathbf{x}'$ without the access of $\mathcal{G}_\text{oracle}$?}}
    \label{fig:main_fig}
\end{figure*}

Images generally consist of both \textit{invariant} and \textit{environmental} features. As shown in Figure \ref{fig:teaser}, when the correlations between environmental features (land and water) and corresponding target labels (landbird, waterbird) are high, neural networks can rely on spurious features to achieve high classification performance~\citep{beery2018recognition,Sagawa2020Distributionally}. It can cause the model to incorrectly make high-confidence predictions for OOD samples with similar environmental features but different semantic content. Moreover, in fine-grained classification settings, the degree of distinction between ID and OOD samples may be as subtle as that between different ID classes. As illustrated in Figure \ref{fig:teaser}, fine-grained OOD for the Waterbirds ID dataset may be ``hen", which differs from ID samples based on subtle fine-grained attributes (Also, see \ref{sec:fine_grained}). Moreover, the overlap of high-level feature sets between fine-grained OOD and ID data complicates the detection of the former. Real-world scenarios frequently involve either spurious or fine-grained settings as ID and OOD are often captured under similar conditions during deployment, highlighting the importance of study under such settings.

A significant majority of OOD detection studies, including recent ones~\citep{du2023dream,liu2024can,doorenbos2024nonlinearoutliersynthesisoutofdistribution,liang2025revisiting,fang2024kpcaood,yang2025oodd,Karunanayake_2025_WACV}, restrict their studies to conventional cases. While few works~\cite{mixoe23wacv,techapanurak2021practical,perera2019deep,shinohara2025logit} study fine-grained OOD detection, they often require the curation of diverse outliers non-overlapping with ID data~\citep{yao2024outofdistribution,jiang2024dos,zhu2023diversified,fukuda2024taylor,bai2023feed}. Some recent works~\citep{du2023dream,liu2024can,doorenbos2024nonlinearoutliersynthesisoutofdistribution,Kwon_2023_BMVC,chen2024fodfom,wahd2024deep,sun2024clipdrivenoutlierssynthesisfewshot,liu2024diffusionbased,Ansari_2025_ICCV} use foundation models to synthesize the outliers in image space. Such approach can be computationally intensive, often requiring multiple steps and careful prompting to curate outliers. The reliance on domain knowledge of foundation model limits its applicability in highly novel scenarios.

On the other hand, Ming \etal~\citep{ming2021impact} and Zhang \etal~\citep{zhang2023robustness} have explored the detrimental effect of spurious correlation on OOD detection, but studies addressing this issue (with virtual outliers) leveraging only training samples remain relatively scarce. A few notable works such as Kirby~\citep{kim2023key} and BackMix~\citep{wang2025backmix} propose to use background image features utilizing inpainting procedure while OEST~\citep{wang2023out} utilizes explicit data augmentations. In this work however, we take a more direct simplistic approach -- we add gradient attribution values to ID inputs to disrupt invariant features while amplifying true-class logit, thereby synthesizing challenging near-manifold virtual outliers.

To summarize, in this work, we propose a unified \textit{\textbf{A}pproach to \textbf{S}purious, fine-grained and \textbf{C}onventional OOD Detection} (dubbed \textbf{\ours}). \ours~consists of: \circled{1} outlier synthesis pipeline and \circled{2} virtual outlier exposure (OE) training pipeline. To synthesize virtual outliers, we perturb invariant features while preserving environmental features. We identify invariant features with the pixel attribution method using the model being learned. Second, we formulate a joint training objective that incentivizes the ID classification and the predictive uncertainty toward the synthesized outliers. To facilitate the joint objective, we employ constrained optimization by leveraging standardized feature representation. Our contributions are:
\begin{itemize}
    \item We propose a novel OE training approach leveraging standardized feature representation, along with an improved variant of posthoc method ODIN~\citep{odin18iclr}.
    \item To the best of our knowledge, we are the first to empirically demonstrate that adding gradient attribution values to ID samples synthesizes effective outliers, whereas subtracting these values does not. We also introduce invariant pixel shuffling as a strong outlier synthesis baseline.
    \item We empirically reveal superiority of z-score over $L_2$ normalization in feature representation for training the OOD detection model.
\end{itemize}

%% file: sec/2_prelim.tex
\section{Preliminaries}
\label{sec:formatting}
\noindent\textbf{Background}: We consider supervised multi-class classification setup. Let $\mathcal{X}_{\text{inv}} \in \mathbb{R}^v$ denote invariant image space, where each invariant feature $\mathbf{x}_{\text{inv}} \in \mathcal{X}_{\text{inv}}$ is essential for class recognition. Let $\mathcal{Y} = \{1, 2, \ldots, C\}$ be label space consisting of $C$ predefined classes with each label $y \in \mathcal{Y}$ associated with an invariant feature $\mathbf{x}_{\text{inv}}$. Let $\mathbf{y}$ be the one-hot vector of $y$. Let $\mathcal{E} \in \mathbb{R}^t$ denote environment space comprising $o$ distinct environments $\{e_1, e_2, \ldots, e_o\}$. Input space $\mathcal{X} \in \mathbb{R}^{v+t}$ is defined such that each input $\mathbf{x} \in \mathcal{X}$ is a function $\psi$ of $\mathbf{x}_{\text{inv}} \in \mathcal{X}_{\text{inv}}$ and $\mathbf{e} \in \mathcal{E}$, i.e., $\mathbf{x} := \psi(\mathbf{x}_{\text{inv}}, \mathbf{e})$, with $\mathbf{e}$ providing non-essential contextual cues. The training dataset $\mathbb{D}_{\text{train}} = \{(\mathbf{x}, \mathbf{y})_i \mid i = 1, 2, \ldots, N\}$ consists of $N$ i.i.d. samples from distribution $P(\mathcal{X}, \mathcal{Y})$. A feature extractor $\phi_\gamma: \mathcal{X} \rightarrow \mathbb{R}^m$ maps input $\mathbf{x} \in \mathcal{X}$ to a feature $\mathbf{h} \in \mathcal{H}$ in feature space $\mathcal{H} \in \mathbb{R}^m$, i.e., $\mathbf{h} := \phi_\gamma(\mathbf{x})$. A classifier $f_\theta: \mathbb{R}^m \rightarrow \mathbb{R}^C$ assigns logits $\mathbf{z} \in \mathbb{R}^C$ to $\mathbf{h}$, which are transformed into probabilities $\mathbf{p} = \rho(\mathbf{z}) \in \mathbb{R}^C$ using softmax function: $\rho(\mathbf{z})_j = \frac{\exp(z_j)}{\sum_{l=1}^C \exp(z_l)}, \forall j \in [1,C]$. The classification model $g = f_\theta \circ \phi_\gamma$ is traditionally optimized under the \textit{closed-world} assumption with empirical risk minimization using $\mathcal{L}$ loss function.: \( \min_{\phi_\gamma, f_\theta} \mathcal{L}(\rho(f_\theta(\phi_\gamma(\mathbf{x}))), \mathbf{y})\).

\noindent\textbf{OOD detection:} The deployment of model $g$ in open world (test distribution $\mathbb{D}_\text{test} = \{\mathbb{D}_\text{train}, \mathbb{D}_\text{out}\} = \{\mathbb{D}_\text{in}, \mathbb{D}_\text{out}\}$) 
violates the closed-world assumption ($\mathbb{D}_\text{test} = \mathbb{D}_\text{train}$), where $\mathbb{D}_\text{out}$ is OOD. OOD input $\mathbf{x}^{\prime} \in \mathbb{D}_\text{out}$ should be correctly identified to ensure the safe operation of model $g$. This is generally achieved through a scoring function $s: \mathbb{R}^m \rightarrow \mathbb{R}$ (possibly incorporating $f_\theta$), that quantifies the alignment of input $\mathbf{x}_\text{test}$ with $\mathbb{D}_\text{in}$ via the score $s(\phi_\gamma(\mathbf{x}_\text{test}))$. Specifically, if $s(\phi_\gamma(\mathbf{x}_\text{test})) \geq \beta$, 
it indicates $\mathbf{x}_\text{test} \in \mathbb{D}_\text{in}$. Conversely, if $s(\phi_\gamma(\mathbf{x}_\text{test})) < \beta$, it indicates $\mathbf{x}_\text{test} \in \mathbb{D}_\text{out}$. 
Here, $\beta$ represents a threshold chosen to have a higher true positive rate (e.g., 95\%) over the input space $\mathcal{X}$. If $\mathcal{E}^{\prime} \in \mathbb{R}^t$ and $\mathcal{X}_{\text{inv}}^{\prime} \in \mathbb{R}^v$ represent another environment and invariant input space respectively 
such that $\mathcal{X}_{\text{inv}} \cap \mathcal{X}_{\text{inv}}^{\prime} = \emptyset$ and
$\mathcal{E} \cap \mathcal{E}^{\prime} = \emptyset$, 
we can formalize three kinds of OOD inputs: an input $\mathbf{x}_\text{test}$ is known as \textit{conventional OOD} if $\mathbf{x}_\text{test} = \psi(\mathbf{x}'_{\text{inv}}, \mathbf{e}')$. It is known as \textit{spurious OOD} if $\mathbf{x}_\text{test} = \psi(\mathbf{x}^{\prime}_{\text{inv}}, \mathbf{e})$. In either case, $\mathbf{x}_\text{test}$ can be \textit{fine-grained OOD} if $\mathbf{x}^{\prime}_\text{inv} \sim \mathbf{x}_\text{inv}$.

%% file: sec/3_method.tex
\section{Method}
\label{sec:method}

In this section, we motivate our method with an example and then formulate our learning framework based on this motivation. We subsequently detail the outlier synthesis and virtual outlier exposure training. \\

\noindent\textbf{Motivation:} As depicted in Figure~\ref{fig:main_fig} (left), we analyze an image $\mathbf{x} = \psi(\mathbf{x}_{\text{inv}}, \mathbf{e}) \in \mathbb{D}_\text{in}$ consisting of invariant feature (bird) $\mathbf{x}_{\text{inv}}$ and environmental feature (land) $\mathbf{e}$. Only a smaller portion contains the invariant feature $\mathbf{x}_{\text{inv}}$ necessary for class recognition, while the remainder comprises non-essential environmental features $\mathbf{e}$. \textit{Can we synthesize challenging outlier $\mathbf{x}^{\prime}$ from $\mathbf{x}$ by perturbing $\mathbf{x}_{\text{inv}}$ while retaining $\mathbf{e}$?} Let $\mathcal{G}_\text{oracle}$ denote an oracle 2D distribution indicating the presence of invariant feature $\mathbf{x}_\text{inv}$ in $\mathbf{x}$. Consider a transformation $\psi^{-1}_{\mathcal{G}_\text{oracle}}$, with access to $\mathcal{G}_\text{oracle}$, that decomposes $\mathbf{x}$ i.e. $\psi^{-1}_{\mathcal{G}_\text{oracle}}(\mathbf{x}) \rightarrow [\mathbf{x}_{\text{inv}}, \mathbf{e}]$. Consider a perturbation function $\mathcal{P}_F$ that disrupts the semantics of $\mathbf{x}_{\text{inv}}$, yielding $\mathbf{e}^{\prime}$ such that $\mathcal{P}_F([\mathbf{x}_{\text{inv}}, \mathbf{e}]) = [\mathbf{e}^{\prime}, \mathbf{e}]$. Using these transformations, we can synthesize an outlier $\mathbf{x}' = \psi(\mathcal{P}_F(\psi^{-1}_{\mathcal{G}_\text{oracle}}(\mathbf{x})))$. \textit{In the absence of $\mathcal{G}_\text{oracle}$, can we approximate it with $\mathcal{G}$ for each $\mathbf{x} \in \mathcal{X}$ to synthesize outlier $\mathbf{x}^{\prime} \in \mathbb{D}_\text{out}$?} Training the network to enhance predictive uncertainty towards these challenging outliers improves the model's uncertainty towards OOD.\\

\noindent\textbf{Learning framework:} With the assumption of access to synthesized virtual $\mathbb{D}_\text{out}$, our learning framework is designed to optimize the parameters $\theta$ (of $f_\theta$) and $\gamma$ (of $\phi_\gamma$) of a classification model $g$, simultaneously focusing on ID classification accuracy and uncertainty on OOD inputs. We define the total loss function, $\mathcal{L}_\text{total}$ as:
\begin{equation}
\label{eq:obj}
\begin{gathered}
    \rightarrow \arg\min_{\theta, \gamma} \underbrace{\mathcal{L}_\text{ID}\left(f_\theta \left( \phi_\gamma(\mathbb{D}_{\text{in}}) \right)\right)}_{\text{ID classification error}} + \underbrace{\mathcal{L}_\text{OOD}\left(f_\theta \left( \phi_\gamma(\mathbb{D}_{\text{out}}) \right)\right)}_{\text{Uncerntainty error}}
\end{gathered}
\end{equation}
We use cross-entropy loss $\mathcal{L}_\text{CE}$ for ID classification loss $\mathcal{L}_\text{ID}$ and KL divergence loss $\mathcal{L}_\text{KL}$ between virtual $\mathbb{D}_{\text{out}}$ and uniform distribution $\mathcal{U}$ for uncertainty loss $\mathcal{L}_\text{OOD}$.
\subsection{Image-based Outlier Synthesis}
\label{sec:outlier_synthesis}
\begin{figure}[t]
  \centering
  \includegraphics[width=0.95\linewidth]{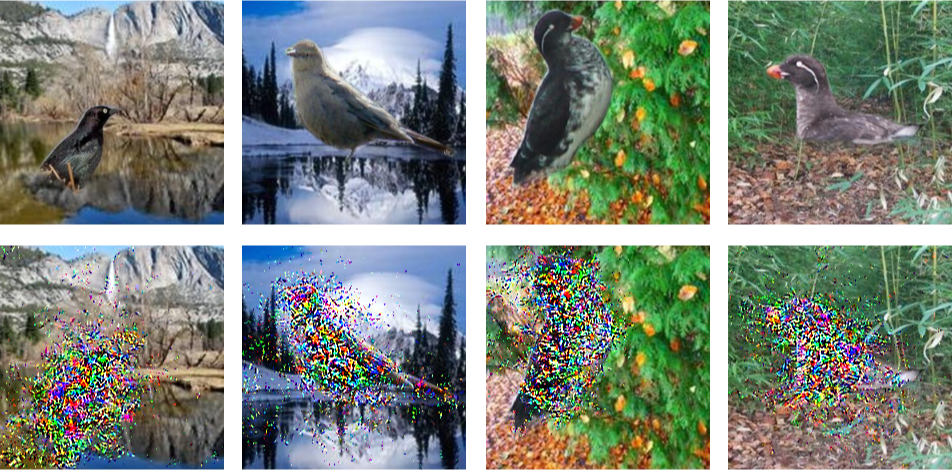}
  \caption{\textbf{Top row}: In-distribution images from the Waterbirds dataset. The first two images show waterbirds in water backgrounds, while the last two show landbirds in land backgrounds. \textbf{Bottom row}: Synthesized virtual outliers corresponding to the images in the top row at the latter stage of training.}
  \label{fig:outlier_synthesis}
\end{figure}
We synthesize virtual outliers from input space $\mathcal{X}$ by approximately perturbing the invariant features $\mathbf{x}_{\text{inv}}$ while preserving the environmental features $\mathbf{e}$ of image $\mathbf{x}$. In the interpretability literature, several methods~\cite{baehrens2010explain,smilkov2017smoothgrad,scott2017unified,wang2023counterfactual,Woerl_2023_CVPR} have been proposed to compute saliency map that quantifies the importance of each pixel. A straightforward approach to computing it involves calculating derivative (i.e. gradient) $\mathbf{G}$ of the logit value of true class ($\mathbf{z}_{c}$) with respect to the input image $\mathbf{x}$:
\begin{equation}
\mathbf{G} = \frac{\partial \mathbf{z}_{c}}{\partial \mathbf{x}}
\label{eq:gradient}
\end{equation}
For an input $\mathbf{x}' = \mathbf{x} + \alpha \cdot \mathbf{G}$, the model $g$ exhibits an increase in logit value of true class compared to the original input $\mathbf{x}$. Since input $\mathbf{x}'$ (with sufficiently high $\alpha$) has its invariant features destroyed (rendering it an outlier), the model should ideally express uncertainty. On the other hand, an increase in the logit value of the true class (roughly speaking) suggests that $\mathbf{x}'$ can serve as a challenging outlier.

We observe similar empirical effects using gradients of either logits or softmax probabilities for outlier synthesis (see \cref{appendix:sec:gradient_logit_softmax}). Since $\mathbf{G}$ assigns larger magnitudes to invariant pixels and smaller ones to environmental pixels, adding $\mathbf{G}$ to $\mathbf{x}$ disproportionately degrades invariant features while minimally impacting environmental features. \textit{Consequently, it effectively perturbs invariant features while preserving environmental features.} $\mathbf{G}$ can be sparsified by masking out the low-magnitude regions. Consider an image $\mathbf{x}$ consisting of $p_\text{inv}\%$ of pixels which are invariant pixels. Let $|\mathbf{G}|^{(100 - p_\text{inv})\%}$ denote the $(100 - p_\text{inv})^\text{th}$ percentile of $|\mathbf{G}|$. The gradient $\mathbf{G}$ with suppressed environment features can be expressed as:
\[
\mathbf{G}_{\text{inv}}^{j} = \begin{cases}
\mathbf{G}^{j}, & \text{if } |\mathbf{G}|^{j} \geq |\mathbf{G}|^{(100 - p_\text{inv})\%} \\
0, & \text{if } |\mathbf{G}|^{j} < |\mathbf{G}|^{(100 - p_\text{inv})\%}
\end{cases}
\]
We compute $\mathbf{G}$ with the model being learned. In highly spurious settings, using $\mathbf{x}' = \mathbf{x} + \alpha \cdot \mathbf{G}_\text{inv}$ better preserves environmental features. The examples of synthesized outliers depicted in Figure \ref{fig:outlier_synthesis} indeed show invariant features of the images being altered. Inspired by such perturbation, we propose improved variant of ODIN~\citep{odin18iclr}, \textbf{invariant-ODIN (i-ODIN)} (See \cref{appendix:sec:iodin}). 

\noindent Additionally, we also propose a novel (to the best of our knowledge) way of synthesizing virtual outlier by shuffling invariant pixels $\mathbf{x}_\text{inv}$ of ID sample (See \cref{appendix:sec:additional_details}). If \texttt{shuffle} denotes pixel-shuffling operation, virtual outliers could be synthesized as:
\[
\mathbf{x}' = \psi(~\texttt{shuffle} (\mathbf{x}_\text{inv}),~\mathbf{e}~)
\]

\subsection{Virtual Outlier Exposure (OE) Training:}
We propose to train model $g$ by simultaneously optimizing ID classification and predictive uncertainty towards the outliers with the learning framework of Equation \ref{eq:obj}.

\begin{proposition} The derivative of $\mathcal{L}_\text{total} = \mathcal{L}_\text{CE} + \mathcal{L}_\text{KL}$ w.r.t $k^{th}$ logit is $(\vp_{k}  - \vy_k) +   \left(\vp^{\prime}_{k} - 1/C \right)$. \label{proposition:gradient_analysis}
\end{proposition}
\begin{proof} 
The cross-entropy loss $\mathcal{L}_\text{CE}$ is given by:
$$\mathcal{L}_\text{CE} = -\sum_{l=1}^C \mathbf{y}_l \log \mathbf{p}_l, ~~ \mathbf{p}_l = \rho(\mathbf{z}_l) = \frac{\exp(\mathbf{z}_l)}{\sum_{r=1}^C \exp(\mathbf{z}_r)}.$$
To compute $\frac{\partial \mathcal{L}_\text{CE}}{\partial \mathbf{z}_k}$, we proceed by substituting the $\mathbf{p}_l$ in $\mathcal{L}_\text{CE}$ and performing $\log$ expansion.
$$\mathcal{L}_\text{CE} = -\sum_{l=1}^C \mathbf{y}_l \mathbf{z}_l + \log \left( \sum_{r=1}^C \exp(\mathbf{z}_r) \right) \notag $$
$$\text{Hence,}~~\frac{\partial \mathcal{L}_\text{CE}}{\partial \mathbf{z}_k} = \left( \mathbf{p}_k - \mathbf{y}_k \right) \notag$$
The Kullback-Leibler divergence loss $\mathcal{L}_\text{KL}$ is given by:
$$\mathcal{L}_\text{KL} = \sum_{l=1}^C \mathbf{p}^{\prime}_l \log (\mathbf{p}^{\prime}_l) - \sum_{l=1}^C \mathbf{p}^{\prime}_l \log \left( \frac{1}{C} \right)$$
$$\text{Hence,}~~\frac{\partial \mathcal{L}_\text{KL}}{\partial \mathbf{z}^{\prime}_k} = \left( \mathbf{p}^{\prime}_k - {1}/{C} \right)$$
\[\text{So,}~~\frac{\partial \mathcal{L}_\text{CE}}{\partial \mathbf{z}_k} + \frac{\partial \mathcal{L}_\text{KL}}{\partial \mathbf{z}^{\prime}_k} = (\mathbf{p}_k - \mathbf{y}_k) + (\mathbf{p}^{\prime}_k - {1}/{C})
\]
\end{proof}
During the initial phase of training, model $g$ lacks a comprehensive understanding of ID features. As we rely on the model for outlier synthesis, it may fail to synthesize true outliers ($\mathbf{x}^{\prime} = \mathbf{x}$) early on. From the proposition \ref{proposition:gradient_analysis}, \textit{ID gradient} $(\mathbf{p}_k - \mathbf{y}_k)$ should dominate \textit{OOD gradient} $(\mathbf{p}^{\prime}_k - {1}/{C})$ to reliably learn ID discrimination as effective outlier synthesis relies on accurately understanding ID features. As the model gets better on ID discrimination, the overconfident nature of neural networks can often lead to high-confidence predictions for both ID and OOD (high $\mathbf{p}^{\prime}_k$ and $\mathbf{p}_k)$, implying $|\mathbf{p}_{k}  - \mathbf{y}_k| < |\mathbf{p}^{\prime}_{k} - 1 / C|$. Though the nature of outlier synthesis determines $\mathbf{p}^{\prime}_{k}$, it is desirable to avoid high-confidence predictions on challenging outliers.
\begin{proposition}
The norm of a standardized feature \(\mathbf{\Tilde{h}} \in \mathbb{R}^m\) with \( (\mu, \sigma) = (0, \sigma) \) is constrained by the upper bound \(\sigma \cdot \sqrt{m-1}\).
\label{proposition:norm}
\end{proposition}
\begin{proof}
We begin by examining the square of the norm of standardized feature \(\mathbf{\Tilde{h}}\) of \(\mathbf{h}\) :
\begin{align}
    \|\mathbf{\Tilde{h}}\|^2 &= \sum_{u=1}^m \Tilde{h}_u^2 = \sum_{u=1}^m \left( \left( \frac{h_u - \mu_{h}}{\sigma_{h}} \right) \cdot \sigma + \mu \right)^2 \nonumber \\
    \|\mathbf{\Tilde{h}}\|^2 &= \frac{\sigma^2}{\sigma_{h}^2} \sum_{u=1}^m (h_u - \mu_{h})^2 \label{first_eq}
\end{align}
From the definition of sample standard deviation, we have:
\begin{equation}
\resizebox{\linewidth}{!}{$
\sigma_{h}^2 = \frac{\sum_{u=1}^m (h_u - \mu_{h})^2}{m-1} \Rightarrow \sigma_{h}^2 \cdot (m - 1) = \sum_{u=1}^m (h_u - \mu_{h})^2
$}
\end{equation}
Substituting this equality into Equation \ref{first_eq}:
\[
\|\mathbf{\Tilde{h}}\|^2 = \frac{{\sigma}^2}{\sigma_{h}^2} \sum_{u=1}^m (h_u - \mu_{h})^2 = \frac{{\sigma}^2}{\sigma_{h}^2} \cdot \sigma_{h}^2 \cdot (m-1) = {\sigma}^2 \cdot \left( m-1 \right)
\]
\[
\text{Hence,}~\|\mathbf{\Tilde{h}}\| = \sigma \cdot \sqrt{m-1}
\]
\end{proof}

\noindent We hypothesize that effective joint optimization of ID classification and outlier uncertainty requires mitigating overconfidence. The proposition \ref{proposition:norm} states that the norm of the standardized feature $\tilde{\mathbf{h}} = \mathcal{S}_h\left(\mathbf{h}\right) = \left(\left(\frac{\mathbf{h} - \mu_{\mathbf{h}} }{\sigma_{\mathbf{h}}}\right) \cdot \sigma \right) (\mu = 0)$ is constrained by the upper bound \(\sigma \cdot \sqrt{m-1}\). We hypothesize that employing constrained optimization by using standardized feature $\tilde{\mathbf{h}}$ instead of raw feature ${\mathbf{h}}$ minimizes overconfidence. Indeed, prior works~\cite{wei2022mitigating,regmi2023t2fnorm} have shown the effectiveness of constrained optimization. Comparatively low values of $\mathbf{p}_{k}$ and $\mathbf{p}^{\prime}_{k}$ can often ensure $|\mathbf{p}_{k}  - \mathbf{y}_k| > |\mathbf{p}^{\prime}_{k} - 1 / C|$. This reduction in overconfidence can assist in maintaining the appropriate balance of \textit{ID gradient} and \textit{OOD gradient}. Furthermore, a hyperparameter $\lambda$ can be introduced to empirically achieve the desired balance in $\mathcal{L}_\text{total}$ such that, $ \mathcal{L}_\text{total} = \mathcal{L}_{CE} + \lambda \cdot \mathcal{L}_{KL}$. The training-time regularization objective \ref{eq:obj} can be expressed as:
\begin{equation}
\label{eq:final_obj}
\resizebox{\linewidth}{!}{$
\arg\min_{\theta, \gamma} \quad
\underbrace{
\mathcal{L}_\text{CE}\left(
f_\theta\left( \mathcal{S}_h\left( \phi_{\gamma}\left( \mathbf{x} \right) \right) \right), \mathbf{y}
\right)
}_{\text{ID classification error}} +\ \lambda \cdot \underbrace{
\mathcal{L}_\text{KL}\left(
f_\theta\left( \mathcal{S}_h\left( \phi_{\gamma}\left( \mathbf{x}^{\prime} \right) \right) \right), \mathcal{U}
\right)
}_{\text{Uncertainty error}}
$}
\end{equation}
We train our model by using this objective in \cref{eq:final_obj}.

%% file: sec/4_experiments.tex
\input{tables/main_table}

\section{Experiments}
\label{sec:experiments}
\begin{table}[!h]
\centering
\label{tab:id_datasets}
\adjustbox{max width=0.99\linewidth}{%
\begin{tabular}{ccc}
\toprule
\textbf{Conventional setting} & \textbf{Spurious setting} & \textbf{Fine-grained setting} \\ \midrule
CIFAR-10/100~\citep{cifar-10,cifar-100} & CelebA~\citep{liu2015faceattributes} & Aircraft $(\text{ID/OOD categories}=90/10)$~\citep{maji2013fine} \\
ImageNet-100~\citep{imagenet} & Waterbirds~\citep{Sagawa2020Distributionally} & Car $(\text{ID/OOD categories}=150/46)$~\cite{jkrause3DRR2013} \\ \bottomrule
\end{tabular}}
\caption{ID datasets used in our experiments (See \ref{appendix:sec:spurious_dataset_formalization}).}
\end{table}

\noindent\textbf{OOD Datasets:} The details regarding spurious OOD (of Waterbirds and CelebA) and fine-grained OOD (of Aircraft and Car) datasets are provided in \cref{appendix:sec:spurious_dataset_formalization}. For conventional OOD datasets under both spurious and fine-grained setup, we use NINCO~\citep{ninco}, SUN~\citep{xiao2010sun}, OpenImage-O~\citep{haoqi2022vim}, iNaturalist~\citep{van2018inaturalist}, and Textures~\citep{cimpoi2014describing} datasets. We use following conventional OOD datasets for CIFAR-10/100 ID datasets: MNIST~\cite{deng2012mnist}, SVHN~\cite{svhn}, iSUN~\cite{xu2015turkergaze} Textures~\citep{cimpoi2014describing}, Places365~\citep{zhou2017places} and for ImageNet-100: SSB-Hard~\citep{vaze2022openset}, OpenImage-O~\citep{haoqi2022vim}, iNaturalist~\citep{van2018inaturalist}, and Textures~\citep{cimpoi2014describing}. \\
\noindent\textbf{Experimental details.}  We adhere closely to the training procedures outlined in OpenOOD~\citep{yang2022openood,zhang2023openood} with a few modifications. The experiments in fine-grained settings are performed with a batch size of 32. We use ResNet-18 model in spurious and conventional settings (CIFAR-10/100), while we use ResNet-50 model in fine-grained settings. We perform experiments in the conventional setting (CIFAR-10/100) from scratch while other settings follow a fine-tuning approach. For spurious and fine-grained settings, we fine-tune the (ImageNet) pre-trained model with an initial learning rate of 0.01 for 30 epochs. For ImageNet-100 experiments, we adopt the experimental setup of Dream-OOD~\citep{du2023dream}. We fine-tune the pre-trained ResNet-34 base model for 20 epochs with a batch size of 40 and a learning rate of 0.0005. For LogitNorm~\citep{wei2022mitigating} training in the spurious setting, we set the temperature to 1. Please refer to \cref{appendix:sec:additional_details} for complete details.

\noindent\textbf{Metrics:} We evaluate OOD detection using AUROC (Area Under Receiver-Operator Characteristics) and FPR@95 (False Positive Rate at 95\% True Positive Rate), where higher AUROC indicates better OOD/ID discrimination and lower FPR reflects fewer ID samples misclassified as OOD.

\noindent\textbf{Baselines:} MSP~\citep{msp17iclr}, GEN~\citep{liu2023gen}, ODIN~\citep{odin18iclr}, MDS~\citep{mahalanobis18nips}, MDSEns~\citep{mahalanobis18nips}, TempScale~\citep{guo2017calibration},  RMDS~\citep{rmd21arxiv}, Gram~\citep{gram20icml}, EBO~\citep{energyood20nips}, GradNorm~\citep{huang2021importance}, ReAct~\citep{react21nips}, MLS~\citep{species22icml}, KLM~\citep{species22icml}, VIM~\citep{haoqi2022vim}, DICE~\citep{sun2021dice}, RankFeat~\citep{song2022rankfeat}, ASH~\citep{djurisic2023extremely}, SHE~\citep{she23iclr}, NNGuide~\citep{park2023nearest}, Relation~\citep{kim2024neural}, SCALE~\citep{xu2024scaling}, FDBD~\citep{fdbd}, ConfBranch~\cite{confbranch2018arxiv}, RotPred~\citep{rotpred}, G-ODIN~\citep{godin20cvpr},  MOS~\citep{mos21cvpr}, VOS~\citep{vos22iclr} , LogitNorm~\citep{wei2022mitigating}, CIDER~\citep{cider2023iclr}, NPOS~\citep{npos2023iclr}, OE~\citep{oe18iclr}, MixOE~\citep{mixoe23wacv}, DreamOOD~\citep{du2023dream}.

%% file: tables/main_table.tex
\begin{table*}[t]
\centering
\adjustbox{max width=0.95\linewidth}{%
\begin{tabular}{l cccccc} 
\toprule
\multirow{3}{*}{Method}
& \multicolumn{2}{c}{\textbf{Fine-grained OOD Detection}} & \multicolumn{2}{c}{\textbf{Spurious OOD Detection}} & \multicolumn{2}{c}{\textbf{Conventional OOD Detection}} \\
\cmidrule(lr){2-3} \cmidrule(lr){4-5} \cmidrule(lr){6-7}
& \multicolumn{1}{c}{\textbf{Aircraft}} & \multicolumn{1}{c}{\textbf{Car}} & \multicolumn{1}{c}{\textbf{Waterbirds}} & \multicolumn{1}{c}{\textbf{CelebA}} & \multicolumn{1}{c}{\textbf{CIFAR-100}} & \multicolumn{1}{c}{\textbf{CIFAR-10}} \\
\midrule
MSP & 63.79{\tiny$\pm$5.71} / 80.53{\tiny$\pm$1.61} & 58.17{\tiny$\pm$0.99} / 87.12{\tiny$\pm$0.16} & 60.41{\tiny$\pm$1.52} / 77.18{\tiny$\pm$0.66} & 56.00{\tiny$\pm$2.73} / 82.53{\tiny$\pm$1.19} & 57.49{\tiny$\pm$0.85} / 78.48{\tiny$\pm$0.42} & 29.94{\tiny$\pm$1.32} / 91.06{\tiny$\pm$0.35}\\
TempScale & 61.72{\tiny$\pm$5.31} / 82.07{\tiny$\pm$1.62} & 57.47{\tiny$\pm$1.35} / 87.74{\tiny$\pm$0.20} & 60.37{\tiny$\pm$1.51} / 77.19{\tiny$\pm$0.66} & 55.50{\tiny$\pm$2.62} / 82.62{\tiny$\pm$1.18} & 56.58{\tiny$\pm$0.99} / 79.55{\tiny$\pm$0.51} & 31.38{\tiny$\pm$1.78} / 91.33{\tiny$\pm$0.42} \\
MDS & 77.41{\tiny$\pm$0.74} / 66.51{\tiny$\pm$0.37} & 67.48{\tiny$\pm$0.61} / 70.45{\tiny$\pm$0.39} & 93.60{\tiny$\pm$1.42} / 73.82{\tiny$\pm$0.91} & 90.91{\tiny$\pm$2.49} / 59.40{\tiny$\pm$3.68} & 73.39{\tiny$\pm$1.69} / 68.32{\tiny$\pm$1.24} & 32.28{\tiny$\pm$3.97} / 89.75{\tiny$\pm$1.60} \\
MDSEns & 94.70{\tiny$\pm$0.81} / 49.67{\tiny$\pm$0.39} & 95.36{\tiny$\pm$0.12} / 49.75{\tiny$\pm$0.03} & 62.19{\tiny$\pm$0.60} / 84.70{\tiny$\pm$0.07} & 92.71{\tiny$\pm$0.46} / 56.38{\tiny$\pm$0.57}  & 63.04{\tiny$\pm$0.16} / 71.05{\tiny$\pm$0.33} & 55.59{\tiny$\pm$2.36} / 77.92{\tiny$\pm$0.57} \\
RMDS & 58.40{\tiny$\pm$4.11} / \underline{86.80{\tiny$\pm$0.44}} & 51.46{\tiny$\pm$1.20} / 88.25{\tiny$\pm$0.15} & 84.06{\tiny$\pm$4.20} / 71.61{\tiny$\pm$1.44} & 88.36{\tiny$\pm$1.89} / 72.04{\tiny$\pm$1.69} & 52.95{\tiny$\pm$0.13} / 82.50{\tiny$\pm$0.10} & 24.34{\tiny$\pm$0.74} / 92.55{\tiny$\pm$0.22} \\
Gram & 96.01{\tiny$\pm$0.25} / 43.53{\tiny$\pm$1.42} & 92.19{\tiny$\pm$0.08} / 55.74{\tiny$\pm$0.68} & 92.71{\tiny$\pm$0.44} / 66.95{\tiny$\pm$1.87} & 72.17{\tiny$\pm$3.63} / 68.46{\tiny$\pm$2.20} & 71.55{\tiny$\pm$1.90} / 61.48{\tiny$\pm$1.03} & 77.65{\tiny$\pm$5.39} / 64.28{\tiny$\pm$2.69} \\
EBO & 51.22{\tiny$\pm$3.88} / 85.80{\tiny$\pm$1.25} & 59.36{\tiny$\pm$2.64} / 86.83{\tiny$\pm$0.35} & 59.17{\tiny$\pm$1.43} / 77.66{\tiny$\pm$0.79} & 55.19{\tiny$\pm$3.10} / 82.50{\tiny$\pm$1.10} & 54.76{\tiny$\pm$1.42} / 80.84{\tiny$\pm$0.60} & 38.82{\tiny$\pm$4.25} / 91.63{\tiny$\pm$0.76}  \\
GradNorm & 90.49{\tiny$\pm$1.67} / 70.78{\tiny$\pm$1.37} & 86.86{\tiny$\pm$0.70} / 72.13{\tiny$\pm$0.37} & 72.16{\tiny$\pm$2.79} / 79.79{\tiny$\pm$1.47} & 62.57{\tiny$\pm$3.88} / 77.16{\tiny$\pm$2.26}  & 84.21{\tiny$\pm$2.34} / 69.60{\tiny$\pm$1.23} & 91.07{\tiny$\pm$1.76} / 59.41{\tiny$\pm$2.67} \\
ReAct & 60.11{\tiny$\pm$6.33} / 83.62{\tiny$\pm$1.22} & {45.46{\tiny$\pm$2.65} / 88.89{\tiny$\pm$0.29}} & 56.37{\tiny$\pm$1.95} / 78.80{\tiny$\pm$0.89} & 55.14{\tiny$\pm$2.85} / 82.77{\tiny$\pm$0.92} & 52.25{\tiny$\pm$1.48} / 81.46{\tiny$\pm$0.55} & 41.09{\tiny$\pm$6.33} / 91.06{\tiny$\pm$1.04} \\
MLS & 52.00{\tiny$\pm$3.80} / 85.99{\tiny$\pm$1.25} & 59.19{\tiny$\pm$2.48} / 87.42{\tiny$\pm$0.35} & 59.17{\tiny$\pm$1.43} / 77.69{\tiny$\pm$0.68} & 55.19{\tiny$\pm$3.10} / 82.52{\tiny$\pm$1.11} & 54.92{\tiny$\pm$1.35} / 80.65{\tiny$\pm$0.57} & 38.79{\tiny$\pm$4.19} / 91.52{\tiny$\pm$0.73} \\
KLM & 82.29{\tiny$\pm$4.20} / 80.74{\tiny$\pm$2.01} & 66.45{\tiny$\pm$1.43} / 85.50{\tiny$\pm$0.15} & 97.68{\tiny$\pm$0.09} / 46.97{\tiny$\pm$0.47} & 98.61{\tiny$\pm$0.16} / 53.12{\tiny$\pm$0.97} & 74.15{\tiny$\pm$1.15} / 76.46{\tiny$\pm$0.50} & 18.16{\tiny$\pm$2.10} / 94.69{\tiny$\pm$0.87} \\
VIM & 57.51{\tiny$\pm$4.10} / 79.21{\tiny$\pm$0.94} & 54.80{\tiny$\pm$1.23} / 82.03{\tiny$\pm$0.16} & 39.66{\tiny$\pm$1.21} / 85.32{\tiny$\pm$0.63} & 58.25{\tiny$\pm$1.67} / 82.73{\tiny$\pm$0.29} & 49.92{\tiny$\pm$1.06} / 81.81{\tiny$\pm$0.76} & 24.19{\tiny$\pm$0.31} / 93.64{\tiny$\pm$0.16} \\
DICE & 70.79{\tiny$\pm$4.24} / 72.20{\tiny$\pm$3.31} & 72.27{\tiny$\pm$0.87} / 77.45{\tiny$\pm$0.31} & 56.40{\tiny$\pm$2.22} / 84.91{\tiny$\pm$1.55} & 50.52{\tiny$\pm$2.50} / 80.90{\tiny$\pm$2.15} & 54.55{\tiny$\pm$0.45} / 80.93{\tiny$\pm$0.30} & 53.31{\tiny$\pm$5.29} / 83.70{\tiny$\pm$2.16} \\
RankFeat & 62.00{\tiny$\pm$4.57} / 82.13{\tiny$\pm$1.74} & 87.45{\tiny$\pm$3.73} / 62.05{\tiny$\pm$3.57} & 70.47{\tiny$\pm$8.53} / 67.50{\tiny$\pm$6.04} & 74.48{\tiny$\pm$2.51} / 67.38{\tiny$\pm$3.39} & 67.90{\tiny$\pm$0.96} / 68.52{\tiny$\pm$1.32}  & 54.44{\tiny$\pm$9.49} / 77.48{\tiny$\pm$5.48} \\
ASH & 86.41{\tiny$\pm$3.76} / 73.98{\tiny$\pm$3.29} & 80.25{\tiny$\pm$4.34} / 74.73{\tiny$\pm$3.93} & 36.69{\tiny$\pm$1.16} / 87.03{\tiny$\pm$0.62} & 58.49{\tiny$\pm$4.15} / 80.98{\tiny$\pm$0.64} & 52.84{\tiny$\pm$1.20} / 82.26{\tiny$\pm$0.54}  & 70.55{\tiny$\pm$5.50} / 85.27{\tiny$\pm$2.05} \\
SHE & 78.90{\tiny$\pm$2.46} / 76.13{\tiny$\pm$1.45} & 86.78{\tiny$\pm$1.34} / 71.16{\tiny$\pm$0.64} & 66.03{\tiny$\pm$2.08} / 82.99{\tiny$\pm$1.82} & 55.19{\tiny$\pm$4.01} / 78.55{\tiny$\pm$2.02} & 62.08{\tiny$\pm$2.73} / 77.99{\tiny$\pm$1.09}  & 63.71{\tiny$\pm$5.22} / 86.18{\tiny$\pm$1.20} \\
GEN & \underline{50.81{\tiny$\pm$5.85}} / 86.41{\tiny$\pm$1.36} & 56.98{\tiny$\pm$2.40} / 88.05{\tiny$\pm$0.51} & 60.37{\tiny$\pm$1.51} / 77.19{\tiny$\pm$0.66} & 55.50{\tiny$\pm$2.62} / 82.62{\tiny$\pm$1.18} & 55.11{\tiny$\pm$1.65} / 80.48{\tiny$\pm$0.93}  & 32.49{\tiny$\pm$1.17} / 91.72{\tiny$\pm$0.57} \\
NNGuide & 52.23{\tiny$\pm$3.23} / 85.13{\tiny$\pm$1.42} & 61.12{\tiny$\pm$2.45} / 85.86{\tiny$\pm$0.34} & 53.69{\tiny$\pm$1.40} / 80.39{\tiny$\pm$0.75} & 52.32{\tiny$\pm$2.39} / 83.22{\tiny$\pm$1.02}  & 51.89{\tiny$\pm$1.35} / 82.33{\tiny$\pm$0.66} & 39.64{\tiny$\pm$3.91} / 91.39{\tiny$\pm$0.57} \\
Relation & 61.35{\tiny$\pm$5.78} / 83.05{\tiny$\pm$1.90} & 56.25{\tiny$\pm$1.30} / 86.77{\tiny$\pm$0.84} & \underline{31.71{\tiny$\pm$0.68}} / 88.35{\tiny$\pm$0.24} & 62.19{\tiny$\pm$2.65} / 81.78{\tiny$\pm$0.45} & 53.84{\tiny$\pm$0.32} / 81.60{\tiny$\pm$0.34} & 26.59{\tiny$\pm$0.64} / 92.49{\tiny$\pm$0.22} \\
SCALE & 90.49{\tiny$\pm$1.67} / 70.78{\tiny$\pm$1.37} & 79.78{\tiny$\pm$1.19} / 73.78{\tiny$\pm$0.35} & 36.49{\tiny$\pm$1.94} / {91.89{\tiny$\pm$0.66}} & 75.82{\tiny$\pm$3.25} / 72.58{\tiny$\pm$1.52} & 53.20{\tiny$\pm$1.27} / 81.97{\tiny$\pm$0.54} & 63.15{\tiny$\pm$5.95} / 87.38{\tiny$\pm$1.51} \\
FDBD & 58.46{\tiny$\pm$6.07} / 83.71{\tiny$\pm$1.50} & 50.30{\tiny$\pm$2.59} / 88.18{\tiny$\pm$0.26} & 50.07{\tiny$\pm$1.19} / 80.34{\tiny$\pm$0.57} & 56.07{\tiny$\pm$3.26} / 83.16{\tiny$\pm$0.50} & 51.66{\tiny$\pm$0.41} / 81.23{\tiny$\pm$0.41} & 23.00{\tiny$\pm$0.80} / 93.44{\tiny$\pm$0.26} \\ \midrule
ConfBranch & 91.24{\tiny$\pm$0.97} / 41.85{\tiny$\pm$0.70} & 80.63{\tiny$\pm$1.80} / 62.19{\tiny$\pm$0.76} & 52.09{\tiny$\pm$5.40} / 85.58{\tiny$\pm$2.42} & 57.02{\tiny$\pm$1.53} / 80.39{\tiny$\pm$0.52} & 77.97{\tiny$\pm$1.16} / 63.82{\tiny$\pm$2.02} & 21.40{\tiny$\pm$1.07} / 93.30{\tiny$\pm$0.50}\\
RotPred & 62.04{\tiny$\pm$1.30} / 81.98{\tiny$\pm$0.46} & 61.86{\tiny$\pm$2.18} / 85.00{\tiny$\pm$0.23} & 42.06{\tiny$\pm$1.55} / 85.32{\tiny$\pm$0.94} & 52.74{\tiny$\pm$2.52} / 83.26{\tiny$\pm$0.47} & \underline{35.57{\tiny$\pm$1.04}} / 88.09{\tiny$\pm$0.33} & \underline{12.80{\tiny$\pm$0.69}} / 96.36{\tiny$\pm$0.12} \\
G-ODIN & 58.43{\tiny$\pm$4.04} / 83.09{\tiny$\pm$0.06} & 64.34{\tiny$\pm$6.38} / 85.94{\tiny$\pm$1.08} & 58.65{\tiny$\pm$17.87} / 80.87{\tiny$\pm$4.18} & 78.17{\tiny$\pm$2.96} / 58.68{\tiny$\pm$1.56} & 37.03{\tiny$\pm$0.43} / \underline{88.49{\tiny$\pm$0.36}} & 20.02{\tiny$\pm$0.91} / 95.79{\tiny$\pm$0.19} \\
VOS & 57.31{\tiny$\pm$1.34} / 85.37{\tiny$\pm$0.27} & 63.16{\tiny$\pm$2.77} / 86.42{\tiny$\pm$0.42} & 57.66{\tiny$\pm$3.13} / 78.65{\tiny$\pm$1.26} & 54.67{\tiny$\pm$2.02} / 82.56{\tiny$\pm$0.65} & 53.28{\tiny$\pm$4.01} / 81.42{\tiny$\pm$1.95} & 38.08{\tiny$\pm$3.29} / 92.04{\tiny$\pm$0.52} \\
LogitNorm & 70.32{\tiny$\pm$5.73} / 79.16{\tiny$\pm$1.67} & 62.48{\tiny$\pm$1.90} / 84.92{\tiny$\pm$0.49} & 64.18{\tiny$\pm$4.19} / 73.87{\tiny$\pm$0.45} & 100.00{\tiny$\pm$0.00} / 81.61{\tiny$\pm$1.99}  & 52.42{\tiny$\pm$1.67} / 80.70{\tiny$\pm$1.15} & 13.98{\tiny$\pm$1.33} / \underline{96.54{\tiny$\pm$0.45}} \\
CIDER & 88.99{\tiny$\pm$1.88} / 54.08{\tiny$\pm$1.50} & 88.95{\tiny$\pm$1.07} / 54.30{\tiny$\pm$1.31} & 49.34{\tiny$\pm$17.95} / 84.61{\tiny$\pm$7.94} & \underline{50.38{\tiny$\pm$4.98} / 85.35{\tiny$\pm$0.92}} & 61.67{\tiny$\pm$1.69} / 77.22{\tiny$\pm$0.89}  & 20.23{\tiny$\pm$2.90} / 94.49{\tiny$\pm$1.18} \\ 
NPOS & 68.71{\tiny$\pm$10.57} / 72.02{\tiny$\pm$6.07} & 89.09{\tiny$\pm$0.53} / 55.95{\tiny$\pm$0.64} & 42.88{\tiny$\pm$5.95} / 89.01{\tiny$\pm$1.83} & 54.35{\tiny$\pm$5.25} / 82.06{\tiny$\pm$4.26} & 52.09{\tiny$\pm$2.02} / 84.06{\tiny$\pm$0.99} & 22.65{\tiny$\pm$1.10} / 93.51{\tiny$\pm$0.18} \\
OE & 56.82{\tiny$\pm$2.01} / 82.86{\tiny$\pm$0.48} & 60.30{\tiny$\pm$3.04} / 85.60{\tiny$\pm$1.73} & 33.26{\tiny$\pm$6.00} / \underline{92.63{\tiny$\pm$1.14}} & 100.0{\tiny$\pm$0.00} / 70.24{\tiny$\pm$3.83} & 46.82{\tiny$\pm$4.65} / 86.84{\tiny$\pm$1.44} & 19.50{\tiny$\pm$3.72} / 92.82{\tiny$\pm$1.80} \\
MixOE & 70.61{\tiny$\pm$12.18} / 83.47{\tiny$\pm$2.41} & \underline{42.04{\tiny$\pm$0.74} / 90.19{\tiny$\pm$0.23}} & 42.13{\tiny$\pm$4.26} / 88.09{\tiny$\pm$3.18} & 77.52{\tiny$\pm$2.05} / 74.31{\tiny$\pm$0.79} & 68.76{\tiny$\pm$3.98} / 74.86{\tiny$\pm$2.42} & 52.09{\tiny$\pm$6.20} / 90.09{\tiny$\pm$0.64} \\ \midrule
\rowcolor{LightGray} \textbf{\ours} & \textbf{47.94{\tiny$\pm$5.38} / 89.75{\tiny$\pm$1.01}} & \textbf{40.76{\tiny$\pm$1.13} / 91.86{\tiny$\pm$0.20}} & \textbf{11.37{\tiny$\pm$0.42} / 97.48{\tiny$\pm$0.10}} & \textbf{42.62{\tiny$\pm$0.76} / 86.50{\tiny$\pm$0.89}} & \textbf{29.90{\tiny$\pm$0.76} / 91.35{\tiny$\pm$0.13}} & \textbf{7.69{\tiny$\pm$0.29} / 98.32{\tiny$\pm$0.07}} \\
\bottomrule
\end{tabular}
}
\caption{OOD detection results (FPR@95$\downarrow$ / AUROC$\uparrow$) on fine-grained, spurious, and conventional OOD detection. Best results are formatted as \textbf{bold} and second-best results are formatted as \underline{underline}. Same formatting is applied to subsequent tables. See Appendix \ref{sec:complete_results}.}
\label{tab:main_benchmark}
\end{table*}

%% file: sec/5_results.tex
\section{Results}
\begin{figure}[t!]
    \centering
    \adjustbox{width=0.99\linewidth}{
    \begin{subfigure}{0.485\linewidth}
        \centering
        \adjustbox{width=\linewidth,clip}{        \includegraphics{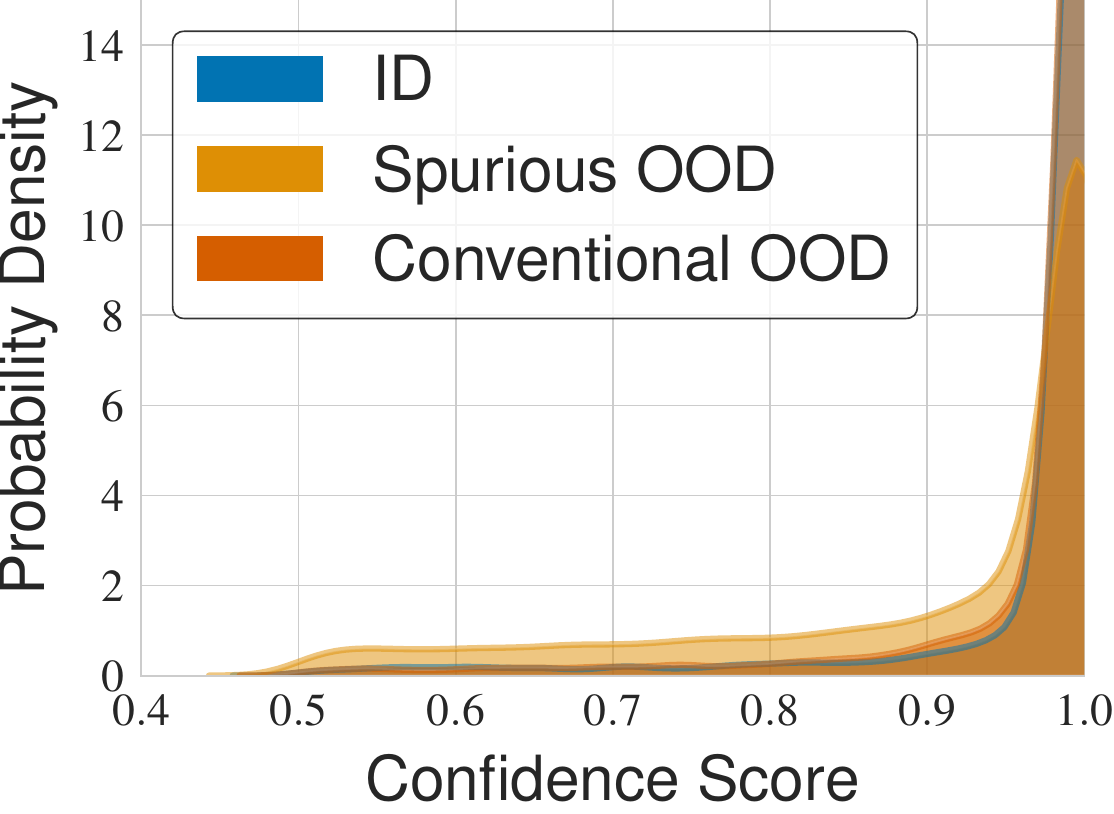}
        }
        \caption{Baseline}
        \label{fig:conf_waterbird_baseline}
    \end{subfigure}
    \hfill
    \begin{subfigure}{0.485\linewidth}
        \centering
        \adjustbox{width=\linewidth,clip}{        \includegraphics{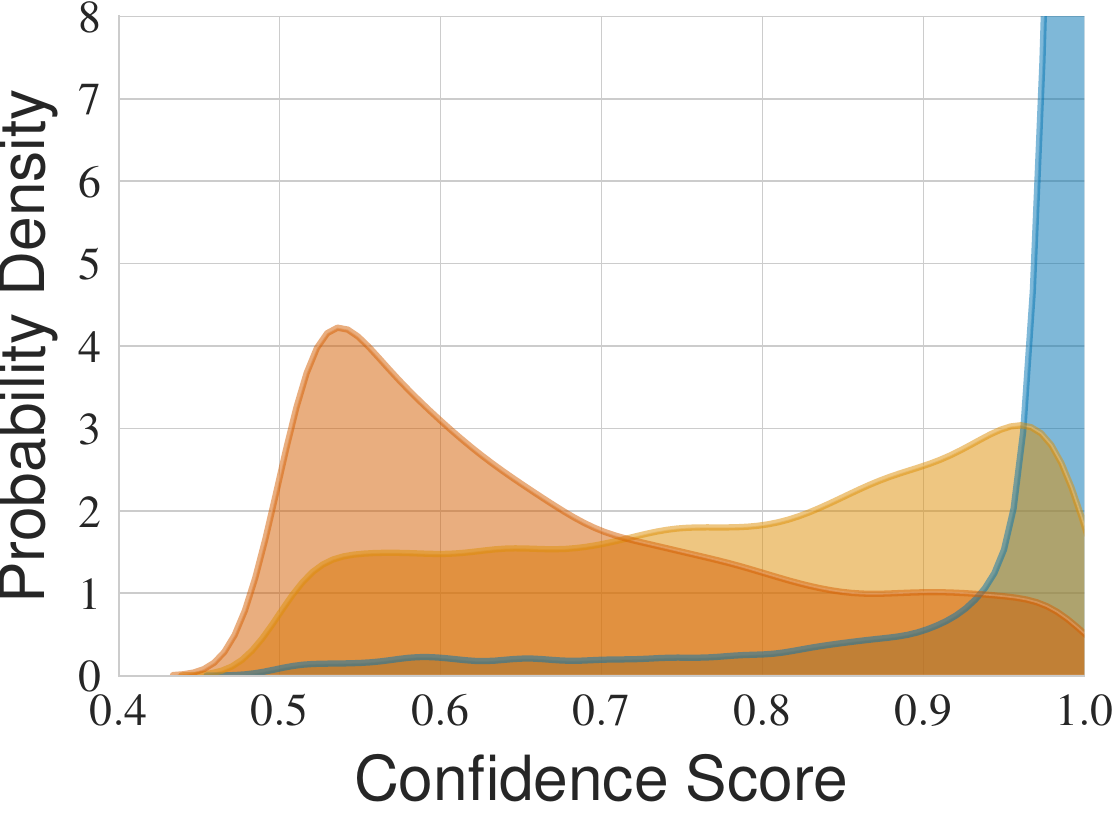}
        }
        \caption{\ours}
        \label{fig:conf_waterbird_ascood}
    \end{subfigure}
    }
    \caption{Visualization of confidence scores (MSP) of (a) cross-entropy baseline and (b) \ours~in Waterbirds benchmark. The confidence scores of ID and (spurious and conventional (iNaturalist)) OOD are relatively well-separated in case of (b) \ours~in comparison to (a) cross-entropy baseline.}
\end{figure}
\noindent\textbf{Fine-grained OOD detection.} We assess OOD detection performance in fine-grained setting using Aircraft and Car benchmarks. Results for fine-grained OOD datasets are presented in Table \ref{tab:main_benchmark}, with performance on conventional OOD datasets deferred to \cref{appendix:sec:complete_aircraft} and \cref{appendix:sec:complete_car}. \ours~outperforms the nearest competitors GEN and RMDS by $\sim$\textbf{3} AUROC points in Aircraft datasets. Notably, many training regularization approaches (RotPred, G-ODIN, VOS, LogitNorm, CIDER, NPOS) fail to even match performance of MSP baseline in Car benchmark. On the other hand, \ours~achieves the best performance among all 30 competing methods, surpassing third-best rival ReAct in FPR@95 / AUROC metric by $\sim$\textbf{5}/\textbf{3} points. MixOE narrows this performance gap in Car datasets by leveraging external OOD datasets $\mathbb{D}_\text{out}$ (SUN datasets), though this benefit doesn't extend to Aircraft datasets under same conditions.

\noindent\textbf{Spurious OOD detection.} Spurious OOD detection is summarized in Table \ref{tab:main_benchmark} using Waterbirds and CelebA benchmarks. On Waterbirds benchmark, \ours~outperforms all 30 competing methods by a substantial margin in both FPR@95 and AUROC metrics. Specifically, \ours~achieves a performance improvement over the nearest competitor Relation by $\sim$\textbf{59\%} in FPR@95 metric. Moreover, \ours~sustains its superiority on the CelebA benchmark too, surpassing second-best FPR@95 metric of CIDER by $\sim$\textbf{15\%}. Such impressive performance of \ours~can be partly attributed to nature of virtual outliers as they potentially contain spurious features that are present in ID samples. However, using external OOD datasets (OE) or mixing them with ID samples (MixOE) may either lack spurious features or lead to their comparatively stronger degradation. Leveraging Car datasets as external OOD datasets, both OE and MixOE outperform other training regularization methods in Waterbirds benchmark, yet still fall short of \ours's performance. Quantitatively, \ours~demonstrates a substantial advantage, outperforming OE and MixOE by \textbf{66\%} and \textbf{73\%} in the FPR@95 metric on the Waterbirds datasets. This superior performance is consistent across datasets, with \ours~outperforming OE by \textbf{57\%} and MixOE by \textbf{45\%} on CelebA datasets. \Cref{fig:conf_waterbird_baseline,fig:conf_waterbird_ascood} visualize confidence scores for (a) cross-entropy baseline and (b) \ours~in Waterbirds benchmark. These plots demonstrate that \ours~achieves relatively more pronounced separation of confidence scores between ID and spurious OOD, as well as between ID and conventional OOD (iNaturalist).

\noindent\textbf{Conventional OOD detection.} We evaluate OOD detection performance in conventional setting using the CIFAR-10/100 benchmarks as presented in Table \ref{tab:main_benchmark}. Consistent with results in the fine-grained and spurious setting, \ours~outperforms all 30 competing methods by a significant margin. Specifically, \ours~exceeds the performance of the strong RotPred baseline on the CIFAR-100 benchmark, improving the FPR@95 metric by $\sim$\textbf{16\%}. Additionally, an AUROC improvement of $\sim$\textbf{3} points is observed between \ours~and RotPred. While CIFAR-10 is considered a comparatively easier benchmark on which many prior methods perform well, \ours~still demonstrates superior performance across both metrics. Apart from CIFAR benchmarks, we conduct experiments in large-scaling settings using ImageNet-100 ID datasets following the experimental settings of Dream-OOD~\citep{du2023dream}. The results presented in \Cref{tab:imagenet100} demonstrate the strong empirical effectiveness of ASCOOD in large-scale settings. Furthermore, when evaluated on SSB-Hard~\citep{vaze2022openset}, ASCOOD achieves the highest AUROC score (\textbf{83.91}), surpassing the second-best score of DreamOOD ({83.30}). (See \cref{appendix:sec:imagenet100} and \cref{appendix:sec:imagenet100_ssb} for complete results.)

\begin{table}[b]
  \centering
  \resizebox{0.99\linewidth}{!}{%
    \begin{tabular}{@{}lcccc}
      \toprule
      Method & \textbf{iNaturalist} & \textbf{Textures} & \textbf{OpenImage} & \textbf{Average} \\
      \midrule
        MSP & 21.09 / 94.87 & 62.82 / 83.45 & 39.82 / 88.08 & 41.24 / 88.80 \\
        ODIN & 21.02 / 95.43 & 79.42 / 80.92 & 50.42 / 85.60 & 50.29 / 87.32 \\
        EBO & 22.24 / 94.03 & 74.67 / 81.04 & 43.98 / 86.66 & 46.96 / 87.24 \\
        ReAct & 20.89 / 94.68 & 65.98 / 82.42 & 41.27 / 87.41 & 42.71 / 88.17 \\        
        SHE & 29.51 / 93.28 & 82.47 / 82.13 & 61.76 / 81.77 & 57.91 / 85.72 \\
        GEN & 18.64 / 94.62 & 65.60 / 82.48 & 39.51 / 88.45 & 41.25 / 88.52 \\
        NNGuide & 19.73 / 95.59 & 71.31 / 84.52 & 43.00 / 87.52 & 44.68 / 89.21 \\
        SCALE & \textbf{12.91} / \textbf{97.32} & 54.13 / 89.77 & 38.58 / 89.96 & 35.21 / 92.35 \\ \midrule
        RotPred & 19.64 / 93.57 & 50.51 / 88.69 & 37.44 / 88.64 & 35.87 / 90.30 \\
        VOS & 21.20 / 94.25 & 64.96 / 83.72 & 36.82 / 88.26 & 40.99 / 88.74 \\
        LogitNorm & 18.38 / 95.75 & 49.96 / 87.23 & 34.69 / 89.51 & 34.34 / 90.83 \\
        CIDER & 24.91 / 95.03 & \textbf{21.84} / \textbf{96.20} & 46.29 / 89.41 & \underline{31.01} / \underline{93.54} \\
        Dream-OOD (EBO) & \underline{14.47} / \underline{96.09} & 60.73 / 84.79 & \underline{32.67} / \underline{90.16} & 35.96 / 90.35 \\ \midrule
        \rowcolor{LightGray} \textbf{ASCOOD} (EBO) & 18.11 / 95.73 & \underline{25.20} / \underline{94.40} & \textbf{26.04 / 91.95} & \textbf{23.12 / 94.02} \\
      \bottomrule
    \end{tabular}
  }
  \caption{Conventional OOD detection (FPR@95 $\downarrow$/ AUROC $\uparrow$) in large-scale setting using ImageNet-100 datasets.}
  \label{tab:imagenet100}
\end{table}

\noindent\textbf{ODIN vs. i-ODIN.} Unlike ODIN, which perturbs all color channel intensities, we propose i-ODIN that modifies the variable number of significant color channel intensities (determined via pixel attribution) of the image. We compare ODIN and i-ODIN in challenging cases as shown in \Cref{tab:odin_comparison} which demonstrate significant gain of i-ODIN over ODIN establishing superiority of the former. For instance, i-ODIN outperforms ODIN in FPR@95 metric by \textbf{20\%} in CIFAR-10 (ID) vs CIFAR-100 (OOD) and by \textbf{33\%} in CIFAR-100 (ID) vs TIN (OOD). Furthermore, a non-trivial performance improvement of i-ODIN over ODIN is also observed in fine-grained setting, especially in the FPR@95 metric. Specifically, we find perturbing only the single most significant color channel intensity yields substantially better performance than perturbing all color channel intensities. However, in trivial cases involving only two classes (e.g., Waterbirds and CelebA datasets), this improvement is not observed. Please see \Cref{appendix:sec:complete_results:odin_compare} for complete results.
\begin{table}[t]
	\centering
	\adjustbox{max width=0.95\linewidth}{%
		\begin{tabular}{l c cc cc}
			\toprule
			\multirow{2}{*}{Benchmark} & \multirow{2}{*}{Method} & \multicolumn{2}{c}{\textbf{CIFAR-100}} & \multicolumn{2}{c}{\textbf{TIN}} \\
            \cmidrule(lr){3-4} \cmidrule(lr){5-6}
 & & FPR@95 $\downarrow$ & AUROC $\uparrow$ & FPR@95 $\downarrow$ & AUROC $\uparrow$ \\
			\midrule
	 \multirow{2}{*}{CIFAR-10} & ODIN & 77.00{\tiny$\pm$5.74} & 82.18{\tiny$\pm$1.87} & 75.38{\tiny$\pm$6.42} & 83.55{\tiny$\pm$1.84} \\ 
		    & \cellcolor{LightGray} i-ODIN & \cellcolor{LightGray} \textbf{61.33{\tiny$\pm$1.18}} & \cellcolor{LightGray} \textbf{86.87{\tiny$\pm$0.14}} & \cellcolor{LightGray} \textbf{50.20{\tiny$\pm$2.39}} & \cellcolor{LightGray} \textbf{88.96{\tiny$\pm$0.25}} \\ \midrule

            \multirow{2}{*}{Benchmark} & \multirow{2}{*}{Method} & \multicolumn{2}{c}{\textbf{Car}} & \multicolumn{2}{c}{\textbf{Aircraft}} \\
            \cmidrule(lr){3-4} \cmidrule(lr){5-6}
            & & FPR@95 $\downarrow$ & AUROC $\uparrow$ & FPR@95 $\downarrow$ & AUROC $\uparrow$ \\ \midrule

            \multirow{2}{*}{Fine-grained setting} & ODIN & 64.78{\tiny$\pm$1.28} & 86.41{\tiny$\pm$0.29} & 54.55{\tiny$\pm$5.54} & \textbf{86.23{\tiny$\pm$1.47}} \\ 

              & \cellcolor{LightGray} i-ODIN & \cellcolor{LightGray} \textbf{58.21{\tiny$\pm$2.41}} & \cellcolor{LightGray} \textbf{87.80{\tiny$\pm$0.36}} & \cellcolor{LightGray} \textbf{51.81{\tiny$\pm$4.82}} & \cellcolor{LightGray} \textbf{86.21{\tiny$\pm$1.28}} \\
        \bottomrule
		\end{tabular}
	}
	\caption{ODIN vs. i-ODIN in challenging cases.}
	\label{tab:odin_comparison}
\end{table}

\noindent\textbf{Accuracy:} It is undesirable to trade off accuracy with OOD detection performance. \ours~achieves accuracies of $\sim$87.27\%, 76.63\%, 94.95\%, 93.59\%, and 96.58\% on ImageNet-100, CIFAR-100, CIFAR-10, Waterbirds, and CelebA respectively, closely aligning with the baseline accuracies of $\sim$87.33\%, 77.25\%, 95.06\%, 93.72\%, and 96.72\%. In fine-grained setting, \ours~achieves slightly better accuracies of $\sim$\textbf{94.20}\% and \textbf{89.61}\% in Car and Aircraft datasets respectively bettering baseline accuracies of $\sim$92.73\% and 87.85\%. It shows efficacy of \ours~in enhancing OOD detection without harming accuracy.

\subsection{Ablation studies}
\noindent\textbf{Outlier synthesis.} For a sufficiently high value of $\alpha$ used in synthesizing virtual outliers $\mathbf{x}^{\prime}$, the invariant features of $\mathbf{x}$ are destroyed, irrespective of whether gradient addition or subtraction is employed. As the subtraction of the gradient reduces the logit value associated with the true class in resulting outliers, the model tends to exhibit increased uncertainty towards them. Consequently, the resulting outliers are less challenging and offer limited potential for enhancing predictive uncertainty towards OOD samples. Therefore, substantial performance gains are not anticipated with this approach. Conversely, gradient addition increases the true class logit while simultaneously compromising invariant features (with high $\alpha$), creating an opportunity to improve predictive uncertainty towards outliers. By incentivizing predictive uncertainty for these virtual outliers, the model learns an improved discrimination between known and unknown data. We empirically analyze impact of these outlier synthesis strategies along with invariant pixel shuffling on fine-grained OOD detection, reporting results on Aircraft and Car datasets in \Cref{tab:gradient_addition_vs_subtraction}. The results verify that gradient addition is a superior choice for outlier synthesis.

\begin{table}[t]
\centering
\resizebox{0.95\linewidth}{!}{
\begin{tabular}{@{}lcccc@{}}
\toprule
 \multirow{2}{*}{$\mathbf{x}^{\prime}$} & \multicolumn{2}{c}{\textbf{Car}} & \multicolumn{2}{c}{\textbf{Aircraft}} \\
\cmidrule(lr){2-3} \cmidrule(lr){4-5}
 & FPR@95 $\downarrow$ & AUROC $\uparrow$ & FPR@95 $\downarrow$ & AUROC $\uparrow$ \\
\midrule
$\mathbf{x}' = \psi(~\texttt{shuffle} (\mathbf{x}_\text{inv}),~\mathbf{e}~)$ & 63.84{\tiny$\pm$3.21} & 85.38{\tiny$\pm$0.29} & \underline{47.98{\tiny$\pm$2.04}} & \underline{83.87{\tiny$\pm$0.38}} \\
$\mathbf{x}^{\prime} = \mathbf{x} - \alpha \cdot \mathbf{G}_\text{inv}$ & \underline{60.20{\tiny$\pm$1.91}} & \underline{86.27{\tiny$\pm$0.34}} & 50.15{\tiny$\pm$4.80} & 83.64{\tiny$\pm$0.82}  \\
\rowcolor{LightGray} $\mathbf{x}^{\prime} = \mathbf{x} + \alpha \cdot \mathbf{G}_\text{inv}$ & \textbf{40.76{\tiny$\pm$1.13}} & \textbf{91.86{\tiny$\pm$0.20}} & \textbf{47.94{\tiny$\pm$5.38}} &  \textbf{89.75{\tiny$\pm$1.01}} \\
\bottomrule
\end{tabular}}
\caption{Ablation of outlier synthesis methods on challenging fine-grained OOD detection. (See \cref{appendix:sec:gradient_addition_vs_subtraction} for complete results.)}
\label{tab:gradient_addition_vs_subtraction}
\end{table}

\begin{figure}[tbh]
    \centering
    \adjustbox{width=0.9\linewidth}{
    \begin{subfigure}{0.45\linewidth}
        \centering
        \adjustbox{width=\linewidth,clip}{\includegraphics{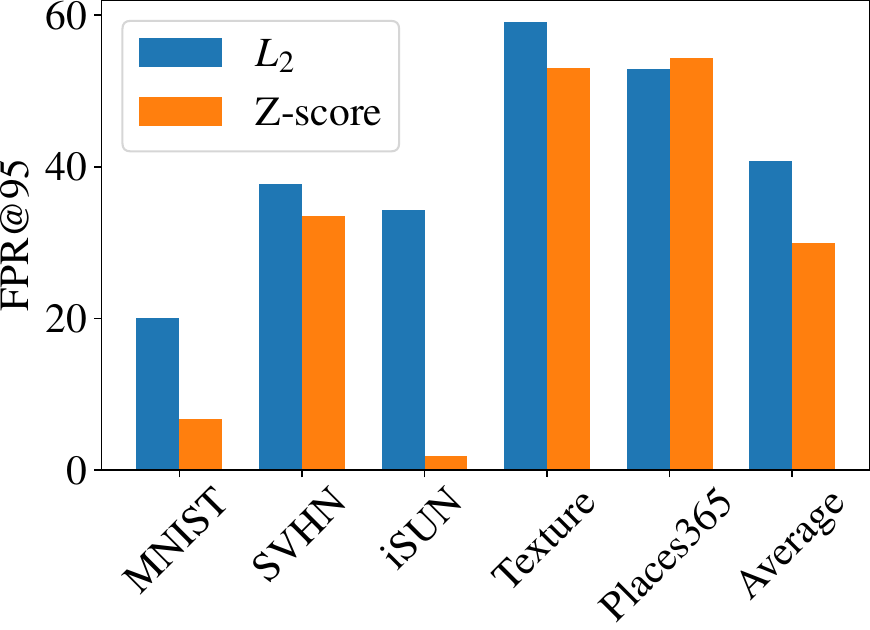}}
        \caption{FPR@95 in various datasets.}
        \label{fig:fpr}
    \end{subfigure}
    \hfill
    \begin{subfigure}{0.45\linewidth}
        \centering
        \adjustbox{width=\linewidth,clip}{\includegraphics{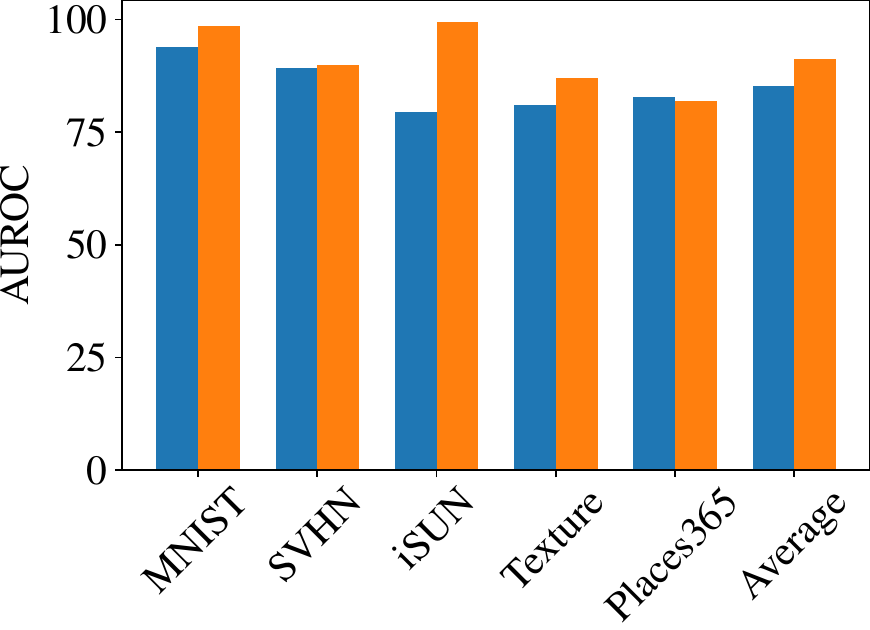}}
        \caption{AUROC in various datasets.}
        \label{fig:auroc}
    \end{subfigure}
    }
    \caption{Comparison of $L_2$ normalization and Z-score normalization in terms of FPR@95 and AUROC in CIFAR-100 datasets.}
\end{figure}

\noindent\textbf{Standardized feature space:} As $L_2$ normalization techniques~\citep{wei2022mitigating,regmi2023t2fnorm,Park_2023_ICCV} have been explored in prior works, its empirical effectiveness in comparison to z-score normalization is unknown. We make this comparison in CIFAR-100 benchmark with bar charts in Figure \ref{fig:fpr} and Figure \ref{fig:auroc} using invariant pixel shuffling for outlier synthesis. We observe that $L_2$ normalization yields average OOD detection performance (FPR@95/AUROC: $40.81${\tiny$\pm3.06$}/$85.26${\tiny$\pm3.10$}) which is inferior to the average performance of \ours~($29.90${\tiny$\pm0.76$}/$91.35${\tiny$\pm0.13$}). Furthermore, we also conduct experiments in fine-grained settings using gradient addition for outlier synthesis. The results, presented in \Cref{tab:standardize_vs_no_standardize}, consistently demonstrate z-score normalization to be superior. For instance, z-score normalization leads to improvement of $\textbf{19}\%$ and $\textbf{31}\%$ in FPR@95 metric in Aircraft and Car datasets, respectively. Furthermore, we can also observe that employing $\mathcal{L}_\text{KL}$ for $\mathcal{L}_\text{OOD}$ yields superior results in comparison to $\mathcal{L}_\text{energy}$ (used in prior works~\citep{vos22iclr,du2023dream}) in fine-grained setting. Please refer \cref{appendix:sec:additional_details} for additional empirical studies.

\begin{table}[t]
  \centering
  \resizebox{0.99\linewidth}{!}{%
    \begin{tabular}{@{}lclcc}
      \toprule
      $\mathbb{D}_{\text{in}}$ & $\mathcal{L}_\text{OOD}$ & Feature space & \textbf{Fine-grained OOD} & \textbf{Conventional OOD} \\
      \midrule
      \multirow{7}{*}{\rotatebox{90}{Aircraft}} & \multirow{3}{*}{$\mathcal{L}_\text{energy}$} & $\mathbf{h}$ & 72.57{\tiny$\pm$2.52} / 78.13{\tiny$\pm$0.81} & 3.21{\tiny$\pm$0.04} / 99.09{\tiny$\pm$0.04} \\
                                & & $\mathbf{h}~/~||\mathbf{h}||_2$ & 57.56{\tiny$\pm$8.62} / 85.17{\tiny$\pm$1.10} & 40.50{\tiny$\pm$7.75} / 88.10{\tiny$\pm$1.77} \\
                                & & $\left(\mathbf{h} - \mu_{\mathbf{h}} \right) \cdot \sigma / \sigma_{\mathbf{h}}$ & 58.70{\tiny$\pm$4.90} / 84.47{\tiny$\pm$0.71} & 84.63{\tiny$\pm$8.10} / 53.51{\tiny$\pm$6.27} \\
                                \cmidrule(lr){2-5}
                                & \multirow{3}{*}{$\mathcal{L}_\text{KL}$} & $\mathbf{h}$ & 55.21{\tiny$\pm$5.98} / \underline{88.73{\tiny$\pm$0.89}} & 6.84{\tiny$\pm$0.95} / 98.47{\tiny$\pm$0.11} \\
                                & & $\mathbf{h}~/~||\mathbf{h}||_2$ &  \underline{54.30{\tiny$\pm$5.87}} / 85.70{\tiny$\pm$1.48} &  \underline{2.58{\tiny$\pm$0.06}} /  \underline{99.37{\tiny$\pm$0.04}} \\
                                & & \cellcolor{LightGray} $\left(\mathbf{h} - \mu_{\mathbf{h}} \right) \cdot \sigma / \sigma_{\mathbf{h}}$ & \cellcolor{LightGray} {\textbf{47.94{\tiny$\pm$5.38}} / \textbf{89.75{\tiny$\pm$1.01}}} & \cellcolor{LightGray} {\textbf{0.55{\tiny$\pm$0.07}} / \textbf{99.84{\tiny$\pm$0.02}}} \\
      \midrule
      \multirow{7}{*}{\rotatebox{90}{Car}} & \multirow{3}{*}{$\mathcal{L}_\text{energy}$} & $\mathbf{h}$ & {54.15{\tiny$\pm$1.23} / 88.09{\tiny$\pm$0.12}} & {6.98{\tiny$\pm$0.55} / 98.73{\tiny$\pm$0.08}} \\
                                    & & $\mathbf{h}~/~||\mathbf{h}||_2$ &  61.80{\tiny$\pm$0.81} / 86.57{\tiny$\pm$0.04} &  7.16{\tiny$\pm$2.87} / 98.70{\tiny$\pm$0.39} \\
                                    & & $\left(\mathbf{h} - \mu_{\mathbf{h}} \right) \cdot \sigma / \sigma_{\mathbf{h}}$ & 58.84{\tiny$\pm$1.05} / 86.46{\tiny$\pm$0.37} & {\textbf{3.07{\tiny$\pm$1.02}} / \textbf{99.41{\tiny$\pm$0.18}}} \\
                                    \cmidrule(lr){2-5}
                                    & \multirow{3}{*}{$\mathcal{L}_\text{KL}$} & $\mathbf{h}$ & 54.89{\tiny$\pm$2.51} / \underline{90.21{\tiny$\pm$0.21}} & {14.32{\tiny$\pm$2.43} / 96.58{\tiny$\pm$0.48}} \\
                                    & & $\mathbf{h}~/~||\mathbf{h}||_2$ &  \underline{51.86{\tiny$\pm$4.10}} / 89.70{\tiny$\pm$0.23} & 9.55{\tiny$\pm$2.82} / 97.82{\tiny$\pm$0.39} \\
                                    & & \cellcolor{LightGray} $\left(\mathbf{h} - \mu_{\mathbf{h}} \right) \cdot \sigma / \sigma_{\mathbf{h}}$ & \cellcolor{LightGray} {\textbf{40.76{\tiny$\pm$1.13}} / \textbf{91.86{\tiny$\pm$0.20}}} & \cellcolor{LightGray} {\underline{4.28{\tiny$\pm$0.53}} / \underline{98.79{\tiny$\pm$0.12}}} \\
      \bottomrule
    \end{tabular}
  }
  \caption{Ablation study of feature space $\mathbf{h}$ and uncertainty loss $\mathcal{L}_\text{OOD}$ in fine-grained setting in FPR@95$\downarrow$ / AUROC$\uparrow$.}
  \label{tab:standardize_vs_no_standardize}
\end{table}

%% file: sec/6_related.tex
\section{Related Works}

\noindent\textbf{OOD detection} Since Nguyen \etal~\citep{nguyen2015deep} highlighted the overconfidence of neural networks towards OOD data, numerous studies have proposed to mitigate this issue. Post-hoc methods for OOD detection offer a practical solution by operating on pre-trained models.  Hendrycks \etal~\citep{msp17iclr} uses the maximum softmax probability (MSP) as a measure of confidence. Liang \etal~\citep{odin18iclr} improves upon MSP by incorporating temperature scaling and input preprocessing. Liu \etal~\citep{energyood20nips} utilizes energy function derived from the softmax denominator. Lee \etal~\citep{mahalanobis18nips} estimates class-conditional Gaussian distributions in feature space and uses maximum Mahalanobis distance to all class centroids as a measure of OOD uncertainty. Methods offered by Sun \etal~\citep{react21nips} and Ahn \etal~\citep{Ahn_2023_CVPR} analyze the activation patterns of neurons to identify and leverage distinctive features for OOD detection. Yang \etal~\citep{yang2022openood} and Zhang \etal~\citep{zhang2023openood} provide the comprehensive benchmark of such many other post-hoc methods~\citep{rmd21arxiv,species22icml,sun2022knnood,she23iclr}.

Beyond post-hoc methods, various regularization strategies have been explored for OOD detection. DeVries \etal~\citep{confbranch2018arxiv} proposes learning confidence estimates by attaching an auxiliary branch to a pre-trained classifier. Hsu \etal~\citep{godin20cvpr} proposes novel confidence scoring decomposition and input preprocessing. Huang \etal~\citep{mos21cvpr} describes a group-based framework to refine decision boundaries. Yu \etal~\citep{Yu_2023_CVPR} demonstrates the importance of feature norms, while Wei \etal~\citep{wei2022mitigating} and Regmi \etal~\citep{regmi2023t2fnorm} study the utility of normalization in OOD Detection. Recent works also investigate the role of contrastive learning~\citep{csi20nips,2021ssd,cider2023iclr,regmi2023reweightood,PALM2024} and self-supervised learning~\citep{rotpred} for OOD detection. More recent advancements include leveraging masked image modeling~\citep{Li_2023_CVPR}, balancing energy regularization~\citep{Choi_2023_CVPR}, decoupling MaxLogit~\citep{Zhang_2023_CVPR}, exploring binary neuron patterns~\citep{Olber_2023_CVPR}, and applying uncertainty-aware optimal transport~\citep{Lu_2023_CVPR}. Additionally, there has been growing interest in simultaneously addressing OOD detection and generalization~\cite{liu2024neuron,yang2023full,bai2024aha}.

Several works such as OE~\citep{oe18iclr}, MCD~\citep{mcd19iccv}, and UDG~\citep{yang2021scood} tackle OOD detection by incorporating external OOD data during training. Though similar works~\citep{mixoe23wacv,techapanurak2021practical,perera2019deep} do the same to enhance OOD detection in fine-grained settings, curating OOD datasets that don't overlap with ID and ensuring their diversity~\citep{yao2024outofdistribution,jiang2024dos,zhu2023diversified} can be a significant hurdle. To avoid reliance on real outlier data, few recent works~\citep{vos22iclr,npos2023iclr,gao2024oalenhancingooddetection,li2025outlier,gong2024outofdistributiondetectionprototypicaloutlier} synthesize virtual outliers in feature space but such approach is computationally expensive. Recent studies have increasingly explored foundation models for OOD detection ~\citep{Wang_2023_ICCV,Gao_2023_ICCV,lapt,Li_2024_CVPR,Bai_2024_CVPR}. In contrast, our approach synthesizes virtual outliers in image space from ID data without relying on foundation models, akin to VoSo~\citep{voso}. While Roy \etal~\citep{roy2022does} and Ahmed \etal~\citep{roy2022does} address challenging scenarios, their studies are limited to cases where ID and OOD share semantic similarities. There is a scarcity of studies addressing spurious correlations in the context of OOD detection. Ming \etal~\citep{ming2021impact}, Zhang \etal~\citep{zhang2023robustness}, Kirby~\citep{kim2023key} and BackMix~\citep{wang2025backmix} share some similarities with our work in terms of motivation and goal. Ming \etal~\citep{ming2021impact} first analyze the impact of spurious settings in of OOD detection. Zhang \etal~\citep{zhang2023robustness} further extend this analysis and propose a reweighting solution. In contrast to these works, our approach incorporates virtual outlier synthesis using pixel attribution. Similar to Kirby~\citep{kim2023key} and BackMix~\citep{wang2025backmix}, we synthesize virtual outliers by removing invariant features. But, our work shows gradient addition is superior; it not only destroys these features but also makes resulting outliers challenging. More broadly, our work aims to address spurious, fine-grained, as well as conventional OOD inputs comprehensively within a unified framework.

%% file: sec/7_conclusion.tex
\section{Conclusion} In this work, we introduce a novel training method \textbf{\ours} designed to improve OOD detection in both conventional and challenging cases. \ours~trains the model by incentivizing joint optimization of ID classification and predictive uncertainty towards virtual outliers. \ours~synthesizes virtual outliers by approximately destroying invariant features from ID images. These invariant features are determined by the pixel attribution method using the model being learned. For effective dual optimization, it employs constrained optimization in a standardized feature space. By mitigating the impact of spurious correlations and promoting the capture of fine-grained attributes, \ours~demonstrates improved performance in spurious, fine-grained, and conventional setups, as evidenced by extensive experiments across seven datasets. Importantly, \ours~operates without relying on external OOD datasets, making it a promising approach for OOD detection.

%% file: sec/X_suppl.tex
\appendix
\setcounter{page}{1}
\onecolumn

{\centering
\Large
\textbf{\thetitle}\\
\vspace{0.5em}Supplementary Material \\
\vspace{1.0em}
}
\section{Datasets} 
\label{appendix:sec:spurious_dataset_formalization}

\subsection{Spurious setting} 
We use Waterbirds and CelebA as ID datasets for the OOD evaluation in the spurious setting.
\begin{figure}[h!]
  \centering
  \resizebox{0.85\textwidth}{!}{%
    \begin{subfigure}[t]{0.45\textwidth}
      \centering
      \begin{tabular}{cccc}
        \includegraphics[width=0.22\linewidth, height=0.22\linewidth]{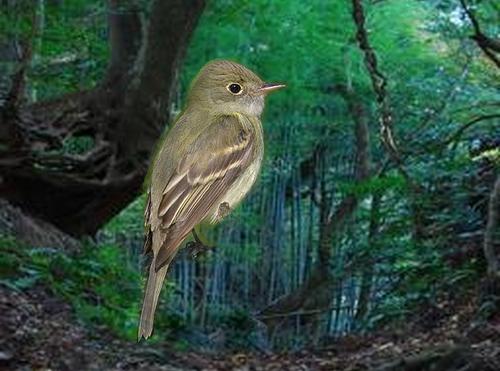} & \includegraphics[width=0.22\linewidth, height=0.22\linewidth]{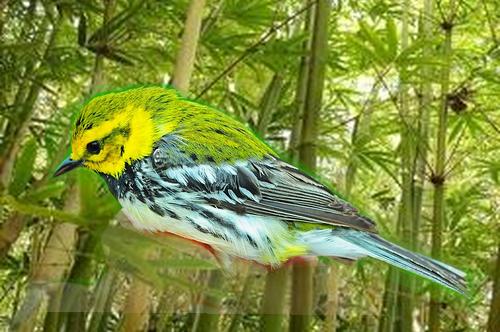} & \includegraphics[width=0.22\linewidth, height=0.22\linewidth]{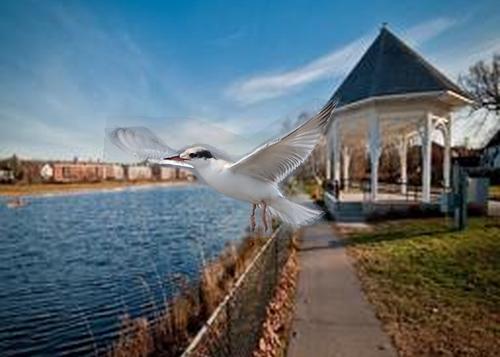} & \includegraphics[width=0.22\linewidth, height=0.22\linewidth]{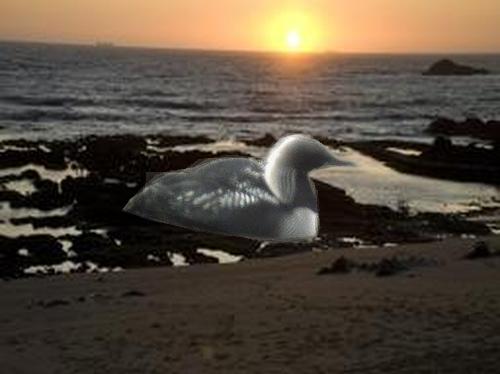} \\
        \includegraphics[width=0.22\linewidth, height=0.22\linewidth]{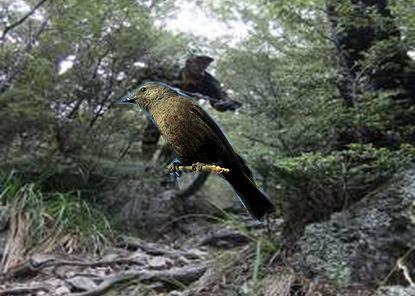} & \includegraphics[width=0.22\linewidth, height=0.22\linewidth]{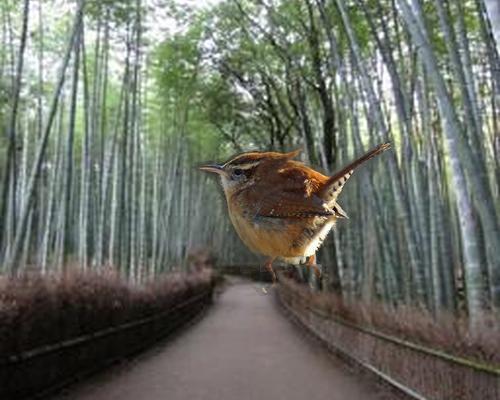} & \includegraphics[width=0.22\linewidth, height=0.22\linewidth]{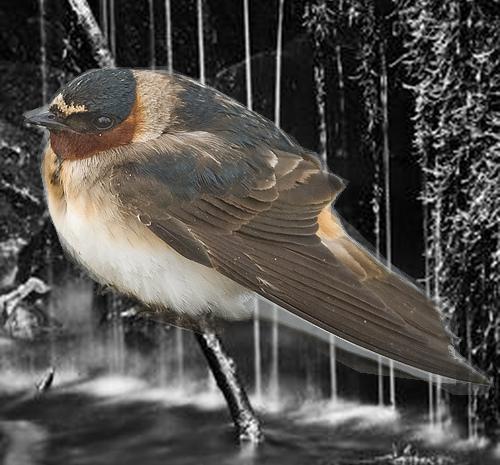} & \includegraphics[width=0.22\linewidth, height=0.22\linewidth]{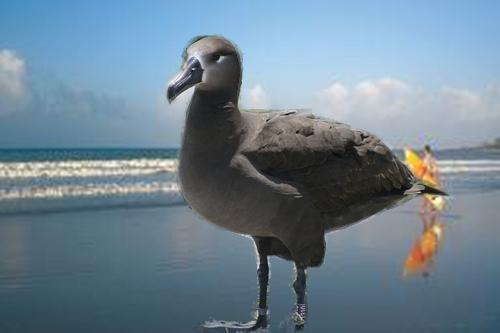} \\
        \includegraphics[width=0.22\linewidth, height=0.22\linewidth]{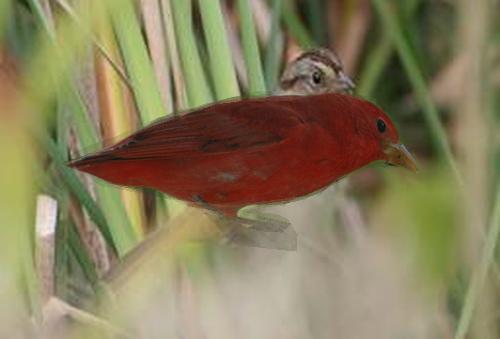} & \includegraphics[width=0.22\linewidth, height=0.22\linewidth]{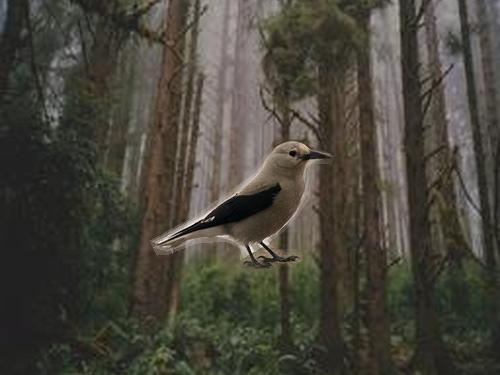} & \includegraphics[width=0.22\linewidth, height=0.22\linewidth]{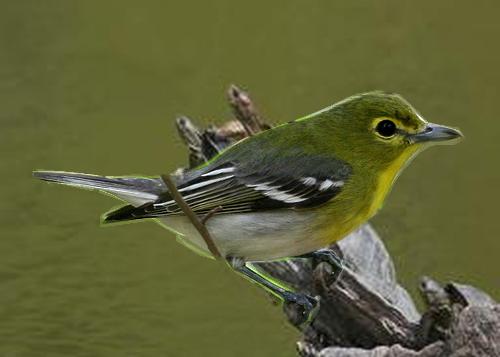} & \includegraphics[width=0.22\linewidth, height=0.22\linewidth]{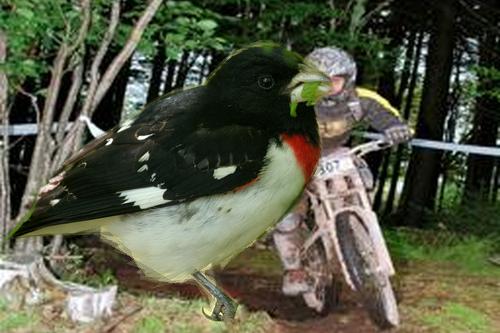} \\
        \includegraphics[width=0.22\linewidth, height=0.22\linewidth]{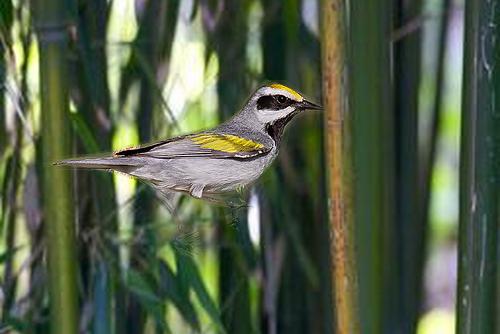} & \includegraphics[width=0.22\linewidth, height=0.22\linewidth]{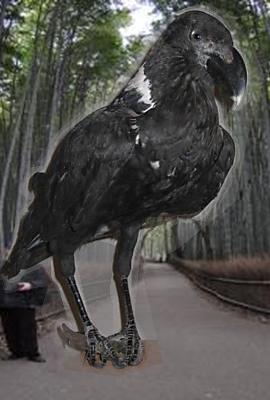} & \includegraphics[width=0.22\linewidth, height=0.22\linewidth]{images/waterbird_id/Acadian_Flycatcher_0007_795600.jpg} & \includegraphics[width=0.22\linewidth, height=0.22\linewidth]{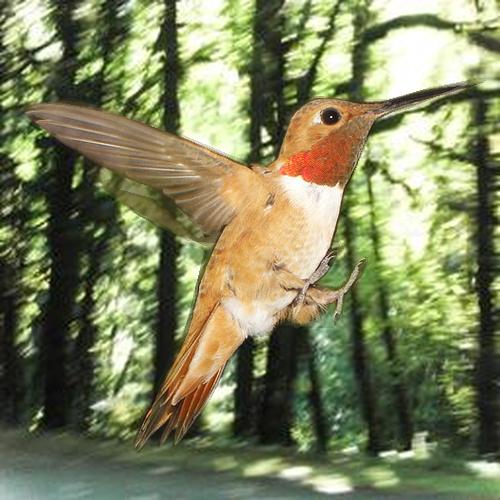}
      \end{tabular}
      \caption{ID images of Waterbirds datasets.}
    \end{subfigure}
    \hspace{0.1\textwidth} 
    \begin{subfigure}[t]{0.45\textwidth}
      \centering
      \begin{tabular}{cccc}
        \includegraphics[width=0.22\linewidth, height=0.22\linewidth]{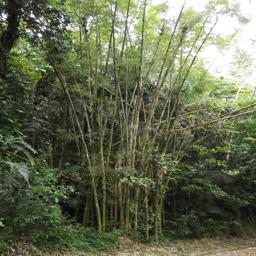} & \includegraphics[width=0.22\linewidth, height=0.22\linewidth]{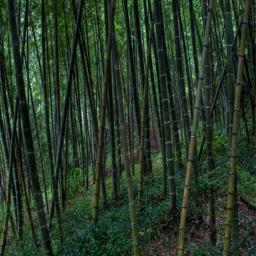} & \includegraphics[width=0.22\linewidth, height=0.22\linewidth]{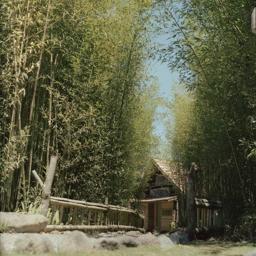} & \includegraphics[width=0.22\linewidth, height=0.22\linewidth]{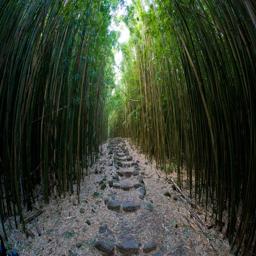} \\
        \includegraphics[width=0.22\linewidth, height=0.22\linewidth]{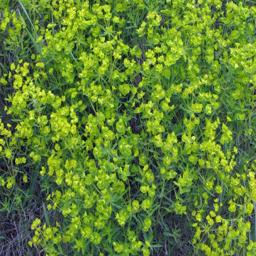} & \includegraphics[width=0.22\linewidth, height=0.22\linewidth]{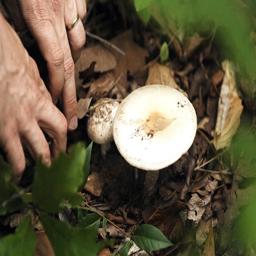} & \includegraphics[width=0.22\linewidth, height=0.22\linewidth]{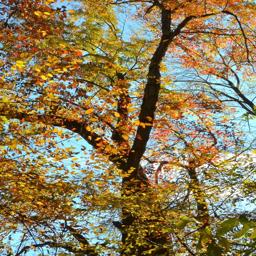} & \includegraphics[width=0.22\linewidth, height=0.22\linewidth]{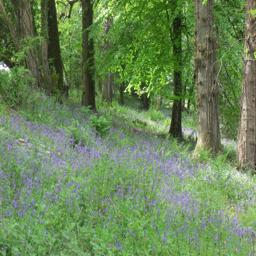} \\
        \includegraphics[width=0.22\linewidth, height=0.22\linewidth]{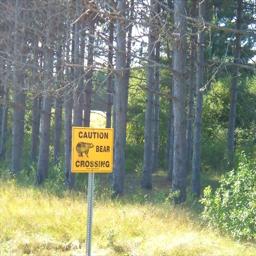} & \includegraphics[width=0.22\linewidth, height=0.22\linewidth]{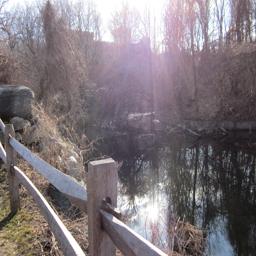} & \includegraphics[width=0.22\linewidth, height=0.22\linewidth]{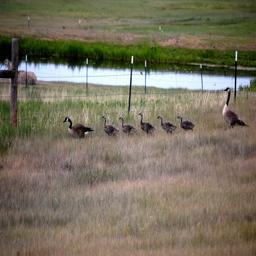} & \includegraphics[width=0.22\linewidth, height=0.22\linewidth]{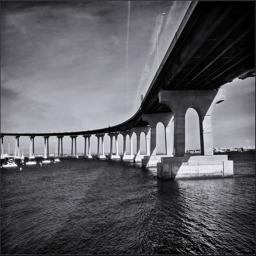} \\
        \includegraphics[width=0.22\linewidth, height=0.22\linewidth]{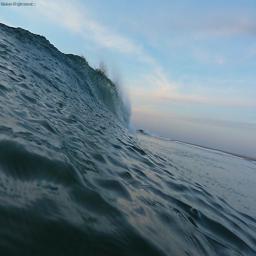} & \includegraphics[width=0.22\linewidth, height=0.22\linewidth]{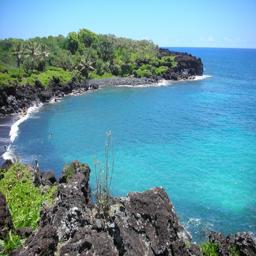} & \includegraphics[width=0.22\linewidth, height=0.22\linewidth]{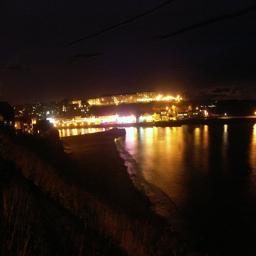} & \includegraphics[width=0.22\linewidth, height=0.22\linewidth]{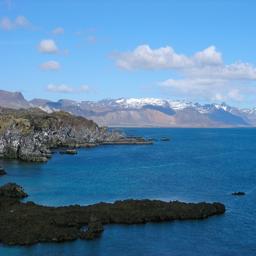}
      \end{tabular}
      \caption{Spurious OOD images of Waterbirds datasets.}
    \end{subfigure}
  }
  \caption{Examples of ID and spurious OOD images of Waterbirds datasets.}
  \label{fig:waterbirds}
\end{figure}

\noindent\textbf{Waterbirds:} Waterbirds $\mathbb{D}_\text{in}$~\citep{Sagawa2020Distributionally} is constructed from the combination of photographs from the Caltech-UCSD Birds-200-2011 (CUB) datasets \cite{wah2011caltech} $\mathcal{X}_\text{inv}$ with image backgrounds $\mathcal{E}$ from the Places datasets \cite{zhou2017places}. Each bird is labeled as $\mathcal{Y} = \{\texttt{waterbird}, \texttt{landbird}\}$ and placed on an environment $\mathcal{E}=\{\texttt{water background}, \texttt{land background}\}$, i.e., $\mathbf{x} =\psi(\mathbf{x}_{\text{inv}}, \mathbf{e})$. Images of land and water alone are considered spurious OOD, i.e., $ \mathbf{x}^{\prime} = \psi(\mathbf{x}^{\prime}_{\text{inv}}, \mathbf{e})$, where $\mathbf{x}^{\prime}_{\text{inv}} = \emptyset$. The correlation between land (water) and landbird (waterbird) in the training set is set to $\sim 0.9$. The examples of ID and spurious OOD images of the Waterbirds datasets are presented in \Cref{fig:waterbirds}.
\begin{figure}[h!]
  \centering
  \resizebox{0.85\textwidth}{!}{%
    \begin{subfigure}[t]{0.45\textwidth}
      \centering
      \begin{tabular}{cccc}
        \includegraphics[width=0.22\linewidth, height=0.22\linewidth]{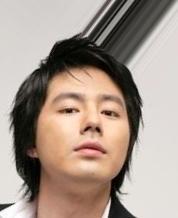} & \includegraphics[width=0.22\linewidth, height=0.22\linewidth]{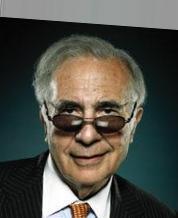} & \includegraphics[width=0.22\linewidth, height=0.22\linewidth]{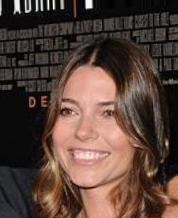} & \includegraphics[width=0.22\linewidth, height=0.22\linewidth]{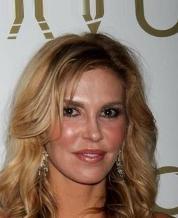} \\
        \includegraphics[width=0.22\linewidth, height=0.22\linewidth]{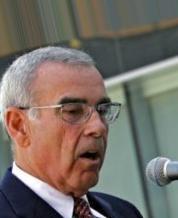} & \includegraphics[width=0.22\linewidth, height=0.22\linewidth]{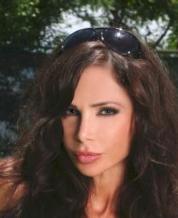} & \includegraphics[width=0.22\linewidth, height=0.22\linewidth]{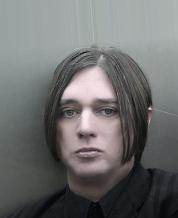} & \includegraphics[width=0.22\linewidth, height=0.22\linewidth]{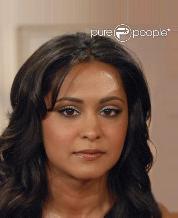} \\
        \includegraphics[width=0.22\linewidth, height=0.22\linewidth]{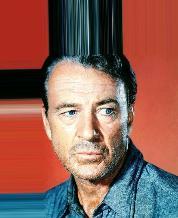} & \includegraphics[width=0.22\linewidth, height=0.22\linewidth]{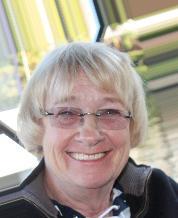} & \includegraphics[width=0.22\linewidth, height=0.22\linewidth]{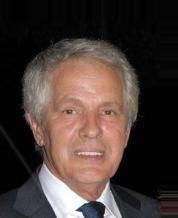} & \includegraphics[width=0.22\linewidth, height=0.22\linewidth]{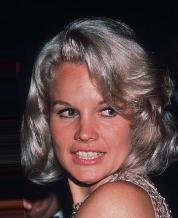} \\
        \includegraphics[width=0.22\linewidth, height=0.22\linewidth]{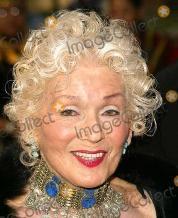} & \includegraphics[width=0.22\linewidth, height=0.22\linewidth]{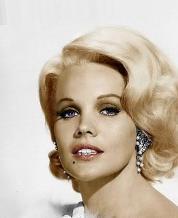} & \includegraphics[width=0.22\linewidth, height=0.22\linewidth]{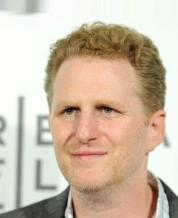} & \includegraphics[width=0.22\linewidth, height=0.22\linewidth]{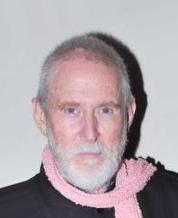}
      \end{tabular}
      \caption{ID images of CelebA datasets.}
    \end{subfigure}
    \hspace{0.1\textwidth}
    \begin{subfigure}[t]{0.45\textwidth}
      \centering
      \begin{tabular}{cccc}
        \includegraphics[width=0.22\linewidth, height=0.22\linewidth]{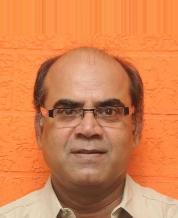} & \includegraphics[width=0.22\linewidth, height=0.22\linewidth]{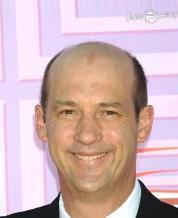} & \includegraphics[width=0.22\linewidth, height=0.22\linewidth]{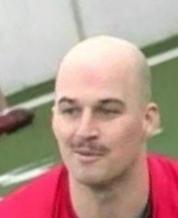} & \includegraphics[width=0.22\linewidth, height=0.22\linewidth]{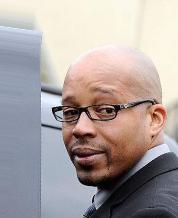} \\
        \includegraphics[width=0.22\linewidth, height=0.22\linewidth]{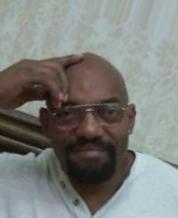} & \includegraphics[width=0.22\linewidth, height=0.22\linewidth]{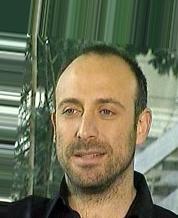} & \includegraphics[width=0.22\linewidth, height=0.22\linewidth]{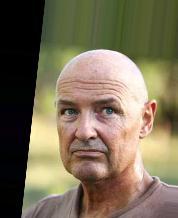} & \includegraphics[width=0.22\linewidth, height=0.22\linewidth]{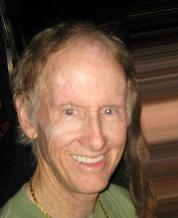} \\
        \includegraphics[width=0.22\linewidth, height=0.22\linewidth]{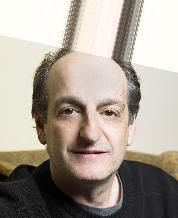} & \includegraphics[width=0.22\linewidth, height=0.22\linewidth]{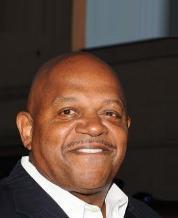} & \includegraphics[width=0.22\linewidth, height=0.22\linewidth]{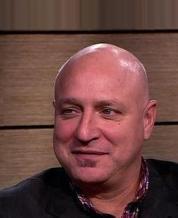} & \includegraphics[width=0.22\linewidth, height=0.22\linewidth]{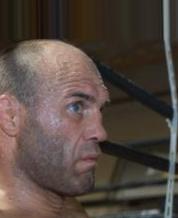} \\
        \includegraphics[width=0.22\linewidth, height=0.22\linewidth]{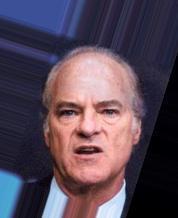} & \includegraphics[width=0.22\linewidth, height=0.22\linewidth]{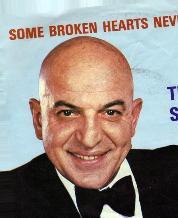} & \includegraphics[width=0.22\linewidth, height=0.22\linewidth]{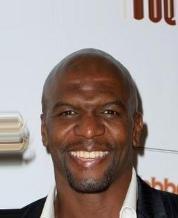} & \includegraphics[width=0.22\linewidth, height=0.22\linewidth]{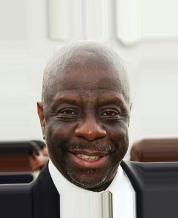}
      \end{tabular}
      \caption{Spurious OOD images of CelebA datasets.}
    \end{subfigure}
  }
  \caption{Examples of ID and spurious OOD images of CelebA datasets.}
  \label{fig:celebA}
\end{figure}

\noindent\textbf{CelebA:} CelebA $\mathbb{D}_\text{in}$~\citep{liu2015faceattributes} is a large-scale face attributes datasets containing hair color attributes $\{\texttt{grey}, \texttt{non-grey}\}$. We consider the label space as $\mathcal{Y} = \{\texttt{grey hair}, \texttt{nongrey hair}\}$. The environments $\mathcal{E} = \{\texttt{male}, \texttt{female}\}$ denote the gender of the person. The correlation between $\texttt{grey hair}$ and $\texttt{male}$ gender (environmental feature) is $\sim0.8$ in the training set. Spurious OOD inputs consist of \texttt{bald male}, which contain environmental feature (gender) without invariant feature (hair). \\

\noindent The examples of ID and spurious OOD images of CelebA datasets are presented in \Cref{fig:celebA}. For a detailed formalization of ID and spurious OOD of Waterbirds and CelebA datasets, we refer readers to Ming \etal~\cite{ming2021impact}. \\

\subsection{Fine-grained setting} 
\label{sec:fine_grained}
In this setting, ID/OOD splits are established through a holdout class approach. Specifically, a subset of categories is designated as ID, while the remaining categories are excluded from the training set and treated as OOD during testing. \\

\begin{figure}[h!]
  \centering
  \resizebox{0.85\textwidth}{!}{%
    \begin{subfigure}[t]{0.45\textwidth}
      \centering
      \begin{tabular}{cccc}
        \includegraphics[width=0.22\linewidth, height=0.22\linewidth]{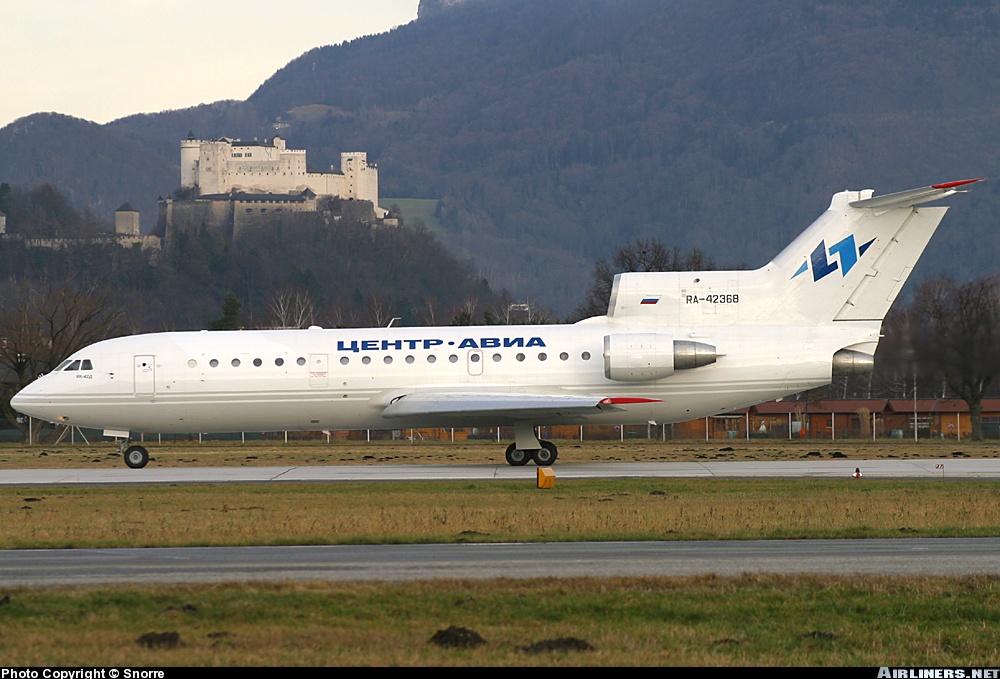} & \includegraphics[width=0.22\linewidth, height=0.22\linewidth]{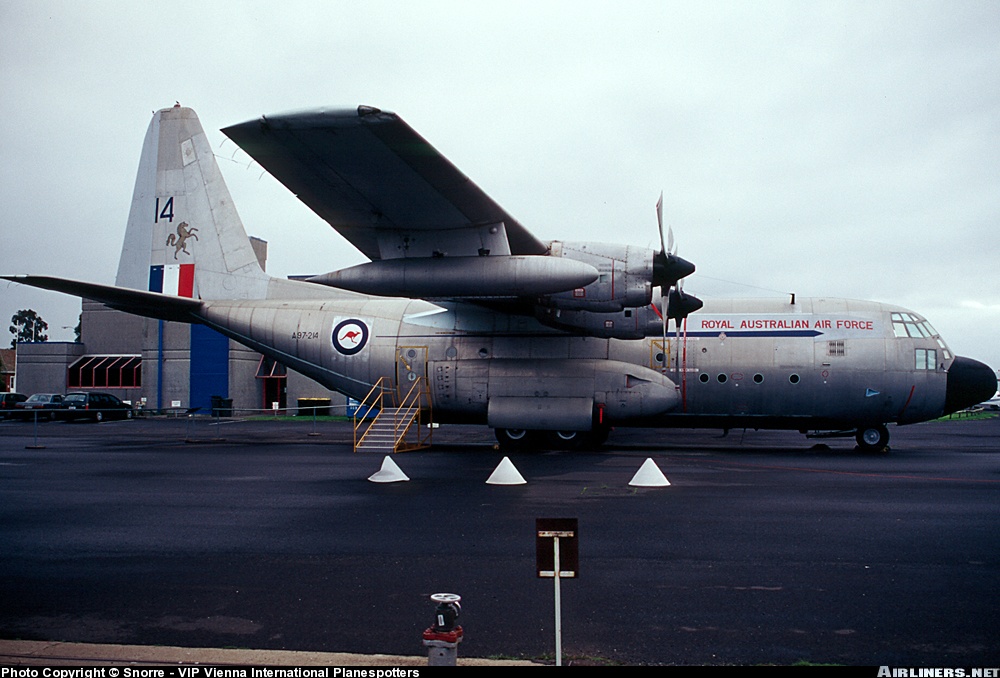} & \includegraphics[width=0.22\linewidth, height=0.22\linewidth]{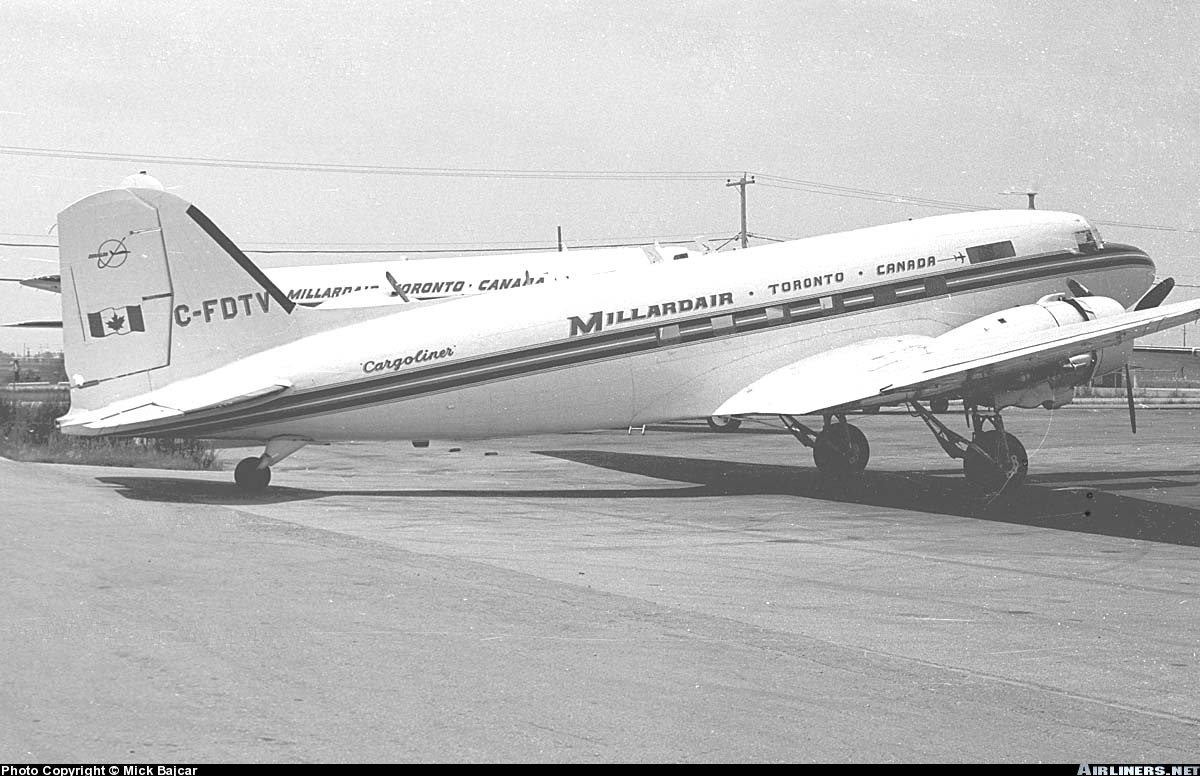} & \includegraphics[width=0.22\linewidth, height=0.22\linewidth]{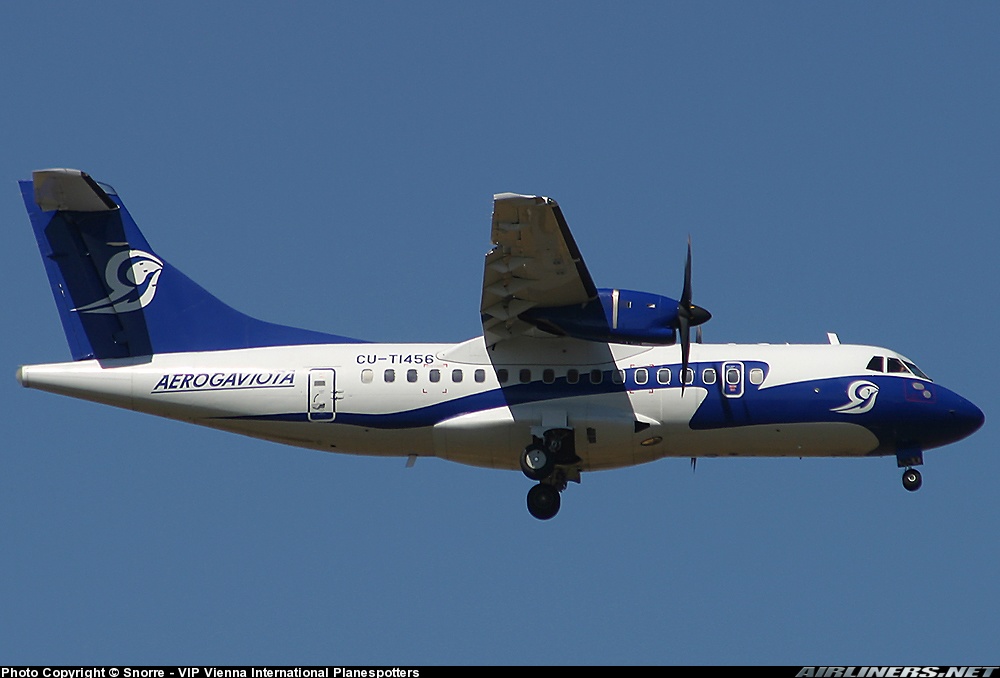} \\
        \includegraphics[width=0.22\linewidth, height=0.22\linewidth]{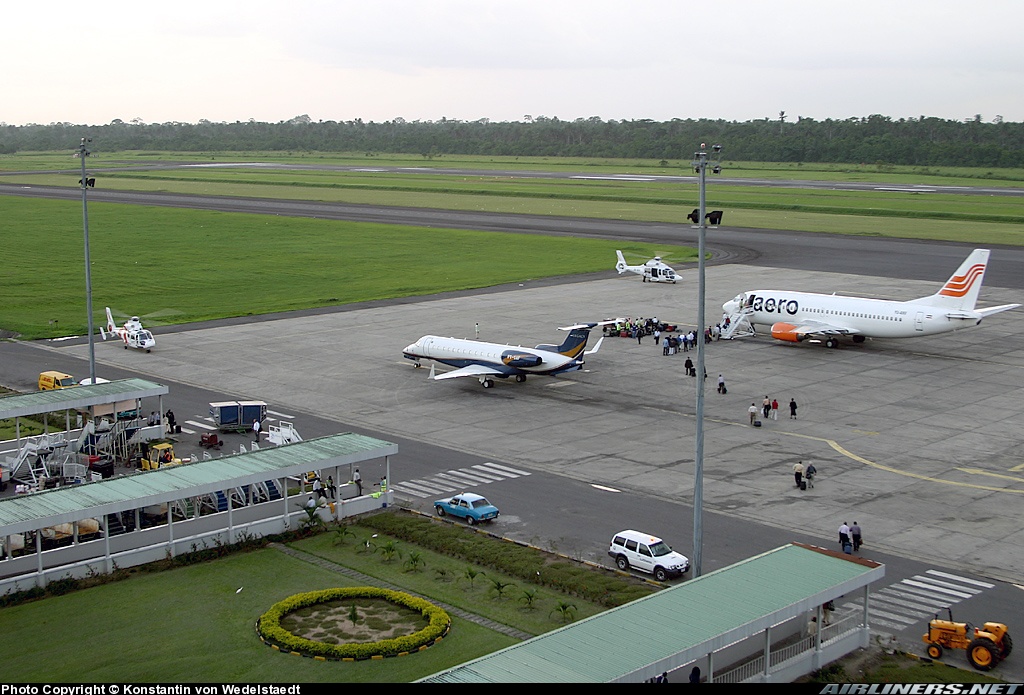} & \includegraphics[width=0.22\linewidth, height=0.22\linewidth]{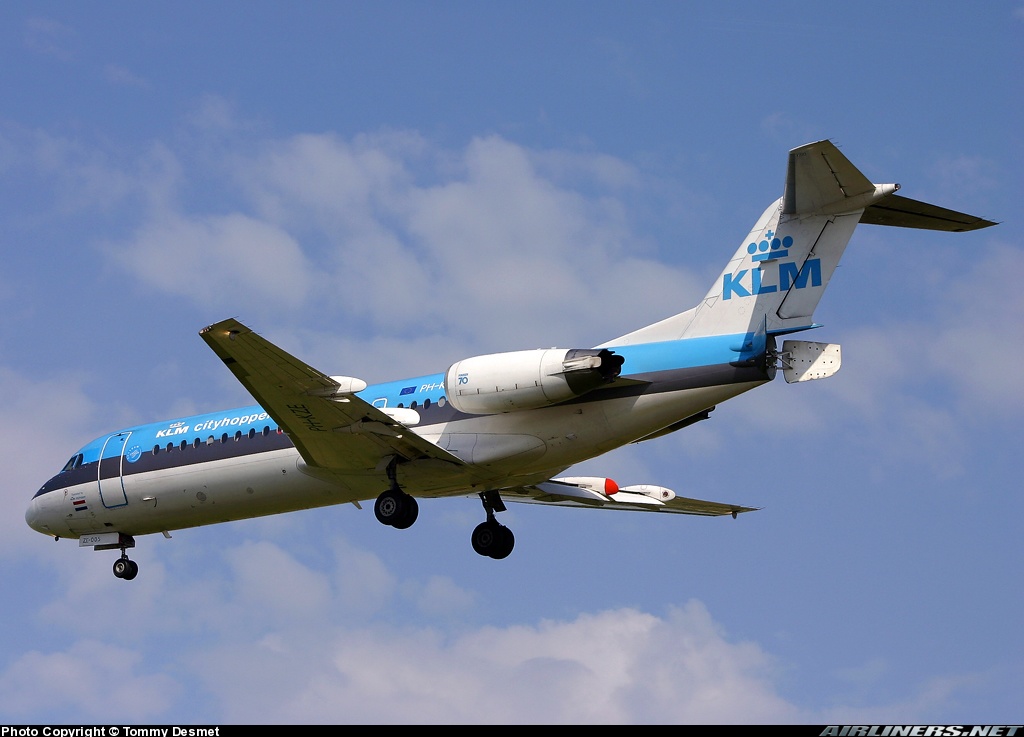} & \includegraphics[width=0.22\linewidth, height=0.22\linewidth]{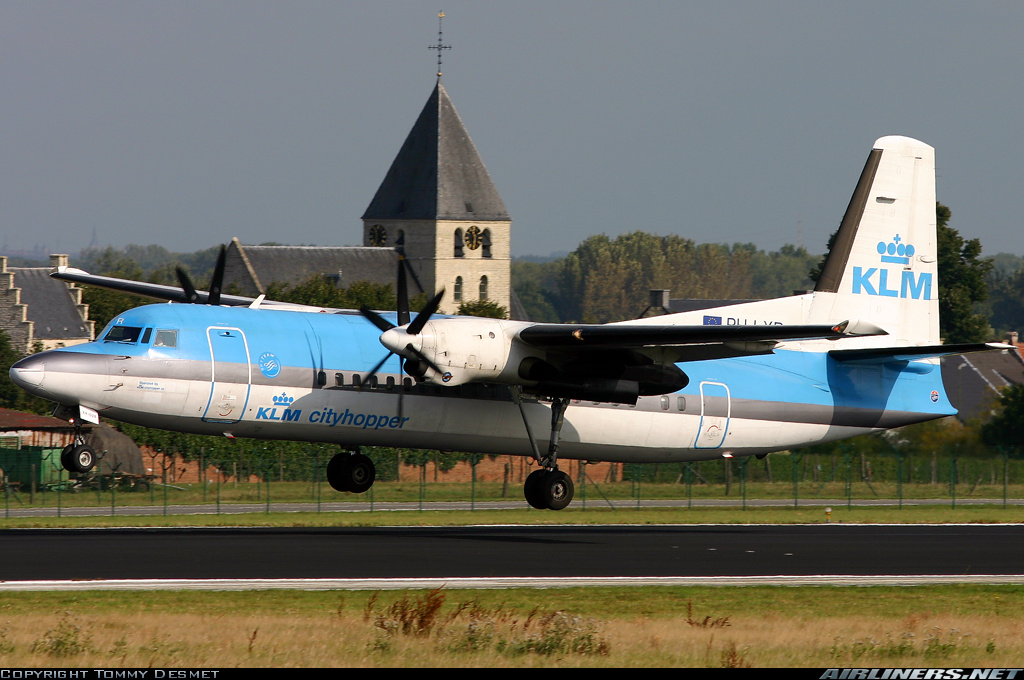} & \includegraphics[width=0.22\linewidth, height=0.22\linewidth]{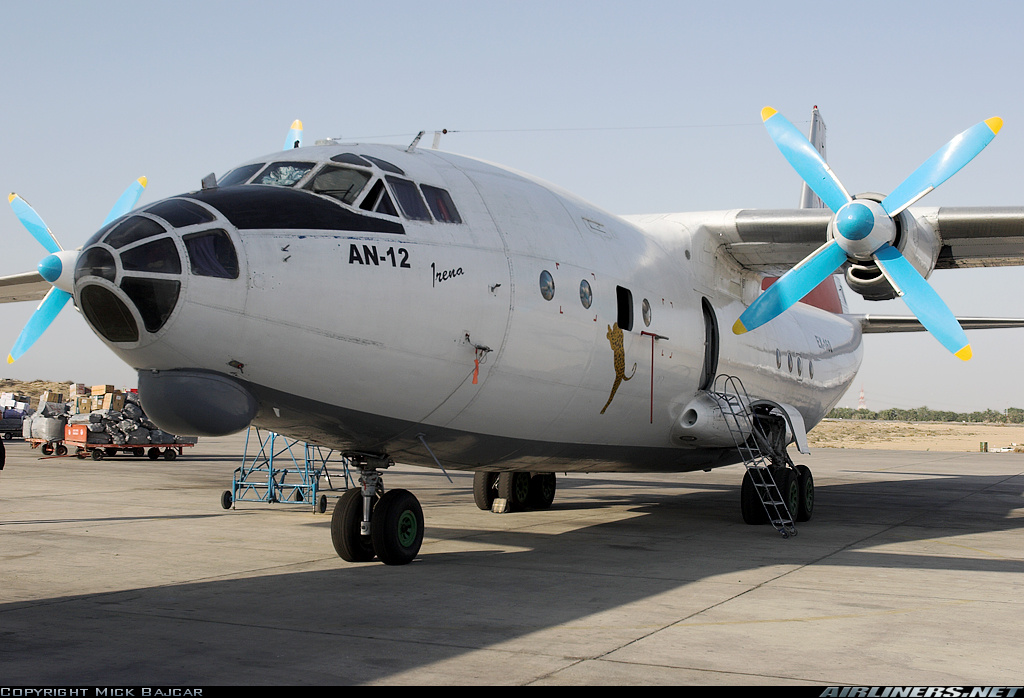} \\
        \includegraphics[width=0.22\linewidth, height=0.22\linewidth]{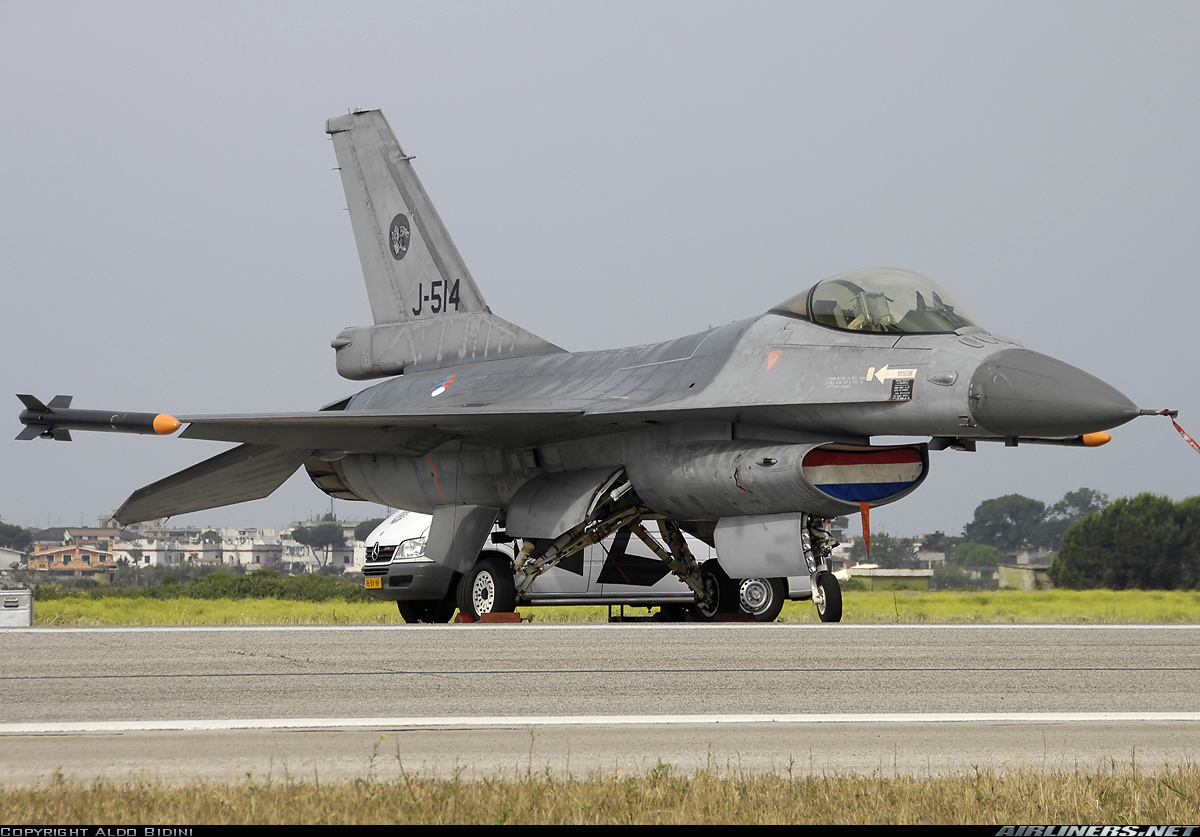} & \includegraphics[width=0.22\linewidth, height=0.22\linewidth]{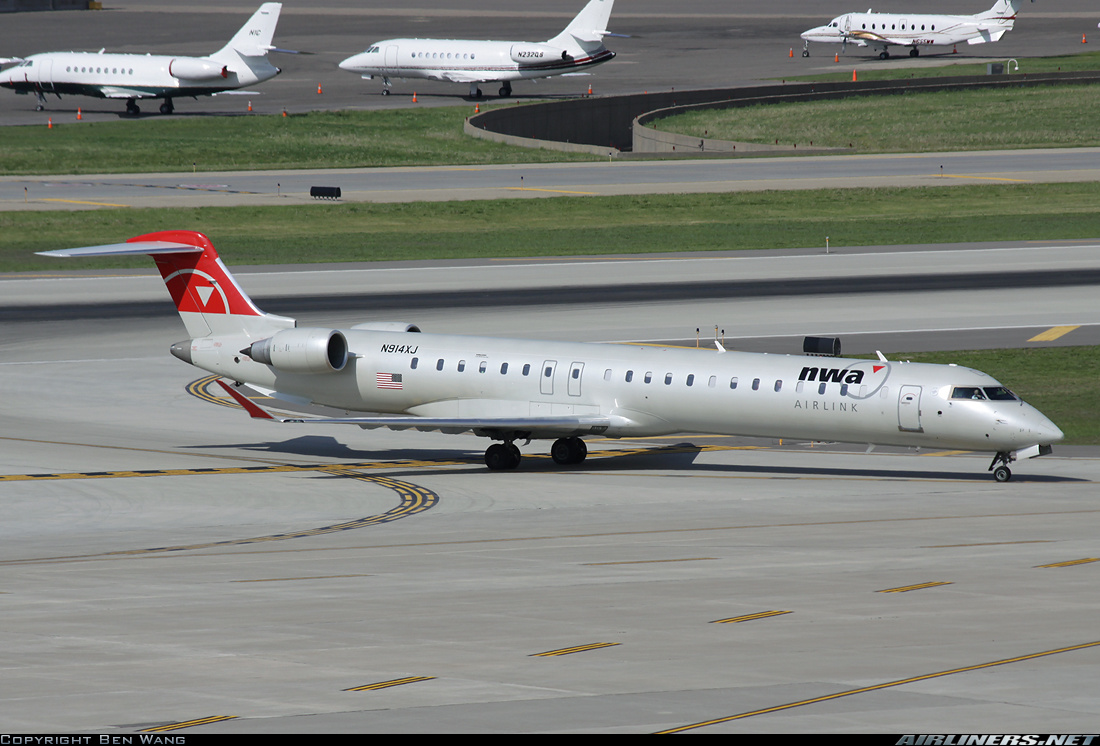} & \includegraphics[width=0.22\linewidth, height=0.22\linewidth]{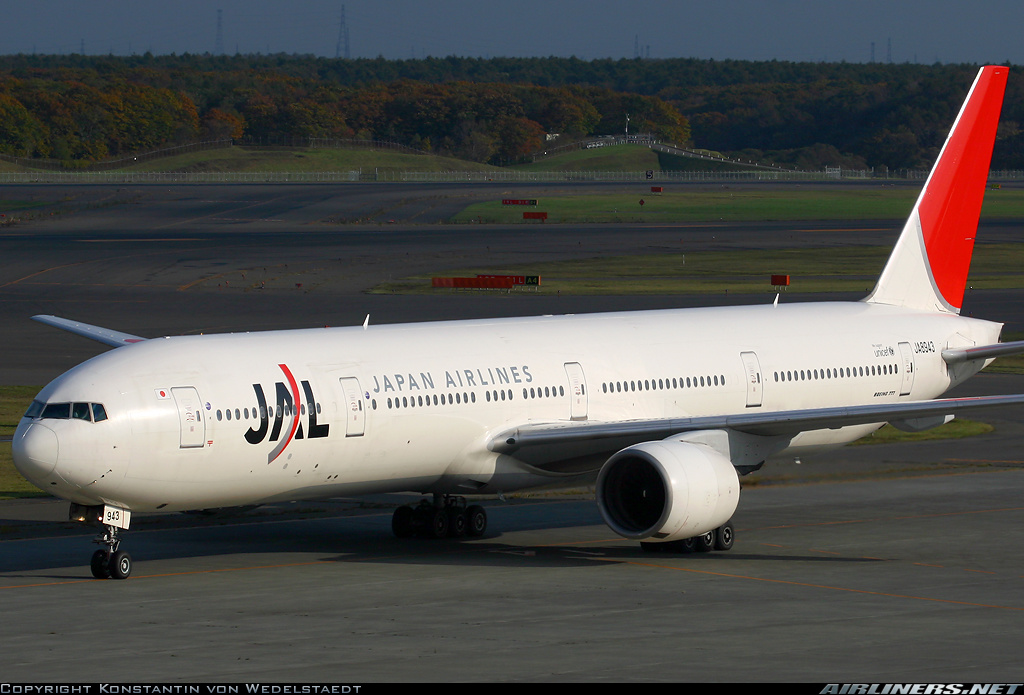} & \includegraphics[width=0.22\linewidth, height=0.22\linewidth]{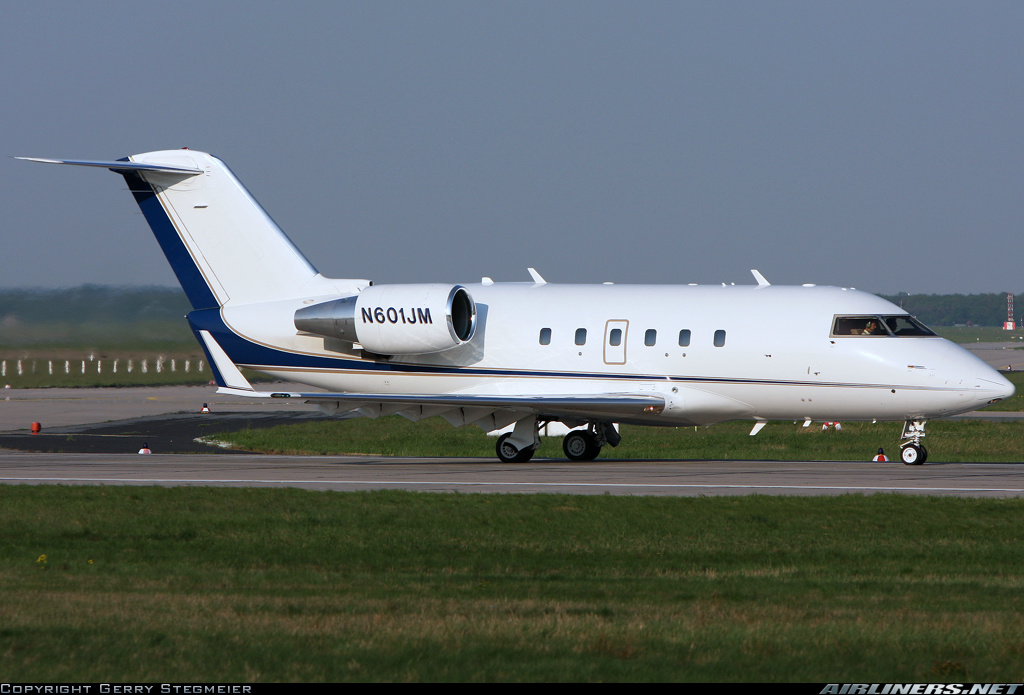} \\
        \includegraphics[width=0.22\linewidth, height=0.22\linewidth]{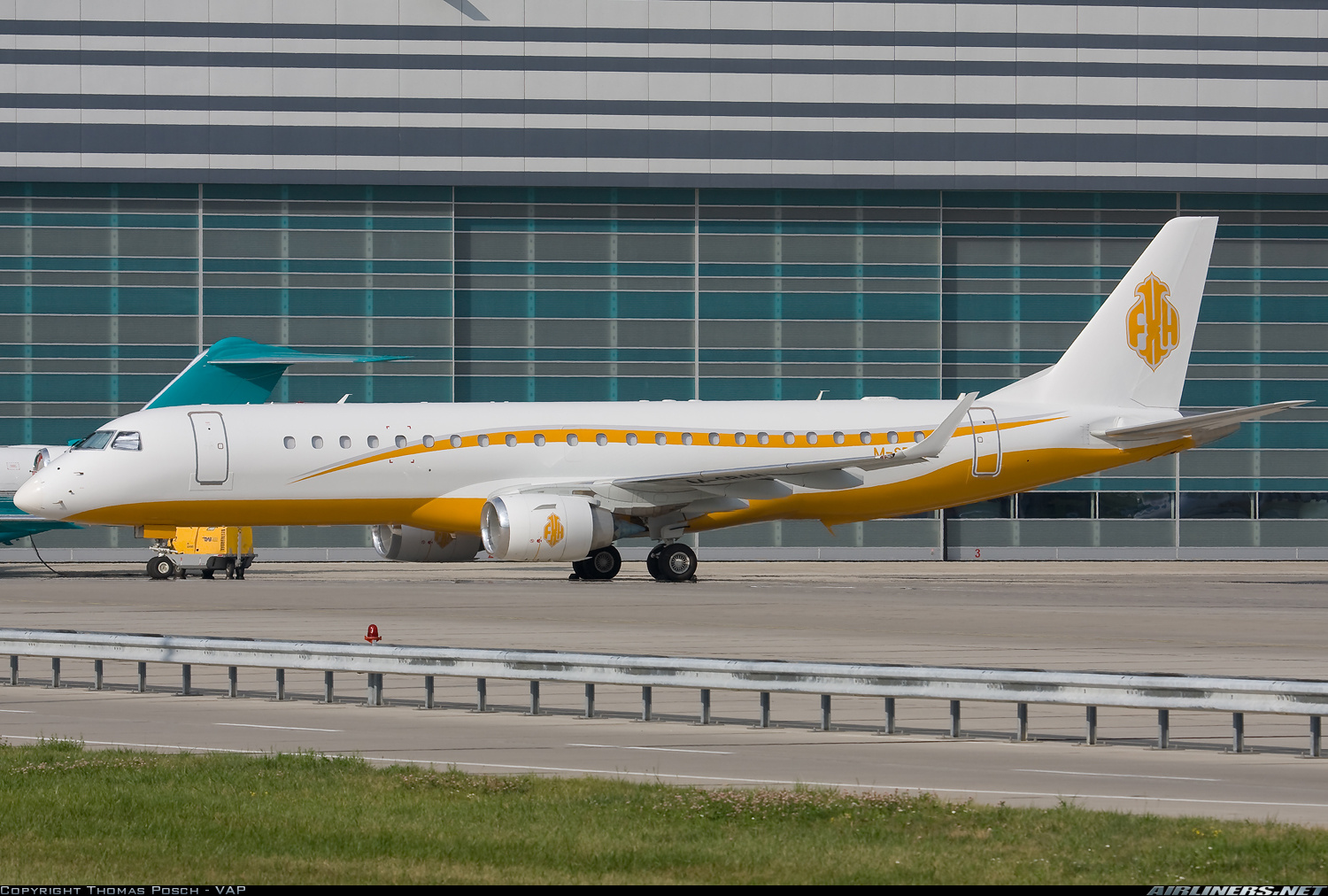} & \includegraphics[width=0.22\linewidth, height=0.22\linewidth]{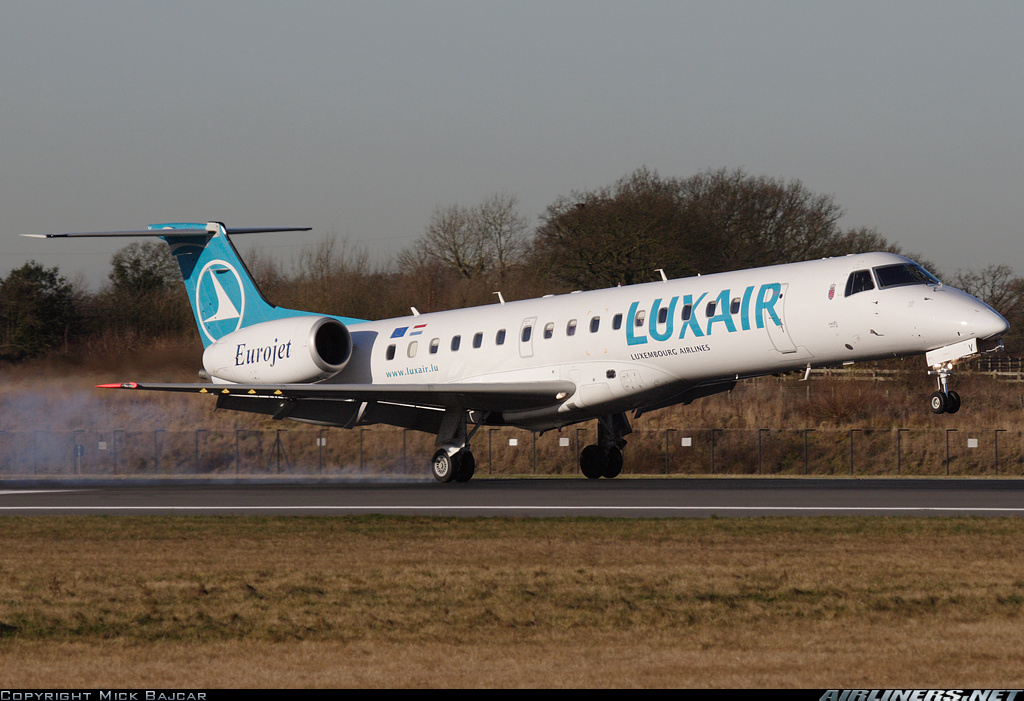} & \includegraphics[width=0.22\linewidth, height=0.22\linewidth]{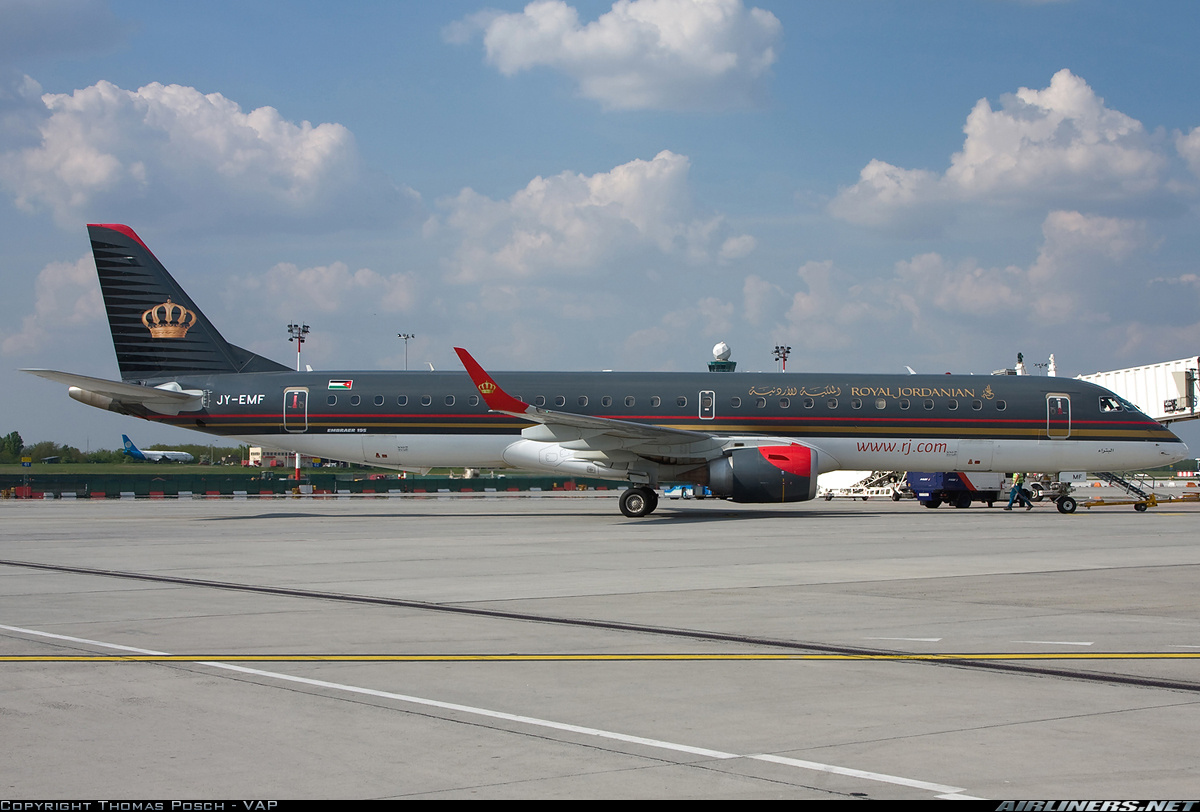} & \includegraphics[width=0.22\linewidth, height=0.22\linewidth]{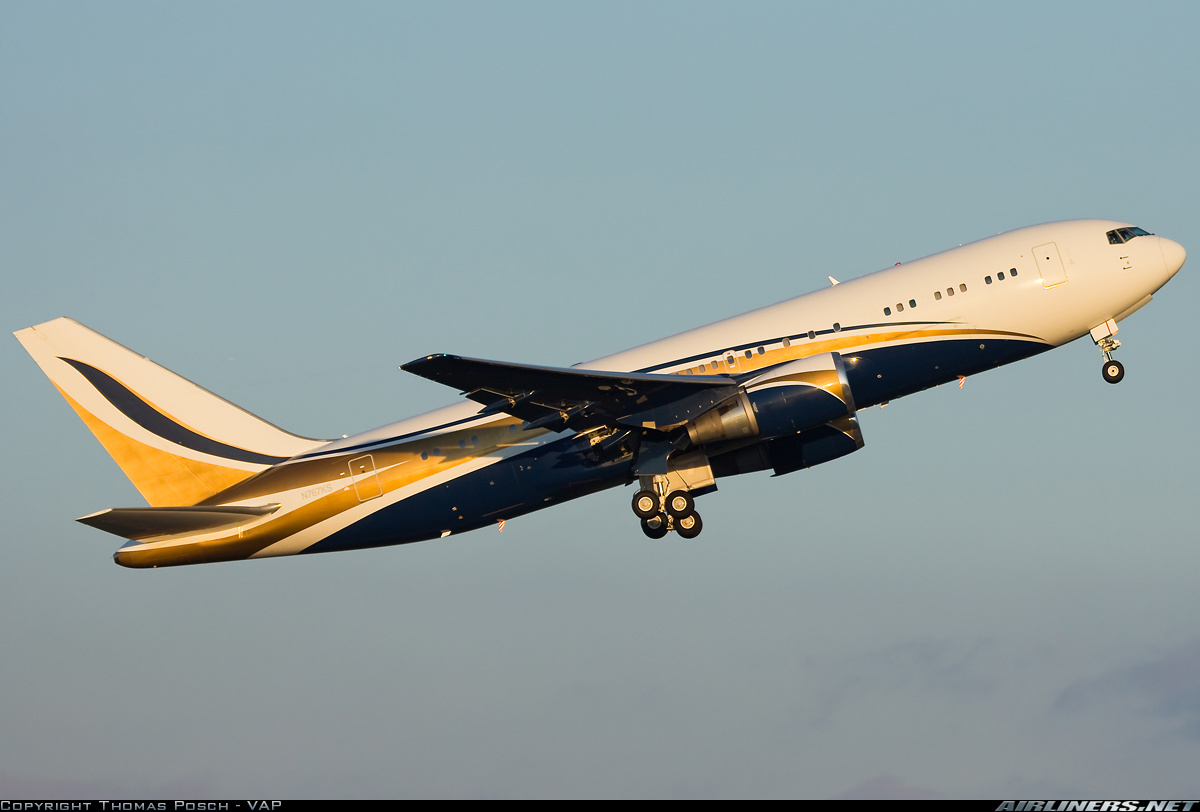}
      \end{tabular}
      \caption{ID images of Aircraft datasets.}
    \end{subfigure}
    \hspace{0.1\linewidth} 
    \begin{subfigure}[t]{0.45\textwidth}
      \centering
      \begin{tabular}{cccc}
        \includegraphics[width=0.22\linewidth, height=0.22\linewidth]{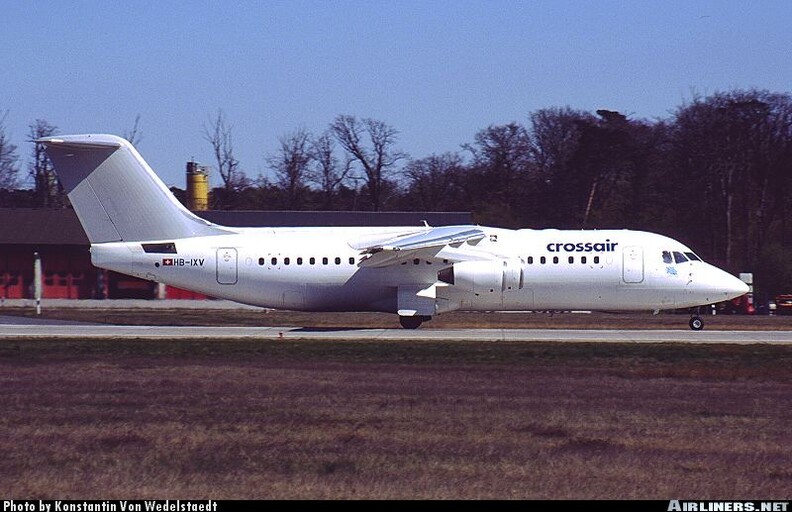} & \includegraphics[width=0.22\linewidth, height=0.22\linewidth]{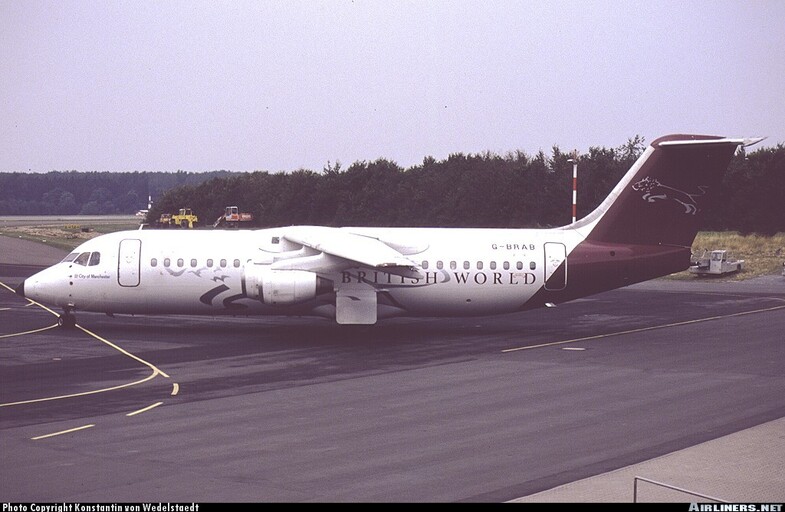} & \includegraphics[width=0.22\linewidth, height=0.22\linewidth]{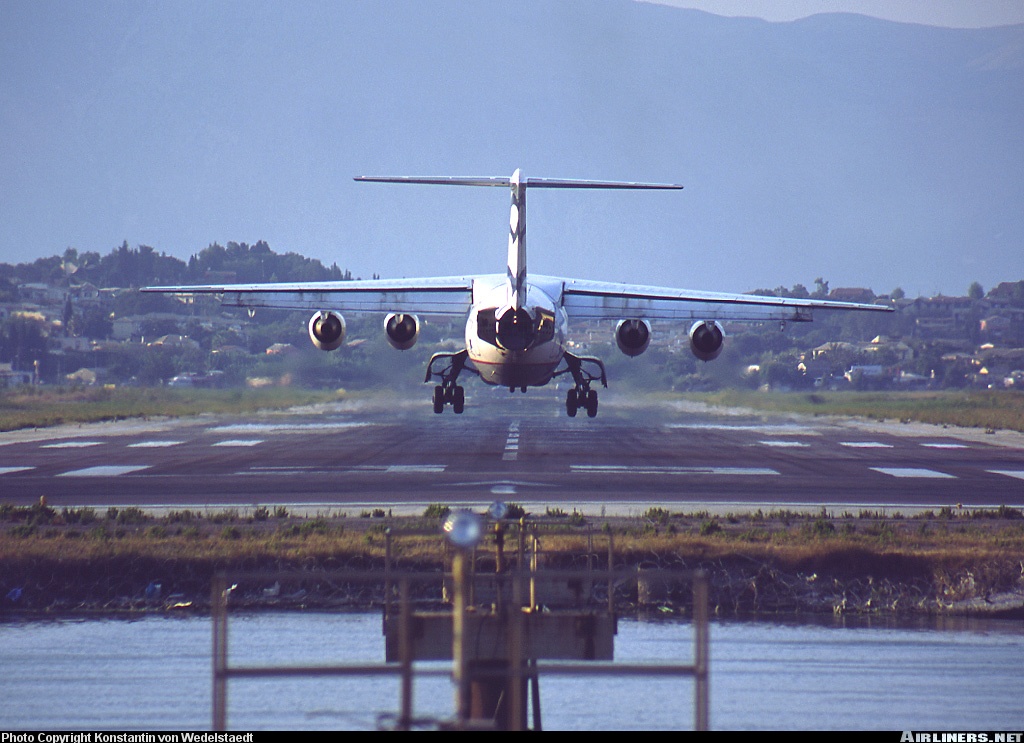} & \includegraphics[width=0.22\linewidth, height=0.22\linewidth]{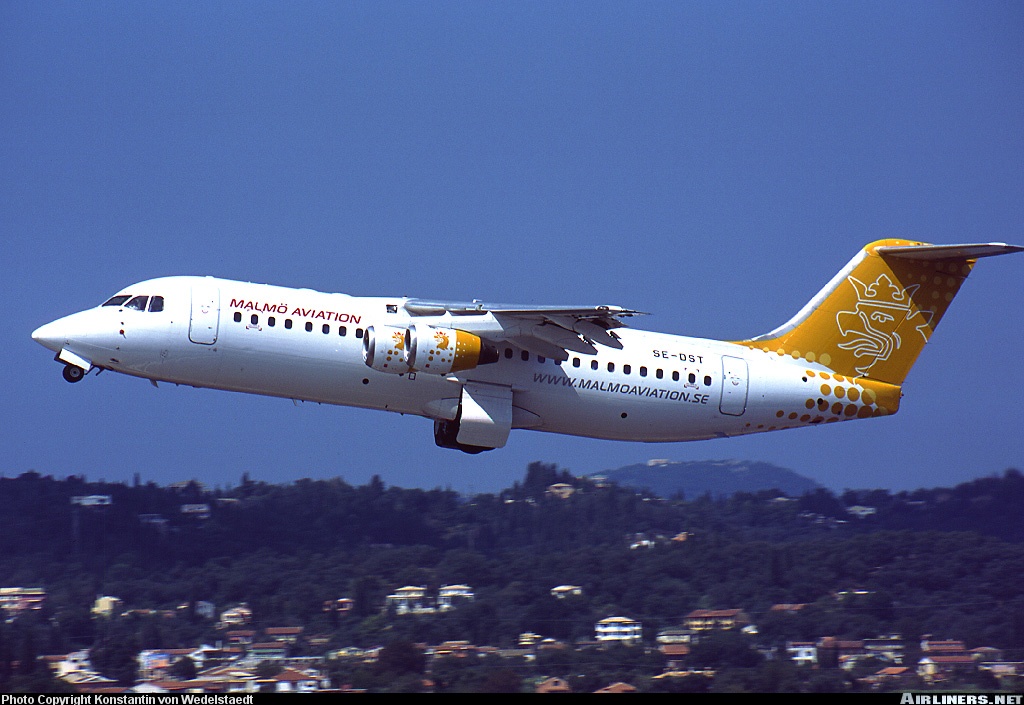} \\
        \includegraphics[width=0.22\linewidth, height=0.22\linewidth]{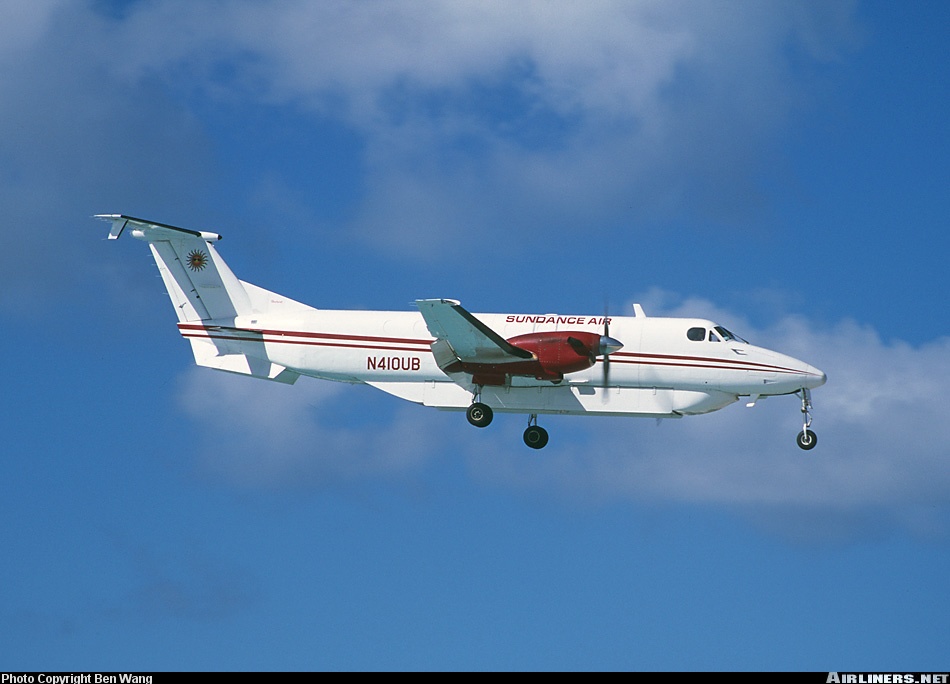} & \includegraphics[width=0.22\linewidth, height=0.22\linewidth]{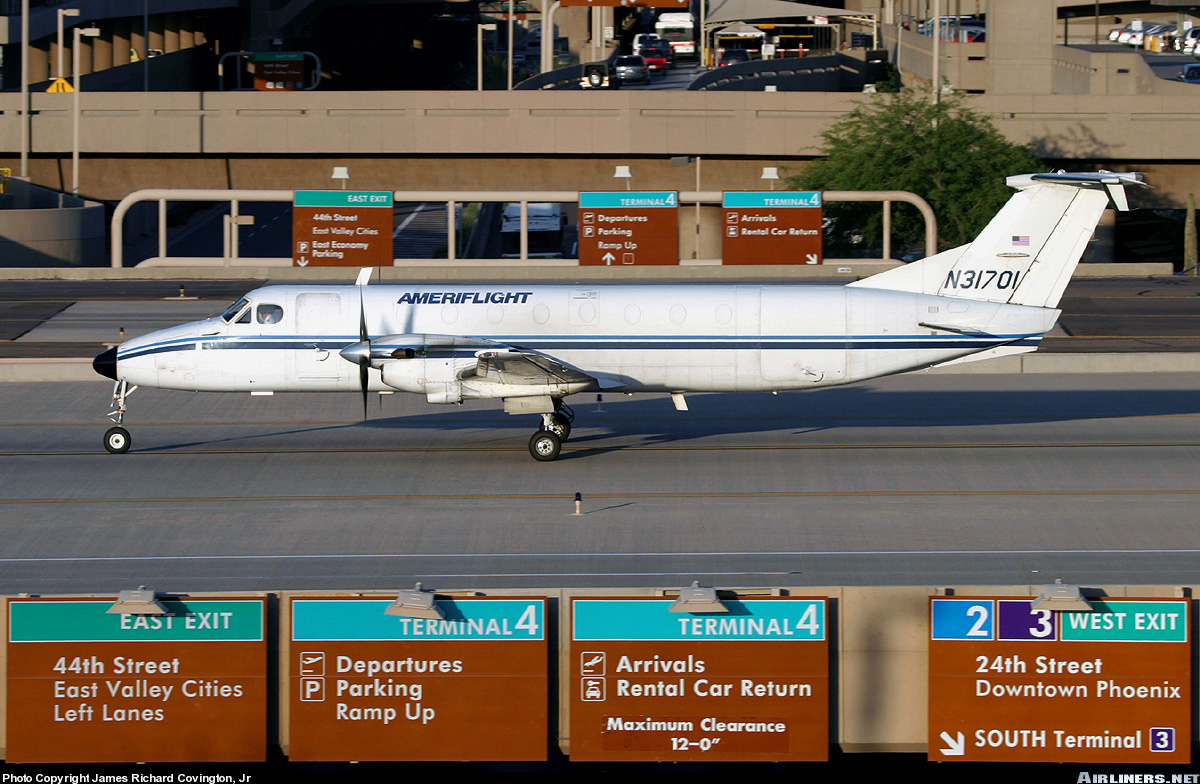} & \includegraphics[width=0.22\linewidth, height=0.22\linewidth]{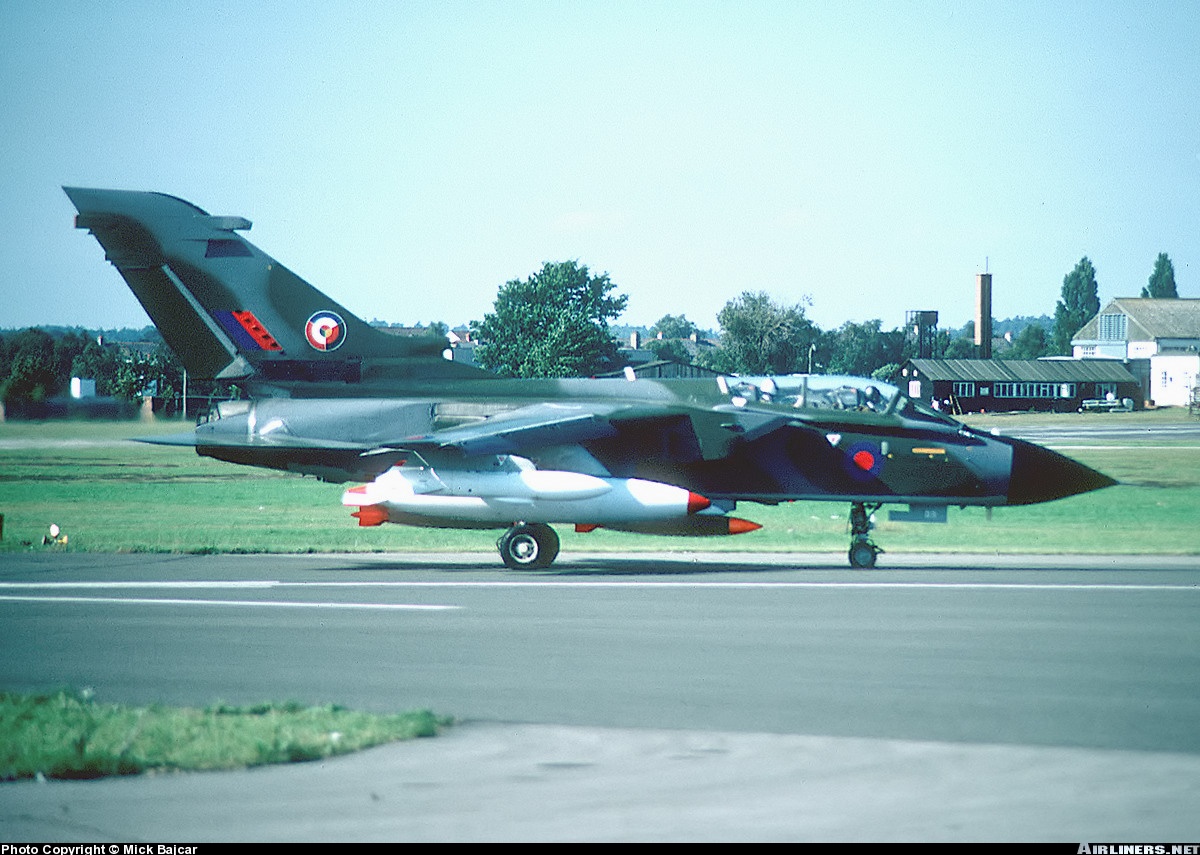} & \includegraphics[width=0.22\linewidth, height=0.22\linewidth]{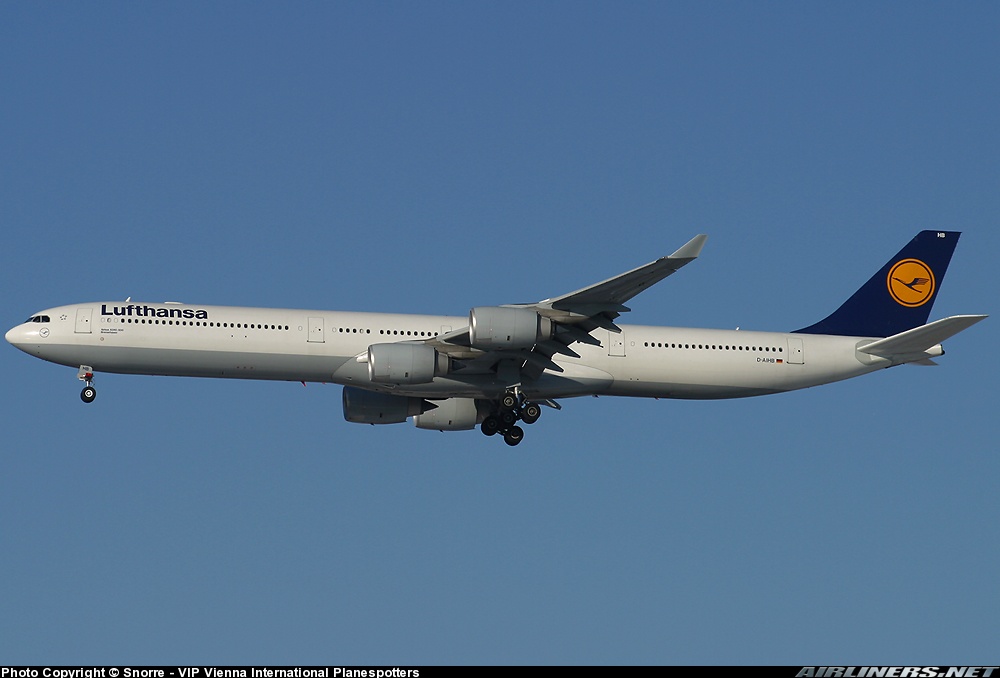} \\
        \includegraphics[width=0.22\linewidth, height=0.22\linewidth]{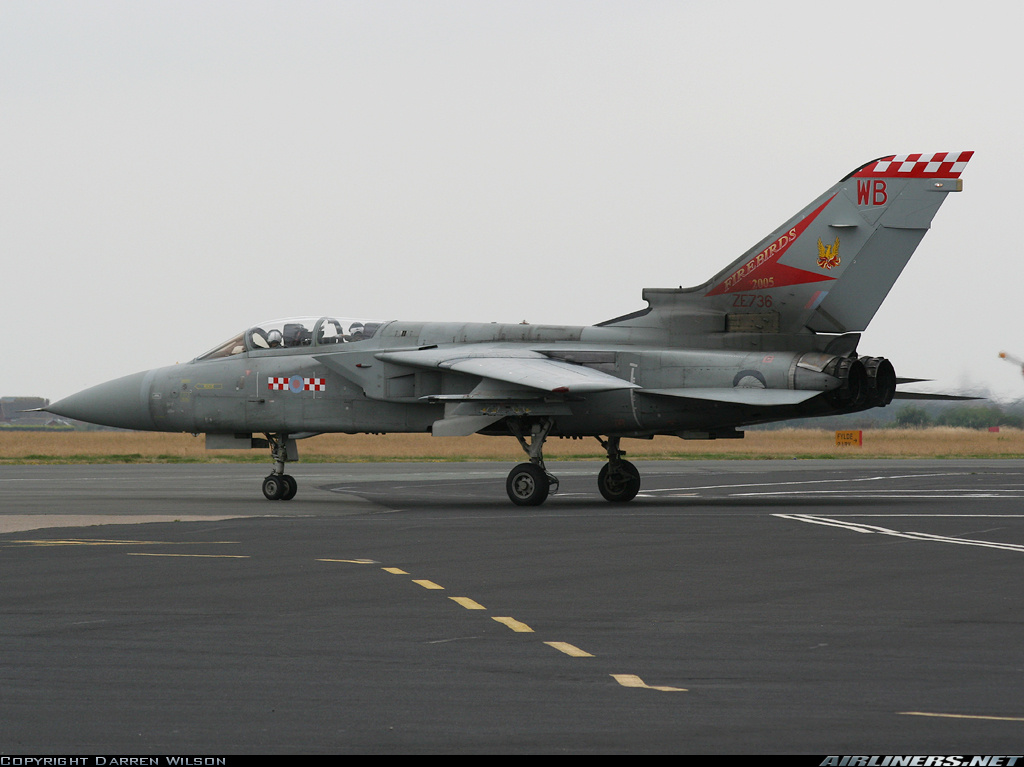} & \includegraphics[width=0.22\linewidth, height=0.22\linewidth]{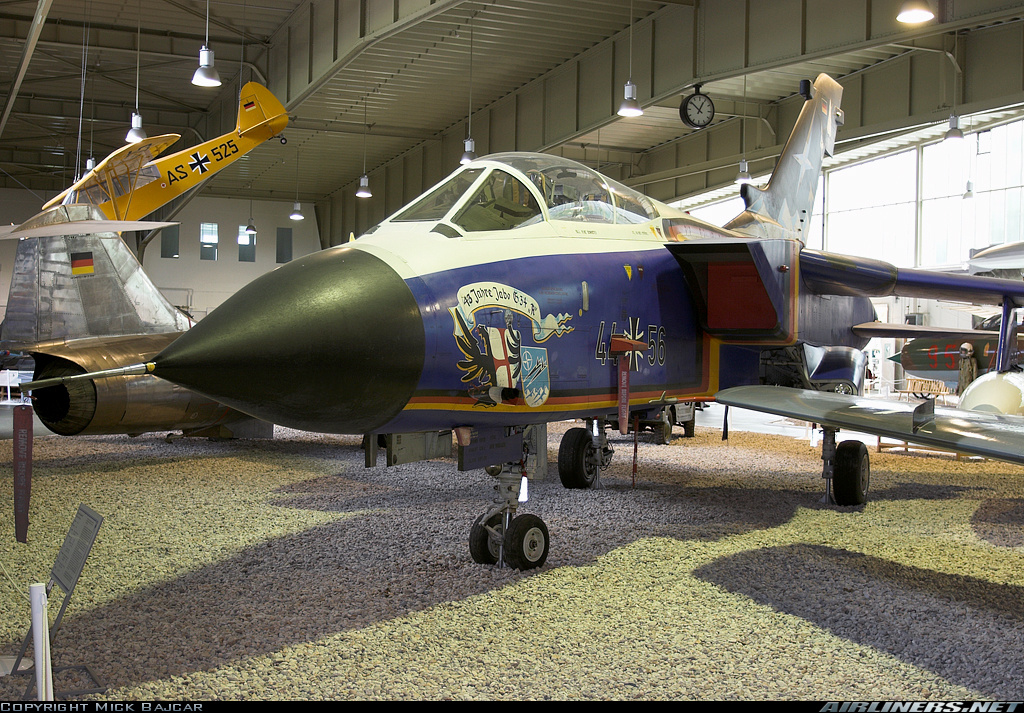} & \includegraphics[width=0.22\linewidth, height=0.22\linewidth]{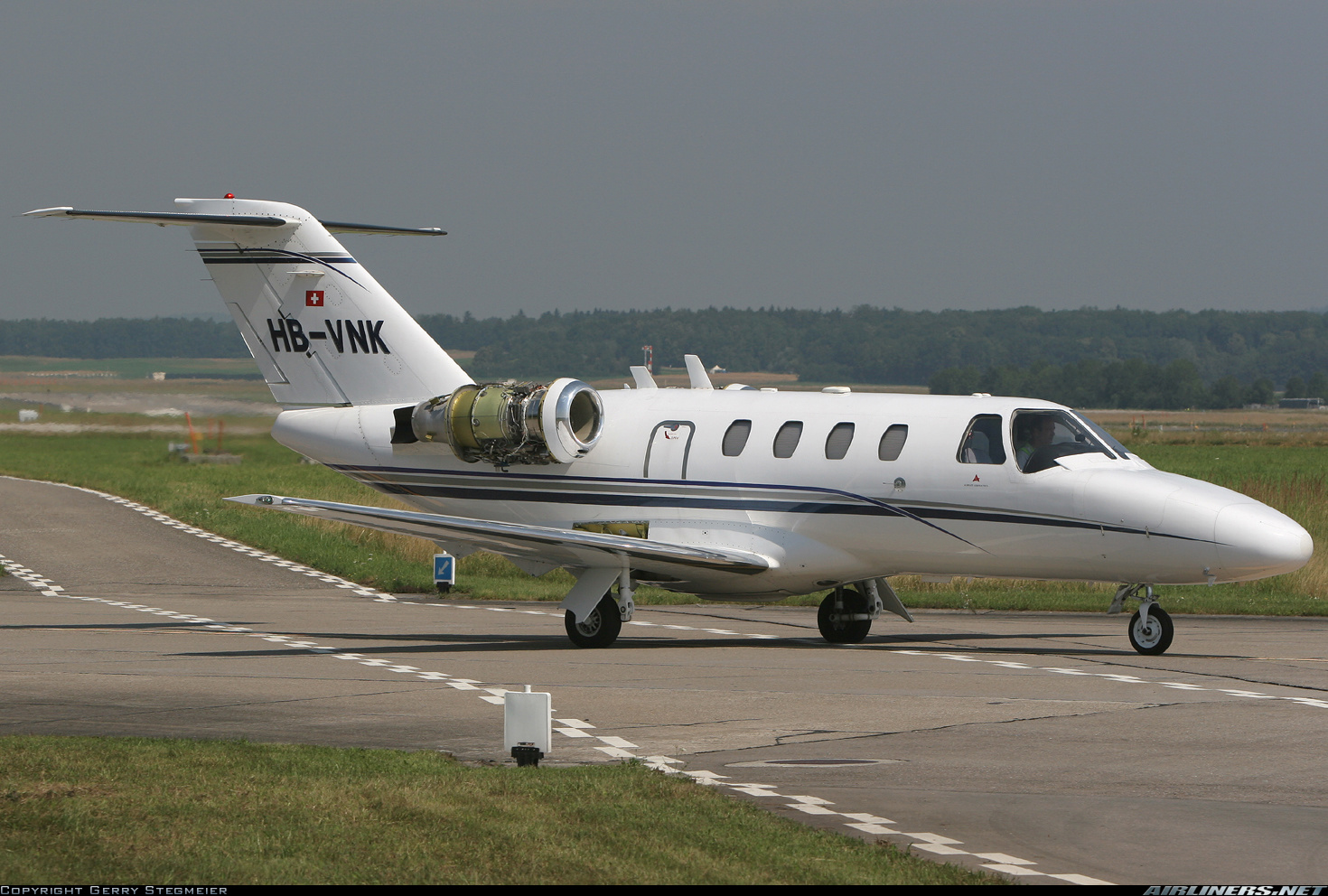} & \includegraphics[width=0.22\linewidth, height=0.22\linewidth]{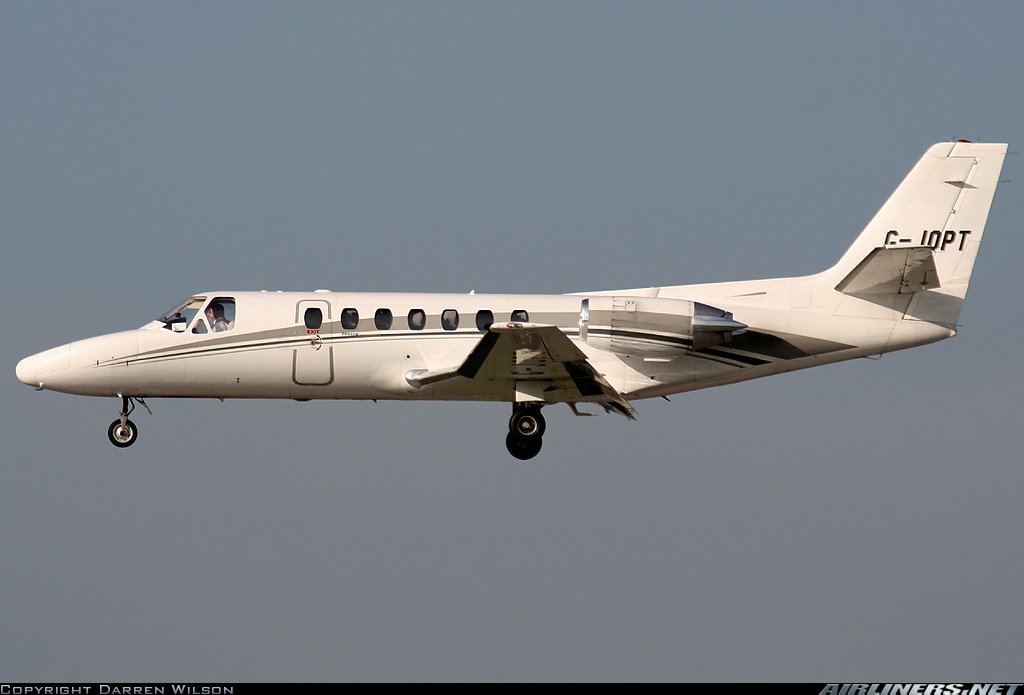} \\
        \includegraphics[width=0.22\linewidth, height=0.22\linewidth]{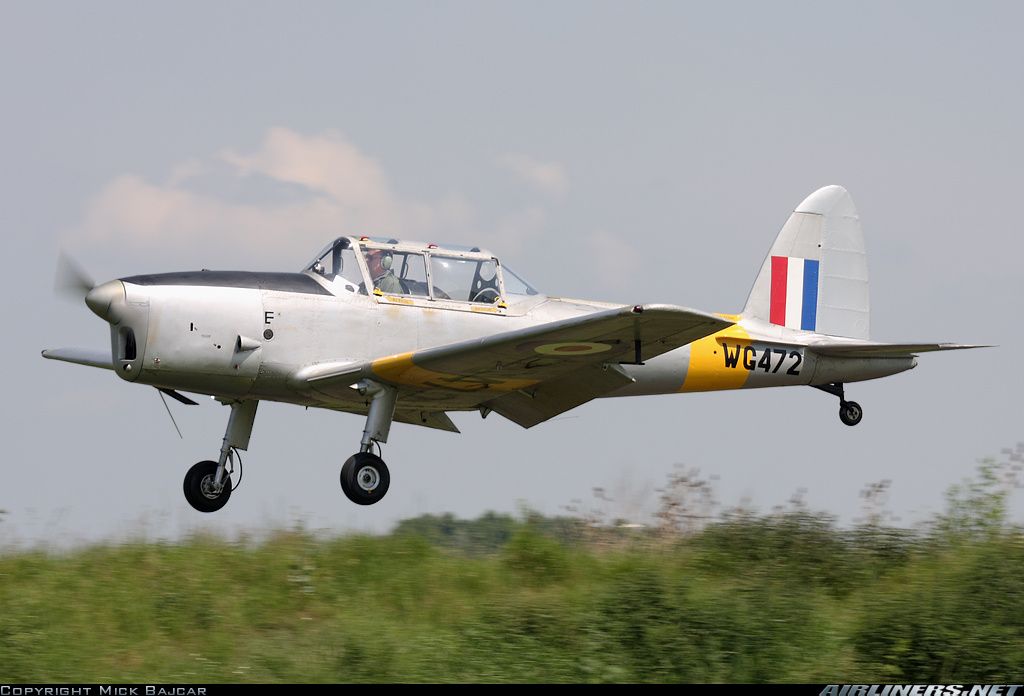} & \includegraphics[width=0.22\linewidth, height=0.22\linewidth]{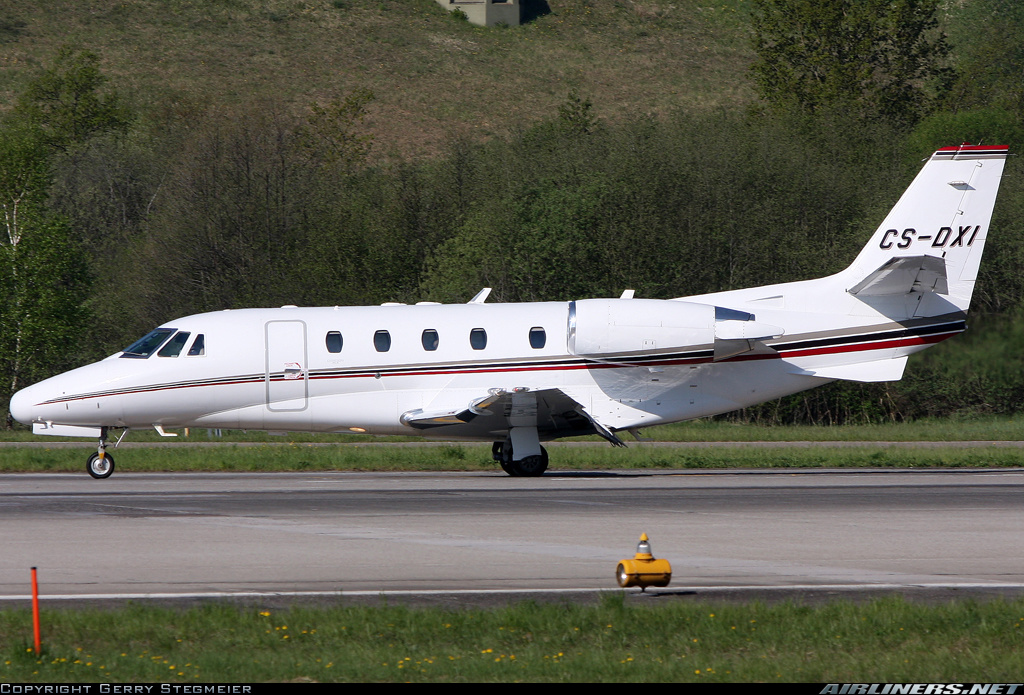} & \includegraphics[width=0.22\linewidth, height=0.22\linewidth]{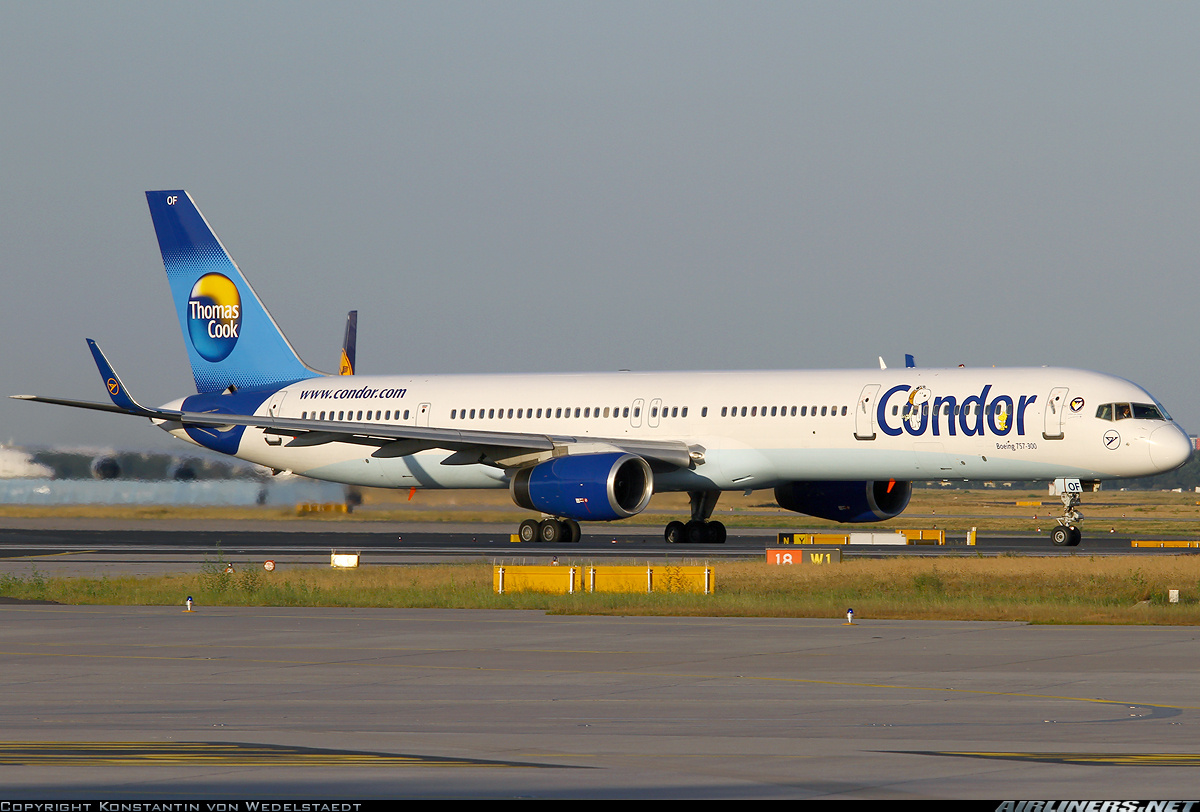} & \includegraphics[width=0.22\linewidth, height=0.22\linewidth]{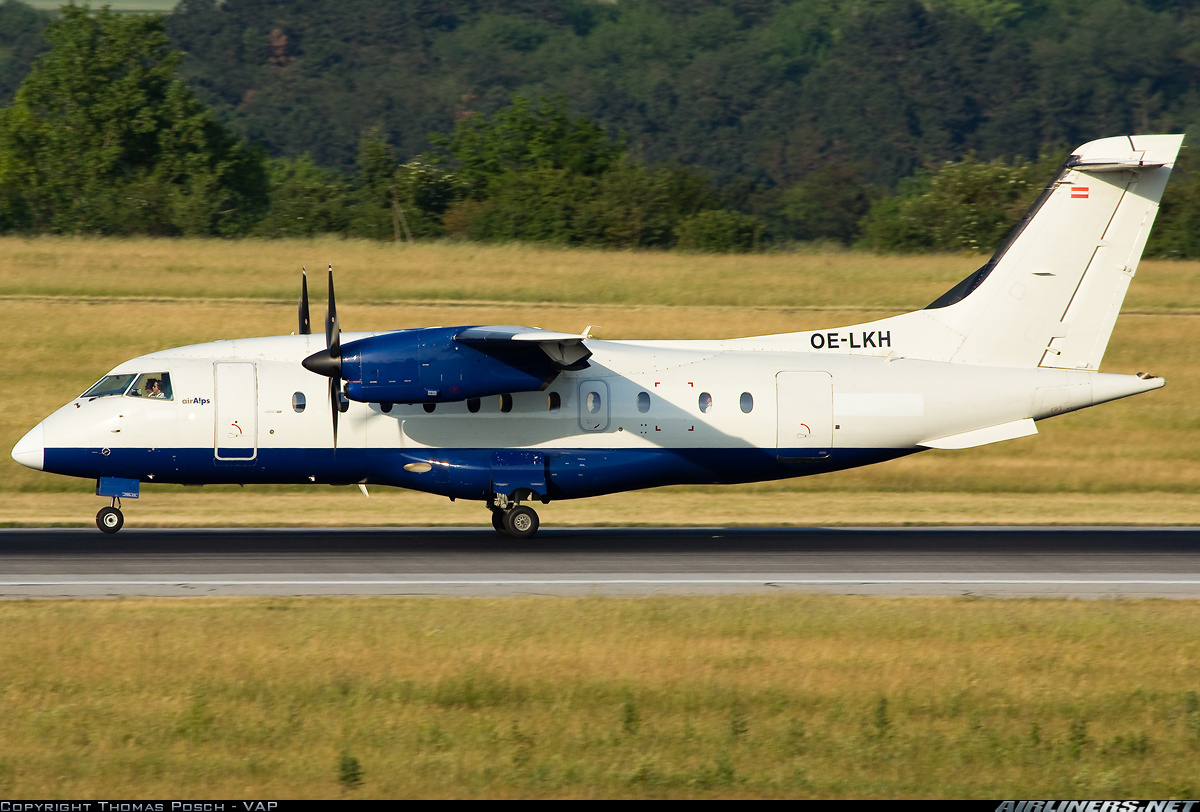}
      \end{tabular}
      \caption{Fine-grained OOD images of Aircraft datasets.}
    \end{subfigure}
  }
  \caption{Examples of ID and fine-grained OOD images of Aircraft datasets.}
  \label{fig:aircraft}
\end{figure}

\noindent\textbf{Aircraft}: We consider 100 classes of the Aircraft datasets. Out of these, 10 classes are randomly selected to serve as fine-grained OOD data, while the remaining classes constitute the ID set. The ID and OOD classes are: \\
\begin{itemize}
    \item ID classes: 0, 1, 2, 3, 4, 5, 6, 7, 8, 9, 10, 11, 12, 13, 14, 16, 18, 19, 20, 21, 22, 23, 24, 25, 26, 27, 28, 29, 30, 31, 33, 34, 35, 36, 37, 39, 41, 42, 43, 44, 45, 46, 47, 48, 51, 52, 53, 54, 55, 56, 57, 59, 60, 61, 62, 64, 65, 66, 67, 68, 69, 70, 71, 72, 73, 74, 75, 76, 77, 78, 79, 80, 81, 82, 83, 84, 85, 86, 87, 88, 89, 90, 91, 92, 93, 94, 95, 97, 98, 99 \\
    \item OOD classes: 15, 17, 32, 38, 40, 49, 50, 58, 63, 96 \\
\end{itemize}
\noindent The examples of ID and fine-grained OOD images of the Aircraft datasets are presented in \Cref{fig:aircraft}.

\begin{figure}[h!]
  \centering
  \resizebox{0.8\textwidth}{!}{%
    \begin{subfigure}[t]{0.45\textwidth}
      \centering
      \begin{tabular}{cccc}
        \includegraphics[width=0.22\linewidth, height=0.22\linewidth]{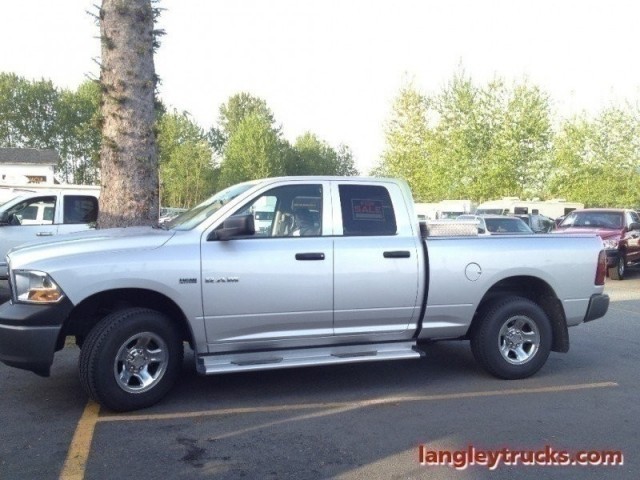} & \includegraphics[width=0.22\linewidth, height=0.22\linewidth]{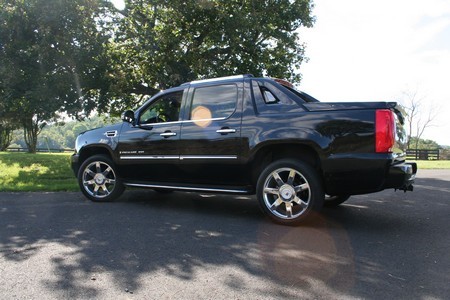} & \includegraphics[width=0.22\linewidth, height=0.22\linewidth]{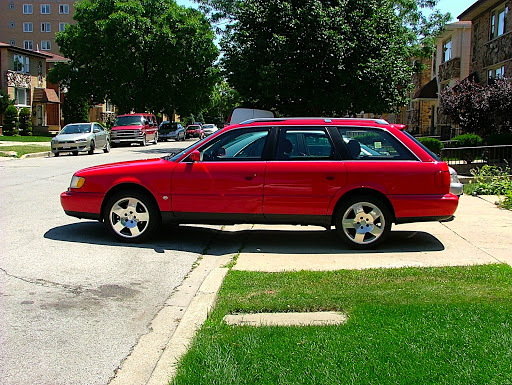} & \includegraphics[width=0.22\linewidth, height=0.22\linewidth]{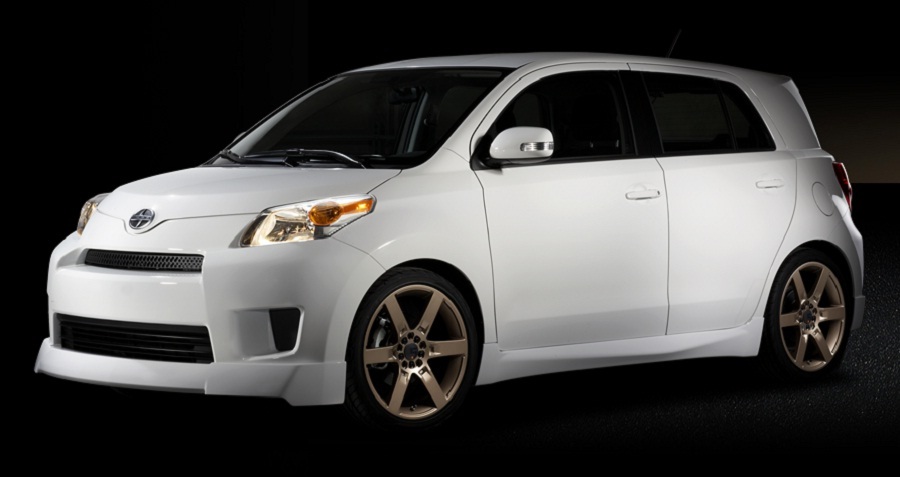} \\
        \includegraphics[width=0.22\linewidth, height=0.22\linewidth]{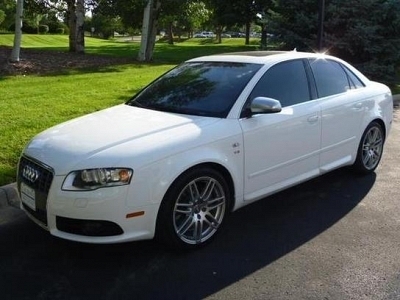} & \includegraphics[width=0.22\linewidth, height=0.22\linewidth]{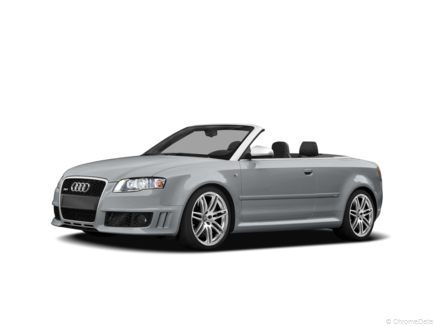} & \includegraphics[width=0.22\linewidth, height=0.22\linewidth]{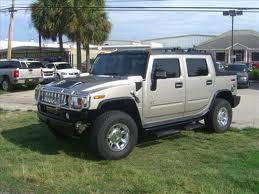} & \includegraphics[width=0.22\linewidth, height=0.22\linewidth]{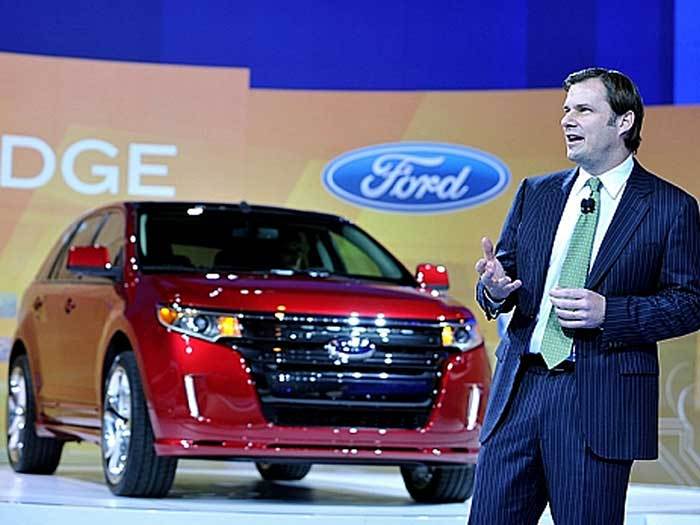} \\
        \includegraphics[width=0.22\linewidth, height=0.22\linewidth]{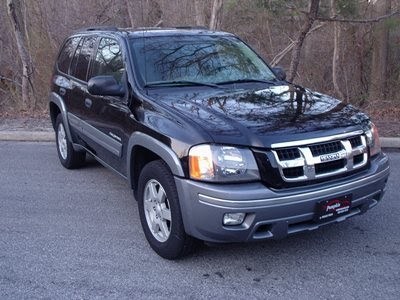} & \includegraphics[width=0.22\linewidth, height=0.22\linewidth]{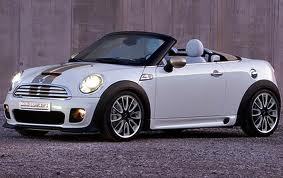} & \includegraphics[width=0.22\linewidth, height=0.22\linewidth]{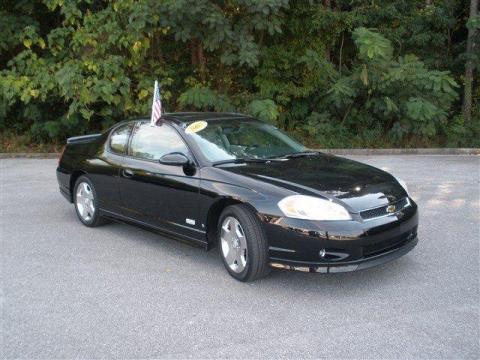} & \includegraphics[width=0.22\linewidth, height=0.22\linewidth]{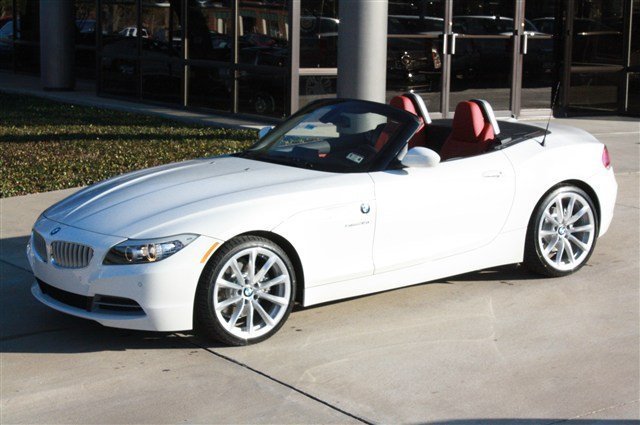} \\
        \includegraphics[width=0.22\linewidth, height=0.22\linewidth]{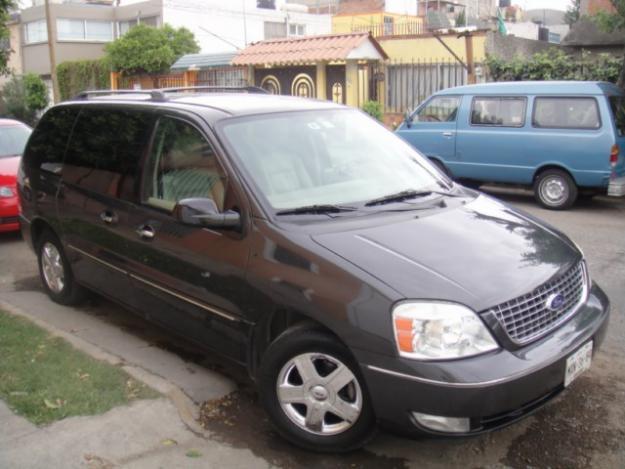} & \includegraphics[width=0.22\linewidth, height=0.22\linewidth]{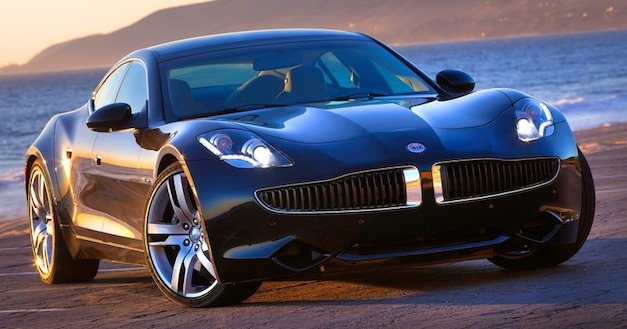} & \includegraphics[width=0.22\linewidth, height=0.22\linewidth]{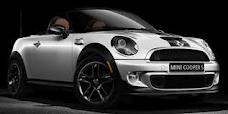} & \includegraphics[width=0.22\linewidth, height=0.22\linewidth]{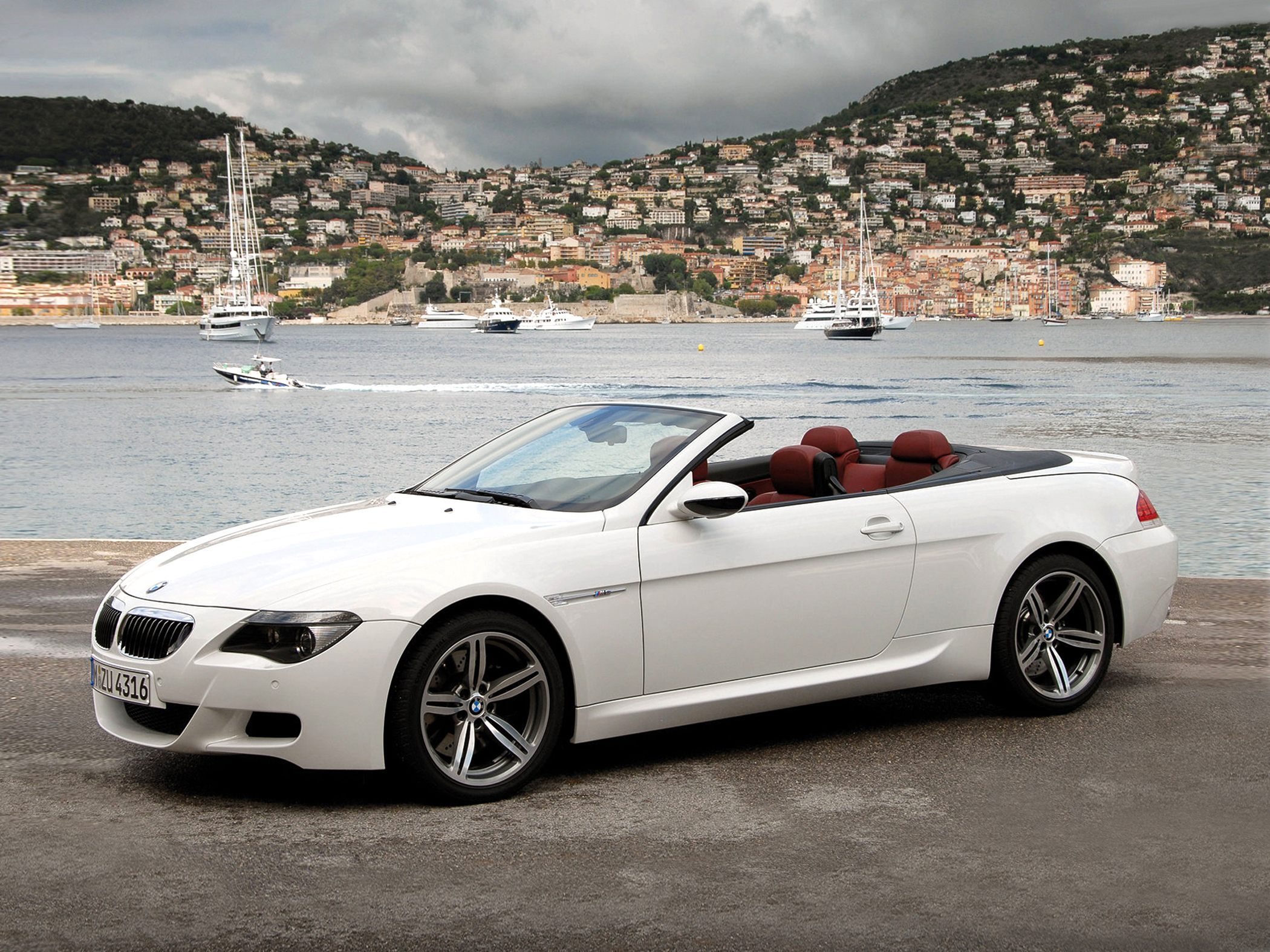}
      \end{tabular}
      \caption{ID images of Car datasets.}
    \end{subfigure}
    \hspace{0.1\linewidth}
    \begin{subfigure}[t]{0.45\textwidth}
      \centering
      \begin{tabular}{cccc}
        \includegraphics[width=0.22\linewidth, height=0.22\linewidth]{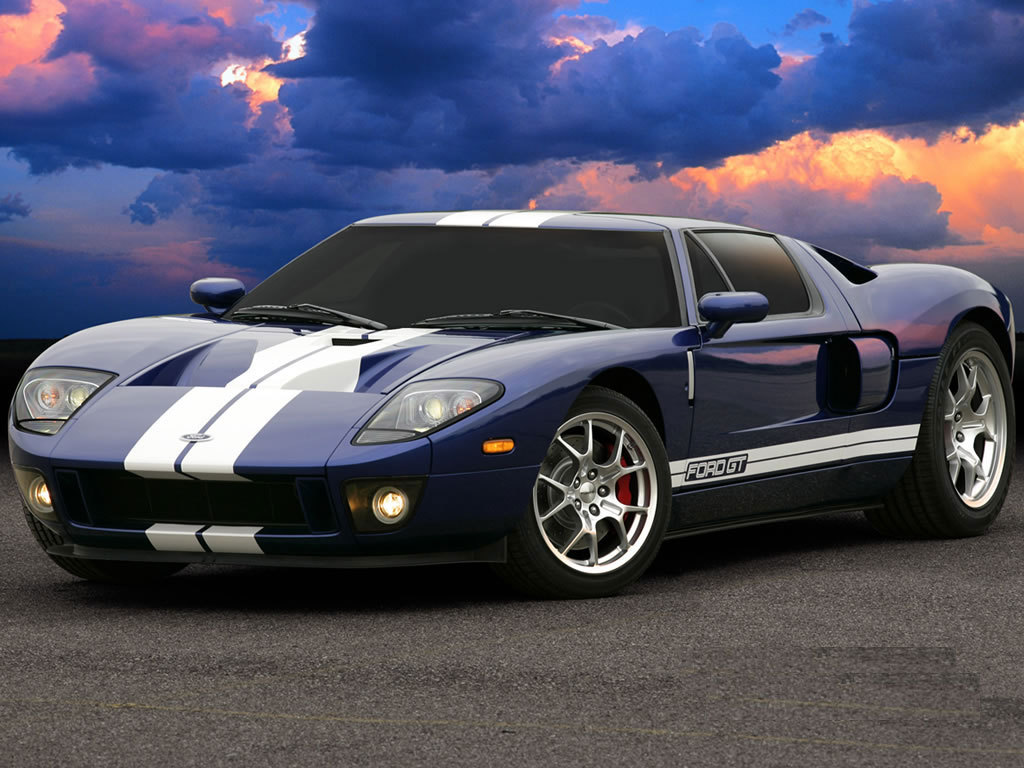} & \includegraphics[width=0.22\linewidth, height=0.22\linewidth]{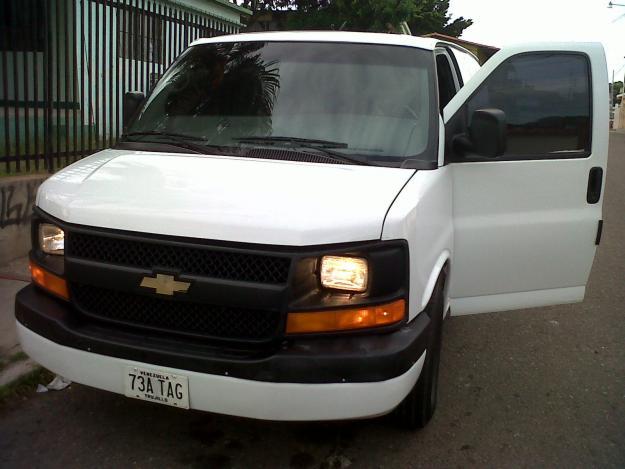} & \includegraphics[width=0.22\linewidth, height=0.22\linewidth]{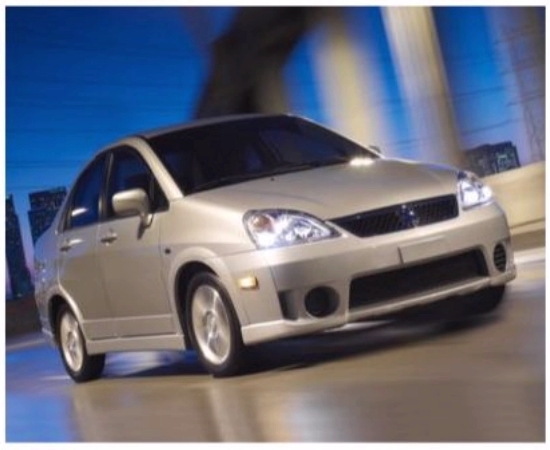} & \includegraphics[width=0.22\linewidth, height=0.22\linewidth]{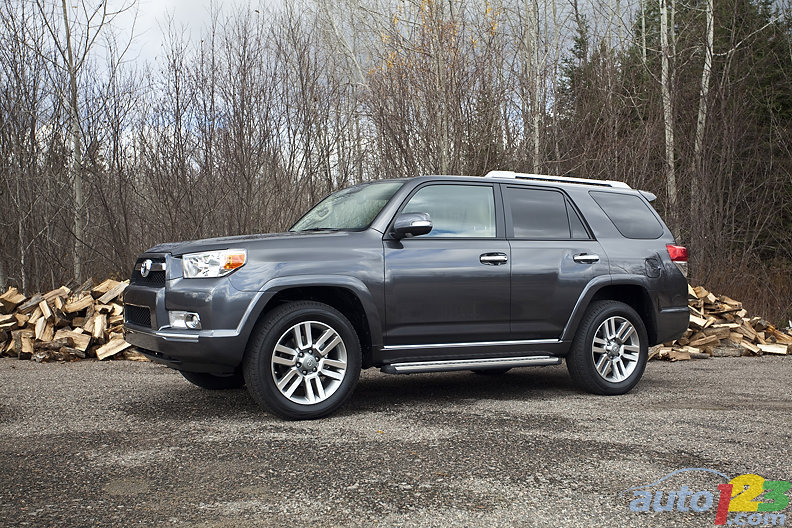} \\
        \includegraphics[width=0.22\linewidth, height=0.22\linewidth]{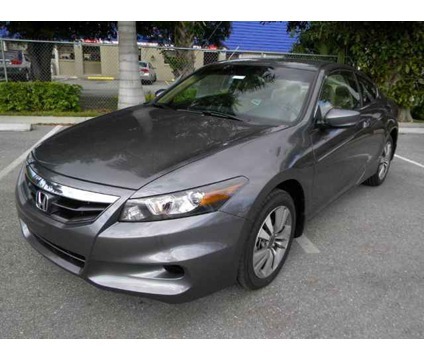} & \includegraphics[width=0.22\linewidth, height=0.22\linewidth]{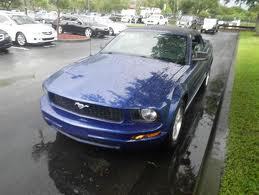} & \includegraphics[width=0.22\linewidth, height=0.22\linewidth]{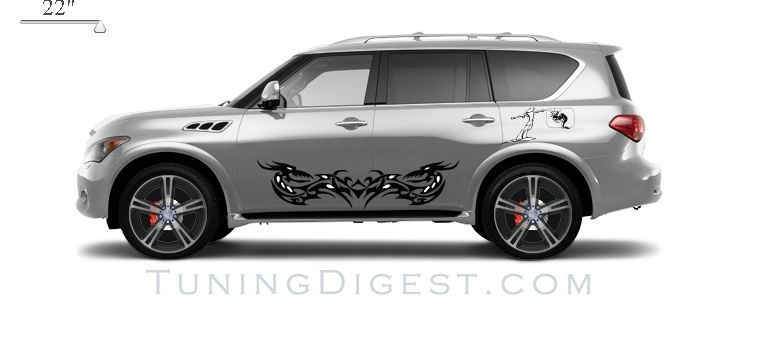} & \includegraphics[width=0.22\linewidth, height=0.22\linewidth]{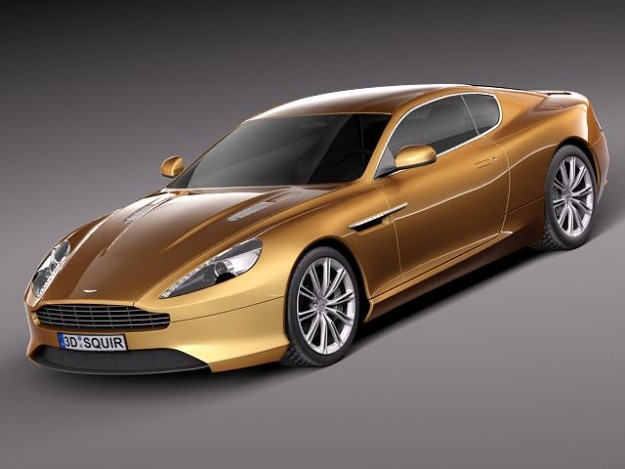} \\
        \includegraphics[width=0.22\linewidth, height=0.22\linewidth]{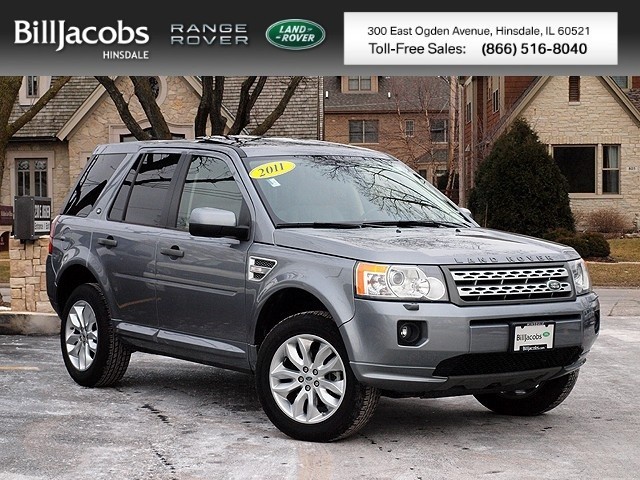} & \includegraphics[width=0.22\linewidth, height=0.22\linewidth]{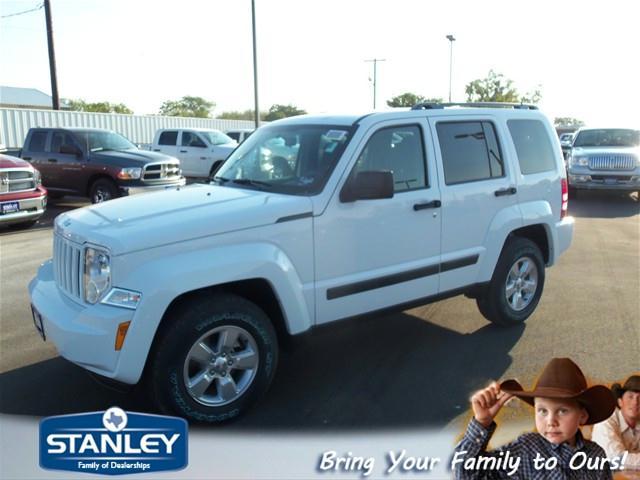} & \includegraphics[width=0.22\linewidth, height=0.22\linewidth]{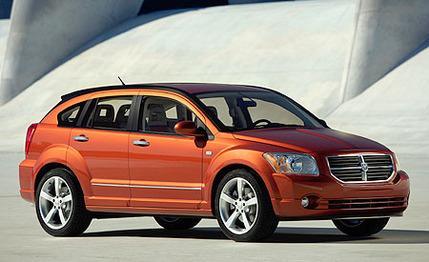} & \includegraphics[width=0.22\linewidth, height=0.22\linewidth]{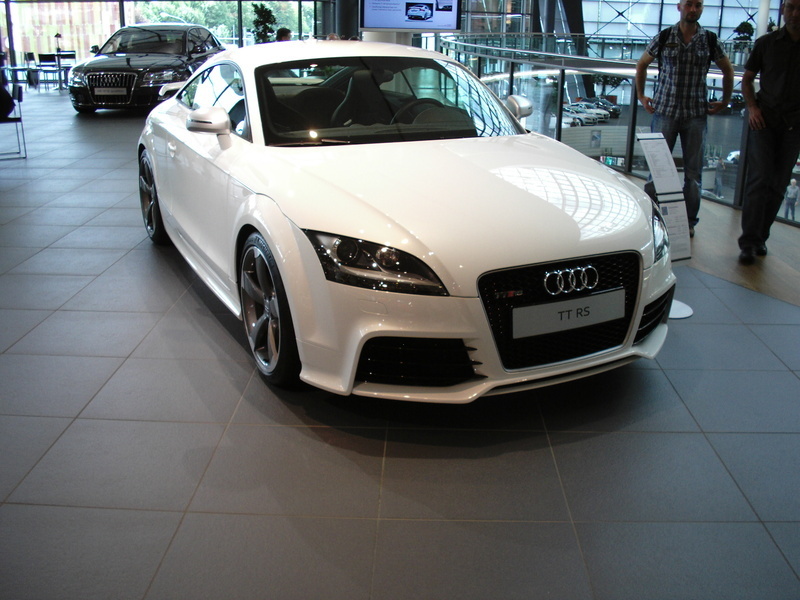} \\
        \includegraphics[width=0.22\linewidth, height=0.22\linewidth]{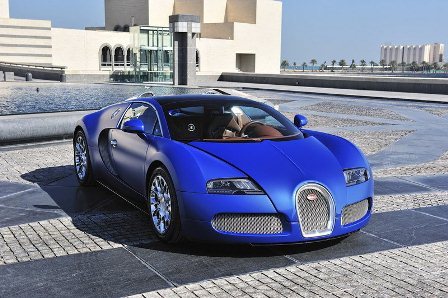} & \includegraphics[width=0.22\linewidth, height=0.22\linewidth]{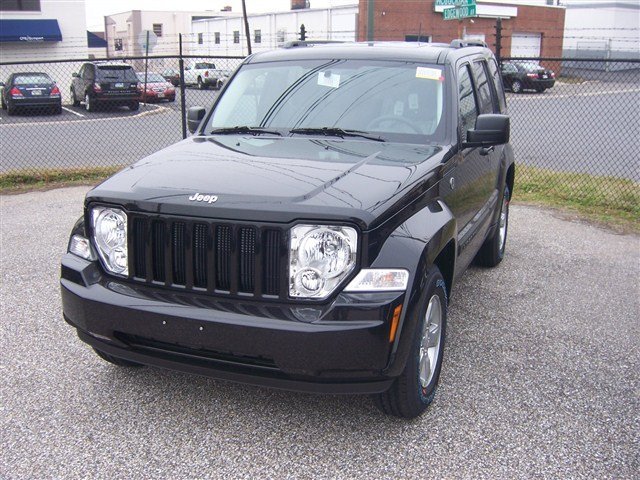} & \includegraphics[width=0.22\linewidth, height=0.22\linewidth]{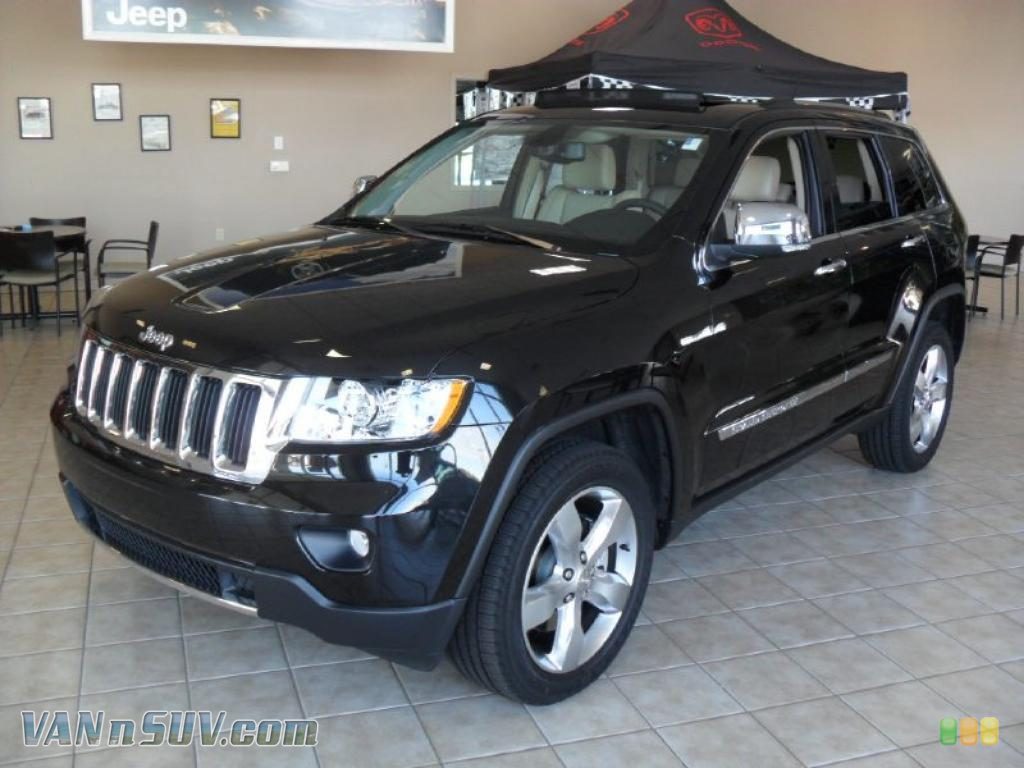} & \includegraphics[width=0.22\linewidth, height=0.22\linewidth]{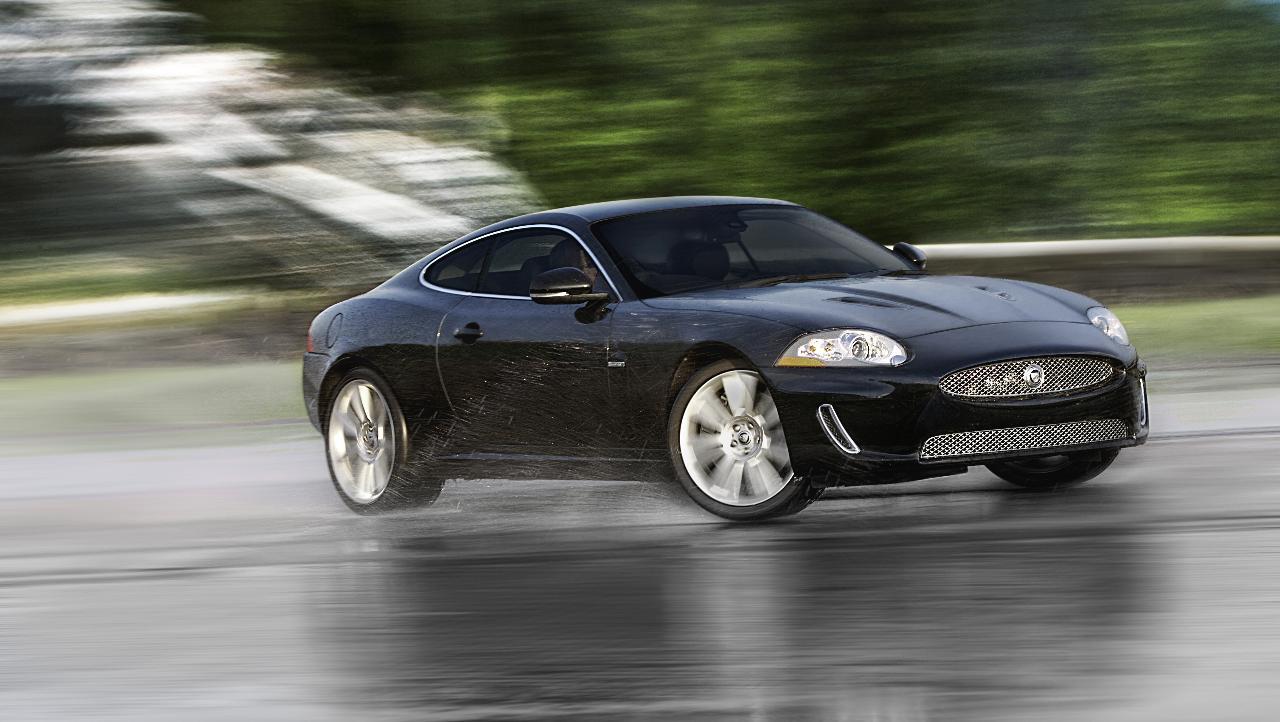}
      \end{tabular}
      \caption{Fine-grained OOD images of Car datasets.}
    \end{subfigure}
  }
  \caption{Examples of ID and fine-grained OOD images of Car datasets.}
  \label{fig:cars}
\end{figure}

\noindent\textbf{Car}: The Car datasets include 196 classes. Among these, 46 classes are randomly assigned as fine-grained OOD data, with the rest retained as ID classes. The ID and OOD classes are:
\begin{itemize}
    \item ID classes: 0, 1, 2, 3, 4, 5, 11, 12, 14, 15, 16, 17, 18, 19, 20, 21, 22, 23, 26, 27, 28, 29, 30, 31, 32, 33, 35, 36, 37, 38, 39, 40, 41, 42, 46, 47, 48, 49, 51, 52, 53, 54, 55, 56, 58, 59, 60, 61, 62, 63, 64, 65, 66, 68, 69, 71, 72, 73, 74, 75, 76, 77, 78, 79, 80, 81, 83, 84, 86, 88, 89, 90, 91, 92, 94, 95, 96, 97, 98, 100, 102, 104, 105, 107, 108, 109, 110, 112, 113, 114, 115, 116, 119, 120, 122, 123, 124, 125, 126, 128, 130, 131, 132, 133, 134, 135, 137, 138, 139, 142, 144, 145, 149, 150, 151, 152, 153, 155, 156, 157, 158, 159, 162, 163, 164, 165, 166, 167, 168, 169, 170, 171, 172, 174, 175, 177, 178, 179, 181, 183, 184, 185, 187, 189, 190, 191, 192, 193, 194, 195 \\
    \item OOD classes: 87, 67, 82, 8, 106, 121, 43, 182, 93, 180, 101, 7, 70, 85, 34, 117, 186, 10, 136, 148, 118, 176, 57, 13, 25, 154, 45, 140, 111, 143, 99, 9, 141, 160, 24, 188, 161, 129, 103, 147, 6, 44, 127, 146, 50, 173 \\
\end{itemize}
\noindent The examples of ID and fine-grained OOD images of the Car datasets are presented in \Cref{fig:cars}.

\section{i-ODIN: invariant-ODIN}
\label{appendix:sec:iodin}
The softmax probability for class $i$ incorporating temperature hyperparameter $T > 0$ is:
\begin{equation}
S_i(\mathbf{x}; T) = \frac{\exp(f_i(\phi_\gamma(\mathbf{x}))/T)}{\sum_{j=1}^C \exp(f_j(\phi_\gamma(\mathbf{x}))/T)}
\end{equation}
The confidence score is given by $S(\mathbf{x}; T) = \max_i S_i(\mathbf{x}; T)$. ODIN improves OOD detection via two components:

\begin{enumerate}
    \item {Temperature Scaling}: Adjusting $T$ sharpens the probability distribution, enhancing ID-OOD separability.
    \item {Input Preprocessing}: Small perturbations refine the confidence gap:
    \begin{equation}
    \tilde{\mathbf{x}} = \mathbf{x} - \varepsilon \cdot \text{sign}(-\nabla_{\mathbf{x}} \log S(\mathbf{x}; T))
    \end{equation}
    where $\varepsilon$ controls perturbation magnitude, increasing softmax scores more for ID than OOD samples. \\
\end{enumerate}

\noindent We propose i-ODIN, an enhanced variant of ODIN that improves OOD detection performance through selective perturbation application. While ODIN applies perturbations across the entire input, i-ODIN introduces a masking mechanism that focuses perturbations on the most discriminative regions. i-ODIN follows the same overall framework as ODIN but introduces a critical modification in the input preprocessing stage. Given input $\mathbf{x} \in \mathcal{X}$, gradients $\nabla_{\mathbf{x}} \log S(\mathbf{x}; T)$ are computed, and the top-$k$ influential features are selected: \\

\begin{enumerate}
    \item Generate mask $\mathcal{M}$ for the top-$p_\text{inv} \%$ gradient magnitudes:
    \begin{equation}
    \mathcal{M} = \text{TopK}(|\nabla_{\mathbf{x}} \log S(\mathbf{x}; T)|,~p_\text{inv}\%)
    \end{equation}
    \item Apply perturbation selectively:
    \begin{equation}
    \tilde{\mathbf{x}} = \mathbf{x} - \varepsilon \cdot \text{sign}(-\nabla_{\mathbf{x}} \log S(\mathbf{x}; T)) \odot \mathcal{M}
    \end{equation}
\end{enumerate}

\noindent This targeted approach enhances ID-OOD separation by focusing on the most relevant features.

\newpage
\section{UMAP Visualization of feature space CIFAR-10 datasets and Conventional OOD (iSUN)}
The UMAP visualizations of the feature space of a ResNet-18 model trained in the CIFAR-10 datasets using (a) the cross-entropy baseline and (b) \ours~(with random pixel shuffle perturbation for outlier synthesis) are presented in the \Cref{fig:umap_baseline} and \Cref{fig:umap_ios} respectively. Samples from the Conventional OOD datasets (iSUN) are visualized in black color within each plot while the rest of the colors denote samples from various classes of ID datasets. The visualizations indicate a reduced intersection between OOD and ID samples in the case of \ours~compared to the cross-entropy baseline. It suggests that \ours~effectively maps OOD samples further from the ID distribution. This separation ultimately contributes to an enhancement in OOD detection performance.
\begin{figure}[!h]
    \centering
    \begin{subfigure}{0.65\linewidth}
        \centering
        \fcolorbox{black}{white}{
        \resizebox{\linewidth}{!}{%
        \includegraphics{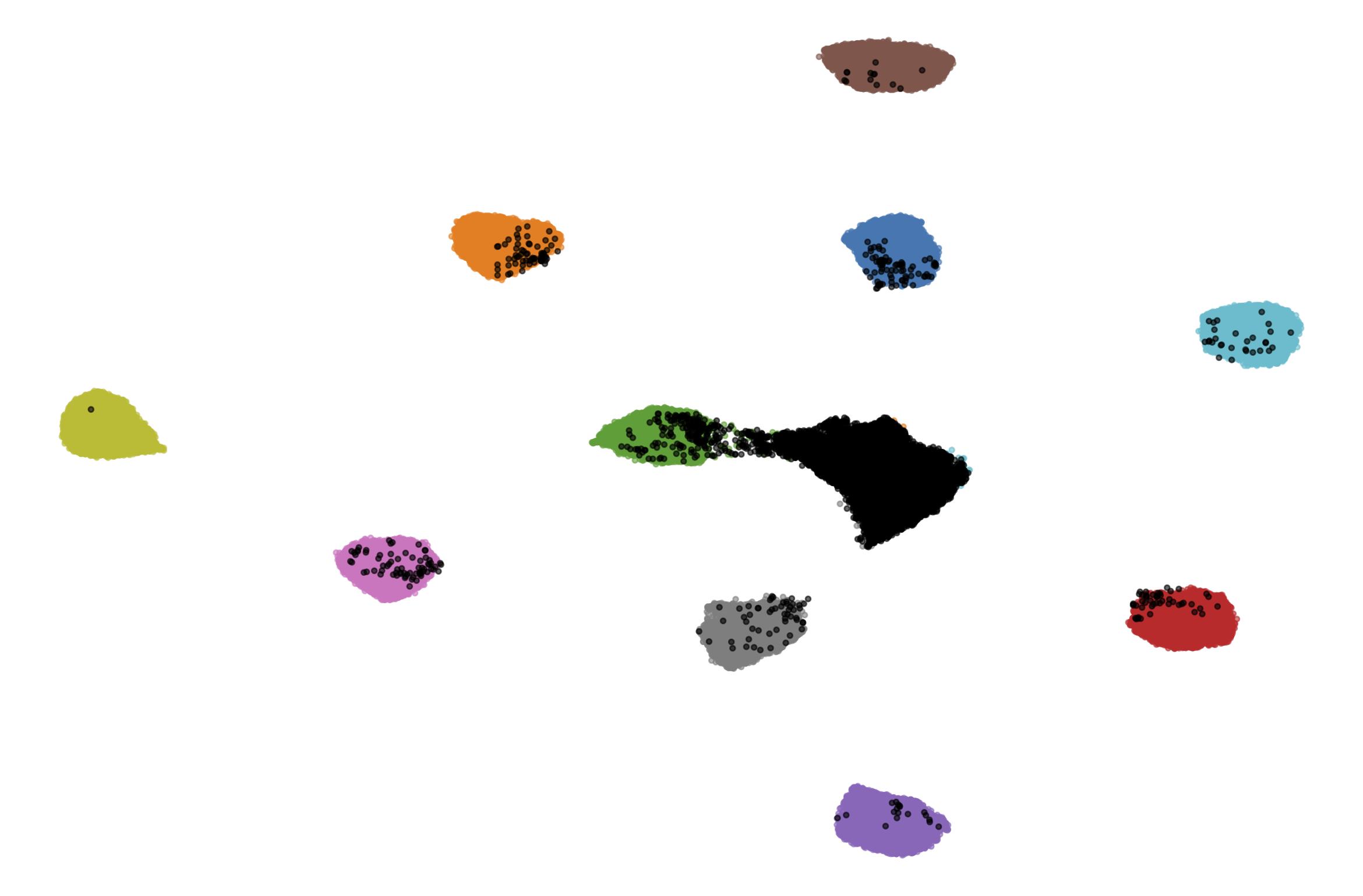}}}
        \caption{Cross-entropy baseline}
        \label{fig:umap_baseline}
    \end{subfigure}
    \begin{subfigure}{0.65\linewidth}
        \centering
        \fcolorbox{black}{white}{
        \resizebox{\linewidth}{!}{%
        \includegraphics{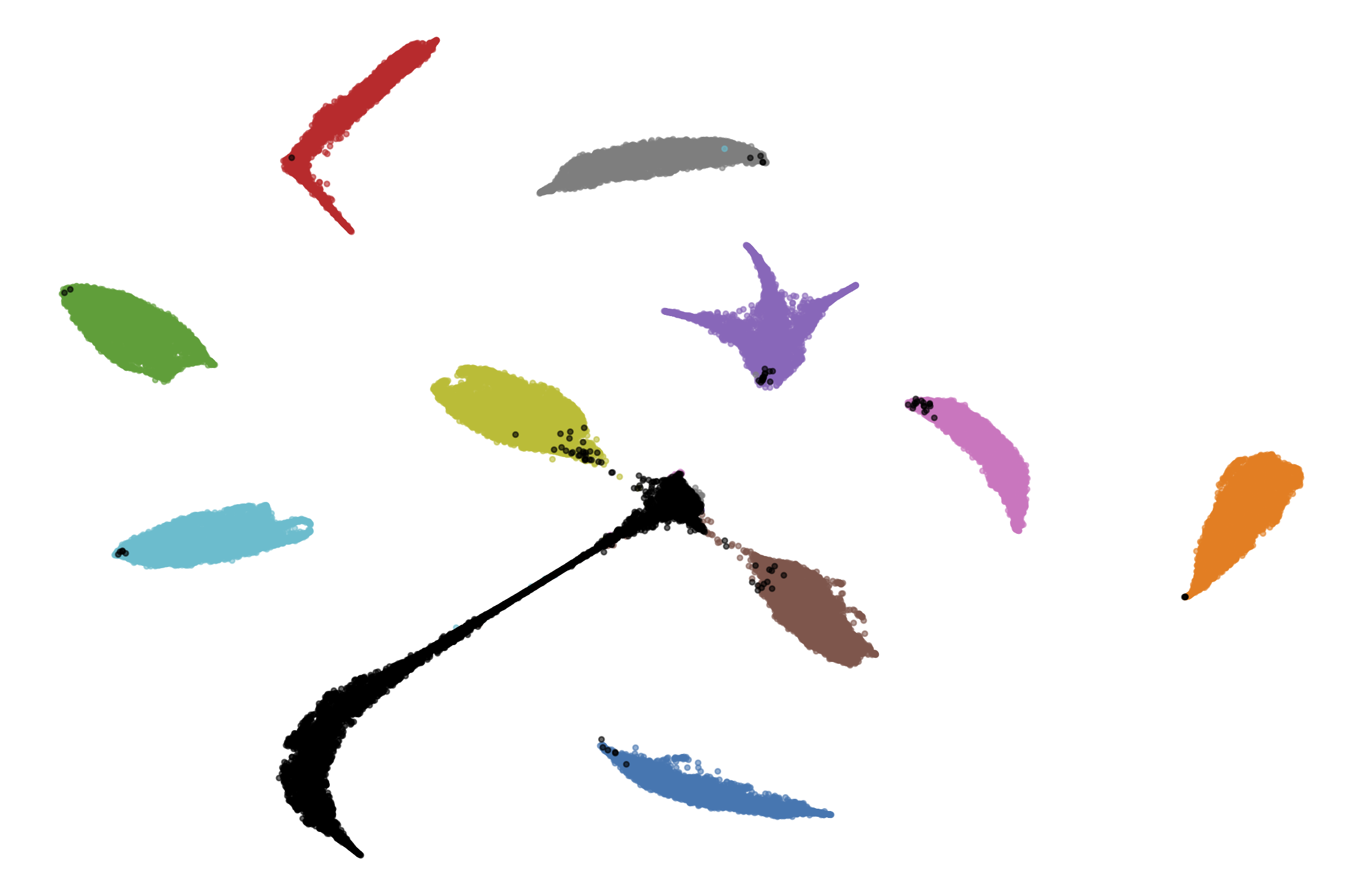}}}
        \caption{\ours}
        \label{fig:umap_ios}
    \end{subfigure}
    \caption{Comparison of UMAP visualizations: (a) Cross-entropy baseline and (b) \ours. OOD samples are in black color, while other colors represent ID samples. The overlap between ID and OOD is reduced in (b) \ours~compared to (a) Cross-entropy baseline.}
    \label{fig:umap}
\end{figure}

\clearpage

\newpage
\section{Additional Details and Experiments}
\label{appendix:sec:additional_details}
\null
\vfill
We use \texttt{PyTorch} deep learning library and \texttt{torchvision} package to conduct all the experiments. For image pre-processing in spurious and fine-grained settings, we use \texttt{RandomResizedCrop} and \texttt{RandomHorizontalFlip} transforms. For conventional setting, we use \texttt{RandomHorizontalFlip} and \texttt{RandomCrop} transforms. We deal with the spatial dimensions of $(32, 32), (224, 224), \text{and}~(448, 448)$ in conventional, spurious, and fine-grained settings respectively. Additionally, we also deal with ImageNet-100 datasets in (224, 224) spatial dimension. \\
\vfill
\null
\noindent \textbf{OOD detection postprocessor.} Unless otherwise noted, we use ODIN and SCALE postprocessors for spurious and conventional settings, and NNGuide and Relation postprocessors for Aircraft and Car benchmarks in ASCOOD experiments. In the case of the ODIN postprocessor, we perform hyperparameter tuning using noise values $\{0.0014, 0.0028, 0.0042, 0.0056, 0.0070, 0.0084, 0.0098\}$ different from just $\{0.0014, 0.0028\}$ specified in OpenOOD for spurious setting. Consistent with OpenOOD, hyperparameters are optimized with respect to the validation OOD datasets (AUROC metric), and model selection is done with respect to validation accuracy. \\
\vfill
\null
\noindent\textbf{\texttt{Shuffle} perturbation:} In this work, we utilize pixel shuffle perturbation for outlier synthesis. Generally, pixel shuffle perturbation requires only the 2D pixel locations to perform shuffling. However, when applying pixel attribution method (using $\mathbf{G}$) to an input image with $C'$ channels, the resulting attribution output contains $C'$ channels containing both positive and negative values. To derive the required 2D pixel locations, we transform $\mathbf{G}$ into a 2D map by computing $\sum_{C'} \exp(\mathbf{G})$. We then select the top $p_\text{inv}\%$ of pixels with the highest magnitudes from the reduced 2D map for shuffling. \\
\vfill
\null
\noindent \textbf{\ours~hyperparameters:} The initial learning rate of FC layer is set to 0.005 for the CIFAR-100 datasets. The final hyperparameters/setup for each datasets is provided in \Cref{appendix:tab:final_hyperparameter}. Additional ablation studies are presented in subsequent sections.
\begin{table}[h!]
\centering
\resizebox{0.9\linewidth}{!}{%
\begin{tabular}{@{}lccccccc}
\toprule
Hyperparameter/Setup & Waterbirds & CelebA & Car & Aircraft & CIFAR-10 & CIFAR-100 & ImageNet-100  \\
\midrule
$\sigma$ & 0.5 & 0.5 & 0.5 & 0.5 & 0.5 & 0.5 & 0.5 \\
$\lambda$ & 0.1 & 1.0 & 1.0 & 1.0 & 1.0 & 5.0 & 1.0 \\
$p_\text{inv}$ & 10\% & 5\% & 10\% & 10\% & 20\% & 10\% & 10\% \\
$\alpha$ ($\mathbf{G}_\text{inv}$) & 300$\rightarrow$30 & 50$\rightarrow$30 & 0.1 & 0.1 & 10 & 10 & 10 \\
$\alpha$ (Gaussian noise) & 0.1 & 0.1 & N/A & N/A & 0.1 & 0.01 & N/A\\
Outlier synthesis & $\mathbf{G}_\text{inv}$ addition & Gaussian noise addition & $\mathbf{G}_\text{inv}$ addition & $\mathbf{G}_\text{inv}$ addition & Random pixel shuffle & Invariant pixel shuffle & $\mathbf{G}_\text{inv}$ addition \\ \bottomrule
\end{tabular}
}
\caption{Final hyperparameter settings for each datasets. (N/A indicates that the corresponding experiment was not conducted/needed.)}
\label{appendix:tab:final_hyperparameter}
\end{table}
\vfill
\null

\newpage

\subsection{ASCOOD vs. OE methods}
\label{sec:oe_comparison_cifar}
We further evaluate the OOD detection performance of our method \ours~against established outlier exposure techniques (OE and MixOE) on the CIFAR-10 and CIFAR-100 benchmarks. Consistent with OpenOOD v1.5, we employ Tiny ImageNet as an external $\mathbb{D}_\text{out}$ for OE and MixOE to promote predictive uncertainty towards OOD samples. We train both OE and MixOE from scratch, unlike the approach taken in OpenOOD v1.5. The results presented in \Cref{appendix:tab:oe_comparison_cifar10} and \Cref{appendix:tab:oe_comparison_cifar100} demonstrate that our method \ours~maintains superior performance on these CIFAR benchmarks. While OE achieves superior performance on some individual datasets within the CIFAR-10 benchmark, our method (\ours) consistently outperforms both OE and MixOE across all datasets in the more challenging CIFAR-100 benchmark.
\begin{table}[h!]
\centering
\resizebox{0.95\linewidth}{!}{%
\begin{tabular}{@{}lcccccc}
\toprule
\multirow{2}{*}{Method} & \multicolumn{6}{c}{\textbf{CIFAR-10}} \\
\cmidrule(lr){2-7}
& MNIST & SVHN & iSUN & Texture & Places365 & Average \\
\midrule
OE & 33.29{\tiny$\pm$8.75} / 86.43{\tiny$\pm$3.94} & \textbf{2.69{\tiny$\pm$1.54}} / \textbf{99.13{\tiny$\pm$0.47}} & 37.95{\tiny$\pm$11.98} / 83.81{\tiny$\pm$5.45} & \textbf{10.74{\tiny$\pm$2.65}} / 97.70{\tiny$\pm$0.51} & \textbf{12.82{\tiny$\pm$2.46}} / \textbf{97.01{\tiny$\pm$0.73}} & 19.50{\tiny$\pm$3.72} / 92.82{\tiny$\pm$1.80} \\
MixOE & 39.29{\tiny$\pm$17.26} / 92.86{\tiny$\pm$2.18} & 34.42{\tiny$\pm$14.02} / 93.37{\tiny$\pm$2.41} & 32.99{\tiny$\pm$6.91} / 93.74{\tiny$\pm$0.60} & 82.73{\tiny$\pm$2.85} / 83.72{\tiny$\pm$0.57} & 71.03{\tiny$\pm$1.24} / 86.73{\tiny$\pm$0.93} & 52.09{\tiny$\pm$6.20} / 90.09{\tiny$\pm$0.64} \\
\midrule
\rowcolor{LightGray} \textbf{\ours} & \textbf{1.37{\tiny$\pm$0.35}} / \textbf{99.68{\tiny$\pm$0.11}} & 6.73{\tiny$\pm$1.37} / 98.61{\tiny$\pm$0.34} &  \textbf{0.94{\tiny$\pm$0.19}} / \textbf{99.79{\tiny$\pm$0.05}} &  10.93{\tiny$\pm$0.52} / \textbf{97.95{\tiny$\pm$0.06}} & 18.47{\tiny$\pm$1.12} / 95.56{\tiny$\pm$0.34} & \textbf{7.69{\tiny$\pm$0.29}} / \textbf{98.32{\tiny$\pm$0.07}} \\
\bottomrule
\end{tabular}
}
\caption{\ours~vs OE methods in FPR@95$\downarrow$ / AUROC$\uparrow$ on CIFAR-10 benchmark.}
\label{appendix:tab:oe_comparison_cifar10}
\end{table}
\vfill
\begin{table}[h!]
\centering
\resizebox{0.95\linewidth}{!}{%
\begin{tabular}{@{}lcccccc}
\toprule
\multirow{2}{*}{Method} & \multicolumn{6}{c}{\textbf{CIFAR-100}} \\
\cmidrule(lr){2-7}
& MNIST & SVHN & iSUN & Texture & Places365 & Average \\
\midrule
OE & 20.21{\tiny$\pm$4.43} / 95.04{\tiny$\pm$0.74} & 40.80{\tiny$\pm$4.65} / 88.99{\tiny$\pm$1.09} & 50.89{\tiny$\pm$19.81} / 87.35{\tiny$\pm$5.83} & 56.14{\tiny$\pm$4.07} / 85.21{\tiny$\pm$0.96} & 66.05{\tiny$\pm$5.77} / 77.60{\tiny$\pm$1.88} & 46.82{\tiny$\pm$4.65} / 86.84{\tiny$\pm$1.44} \\
MixOE & 68.92{\tiny$\pm$6.85} / 74.86{\tiny$\pm$3.88} & 61.62{\tiny$\pm$16.37} / 77.06{\tiny$\pm$7.63} & 68.69{\tiny$\pm$6.71} / 72.69{\tiny$\pm$5.06} & 72.58{\tiny$\pm$3.37} / 74.69{\tiny$\pm$2.19} & 71.97{\tiny$\pm$1.01} / 74.98{\tiny$\pm$0.30} & 68.76{\tiny$\pm$3.98} / 74.86{\tiny$\pm$2.42} \\
\midrule
\rowcolor{LightGray} \textbf{\ours} & \textbf{6.76{\tiny$\pm$1.50}} / \textbf{98.58{\tiny$\pm$0.27}} & \textbf{33.50{\tiny$\pm$3.76}} / \textbf{89.93{\tiny$\pm$1.24}} & \textbf{1.88{\tiny$\pm$0.56}} /
\textbf{99.37{\tiny$\pm$0.09}} & \textbf{53.06{\tiny$\pm$0.81}} / \textbf{86.98{\tiny$\pm$0.81}} & \textbf{54.30{\tiny$\pm$1.00}} / \textbf{81.91{\tiny$\pm$0.77}} & \textbf{29.90{\tiny$\pm$0.76}} / \textbf{91.35{\tiny$\pm$0.13}} \\ \bottomrule
\end{tabular}
}
\caption{\ours~vs OE methods in FPR@95$\downarrow$ / AUROC$\uparrow$ on CIFAR-100 benchmark.}
\label{appendix:tab:oe_comparison_cifar100}
\end{table}
\vfill

\subsection{Ablation study of outlier synthesis in fine-grained settings}

\begin{table}[!h]
\centering
\resizebox{\linewidth}{!}{
\begin{tabular}{@{}lccccc ccccc@{}}
\toprule
 \multirow{3}{*}{Outlier synthesis} & \multirow{3}{*}{$p_\text{inv}$}  & \multicolumn{4}{c}{\textbf{Aircraft}} & \multicolumn{4}{c}{\textbf{Car}} \\
\cmidrule(lr){3-6}\cmidrule(lr){7-10}
 & & \multicolumn{2}{c}{Fine-grained OOD} & \multicolumn{2}{c}{Conventional OOD} & \multicolumn{2}{c}{Fine-grained OOD} & \multicolumn{2}{c}{Conventional OOD} \\
\cmidrule(lr){3-4} \cmidrule(lr){5-6} \cmidrule(lr){7-8} \cmidrule(lr){9-10}
 & & FPR@95$\downarrow$ & AUROC$\uparrow$ & FPR@95$\downarrow$ & AUROC$\uparrow$ & FPR@95$\downarrow$ & AUROC$\uparrow$ & FPR@95$\downarrow$ & AUROC$\uparrow$ \\
\midrule

\multirow{3}{*}{Invariant pixel shuffle}
& 5\% & 48.63{\tiny$\pm$7.37} & 83.75{\tiny$\pm$1.71} & 1.34{\tiny$\pm$0.09} & 99.68{\tiny$\pm$0.03} & 65.23{\tiny$\pm$0.82} & 85.51{\tiny$\pm$0.30} & \textbf{1.33{\tiny$\pm$0.18}} & \textbf{99.71{\tiny$\pm$0.02}} \\
& 10\% & 47.98{\tiny$\pm$2.04} & 83.87{\tiny$\pm$0.38} & 1.19{\tiny$\pm$0.23} & 99.69{\tiny$\pm$0.07} & 63.84{\tiny$\pm$3.21} & 85.38{\tiny$\pm$0.29} & 2.09{\tiny$\pm$0.04} & 99.59{\tiny$\pm$0.02} \\
& 20\% & \textbf{47.35{\tiny$\pm$4.82}} & 83.40{\tiny$\pm$0.74} & 1.24{\tiny$\pm$0.09} & 99.69{\tiny$\pm$0.02} & 62.53{\tiny$\pm$2.97} & 85.48{\tiny$\pm$0.68} & 1.90{\tiny$\pm$0.42} & 99.62{\tiny$\pm$0.06} \\

\midrule

\multirow{5}{*}{$\mathbf{G}_\text{inv}$ addition}
& 1\% & 59.40{\tiny$\pm$2.24} & 88.25{\tiny$\pm$0.42} & 0.60{\tiny$\pm$0.10} & 99.81{\tiny$\pm$0.02} & 46.24{\tiny$\pm$2.54} & 91.02{\tiny$\pm$0.27} & 4.02{\tiny$\pm$0.31} & 98.85{\tiny$\pm$0.11} \\
& 5\% & 56.34{\tiny$\pm$4.31} & 88.32{\tiny$\pm$0.74} & 0.58{\tiny$\pm$0.07} & 99.84{\tiny$\pm$0.02} & 42.96{\tiny$\pm$3.10} & 91.56{\tiny$\pm$0.37} & 3.38{\tiny$\pm$0.23} & 98.92{\tiny$\pm$0.06} \\
& \cellcolor{LightGray} 10\% & \cellcolor{LightGray} {47.94{\tiny$\pm$5.38}} & \cellcolor{LightGray} \textbf{89.75{\tiny$\pm$1.01}} & \cellcolor{LightGray} \textbf{0.55{\tiny$\pm$0.07}} & \cellcolor{LightGray} \textbf{99.84{\tiny$\pm$0.02}} & \cellcolor{LightGray} \textbf{40.76{\tiny$\pm$1.13}} & \cellcolor{LightGray} \textbf{91.86{\tiny$\pm$0.20}} & \cellcolor{LightGray} 4.28{\tiny$\pm$0.53} & \cellcolor{LightGray} 98.79{\tiny$\pm$0.12} \\
& 20\% & 53.13{\tiny$\pm$4.00} & 88.81{\tiny$\pm$0.91} & 0.59{\tiny$\pm$0.08} & 99.83{\tiny$\pm$0.01} & 44.55{\tiny$\pm$1.01} & 91.57{\tiny$\pm$0.06} & 4.59{\tiny$\pm$0.55} & 98.73{\tiny$\pm$0.08} \\
& 100\% & 48.45{\tiny$\pm$5.04} & 89.14{\tiny$\pm$0.78} & 0.60{\tiny$\pm$0.10} & 99.81{\tiny$\pm$0.01} & 43.00{\tiny$\pm$3.33} & 91.32{\tiny$\pm$0.40} & 4.20{\tiny$\pm$0.21} & 98.81{\tiny$\pm$0.02} \\
\bottomrule
\end{tabular}}
\caption{Ablation study of outlier synthesis methods on Aircraft and Car datasets.}
\label{appendix:tab:outlier_ablation_finegrained}
\end{table}

\noindent We evaluate two effective outlier synthesis methods: invariant pixel shuffle and $\mathbf{G}_\text{inv}$ addition in fine-grained settings (Aircraft and Car datasets). The results summarized in \Cref{appendix:tab:outlier_ablation_finegrained} present OOD detection performance for Aircraft and Car datasets. The results show that $\mathbf{G}_\text{inv}$ addition demonstrates overall superior performance. Since the satisfactorily high accuracy of $89.59\%$ and $94.14\%$ is achieved in Aircraft and Car datasets, the gradient $\mathbf{G}$ becomes highly reliable for outlier synthesis as soon as satisfactory accuracy is achieved. So, it leads to effective perturbation of invariant features. The results also indicate that using an invariant pixel shuffle for outlier synthesis may be effective for conventional OOD detection. Still, its performance falls quite short in fine-grained OOD detection.

\newpage

\subsection{Gradient of the logit versus gradient of softmax probability for outlier synthesis}
\label{appendix:sec:gradient_logit_softmax}
We investigate how the choice between using gradient of the logit and gradient of softmax probability for outlier synthesis affects OOD detection performance. \Cref{appendix:tab:gradient_of_logit_vs_prob} reports results obtained in fine-grained settings using $\mathbf{G}_\text{inv}$ addition for outlier synthesis. For gradient of the logit, optimal $\alpha$ is set to 1, while for the gradient of the softmax probability, $\alpha$ is set to 0.1. The results demonstrate that both methods achieve comparable performance with a gradient of the logit slightly performing better in overall performance. Consequently, we simply adopt gradient of the logit for outlier synthesis.

\begin{table}[h!]
\centering
\resizebox{\linewidth}{!}{
\begin{tabular}{@{}lcccccc cccccc@{}}
\toprule
 \multirow{3}{*}{$\mathbf{G}$} & \multirow{3}{*}{$p_\text{inv}$} & \multicolumn{5}{c}{\textbf{Aircraft}} & \multicolumn{5}{c}{\textbf{Car}} \\
\cmidrule(lr){3-7}\cmidrule(lr){8-12}
& & \multicolumn{2}{c}{Fine-grained OOD} & \multicolumn{2}{c}{Conventional OOD} & \multirow{2}{*}{Acc} & \multicolumn{2}{c}{Fine-grained OOD} & \multicolumn{2}{c}{Conventional OOD} & \multirow{2}{*}{Acc}  \\
\cmidrule(lr){3-6} \cmidrule(lr){8-11}
& & FPR@95$\downarrow$ & AUROC$\uparrow$ & FPR@95$\downarrow$ & AUROC$\uparrow$ & & FPR@95$\downarrow$ & AUROC$\uparrow$ & FPR@95$\downarrow$ & AUROC$\uparrow$ & \\
\midrule
\multirow{3}{*}{$\mathbf{G} = \frac{\partial \mathbf{z}_{c}}{\partial \mathbf{x}}$} & 5\% & 56.34{\tiny$\pm$4.31} & 88.32{\tiny$\pm$0.74} & 0.58{\tiny$\pm$0.07} & 99.84{\tiny$\pm$0.02} & 88.42{\tiny$\pm$0.13}\% & 42.96{\tiny$\pm$3.10} & 91.56{\tiny$\pm$0.37} & 3.38{\tiny$\pm$0.23} & 98.92{\tiny$\pm$0.06} & 93.51{\tiny$\pm$0.08}\% \\
& \cellcolor{LightGray} 10\% & \cellcolor{LightGray} \textbf{47.94{\tiny$\pm$5.38}} & \cellcolor{LightGray} \textbf{89.75{\tiny$\pm$1.01}} & \cellcolor{LightGray} \textbf{0.55{\tiny$\pm$0.07}} & \cellcolor{LightGray} \textbf{99.84{\tiny$\pm$0.02}} & \cellcolor{LightGray} \textbf{89.61{\tiny$\pm$0.19}}\% & \cellcolor{LightGray} \textbf{40.76{\tiny$\pm$1.13}} & \cellcolor{LightGray} \textbf{91.86{\tiny$\pm$0.20}} & \cellcolor{LightGray} 4.28{\tiny$\pm$0.53} & \cellcolor{LightGray} 98.79{\tiny$\pm$0.12} & \cellcolor{LightGray} \textbf{94.20{\tiny$\pm$0.11}}\% \\
& 100\% & 48.45{\tiny$\pm$5.04} & 89.14{\tiny$\pm$0.78} & 0.60{\tiny$\pm$0.10} & 99.81{\tiny$\pm$0.01} & 88.93{\tiny$\pm$0.07}\% & 43.00{\tiny$\pm$3.33} & 91.32{\tiny$\pm$0.40} & 4.20{\tiny$\pm$0.21} & 98.81{\tiny$\pm$0.02} & 94.00{\tiny$\pm$0.12}\% \\
\midrule
\multirow{3}{*}{$\mathbf{G} = \frac{\partial \mathbf{p}_{c}}{\partial \mathbf{x}}$} & 5\% & 51.55{\tiny$\pm$4.12} & 89.07{\tiny$\pm$1.13} & 0.56{\tiny$\pm$0.06} & 99.82{\tiny$\pm$0.01} & 89.46{\tiny$\pm$0.29}\% & 44.59{\tiny$\pm$2.47} & 91.37{\tiny$\pm$0.29} & 4.27{\tiny$\pm$0.40} & 98.79{\tiny$\pm$0.09} & 94.15{\tiny$\pm$0.19}\% \\
& 10\% & 50.07{\tiny$\pm$4.55} & 88.96{\tiny$\pm$0.56} & 0.60{\tiny$\pm$0.04} & 99.82{\tiny$\pm$0.00} & 89.16{\tiny$\pm$0.13}\% & 42.11{\tiny$\pm$0.83} & {91.56{\tiny$\pm$0.12}} & 4.34{\tiny$\pm$0.69} & 98.00{\tiny$\pm$0.13} & {94.00{\tiny$\pm$0.02}}\% \\
& 100\% & 48.73{\tiny$\pm$3.02} & 89.37{\tiny$\pm$0.44} & 0.59{\tiny$\pm$0.06} & 99.83{\tiny$\pm$0.02} & 89.26{\tiny$\pm$0.25}\% & 40.94{\tiny$\pm$2.22} & 91.77{\tiny$\pm$0.23} & \textbf{4.02{\tiny$\pm$0.37}} & \textbf{98.85{\tiny$\pm$0.14}} & 94.14{\tiny$\pm$0.10}\% \\ \bottomrule
\end{tabular}} 
\caption{Gradient of logit vs gradient of softmax probability for outlier synthesis.}
\label{appendix:tab:gradient_of_logit_vs_prob}
\end{table}

\subsection{Sensitivity of $\lambda$}
 The sensitivity analysis of $\lambda$ in Aircraft and Car datasets is shown in \Cref{appendix:tab:sensitivity_lambda}. We use $\mathbf{x}^{\prime} = \mathbf{x} + \alpha \cdot \mathbf{G}_\text{inv}$ for outlier synthesis and report the OOD detection results. The results indicate that improved classification accuracy correlates with enhanced OOD detection performance. Setting $\lambda = 1.0$ leads to better accuracy and enhanced OOD detection simultaneously.
\begin{table}[h!]
\centering
\resizebox{0.99\linewidth}{!}{
\begin{tabular}{@{}lccccc cccccc@{}}
\toprule
 \multirow{3}{*}{$\lambda$} & \multicolumn{5}{c}{\textbf{Aircraft}} & \multicolumn{5}{c}{\textbf{Car}} \\
\cmidrule(lr){2-6}\cmidrule(lr){7-11}
 & \multicolumn{2}{c}{Fine-grained OOD} & \multicolumn{2}{c}{Conventional OOD} & \multirow{2}{*}{Acc} & \multicolumn{2}{c}{Fine-grained OOD} & \multicolumn{2}{c}{Conventional OOD} & \multirow{2}{*}{Acc}  \\
\cmidrule(lr){2-5} \cmidrule(lr){7-10}
 & FPR@95$\downarrow$ & AUROC$\uparrow$ & FPR@95$\downarrow$ & AUROC$\uparrow$ & & FPR@95$\downarrow$ & AUROC$\uparrow$ & FPR@95$\downarrow$ & AUROC$\uparrow$ & \\
\midrule
0.1 & 53.56{\tiny$\pm$4.89} & 88.22{\tiny$\pm$0.98} & 1.13{\tiny$\pm$0.22} & 99.69{\tiny$\pm$0.04} & 88.42{\tiny$\pm$0.13}\% & 45.46{\tiny$\pm$2.08} & 90.89{\tiny$\pm$0.28} & \textbf{1.97{\tiny$\pm$0.25}} & \textbf{99.61{\tiny$\pm$0.04}} & 93.51{\tiny$\pm$0.08}\% \\
\rowcolor{LightGray} 1.0 & \textbf{47.94{\tiny$\pm$5.38}} &  \textbf{89.75{\tiny$\pm$1.01}} & \textbf{0.55{\tiny$\pm$0.07}} & \textbf{99.84{\tiny$\pm$0.02}} & \textbf{89.61{\tiny$\pm$0.19}}\% & \textbf{40.76{\tiny$\pm$1.13}} & \textbf{91.86{\tiny$\pm$0.20}} & 4.28{\tiny$\pm$0.53} & 98.79{\tiny$\pm$0.12} & \textbf{94.20{\tiny$\pm$0.11}}\% \\
10.0 & 58.10{\tiny$\pm$4.82} & 78.79{\tiny$\pm$0.88} & 0.63{\tiny$\pm$0.09} & 99.83{\tiny$\pm$0.01} & 88.93{\tiny$\pm$0.07}\% & 99.99{\tiny$\pm$0.01} & 50.00{\tiny$\pm$0.00} & 99.99{\tiny$\pm$0.01} & 50.0{\tiny$\pm$0.00} & 94.00{\tiny$\pm$0.12}\% \\
\bottomrule
\end{tabular}}
\caption{Ablation study of outlier synthesis methods on Aircraft and Car datasets.}
\label{appendix:tab:sensitivity_lambda}
\end{table}

\subsection{Gradient addition vs. subtraction for outlier synthesis}
\label{appendix:sec:gradient_addition_vs_subtraction}
The complete result of \Cref{tab:gradient_addition_vs_subtraction} is provided in \cref{appendix:tab:gradient_addition_vs_subtraction}. We set the hyperparameters as: $\lambda=1$, $p_\text{inv}=10\%$, and $\alpha=0.1$. Interestingly, gradient addition for outlier synthesis not only leads to significantly better fine-grained OOD detection performance but also better classification accuracy.

\begin{table}[h!]
\centering
\resizebox{0.99\linewidth}{!}{
\begin{tabular}{@{}lccccc cccccc@{}}
\toprule
 \multirow{4}{*}{$\mathbf{x}^{\prime}$} & \multicolumn{5}{c}{\textbf{Aircraft}} & \multicolumn{5}{c}{\textbf{Car}} \\
\cmidrule(lr){2-6}\cmidrule(lr){7-11}
 & \multicolumn{2}{c}{Fine-grained OOD} & \multicolumn{2}{c}{Conventional OOD} & \multirow{2}{*}{Acc} & \multicolumn{2}{c}{Fine-grained OOD} & \multicolumn{2}{c}{Conventional OOD} & \multirow{2}{*}{Acc}  \\
\cmidrule(lr){2-5} \cmidrule(lr){7-10}
 & FPR@95$\downarrow$ & AUROC$\uparrow$ & FPR@95$\downarrow$ & AUROC$\uparrow$ & & FPR@95$\downarrow$ & AUROC$\uparrow$ & FPR@95$\downarrow$ & AUROC$\uparrow$ & \\
\midrule
$\mathbf{x}^{\prime} = \mathbf{x} - \alpha \cdot \mathbf{G}_\text{inv}$ & 50.15{\tiny$\pm$4.80} & 83.64{\tiny$\pm$0.82} & 1.30{\tiny$\pm$0.19} & 99.68{\tiny$\pm$0.04} & 87.43{\tiny$\pm$0.20}\% & 60.20{\tiny$\pm$1.91} & 86.27{\tiny$\pm$0.34} & \textbf{2.00{\tiny$\pm$0.45}} & \textbf{99.57{\tiny$\pm$0.08}} & 93.59{\tiny$\pm$0.13}\% \\
\rowcolor{LightGray} $\mathbf{x}^{\prime} = \mathbf{x} + \alpha \cdot \mathbf{G}_\text{inv}$ & \textbf{47.94{\tiny$\pm$5.38}} &  \textbf{89.75{\tiny$\pm$1.01}} & \textbf{0.55{\tiny$\pm$0.07}} & \textbf{99.84{\tiny$\pm$0.02}} & \textbf{89.61{\tiny$\pm$0.19}}\% & \textbf{40.76{\tiny$\pm$1.13}} & \textbf{91.86{\tiny$\pm$0.20}} & 4.28{\tiny$\pm$0.53} & 98.79{\tiny$\pm$0.12} & \textbf{94.20{\tiny$\pm$0.11}}\% \\
\bottomrule
\end{tabular}}
\caption{Ablation of gradient-based outlier synthesis methods (addition vs. subtraction) on Aircraft and Car datasets.}
\label{appendix:tab:gradient_addition_vs_subtraction}
\end{table}

\newpage
\subsection{Hyperparameter studies in Waterbirds datasets}
\begin{figure}[h!]
    \centering
    \adjustbox{width=0.75\linewidth}{
    \begin{subfigure}{0.31\linewidth}
        \centering
        \fcolorbox{black}{white}{\adjustbox{width=\linewidth,clip}{\includegraphics{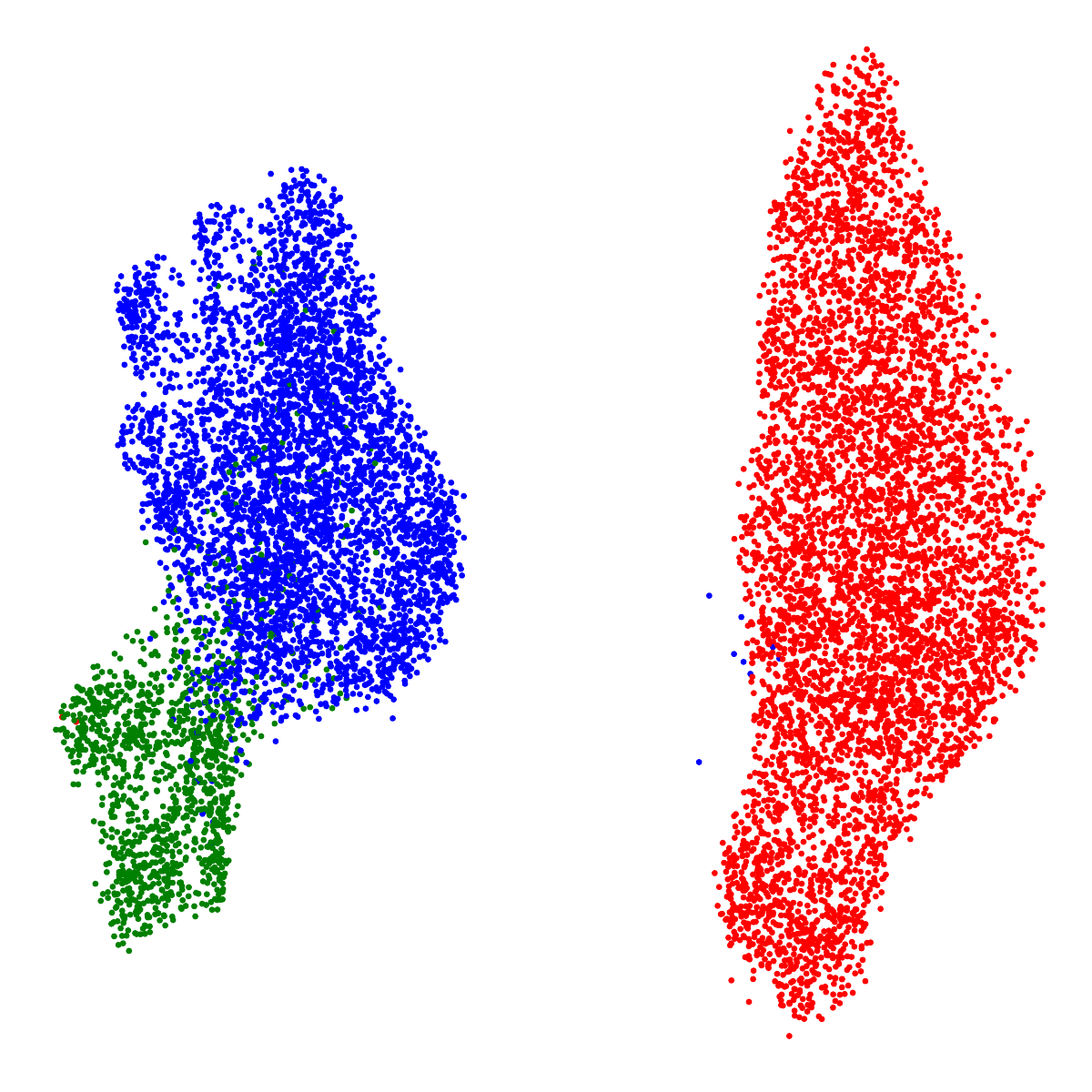}}}
        \caption{$\alpha\sim300$}
        \label{fig:alpha_1}
    \end{subfigure}
    \hfill
    \begin{subfigure}{0.31\linewidth}
        \centering
        \fcolorbox{black}{white}{\adjustbox{width=\linewidth,clip}{\includegraphics{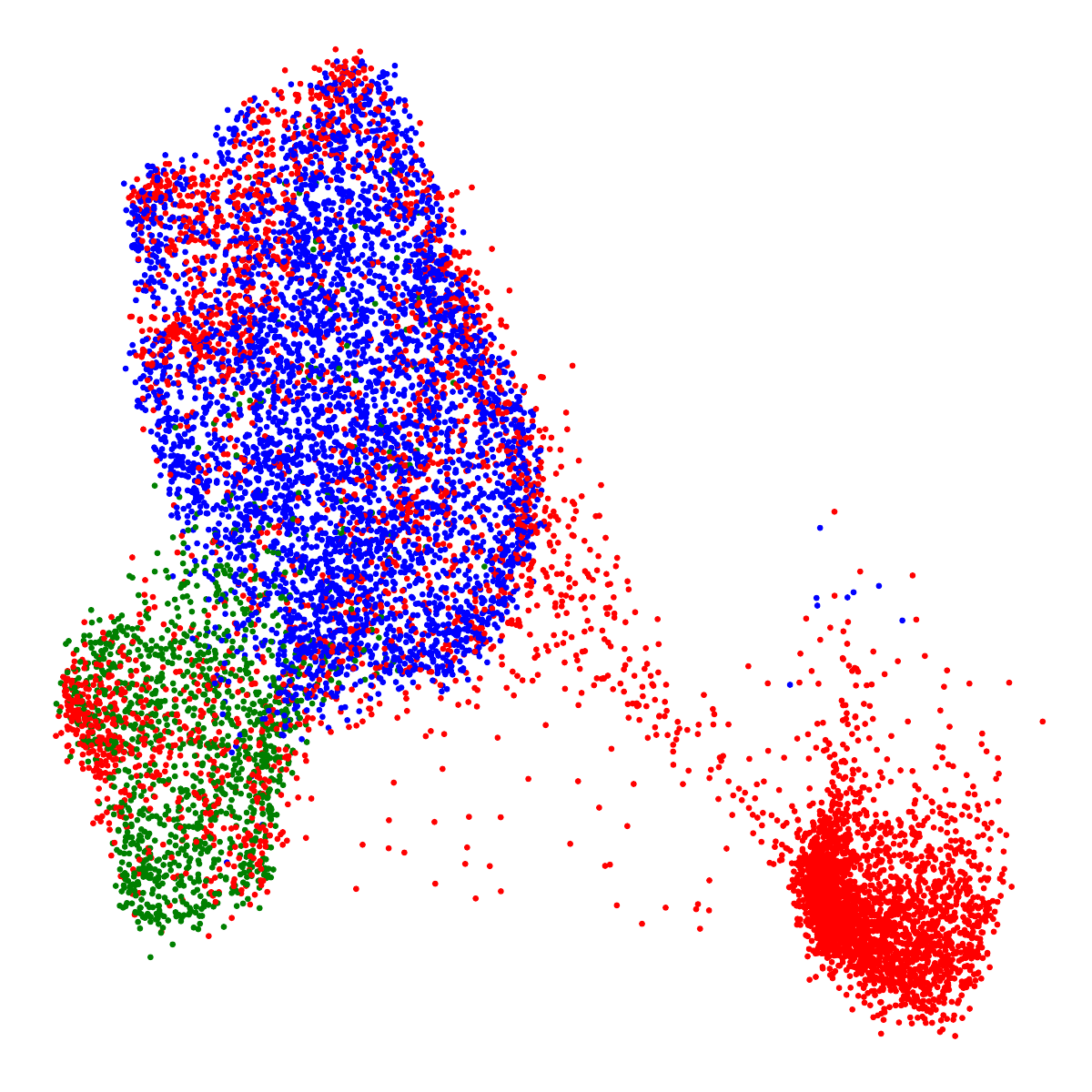}}}
        \caption{$\alpha\sim30$}
        \label{fig:alpha_0.1}
    \end{subfigure}
    \hfill
    \begin{subfigure}{0.31\linewidth}
        \centering
        \fcolorbox{black}{white}{\adjustbox{width=\linewidth,clip}{\includegraphics{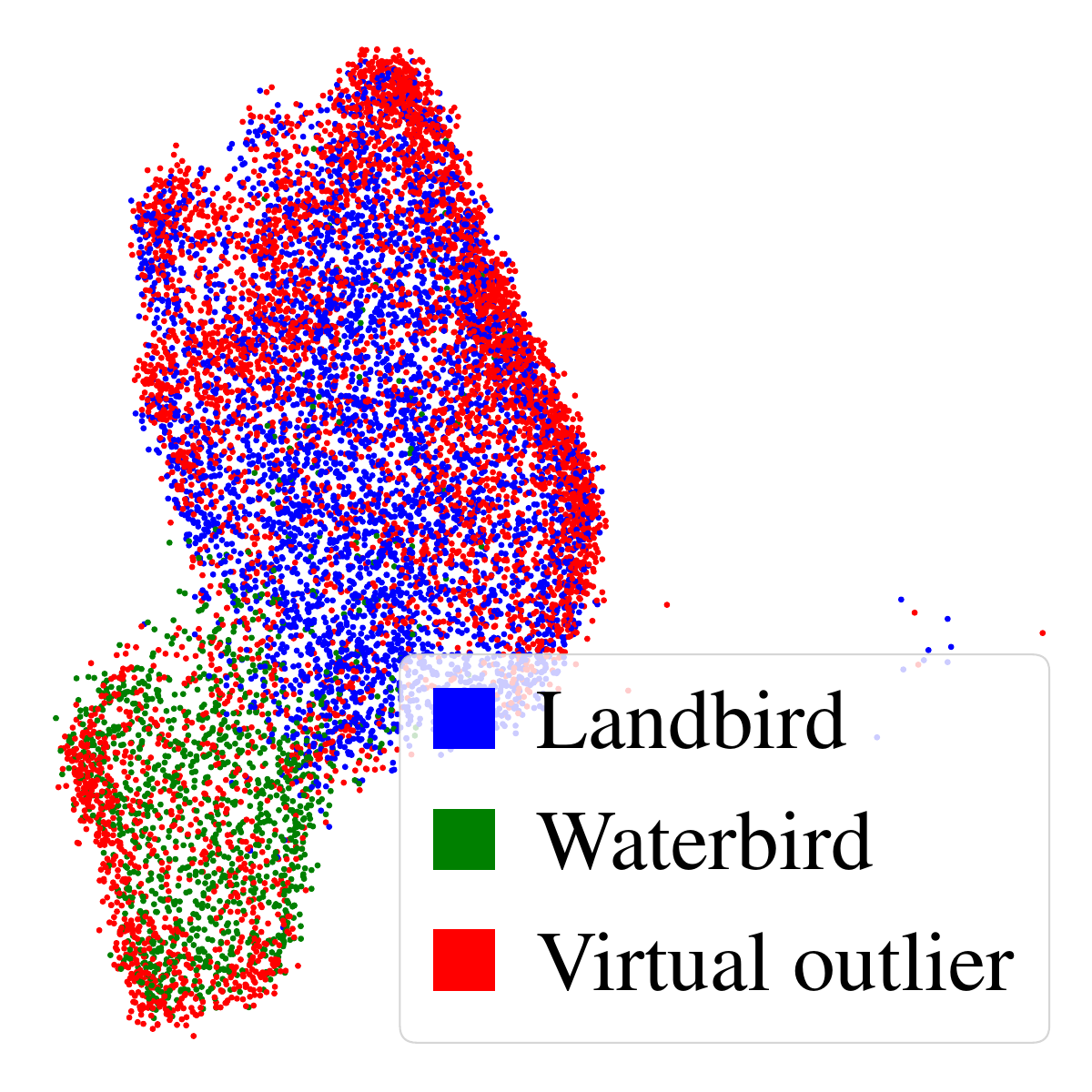}}}
        \caption{$\alpha\sim1$}
        \label{fig:alpha_0.01}
    \end{subfigure}
    }
    \caption{UMAP visualization of feature space of Waterbirds datasets with varying values of $\alpha$ (used for outlier synthesis).}
\end{figure}
\noindent The hyperparameter $\alpha$ determines the nature of synthesized outliers. To illustrate its effect in feature space, \Cref{fig:alpha_1}, \Cref{fig:alpha_0.1}, and \Cref{fig:alpha_0.01} present UMAP visualizations (with cross-entropy baseline on Waterbirds) of virtual outliers alongside ID samples for different values of $\alpha$. \Cref{fig:alpha_1} shows a clear separation of virtual outliers from the ID cluster at high $\alpha$, while \Cref{fig:alpha_0.1} demonstrates partial overlap with some outliers remaining distant, reflecting a moderate $\alpha$. In \Cref{fig:alpha_0.01}, virtual outliers and ID samples significantly overlap due to a very low $\alpha$. Ideally, virtual outliers should be sufficiently challenging but lack invariant features of ID data for optimal spurious OOD detection. The results presented in \Cref{appendix:tab:sensitivity_alpha_waterbirds} demonstrate that optimal spurious OOD detection is achieved by linearly decreasing $\alpha$ from 300 to 30 during training. Dynamically adjusting $\alpha$ during training -- starting high to compensate for $\mathbf{G}$'s initial unreliability and then decreasing to generate more challenging outliers as $\mathbf{G}$ improves -- facilitates enhanced spurious OOD detection. This dynamic adjustment, however, did not show gains in fine-grained and conventional OOD detection. The sensitivity of
$p_\text{inv}$ with both dynamic and static values of $\alpha$ is presented in \Cref{appendix:tab:p_inv_waterbird}, which reinforces the same prior conclusion of dynamic $\alpha$ being superior. It is also evident that $\alpha$ is a more critical hyperparameter than $p_\text{inv}$. Regardless of $p_\text{inv}$ not being highly critical, destroying invariant features while preserving environmental ones still occurs as gradients are higher in invariant regions. Moreover, we also present a sensitivity analysis of $\sigma$ in Waterbirds datasets in \Cref{appendix:tab:sensitivity_sigma_waterbirds}. It suggests the optimal value of $\sigma$ is around 0.5.

\begin{table}[h]
  \centering
    \subfloat[Sensitivity of $p_\text{inv}$]{
      \resizebox{0.5\linewidth}{!}{%
        \begin{tabular}{@{}lccccc}
          \toprule
          \multirow{2}{*}{$\alpha$} & \multirow{2}{*}{$p_\text{inv}$} & \multicolumn{2}{c}{\textbf{Spurious OOD}} & \multicolumn{2}{c}{\textbf{Conventional OOD}} \\
            \cmidrule(lr){3-4} \cmidrule(lr){5-6}
          & & FPR@95$\downarrow$ & AUROC$\uparrow$ & FPR@95$\downarrow$ & AUROC$\uparrow$ \\
          \midrule
          \multirow{4}{*}{$30$} 
                & 1\% & 23.41{\tiny$\pm$2.52} & 94.71{\tiny$\pm$0.41} & 13.98{\tiny$\pm$1.58} & 97.19{\tiny$\pm$0.22} \\
                & 10\% & 17.26{\tiny$\pm$1.19} & 96.27{\tiny$\pm$0.28} & 10.18{\tiny$\pm$0.54} & 98.05{\tiny$\pm$0.12} \\  
                & 30\% & 18.18{\tiny$\pm$1.13} & 96.20{\tiny$\pm$0.27} & \underline{9.52{\tiny$\pm$0.79}} & \underline{98.24{\tiny$\pm$0.15}} \\
                & 100\% & 17.29{\tiny$\pm$1.65} & 96.39{\tiny$\pm$0.27} & \textbf{9.23{\tiny$\pm$0.93}} &  \textbf{98.27{\tiny$\pm$0.17}} \\
          \midrule
          \multirow{4}{*}{$300\rightarrow30$} 
                & 1\% & 14.50{\tiny$\pm$0.37} & 96.97{\tiny$\pm$0.06} & 12.98{\tiny$\pm$0.79} & 97.49{\tiny$\pm$0.18} \\
                & \cellcolor{LightGray} 10\% & \cellcolor{LightGray} \textbf{11.37{\tiny$\pm$0.42}} & \cellcolor{LightGray} \textbf{97.48{\tiny$\pm$0.10}} & \cellcolor{LightGray} 10.02{\tiny$\pm$0.38} & \cellcolor{LightGray} 98.04{\tiny$\pm$0.02} \\
                & 30\% & \underline{11.74{\tiny$\pm$1.19}} & \underline{97.42{\tiny$\pm$0.11}} & 10.61{\tiny$\pm$1.33} & 97.98{\tiny$\pm$0.16} \\
                & 100\% & 13.92{\tiny$\pm$1.38} & 97.06{\tiny$\pm$0.36} & 10.89{\tiny$\pm$0.75} & 97.91{\tiny$\pm$0.19} \\
          \bottomrule
        \end{tabular}%
      }
      \label{appendix:tab:p_inv_waterbird}
    }
    \subfloat[Sensitivity of $\alpha$]{
      \resizebox{0.5\linewidth}{!}{%
        \begin{tabular}{@{}lcccc}
          \toprule
          \multirow{2}{*}{$\alpha$} & \multicolumn{2}{c}{\textbf{Spurious OOD}} & \multicolumn{2}{c}{\textbf{Conventional OOD}} \\
            \cmidrule(lr){2-3} \cmidrule(lr){4-5}
          & FPR@95$\downarrow$ & AUROC$\uparrow$ & FPR@95$\downarrow$ & AUROC$\uparrow$ \\
          \midrule
          300 & 17.92{\tiny$\pm$2.20} & 96.37{\tiny$\pm$0.46} & 15.47{\tiny$\pm$1.28} & 97.04{\tiny$\pm$0.19} \\
          100 & \underline{15.93{\tiny$\pm$1.81}} & \underline{96.86{\tiny$\pm$0.40}} & 12.34{\tiny$\pm$0.77} & 97.72{\tiny$\pm$0.18} \\
          30 & 17.26{\tiny$\pm$1.19} & 96.27{\tiny$\pm$0.28} & \underline{10.18{\tiny$\pm$0.54}} & \underline{98.04{\tiny$\pm$0.12}} \\
          10 & 35.73{\tiny$\pm$6.31} & 90.61{\tiny$\pm$1.53} & 18.41{\tiny$\pm$1.49} & 96.34{\tiny$\pm$0.23} \\
          \rowcolor{LightGray} 300 $\rightarrow$ 30 & \textbf{11.37{\tiny$\pm$0.42}} & \textbf{97.48{\tiny$\pm$0.10}} & \textbf{10.02{\tiny$\pm$0.38}} & 
          \textbf{98.04{\tiny$\pm$0.02}} \\
          \bottomrule
        \end{tabular}
      }
      \label{appendix:tab:sensitivity_alpha_waterbirds}
    } \\ \qquad \qquad \qquad
    \subfloat[Sensitivity of $\alpha$]{
      \resizebox{0.3\linewidth}{!}{%
        \begin{tabular}{@{}lcccc}
        \toprule
        \multirow{2}{*}{$\sigma$} & \multicolumn{2}{c}{\textbf{Spurious OOD}} \\  \cmidrule{2-3}
        & FPR@95$\downarrow$ & AUROC$\uparrow$ \\
        \midrule
        0.1 & 34.01{\tiny$\pm$0.70} & 90.21{\tiny$\pm$1.16} \\
        \rowcolor{LightGray} 0.5 & \cellcolor{LightGray}\textbf{11.37{\tiny$\pm$0.42}} & \textbf{97.48{\tiny$\pm$0.10}} \\
        2.5 & 51.75{\tiny$\pm$31.35} & 80.33{\tiny$\pm$18.22} \\ \bottomrule
        \end{tabular}
      }
      \label{appendix:tab:sensitivity_sigma_waterbirds}
    }
  \caption{Sensitivity analysis in Waterbirds datasets}
\end{table}

\subsection{Hyperparameter studies in CelebA datasets} 
\Cref{appendix:tab:sensitivity_alpha_celeba} presents the sensitivity analysis of the $\alpha$ hyperparameter in CelebA datasets. Due to the relative simplicity of the ID task, nearly perfect performance in conventional OOD detection is achieved across all values of $\alpha$. The sensitivity analysis of $\alpha$ (keeping $p_\text{inv}$ = 5\%) and $p_\text{inv}$ (keeping $\alpha$ = $50\rightarrow30$) is presented in the \Cref{appendix:tab:sensitivity_alpha_celeba} and Table \Cref{appendix:tab:p_inv_celeba} respectively. The results indicate that optimal performance is achieved when $\alpha$ is reduced from 50 to 30 and $p_\text{inv}$ is set to 5\%. \Cref{appendix:tab:sensitivity_sigma_celeba} shows the sensitivity analysis of $\sigma$ for spurious OOD detection in CelebA ($p_\text{inv}=5\%$) datasets. The results demonstrate that setting $\sigma$ to 0.5 yields the best OOD detection performance in CelebA datasets. In each study, outliers are synthesized as $\mathbf{x}^{\prime} = \mathbf{x} + \alpha \cdot \mathbf{G}_\text{inv}$. 
\vfill
\begin{table}[!ht]
\centering
    \resizebox{0.55\linewidth}{!}{%
    \begin{tabular}{@{}lcccc}
        \toprule
        \multirow{2}{*}{$\alpha$} & \multicolumn{2}{c}{\textbf{Spurious OOD}} & \multicolumn{2}{c}{\textbf{Conventional OOD}} \\
        \cmidrule(lr){2-3} \cmidrule(lr){4-5}
        & FPR@95$\downarrow$ & AUROC$\uparrow$ & FPR@95$\downarrow$ & AUROC$\uparrow$ \\
        \midrule
        40 & 51.90{\tiny$\pm$8.35} & 83.14{\tiny$\pm$1.81} & \textbf{0.01{\tiny$\pm$0.01}} & \textbf{99.99{\tiny$\pm$0.00}} \\
        30 & 47.99{\tiny$\pm$7.47} & 83.98{\tiny$\pm$2.01} & \textbf{0.01{\tiny$\pm$0.02}} & \textbf{99.98{\tiny$\pm$0.00}} \\
        10 & 50.49{\tiny$\pm$3.86} & 83.09{\tiny$\pm$0.99} & \textbf{0.00{\tiny$\pm$0.00}} & \textbf{99.99{\tiny$\pm$0.00}} \\
        5.0 & 55.03{\tiny$\pm$8.16} & 82.16{\tiny$\pm$2.04} & \textbf{0.02{\tiny$\pm$0.03}} & \textbf{99.98{\tiny$\pm$0.01}} \\
        \rowcolor{LightGray} 50 $\rightarrow$ 30 & \textbf{46.61{\tiny$\pm$0.63}} & \textbf{84.12{\tiny$\pm$0.55}} & \textbf{0.01{\tiny$\pm$0.01}} & \textbf{99.99{\tiny$\pm$0.00}} \\
        \bottomrule
    \end{tabular}}
    \caption{Sensitivity of $\alpha$ hyperparameter in CelebA datasets.}
    \label{appendix:tab:sensitivity_alpha_celeba}
\end{table}
\vfill
\begin{table}[h!]
\centering
\resizebox{0.75\linewidth}{!}{%
\begin{tabular}{@{}lccccc}
\toprule
\multirow{2}{*}{Metrics} & \multicolumn{3}{c}{\textbf{$p_\text{inv}$}} \\
\cmidrule(lr){2-4}
& 1\% & 5\% & 10\%\\
\midrule
FPR@95 $\downarrow$ / AUROC $\uparrow$ & 49.42{\tiny$\pm$4.25} / 83.62{\tiny$\pm$0.55} & \cellcolor{LightGray} \textbf{46.61{\tiny$\pm$0.63}} / \textbf{84.12{\tiny$\pm$0.55}} & 52.32{\tiny$\pm$5.67} / 82.19{\tiny$\pm$2.15} \\ \bottomrule
\end{tabular}
}
\caption{Sensitivity of $p_\text{inv}$ in terms of spurious OOD detection in CelebA datasets.}
\label{appendix:tab:p_inv_celeba}
\end{table}
\vfill

\begin{table}[h!]
\centering
\resizebox{0.25\linewidth}{!}{%
    \begin{tabular}{@{}lcc}
        \toprule
        \multirow{2}{*}{$\sigma$} & \multicolumn{2}{c}{\textbf{Spurious OOD}} \\ \cmidrule{2-3}
        & FPR@95$\downarrow$ & AUROC$\uparrow$ \\
        \midrule
        0.1 & 66.16{\tiny$\pm$0.59} & 66.08{\tiny$\pm$3.24} \\
        \rowcolor{LightGray} 0.5 & \cellcolor{LightGray}\textbf{46.61{\tiny$\pm$0.63}} & \textbf{84.12{\tiny$\pm$0.55}} \\
        2.5 & 61.08{\tiny$\pm$12.02} & 80.89{\tiny$\pm$3.43} \\
        \bottomrule
    \end{tabular}}
    \caption{Sensitivity of $\sigma$ hyperparameter in terms of spurious OOD detection in CelebA datasets.}
    \label{appendix:tab:sensitivity_sigma_celeba}
\end{table}
\vfill
\newpage

\vfill

\newpage
\subsection{Ablation study of outlier synthesis in CIFAR-10/100 datasets}
The ablation study of outlier synthesis in CIFAR-10/100 datasets is presented in \Cref{appendix:tab:cifar_outlier_ablation}. Our study examines various mechanisms for synthesizing outliers and reveals that shuffling invariant pixels yields optimal results for the CIFAR-100 benchmark. Furthermore, we find that even utilizing mere random shuffling of pixels can achieve good OOD detection performance in the context of the relatively easier CIFAR-10 benchmark. This highlights the effectiveness of pixel shuffle perturbation in outlier synthesis. We observe that adding Gaussian noise and $\mathbf{G}_{\text{inv}}$ to ID for synthesizing virtual outliers performs comparably to other methods on the CIFAR-10 benchmark but significantly underperforms in average OOD detection on the CIFAR-100 benchmark. The datasets yielding higher accuracy are more likely to benefit from utilizing the gradient
$\mathbf{G}$ addition for outlier synthesis, as increased accuracy suggests reliable outlier synthesis in later training stages. Hence, we hypothesize that using $\mathbf{G}_\text{inv}$ addition for outlier synthesis is less effective than mere random pixel shuffle perturbation in CIFAR-100 because of the modest accuracy ($\sim76.63\%$).
\begin{table}[h!]
  \centering
  \resizebox{0.80\linewidth}{!}{
    \begin{tabular}{@{}lcccccc@{}}
      \toprule
      \multirow{2}{*}{Outlier synthesis} & \multirow{2}{*}{$p_\text{inv}/\alpha$} & \multicolumn{2}{c}{\textbf{CIFAR-10}} & \multicolumn{2}{c}{\textbf{CIFAR-100}} \\
      \cmidrule(lr){3-4} \cmidrule(lr){5-6}
      &  & FPR@95$\downarrow$ & AUROC$\uparrow$ & FPR@95$\downarrow$ & AUROC$\uparrow$ \\
      \midrule

      \multirow{6}{*}{Random pixel shuffle} 
      & 5\%  & 9.17{\tiny$\pm$0.76} & 97.95{\tiny$\pm$0.19} &  35.05{\tiny$\pm$1.24} & 89.41{\tiny$\pm$0.54} \\
      & 10\% & 8.67{\tiny$\pm$0.49} & 98.13{\tiny$\pm$0.10} &  34.15{\tiny$\pm$0.77} & 89.13{\tiny$\pm$0.26} \\
      & 15\% & 10.28{\tiny$\pm$0.76} & 97.79{\tiny$\pm$0.18} &  35.60{\tiny$\pm$2.69} & 89.58{\tiny$\pm$0.93} \\
      & 20\% & \cellcolor{LightGray} \textbf{7.69{\tiny$\pm$0.29}} & \cellcolor{LightGray} \textbf{98.32{\tiny$\pm$0.07}} &  34.18{\tiny$\pm$1.37} & 90.23{\tiny$\pm$0.88} \\
      & 30\% & 9.04{\tiny$\pm$0.58} & 98.05{\tiny$\pm$0.10} &  35.02{\tiny$\pm$2.44} & 89.99{\tiny$\pm$0.63} \\
      & 100\% & 10.04{\tiny$\pm$0.74} & 97.78{\tiny$\pm$0.15} & 36.60{\tiny$\pm$0.90} & 89.85{\tiny$\pm$0.33} \\
      \midrule

      \multirow{4}{*}{Invariant pixel shuffle} 
      & 5\%  & 8.88{\tiny$\pm$0.29} & 98.07{\tiny$\pm$0.07} &  32.47{\tiny$\pm$3.19} & 89.91{\tiny$\pm$1.31} \\
      & 10\% & 8.09{\tiny$\pm$0.42} & 98.29{\tiny$\pm$0.06} &  \cellcolor{LightGray} \textbf{29.90{\tiny$\pm$0.76}} & \cellcolor{LightGray} \textbf{91.35{\tiny$\pm$0.13}} \\
      & 15\% & {7.84{\tiny$\pm$0.27}} & {98.31{\tiny$\pm$0.07}} &  32.15{\tiny$\pm$1.30} & 90.42{\tiny$\pm$0.31} \\
      & 20\% & 8.66{\tiny$\pm$0.55} & 98.14{\tiny$\pm$0.15} &  35.43{\tiny$\pm$1.67} & 89.63{\tiny$\pm$0.87} \\      
      \midrule

      \multirow{4}{*}{$\mathbf{G}_\text{inv}$ addition} 
      & 1.0  & 10.43{\tiny$\pm$0.97} & 97.71{\tiny$\pm$0.24} &  46.39{\tiny$\pm$8.48} & 85.56{\tiny$\pm$3.10} \\
      & 10.0 & 9.28{\tiny$\pm$0.95} & 98.00{\tiny$\pm$0.19} &  39.23{\tiny$\pm$1.63} & 88.31{\tiny$\pm$0.95} \\
      & 100.0 & 9.99{\tiny$\pm$0.88} & 97.85{\tiny$\pm$0.23} &  48.22{\tiny$\pm$2.11} & 84.60{\tiny$\pm$0.86} \\      
      \midrule

      \multirow{4}{*}{Gaussian noise addition} 
      & 0.01  & 8.74{\tiny$\pm$0.26} & 97.95{\tiny$\pm$0.18} &  42.64{\tiny$\pm$3.98} & 88.33{\tiny$\pm$1.50} \\
      & 0.1 & 8.61{\tiny$\pm$0.25} & 98.10{\tiny$\pm$0.05} &  43.81{\tiny$\pm$5.01} & 86.48{\tiny$\pm$1.44} \\
      & 1.0 & 12.58{\tiny$\pm$1.75} & 97.12{\tiny$\pm$0.43} &  47.34{\tiny$\pm$3.92} & 84.86{\tiny$\pm$1.20} \\
      & 10.0 & 13.90{\tiny$\pm$1.21} & 96.75{\tiny$\pm$0.28} &  53.82{\tiny$\pm$2.08} & 81.40{\tiny$\pm$1.32} \\      
      \bottomrule
    \end{tabular}}
  \caption{Ablation study of outlier synthesis in CIFAR-10/100 datasets.}
  \label{appendix:tab:cifar_outlier_ablation}
\end{table}
\vfill

\subsection{Ablation study of outlier synthesis in Waterbirds datasets}
\Cref{appendix:tab:waterbird_outlier_ablation} presents an ablation study of outlier synthesis approaches in the Waterbirds datasets. From the results, we can observe a trend in both random pixel shuffle and invariant pixel shuffle: as the level of pixel shuffling increases, the presence of ID environmental features in the outliers diminishes, making the outliers increasingly trivial. For a given level of pixel shuffling $p_\text{inv}$, shuffling invariant pixels derived from $\mathbf{G}_{\text{inv}}$ results in superior performance compared to random pixel shuffling. This indicates that $\mathbf{G}_{\text{inv}}$, as the model continues learning ID discrimination, indicates pixels most crucial for class recognition. Shuffling these pixels effectively disrupts the semantic content synthesizing reliable outliers. Moreover, synthesizing outliers by adding Gaussian noise to ID data leads to a relative decline in OOD detection performance. Gaussian noise introduces extraneous information and results in relatively more distortion of environmental/background features within the synthesized outliers. However, preserving environmental/background features in these outliers is critical in mitigating the effect of high spurious correlations in the training set. Hence, the outlier synthesis approach resulting in minimal alteration in environmental features is favorable in Waterbirds datasets.

\begin{table}[h!]
  \centering
  \resizebox{0.99\linewidth}{!}{%
    \begin{tabular}{@{}l ccc ccc ccc@{}}
      \toprule
      \multirow{2}{*}{OOD type} & \multicolumn{3}{c}{\textbf{Random pixel shuffle}} & \multicolumn{3}{c}{\textbf{Invariant pixel shuffle}} & \multicolumn{3}{c}{\textbf{Gaussian noise}} \\
      \cmidrule(lr){2-4} \cmidrule(lr){5-7} \cmidrule(lr){8-10}
      & $p_\text{inv}$ & FPR@95$\downarrow$ & AUROC$\uparrow$ & $p_\text{inv}$ & FPR@95$\downarrow$ & AUROC$\uparrow$ & $\alpha$ & FPR@95$\downarrow$ & AUROC$\uparrow$ \\
      \midrule

      \multirow{6}{*}{Spurious} 
      & 5\%  & \cellcolor{LightGray} \textbf{19.48{\tiny$\pm$1.18}} & \cellcolor{LightGray} \textbf{96.49{\tiny$\pm$0.10}} & 5\%  & \cellcolor{LightGray} \textbf{13.21{\tiny$\pm$1.24}} & \cellcolor{LightGray} \textbf{97.43{\tiny$\pm$0.16}} & 0.01  & 77.37{\tiny$\pm$3.56} & 74.52{\tiny$\pm$3.23} \\
      & 10\% & 24.32{\tiny$\pm$2.09} & 95.33{\tiny$\pm$0.39} & 10\% & 14.80{\tiny$\pm$0.19} & 97.26{\tiny$\pm$0.05} & 0.05 & \textbf{36.48{\tiny$\pm$6.03}} & 89.89{\tiny$\pm$2.37} \\
      & 15\% & 28.71{\tiny$\pm$2.45} & 94.32{\tiny$\pm$0.35} & 15\% & 17.63{\tiny$\pm$3.78} & 96.73{\tiny$\pm$0.60} & 0.1 & 40.44{\tiny$\pm$6.72} & 90.75{\tiny$\pm$1.81} \\
      & 20\% & 30.23{\tiny$\pm$3.79} & 93.87{\tiny$\pm$0.82} & 20\% & 20.91{\tiny$\pm$1.20} & 96.17{\tiny$\pm$0.20} & 0.5 & \cellcolor{LightGray} 37.86{\tiny$\pm$4.93} & \cellcolor{LightGray} \textbf{91.71{\tiny$\pm$1.38}} \\
      & 30\% & 32.43{\tiny$\pm$3.45} & 92.81{\tiny$\pm$1.02} & 30\% & 28.02{\tiny$\pm$1.76} & 94.50{\tiny$\pm$0.61} & 1.0 & 40.47{\tiny$\pm$2.21} & 90.48{\tiny$\pm$0.56} \\
      & 100\% & 45.16{\tiny$\pm$2.75} & 89.33{\tiny$\pm$1.22} & 100\% & 45.16{\tiny$\pm$2.75} & 89.33{\tiny$\pm$1.22} & 10.0 & 49.93{\tiny$\pm$4.98} & 86.40{\tiny$\pm$1.27} \\
      \midrule

      \multirow{6}{*}{Conventional} 
      & 5\%  & \cellcolor{LightGray} \textbf{22.53{\tiny$\pm$1.44}} & \cellcolor{LightGray} \textbf{95.81{\tiny$\pm$0.35}} & 5\%  & \cellcolor{LightGray} \textbf{15.36{\tiny$\pm$1.79}} & \cellcolor{LightGray} \textbf{97.12{\tiny$\pm$0.24}} & 0.01  & 39.77{\tiny$\pm$6.38} & 86.20{\tiny$\pm$5.50} \\
      & 10\% & 25.95{\tiny$\pm$2.26} & 94.82{\tiny$\pm$0.65} & 10\% & 16.58{\tiny$\pm$0.74} & 96.99{\tiny$\pm$0.07} & 0.05 & 29.12{\tiny$\pm$2.01} & 94.01{\tiny$\pm$0.57} \\
      & 15\% & 27.76{\tiny$\pm$1.83} & 94.14{\tiny$\pm$0.33} & 15\% & 20.35{\tiny$\pm$2.04} & 96.27{\tiny$\pm$0.28} & 0.1 & \cellcolor{LightGray} \textbf{23.49{\tiny$\pm$1.18}} & \cellcolor{LightGray} \textbf{95.88{\tiny$\pm$0.28}} \\
      & 20\% & 28.52{\tiny$\pm$3.30} & 94.10{\tiny$\pm$0.78} & 20\% & 21.51{\tiny$\pm$1.48} & 96.05{\tiny$\pm$0.32} & 0.5 & 33.57{\tiny$\pm$2.17} & 92.43{\tiny$\pm$0.65} \\
      & 30\% & 31.27{\tiny$\pm$2.66} & 92.94{\tiny$\pm$0.92} & 30\% & 27.43{\tiny$\pm$1.44} & 94.37{\tiny$\pm$0.14} & 1.0 & 37.66{\tiny$\pm$1.91} & 90.99{\tiny$\pm$0.93} \\
      & 100\% & 36.05{\tiny$\pm$2.38} & 91.46{\tiny$\pm$1.01} & 100\% & 36.05{\tiny$\pm$2.38} & 91.46{\tiny$\pm$1.01} & 10.0 & 45.55{\tiny$\pm$2.06} & 87.72{\tiny$\pm$0.65} \\
      \bottomrule
    \end{tabular}}
  \caption{Ablation study of outlier synthesis in Waterbirds datasets.}
  \label{appendix:tab:waterbird_outlier_ablation}
\end{table}

\subsection{Ablation study of outlier synthesis in CelebA datasets}
\Cref{appendix:tab:celeba_outlier_ablation} presents the ablation study of outlier synthesis approaches in the CelebA datasets. The table compares three methods of outlier synthesis. In any setting, it can be observed that \ours~achieves near-perfect conventional OOD Detection performance. It is due to the trivial nature of ID classification which is the task of color identification. Similar to the case in Waterbirds datasets, a notable trend observed in random pixel shuffling is that increasing the percentage of shuffled pixels degrades OOD detection performance. A similar effect is evident when using the invariant pixel shuffle approach. We attribute this to the disruption of environmental features as more pixels are shuffled, making the outliers less challenging and thereby diminishing detection efficacy. Additionally, we observe that Gaussian noise addition for outlier synthesis yields the best performance on the spurious OOD Detection in CelebA benchmark among these compared approaches. This approach also forms challenging outliers because of the use of ID data in outlier synthesis.

\begin{table}[!h]
  \centering
  \resizebox{0.99\linewidth}{!}{%
    \begin{tabular}{@{}l ccc ccc ccc@{}}
      \toprule
      \multirow{2}{*}{OOD type} & \multicolumn{3}{c}{\textbf{Random pixel shuffle}} & \multicolumn{3}{c}{\textbf{Invariant pixel shuffle}} & \multicolumn{3}{c}{\textbf{Gaussian noise addition}} \\
      \cmidrule(lr){2-4} \cmidrule(lr){5-7} \cmidrule(lr){8-10}
      & $p_\text{inv}$ & FPR@95$\downarrow$ & AUROC$\uparrow$ & $p_\text{inv}$ & FPR@95$\downarrow$ & AUROC$\uparrow$ & $\alpha$ & FPR@95$\downarrow$ & AUROC$\uparrow$ \\
      \midrule

      \multirow{6}{*}{Spurious} 
      & 5\%  & \cellcolor{LightGray} \textbf{52.79{\tiny$\pm$1.29}} & \cellcolor{LightGray} \textbf{84.77{\tiny$\pm$0.63}} & 5\%  & \cellcolor{LightGray} \textbf{47.27{\tiny$\pm$1.39}} & \cellcolor{LightGray} \textbf{85.37{\tiny$\pm$0.38}} & 0.01  & 76.83{\tiny$\pm$6.17} & 61.75{\tiny$\pm$5.48} \\
      & 10\% & 55.05{\tiny$\pm$2.80} & 83.59{\tiny$\pm$0.89} & 10\% & 50.34{\tiny$\pm$5.00} & 84.70{\tiny$\pm$0.70} & 0.05 & 49.30{\tiny$\pm$3.07} & 83.38{\tiny$\pm$0.87} \\
      & 15\% & 55.41{\tiny$\pm$3.38} & 83.29{\tiny$\pm$1.02} & 15\% & 51.39{\tiny$\pm$4.82} & 84.50{\tiny$\pm$1.86} & 0.1 & \cellcolor{LightGray} \textbf{42.62{\tiny$\pm$0.76}} & \cellcolor{LightGray} \textbf{86.50{\tiny$\pm$0.89}} \\
      & 20\% & 59.05{\tiny$\pm$2.85} & 82.19{\tiny$\pm$0.78} & 20\% & 59.05{\tiny$\pm$2.85} & 82.19{\tiny$\pm$0.78} & 0.5 & 52.90{\tiny$\pm$1.54} & 83.05{\tiny$\pm$2.00} \\
      & 30\% & 58.60{\tiny$\pm$3.99} & 82.41{\tiny$\pm$1.80} & 30\% & 58.60{\tiny$\pm$3.99} & 82.41{\tiny$\pm$1.80} & 1.0 & 60.30{\tiny$\pm$4.89} & 82.58{\tiny$\pm$1.63} \\
      & 100\% & 62.05{\tiny$\pm$1.81} & 81.74{\tiny$\pm$1.07} & 100\% & 62.05{\tiny$\pm$1.81} & 81.74{\tiny$\pm$1.07} & 10.0 & 60.40{\tiny$\pm$8.12} & 80.85{\tiny$\pm$2.61} \\
      \midrule

      \multirow{6}{*}{Conventional} 
      & 5\%  & \cellcolor{LightGray} \textbf{0.02{\tiny$\pm$0.02}} & \cellcolor{LightGray} \textbf{99.97{\tiny$\pm$0.01}} & 5\%  & \cellcolor{LightGray} \textbf{0.02{\tiny$\pm$0.03}} & \cellcolor{LightGray} \textbf{99.98{\tiny$\pm$0.00}} & 0.01  & 74.26{\tiny$\pm$6.23} & 37.09{\tiny$\pm$7.04} \\
      & 10\% & 0.03{\tiny$\pm$0.04} & 99.97{\tiny$\pm$0.01} & 10\% & 0.03{\tiny$\pm$0.04} & 99.97{\tiny$\pm$0.00} & 0.05 & 0.02{\tiny$\pm$0.02} & 99.98{\tiny$\pm$0.00} \\
      & 15\% & 0.04{\tiny$\pm$0.02} & 99.97{\tiny$\pm$0.01} & 15\% & 0.03{\tiny$\pm$0.02} & 99.97{\tiny$\pm$0.00} & 0.1 & \cellcolor{LightGray} \textbf{0.02{\tiny$\pm$0.02}} & \cellcolor{LightGray} \textbf{99.98{\tiny$\pm$0.00}} \\
      & 20\% & 0.05{\tiny$\pm$0.03} & 99.97{\tiny$\pm$0.01} & 20\% & 0.05{\tiny$\pm$0.03} & 99.97{\tiny$\pm$0.01} & 0.5 & 0.32{\tiny$\pm$0.05} & 99.90{\tiny$\pm$0.01} \\
      & 30\% & 0.06{\tiny$\pm$0.01} & 99.97{\tiny$\pm$0.00} & 30\% & 0.06{\tiny$\pm$0.01} & 99.97{\tiny$\pm$0.00} & 1.0 & 0.08{\tiny$\pm$0.05} & 99.95{\tiny$\pm$0.02}\\
      & 100\% & 0.13{\tiny$\pm$0.09} & 99.95{\tiny$\pm$0.01} & 100\% & 0.13{\tiny$\pm$0.09} & 99.95{\tiny$\pm$0.01} & 10.0 & 0.32{\tiny$\pm$0.05} & 99.90{\tiny$\pm$0.01} \\
      \bottomrule
    \end{tabular}
  }
  \caption{Ablation study of outlier synthesis in CelebA datasets.}
  \label{appendix:tab:celeba_outlier_ablation}
\end{table}

\newpage
\section{Confidence score plots with conventional OOD datasets:}
\label{sec:confidence_plot_conventional}

\subsection{Waterbirds datasets}
\begin{figure}[htbp]
    \centering
    \adjustbox{width=0.95\linewidth}{
        \begin{tabular}{c c c c c}
            \includegraphics[width=0.19\linewidth]{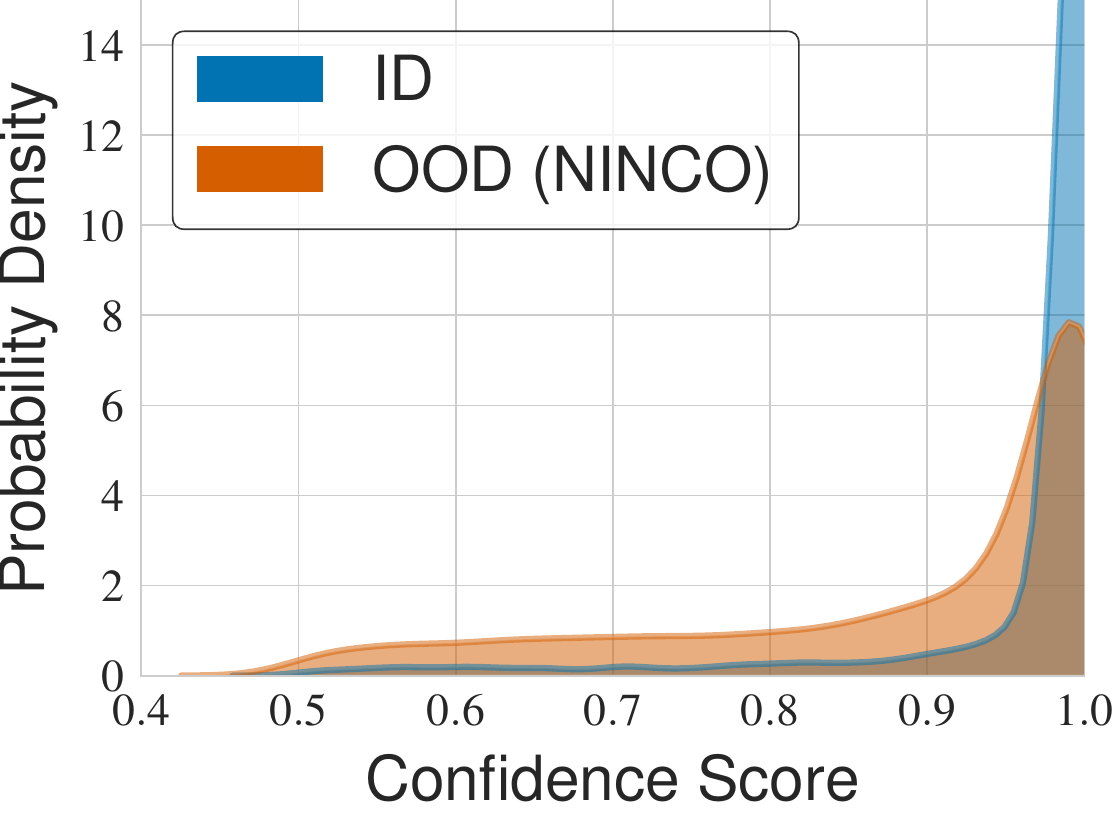} &
            \includegraphics[width=0.19\linewidth]{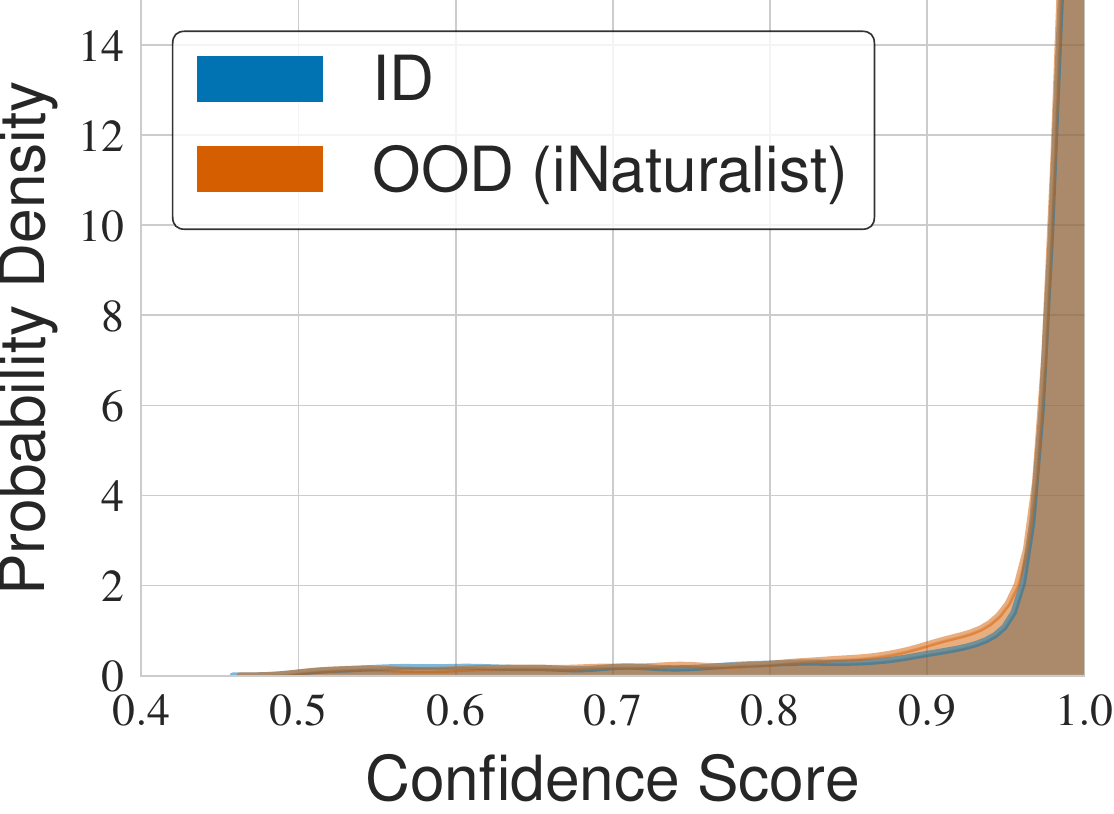} &
            \includegraphics[width=0.19\linewidth]{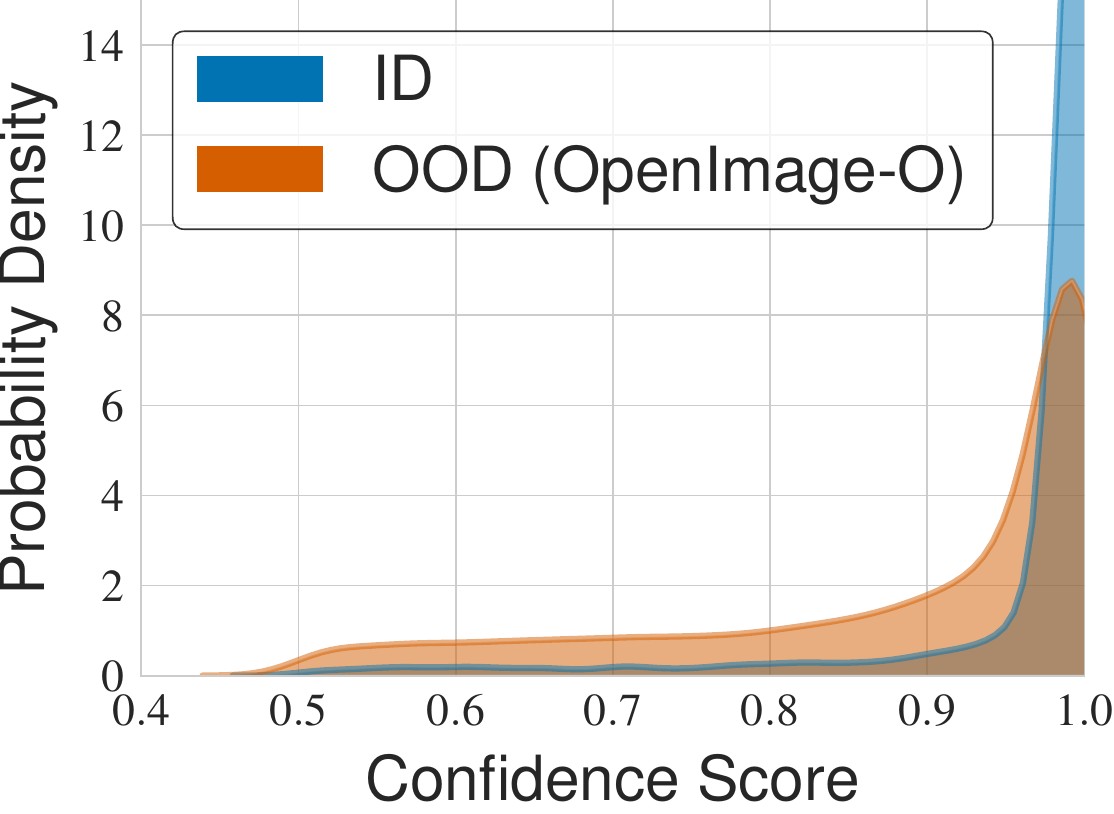} &
            \includegraphics[width=0.19\linewidth]{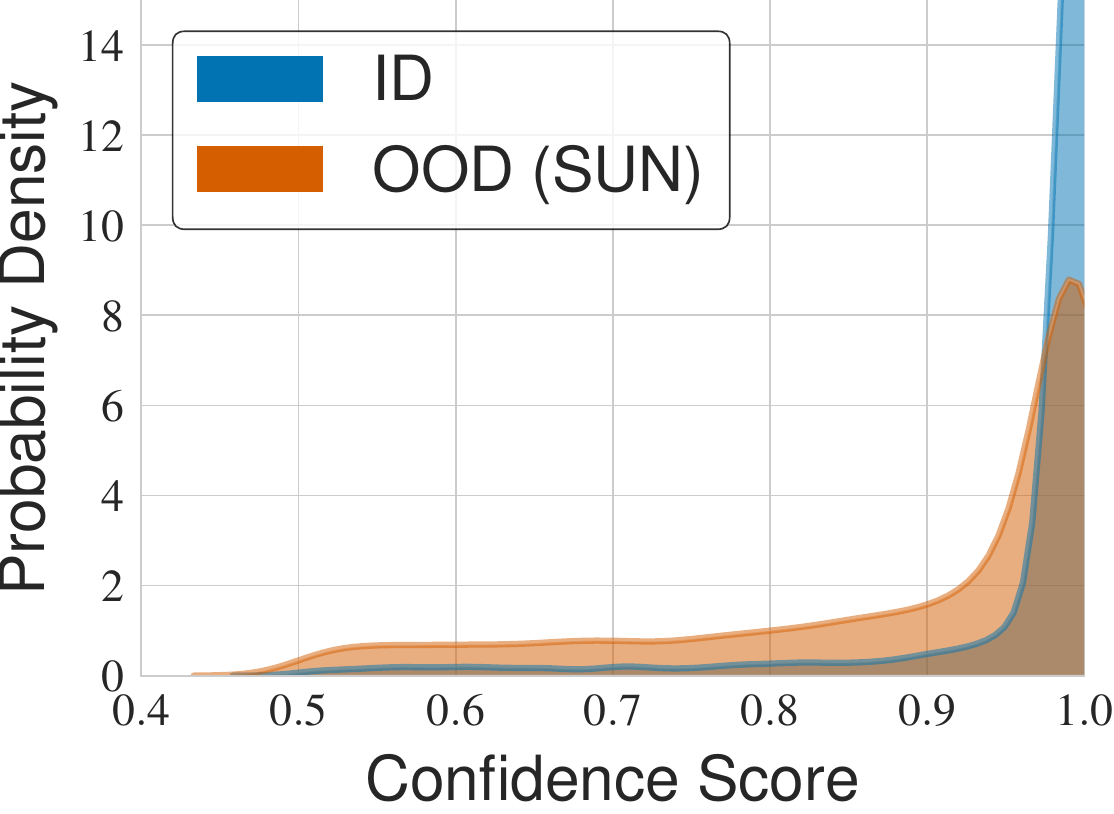} &
            \includegraphics[width=0.19\linewidth]{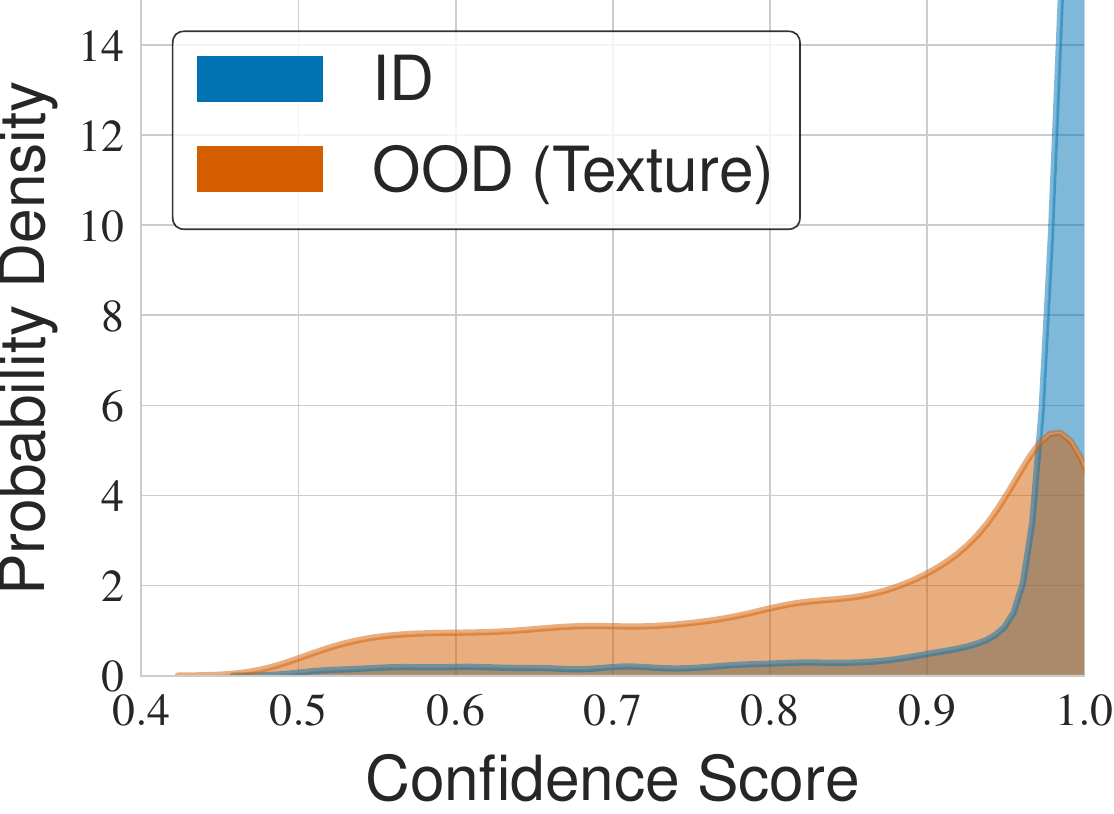} \\

            \includegraphics[width=0.19\linewidth]{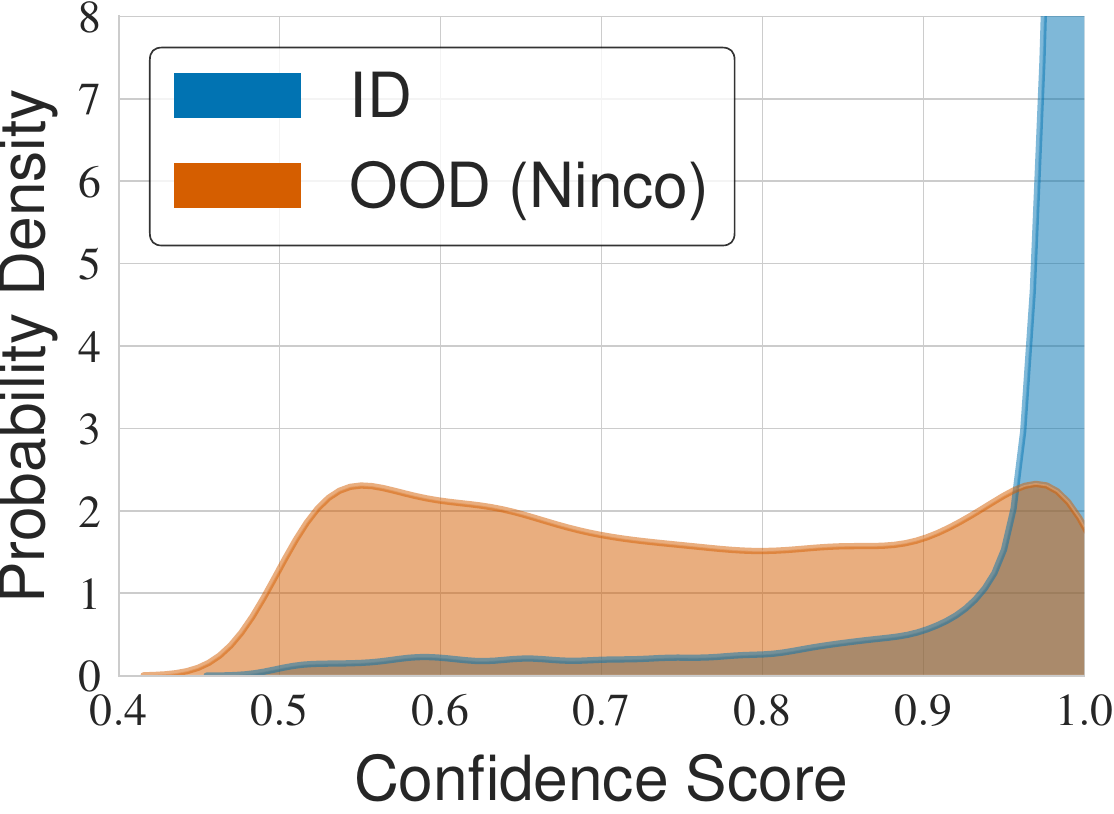} &
            \includegraphics[width=0.19\linewidth]{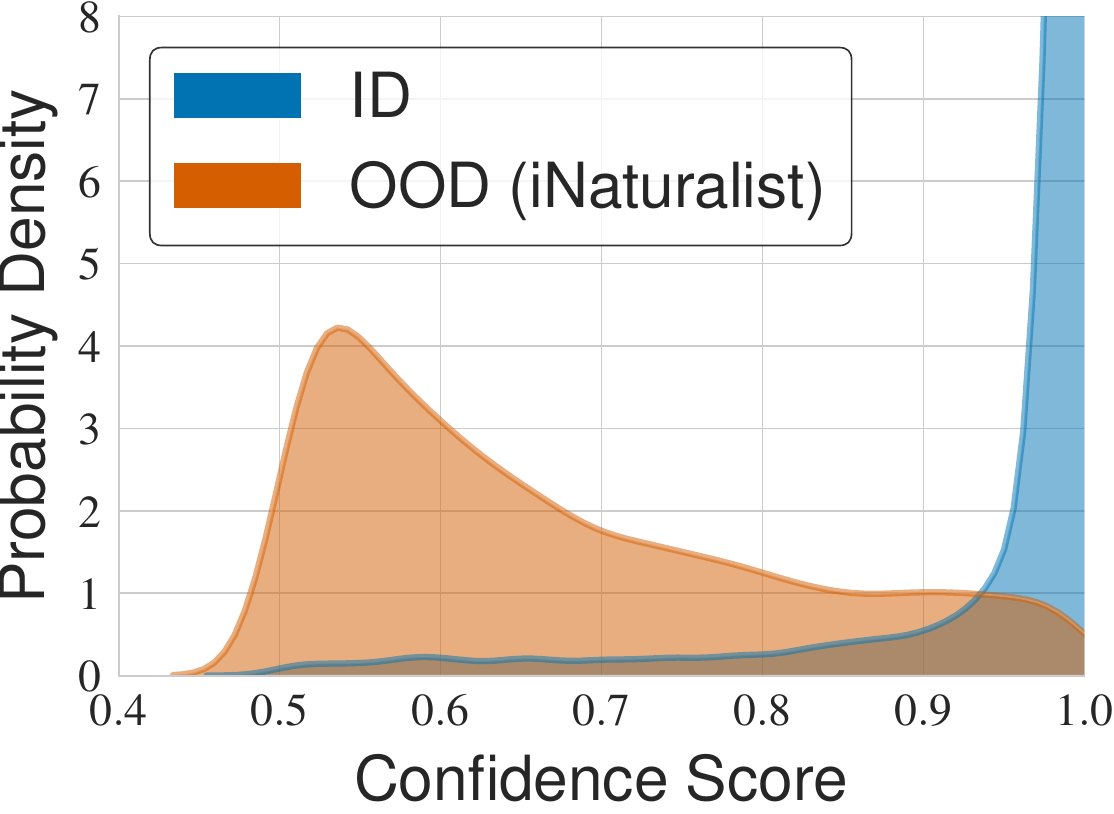} &
            \includegraphics[width=0.19\linewidth]{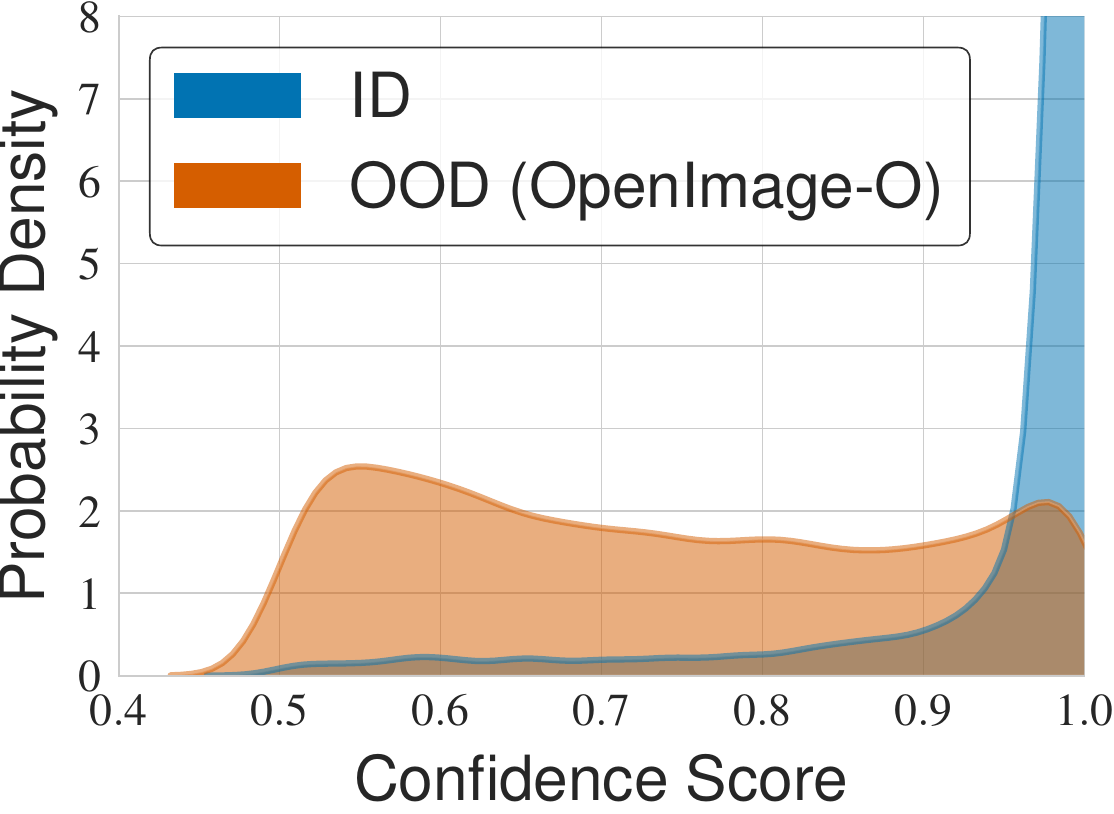} &
            \includegraphics[width=0.19\linewidth]{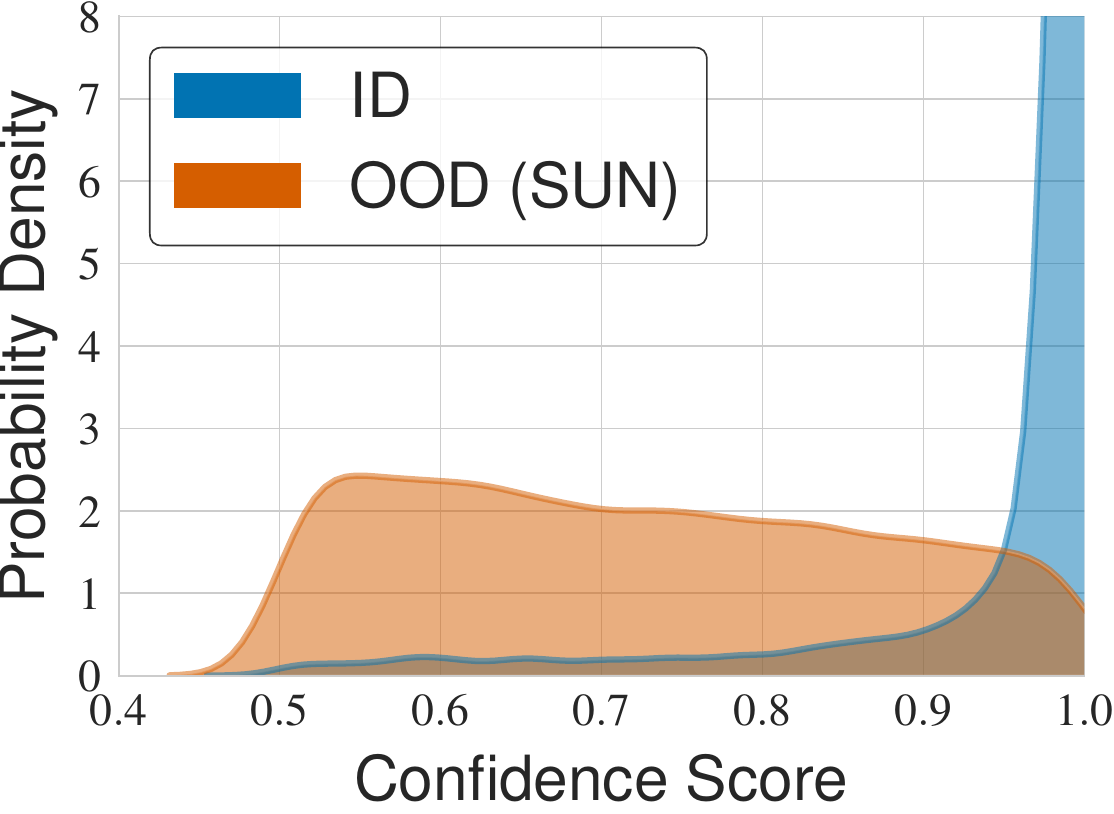} &
            \includegraphics[width=0.19\linewidth]{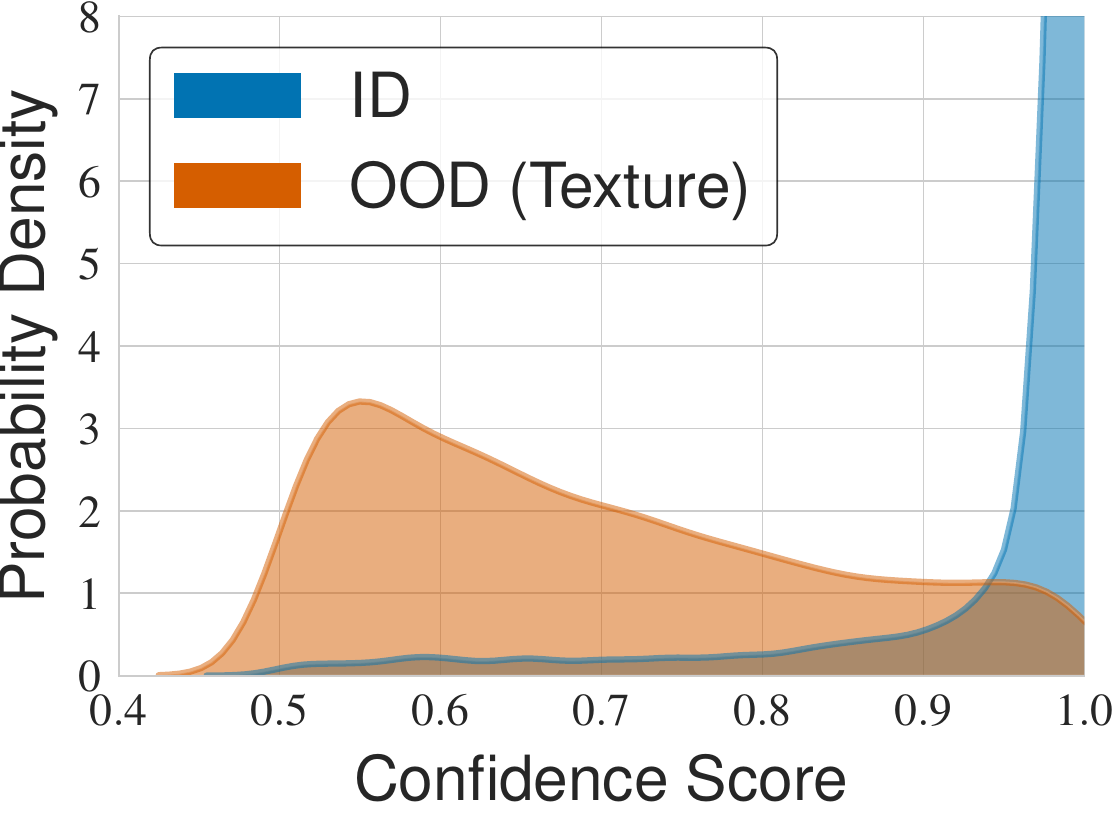}
        \end{tabular}
    }
    \caption{\textbf{Top row:} Visualization of confidence scores of cross-entropy baseline for Waterbirds ID and various conventional OOD datasets. \textbf{Bottom row:} Visualization of confidence scores of \ours~for Waterbirds ID and various conventional OOD datasets.}
\end{figure}

\subsection{CelebA datasets}
\begin{figure}[htbp]
    \centering
    \adjustbox{width=0.95\linewidth}{
        \begin{tabular}{c c c c c}
            \includegraphics[width=0.19\linewidth]{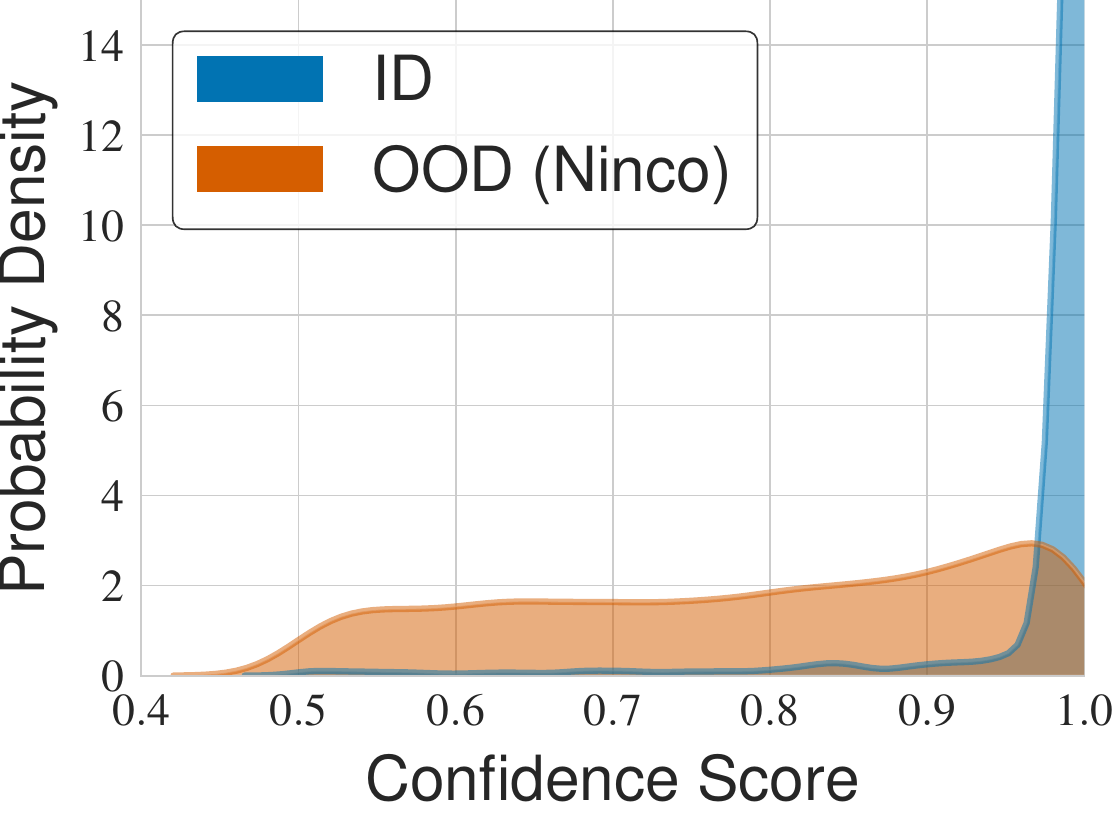} &
            \includegraphics[width=0.19\linewidth]{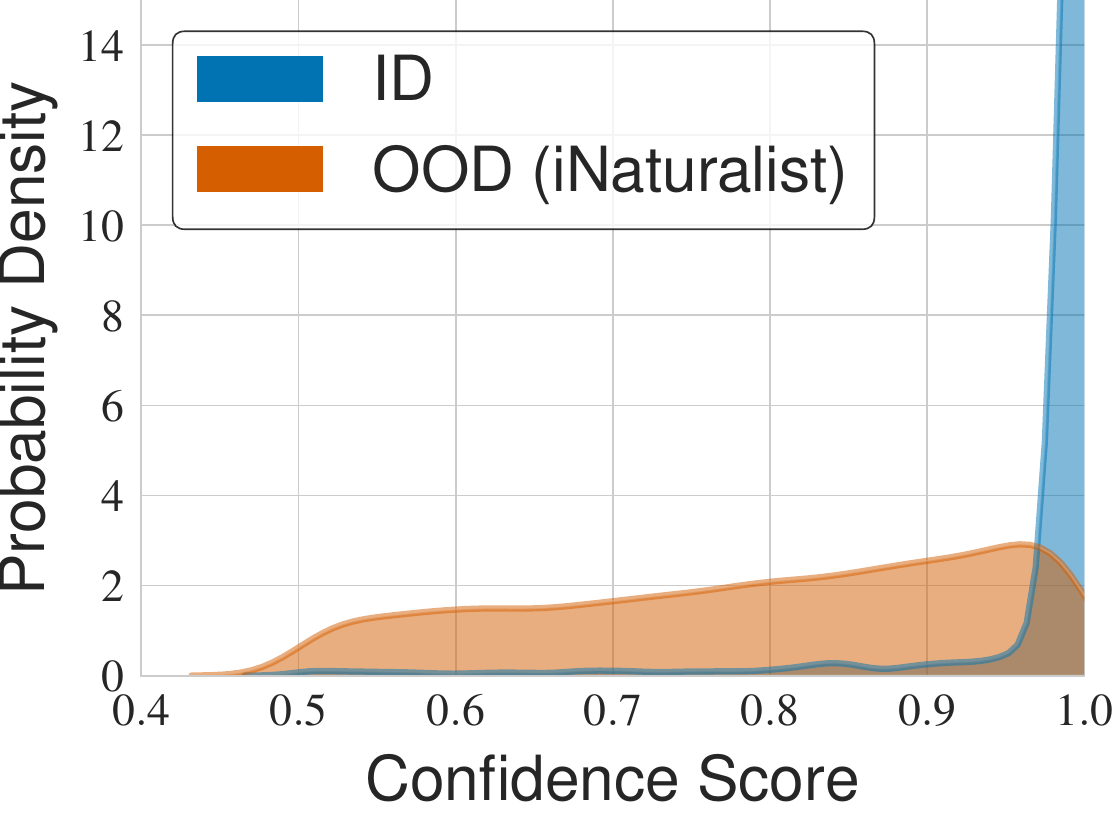} &
            \includegraphics[width=0.19\linewidth]{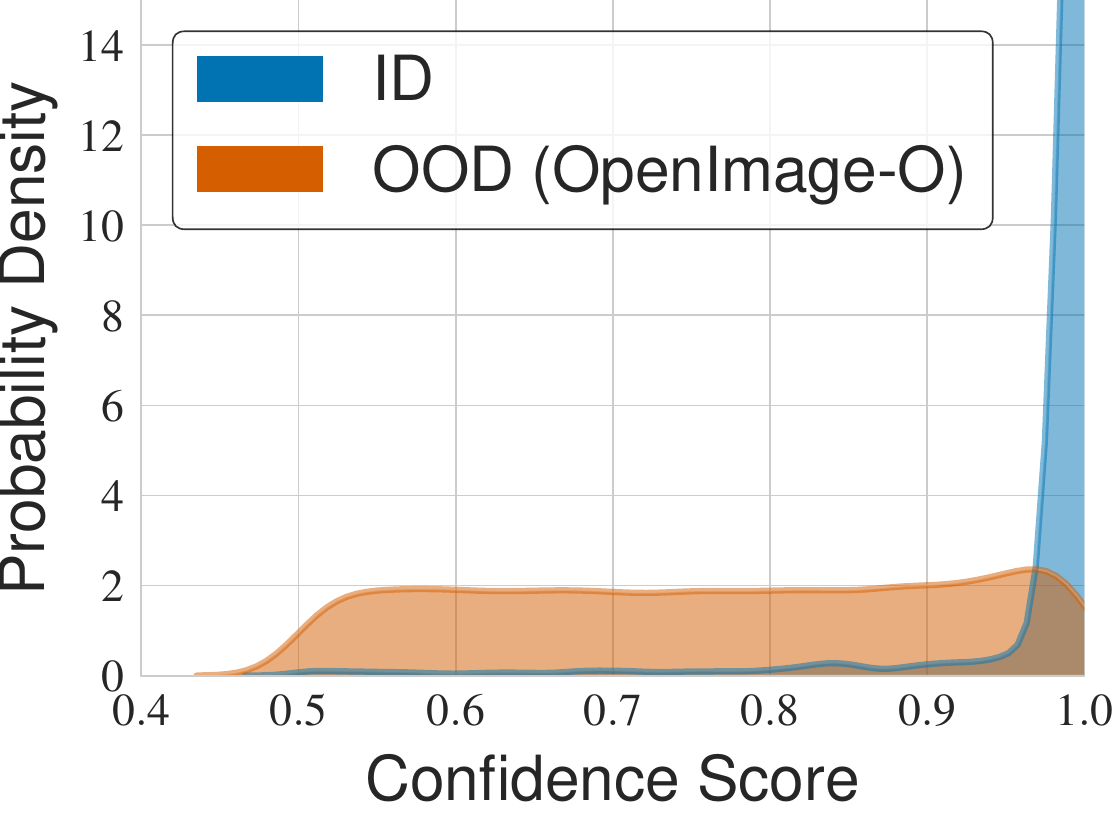} &
            \includegraphics[width=0.19\linewidth]{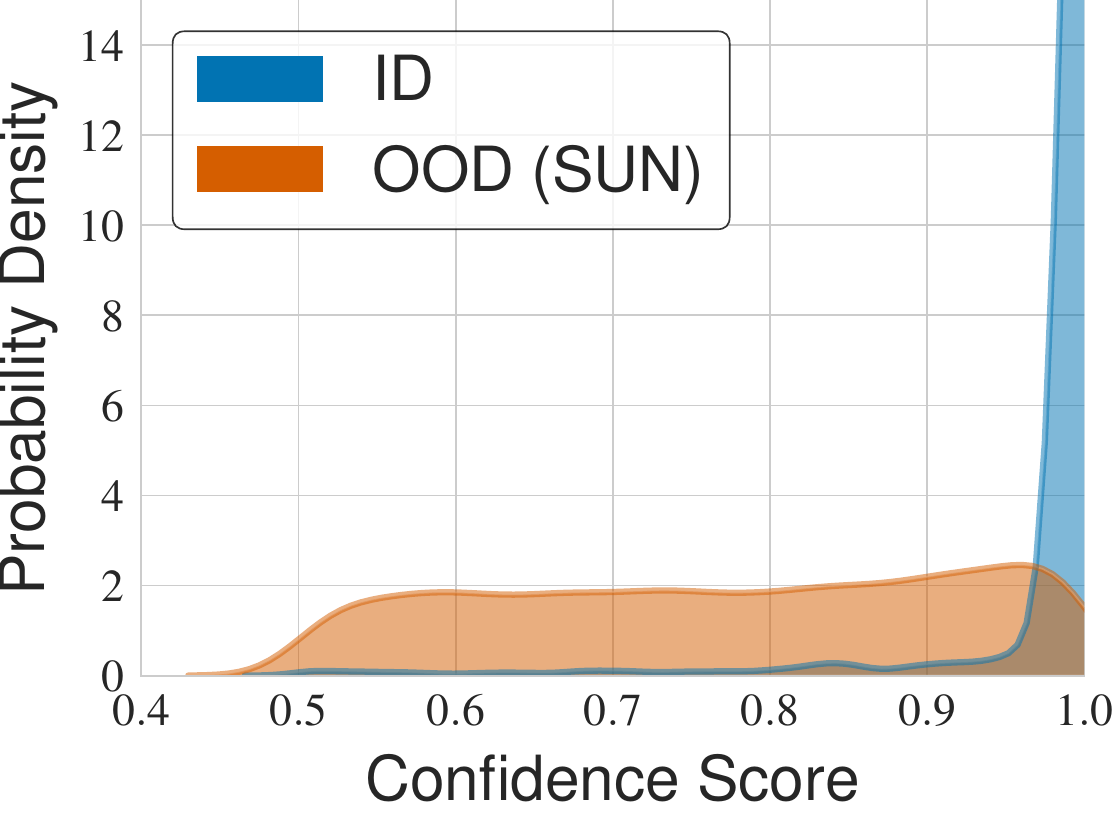} &
            \includegraphics[width=0.19\linewidth]{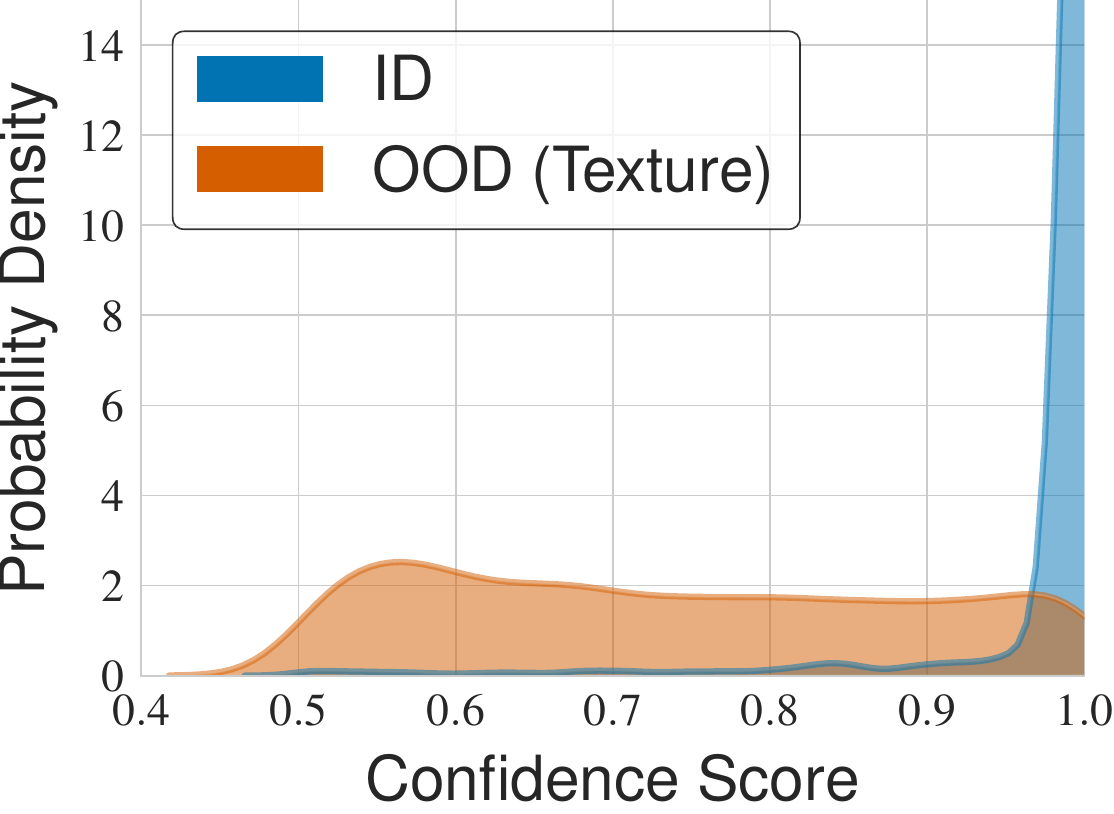} \\

            \includegraphics[width=0.19\linewidth]{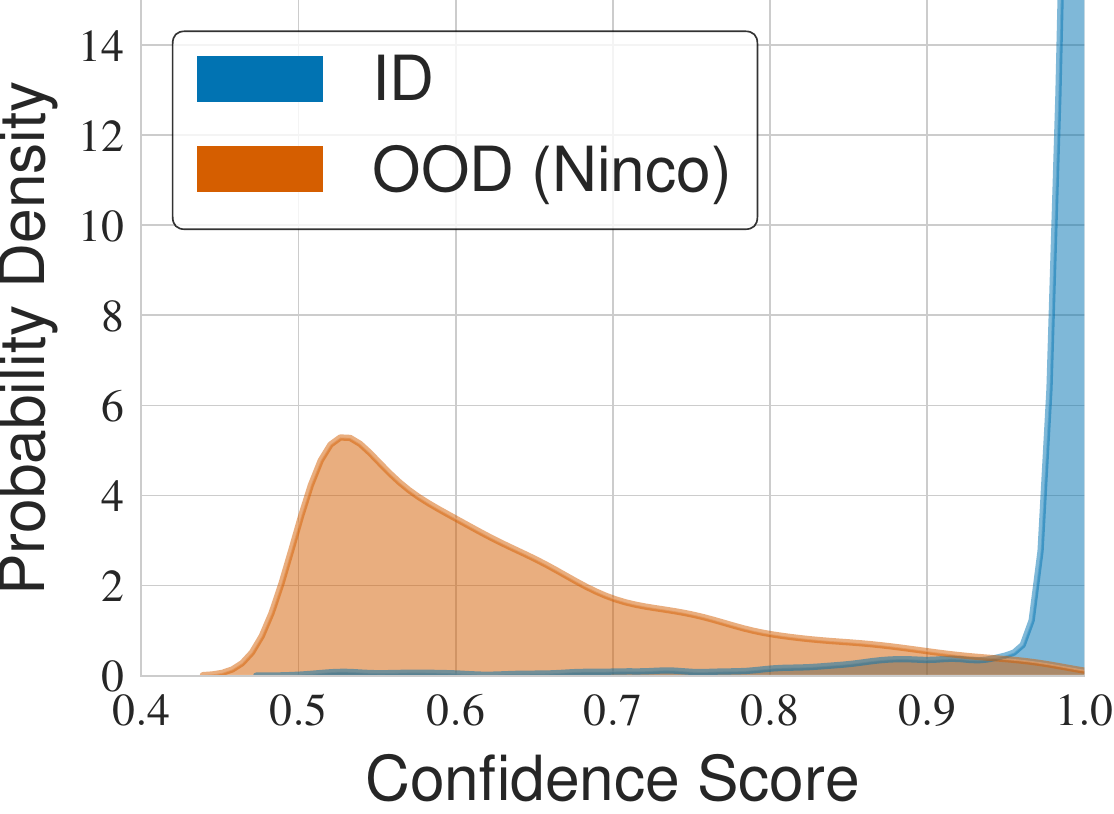} &
            \includegraphics[width=0.19\linewidth]{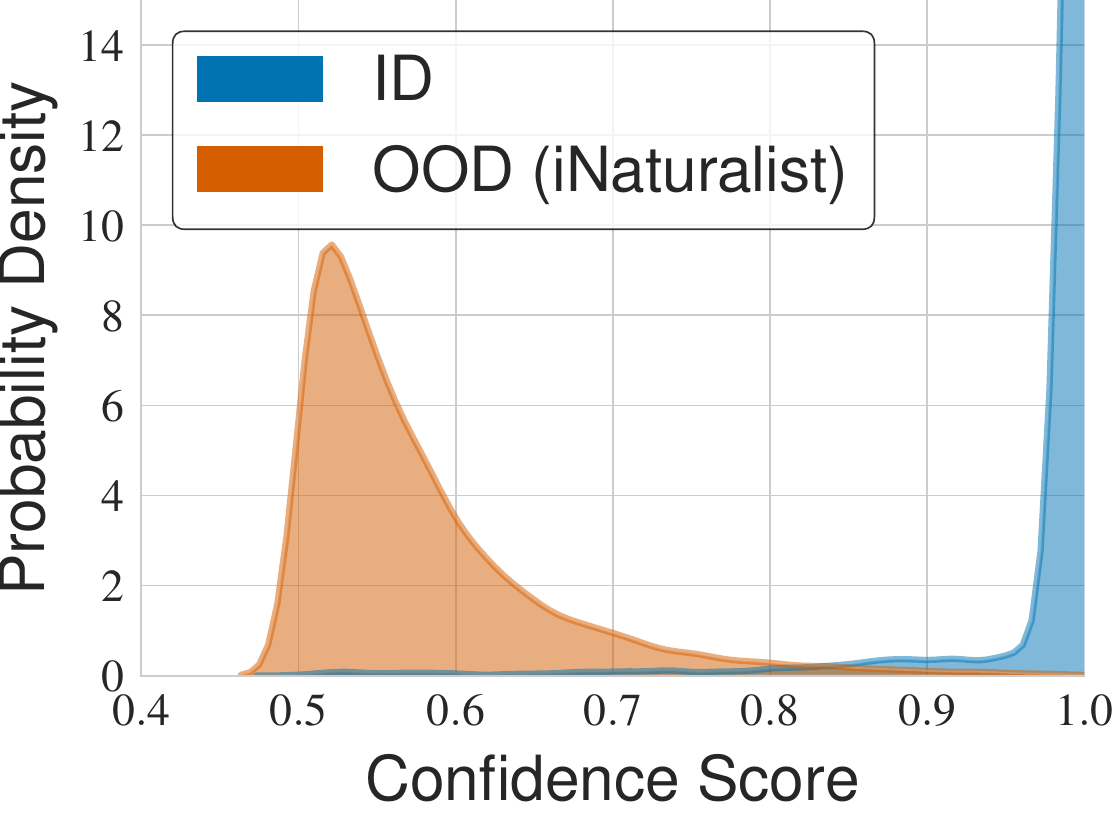} &
            \includegraphics[width=0.19\linewidth]{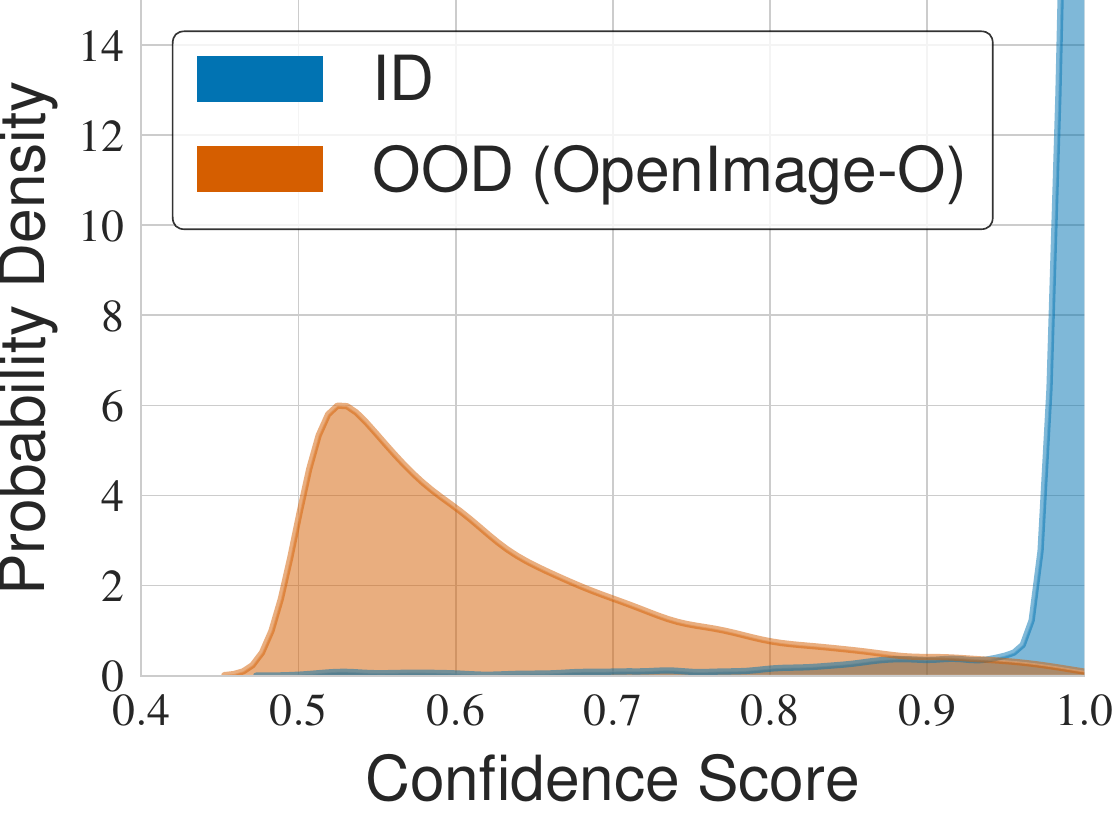} &
            \includegraphics[width=0.19\linewidth]{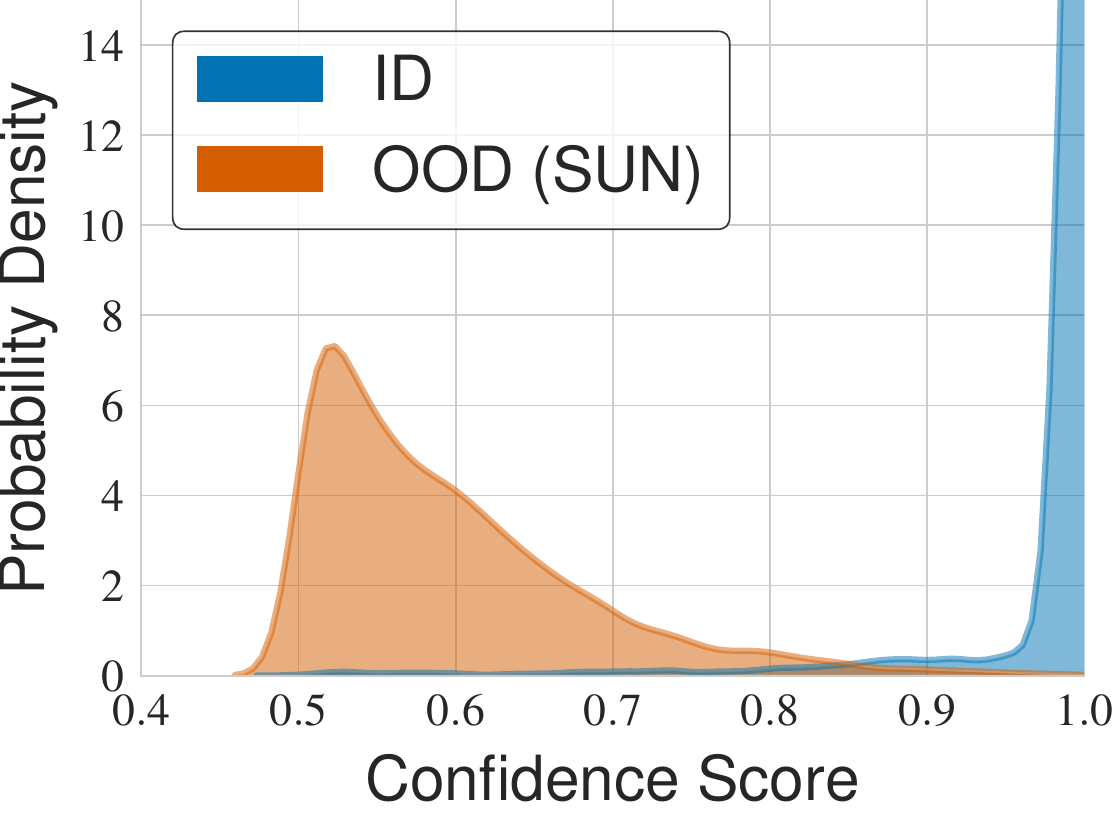} &
            \includegraphics[width=0.19\linewidth]{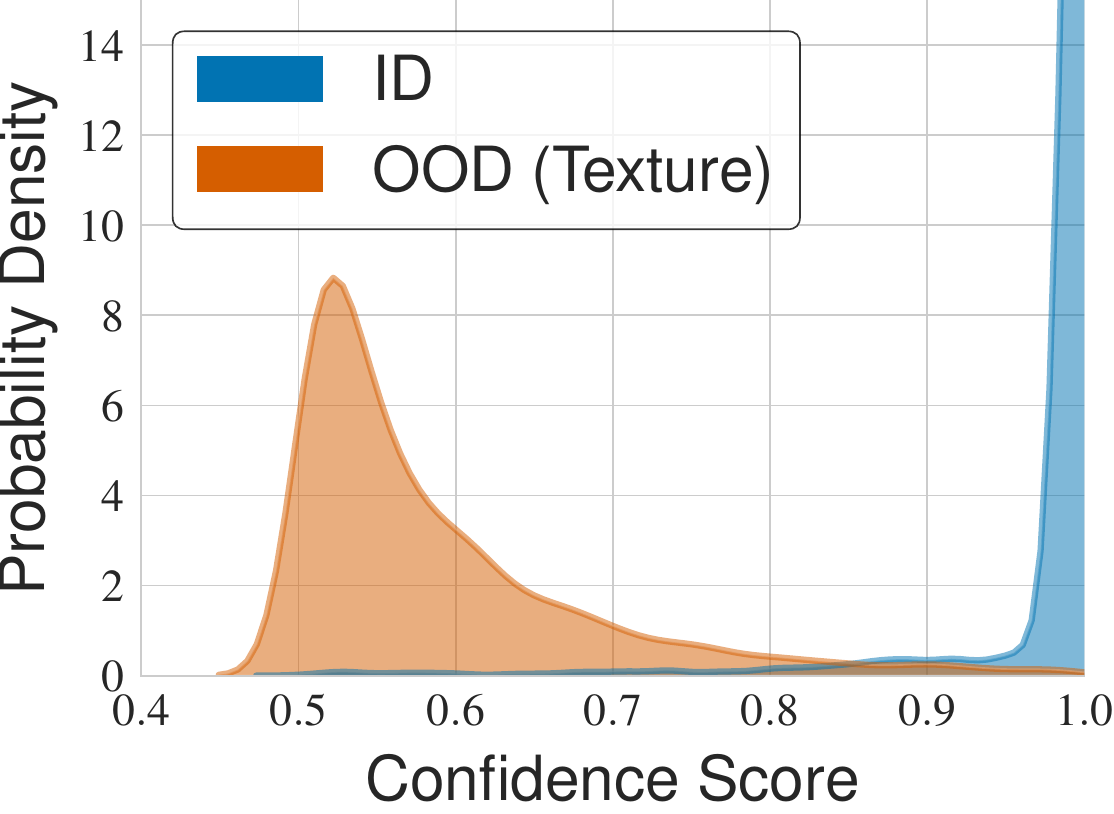}
        \end{tabular}
    }
    \caption{\textbf{Top row:} Visualization of confidence scores of cross-entropy baseline for CelebA ID and various conventional OOD datasets. \textbf{Bottom row:} Visualization of confidence scores of \ours~for CelebA ID and various conventional OOD datasets.}
\end{figure}

\subsection{CIFAR-10 datasets}
\begin{figure}[htbp]
    \centering
    \adjustbox{width=0.95\linewidth}{
        \begin{tabular}{c c c c c}
            \includegraphics[width=0.19\linewidth]{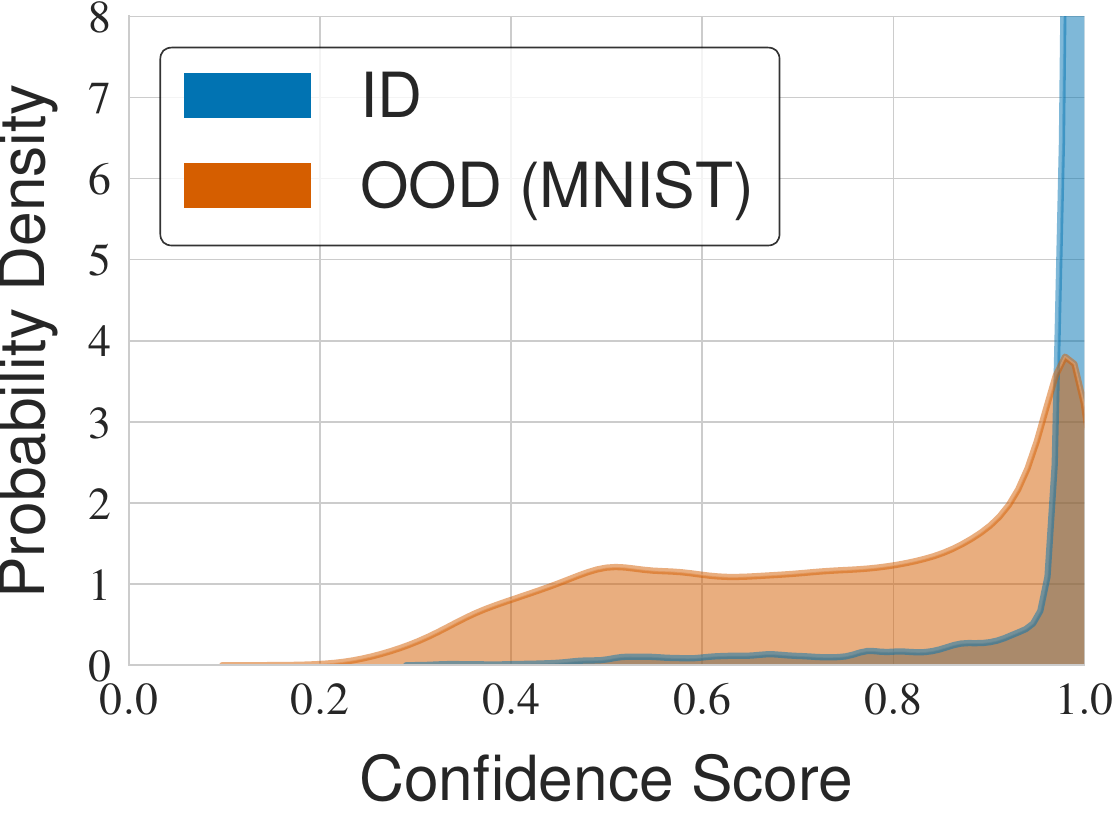} &
            \includegraphics[width=0.19\linewidth]{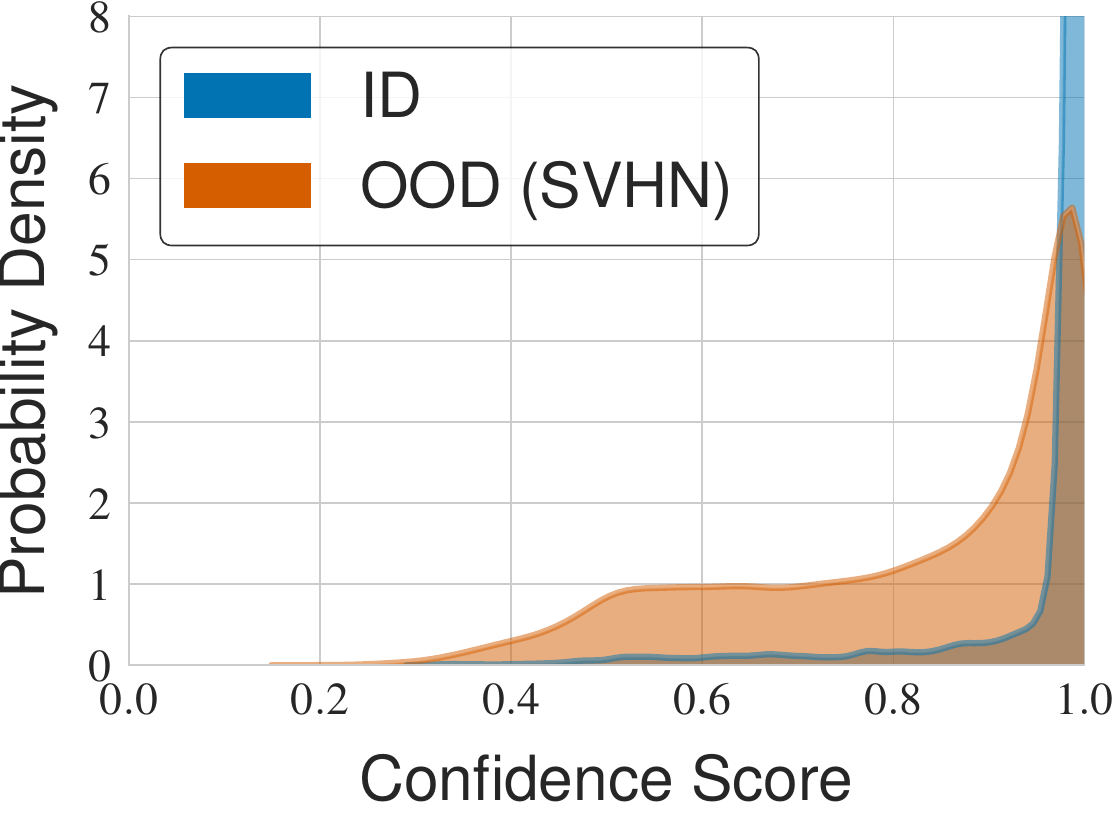} &
            \includegraphics[width=0.19\linewidth]{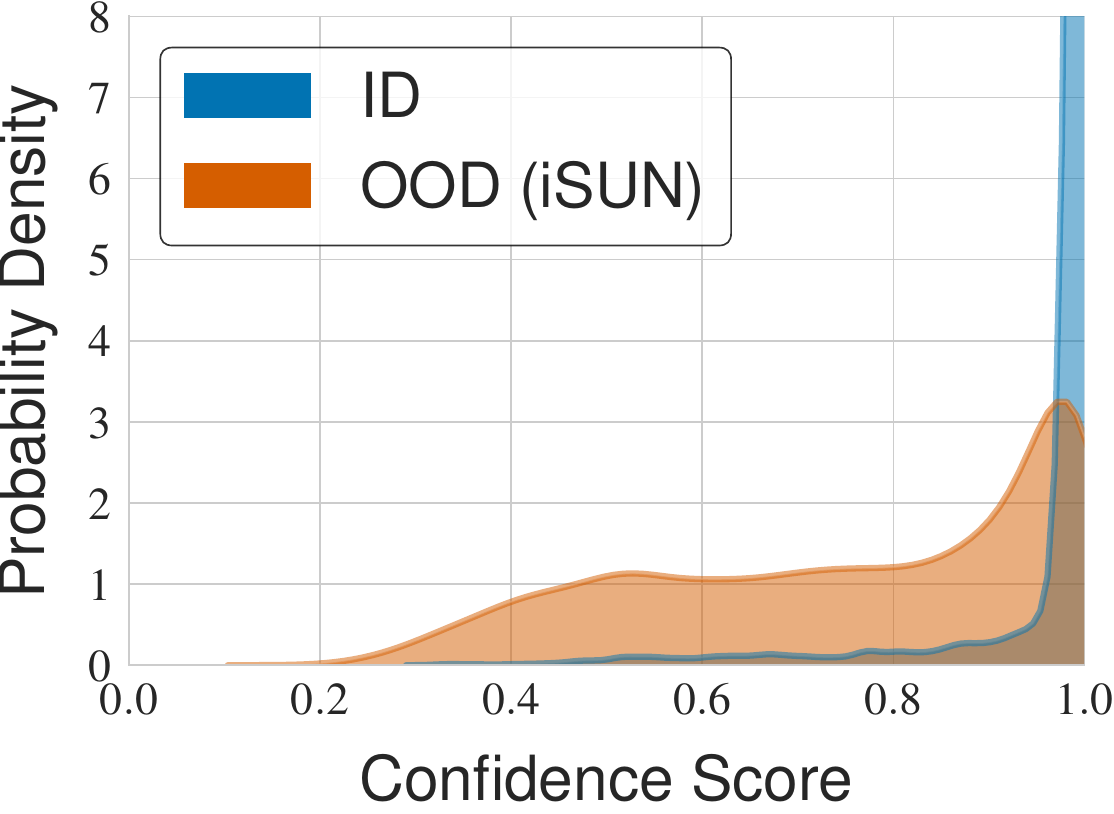} &
            \includegraphics[width=0.19\linewidth]{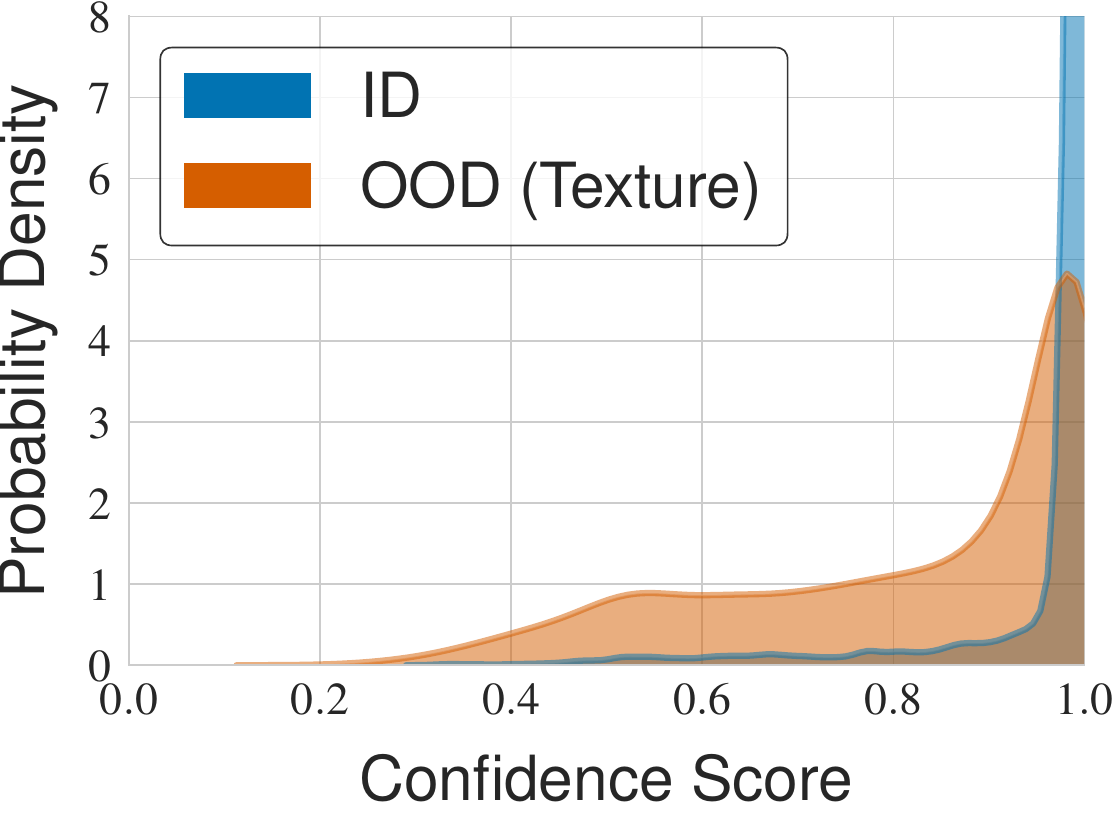} &
            \includegraphics[width=0.19\linewidth]{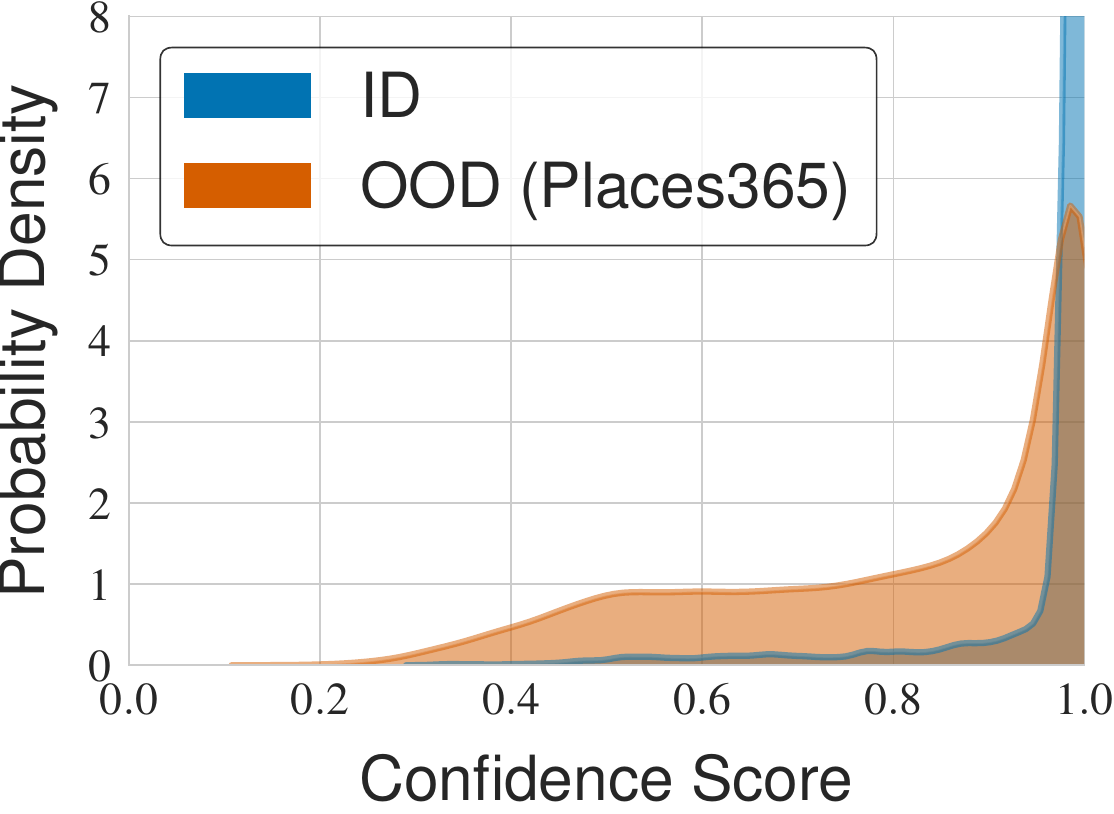} \\

            \includegraphics[width=0.19\linewidth]{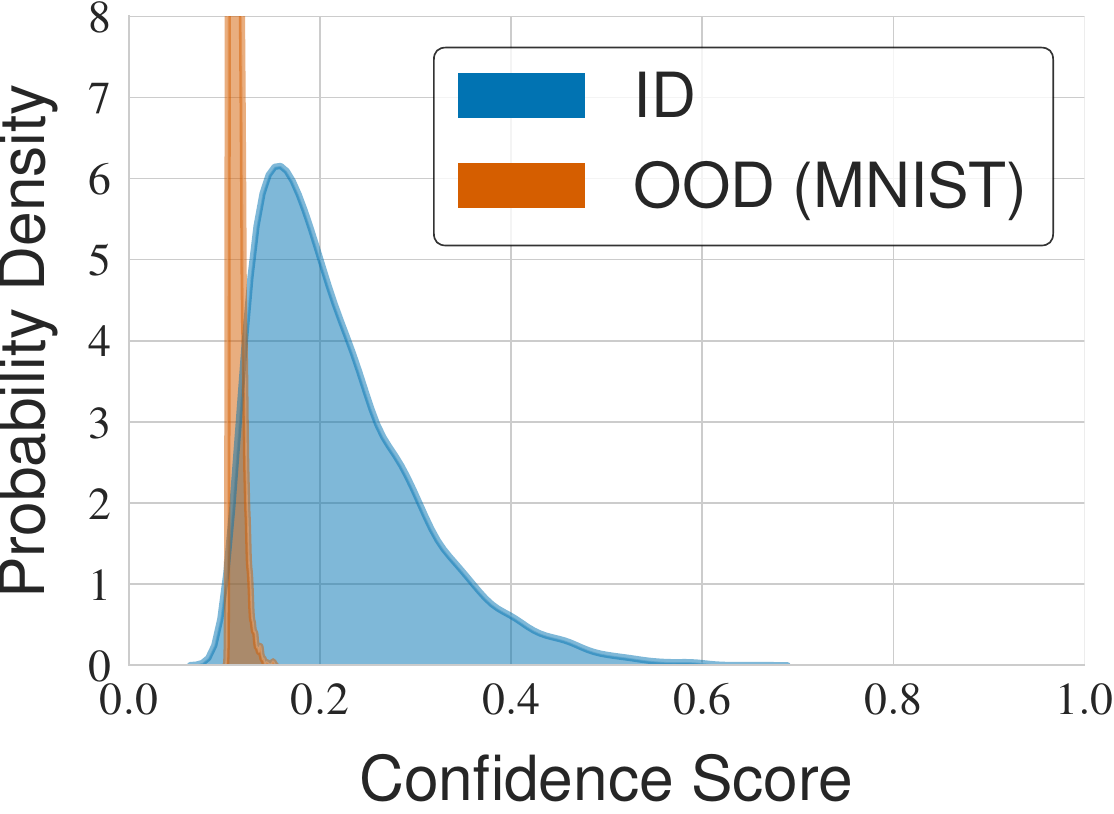} &
            \includegraphics[width=0.19\linewidth]{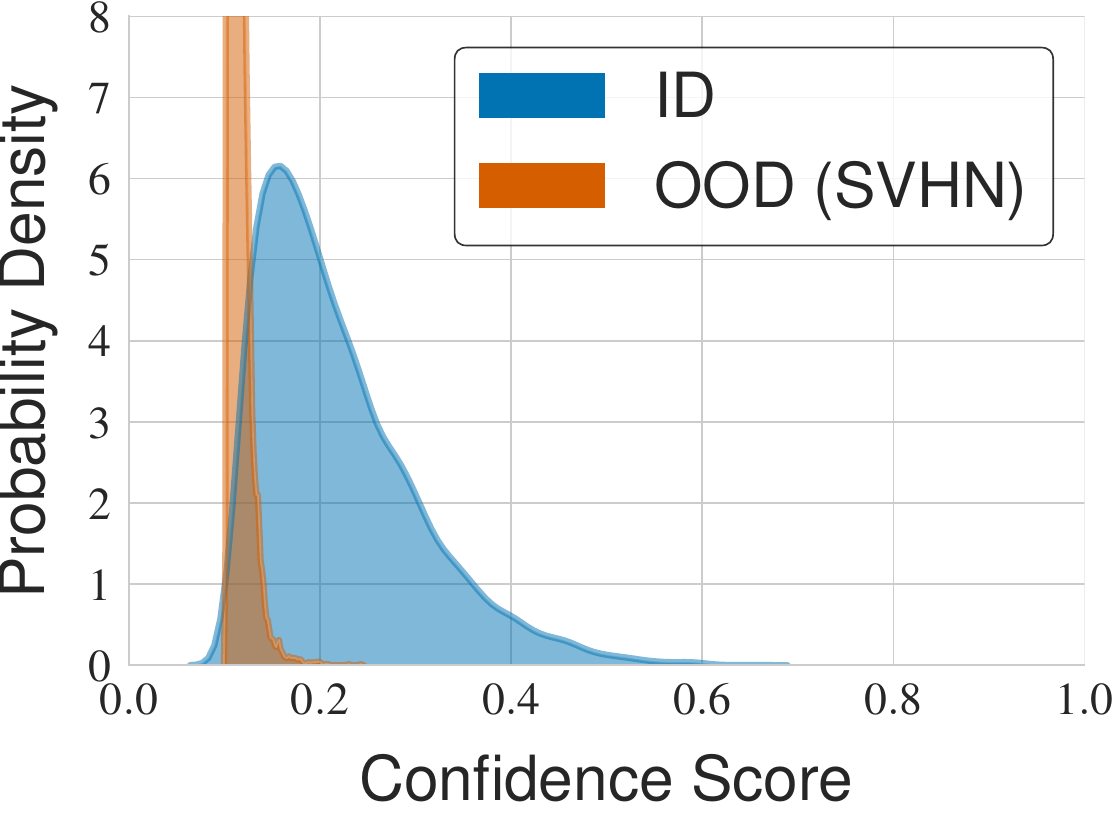} &
            \includegraphics[width=0.19\linewidth]{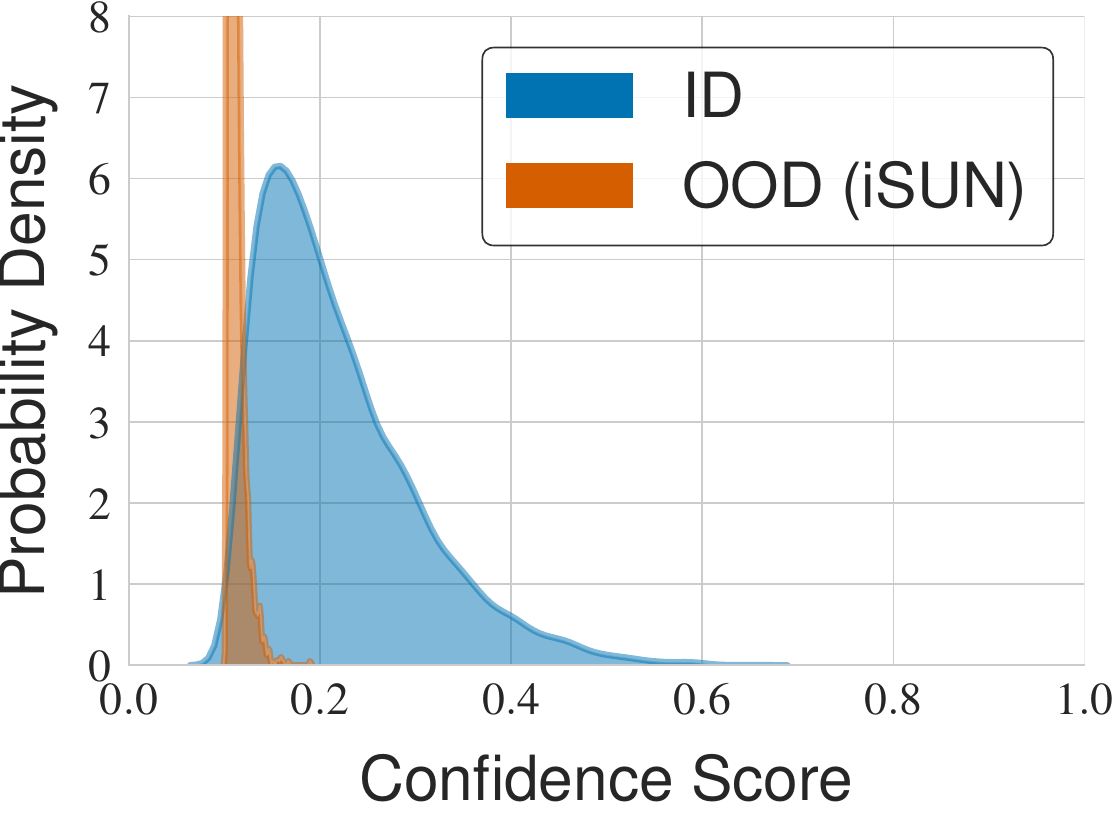} &
            \includegraphics[width=0.19\linewidth]{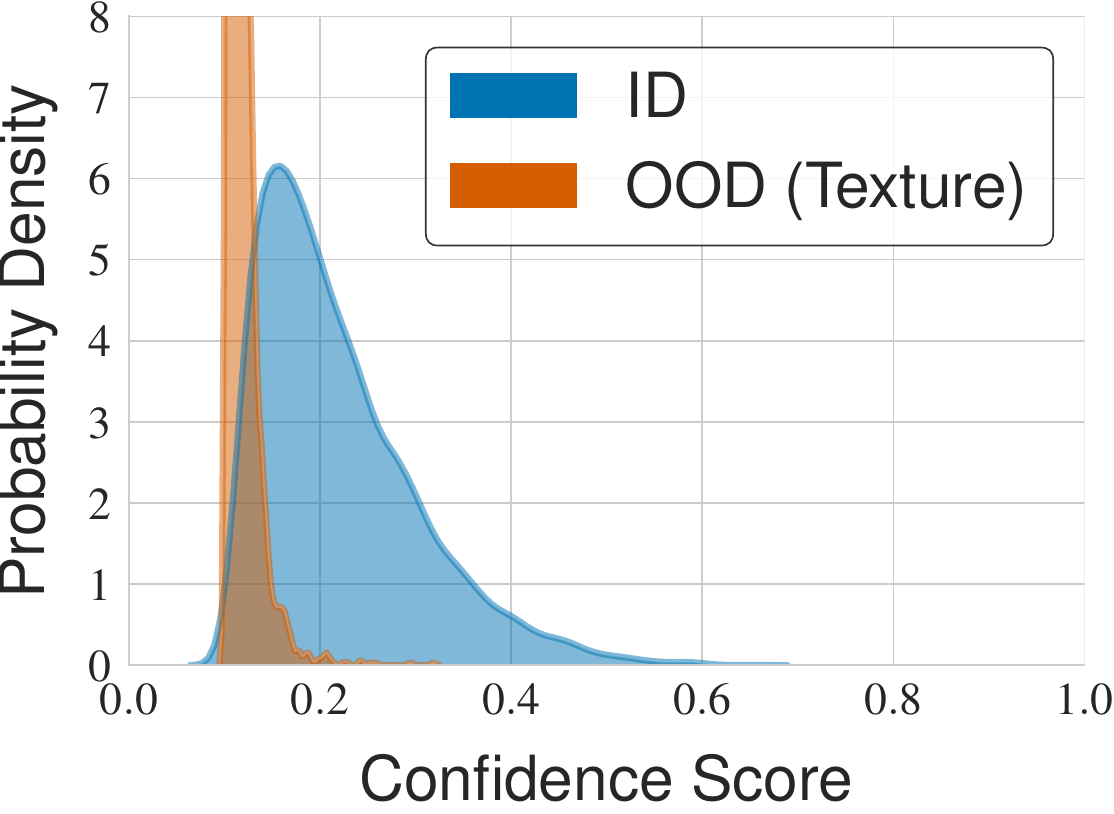} &
            \includegraphics[width=0.19\linewidth]{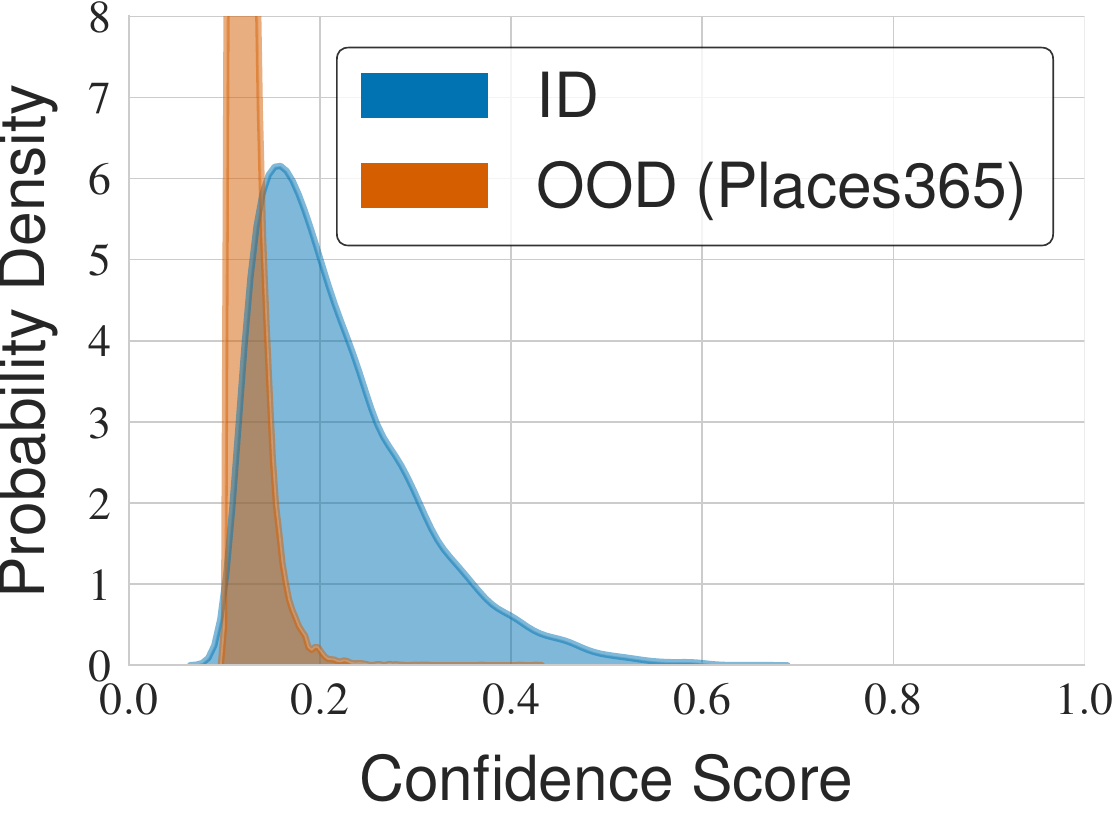}
        \end{tabular}
    }
    \caption{\textbf{Top row:} Visualization of confidence scores of cross-entropy baseline for CIFAR-10 ID and various conventional OOD datasets. \textbf{Bottom row:} Visualization of confidence scores of \ours~for CIFAR-10 ID and various conventional OOD datasets.}
\end{figure}

\newpage
\subsection{CIFAR-100 datasets}
\begin{figure}[h!]
    \centering
    \adjustbox{width=0.95\linewidth}{
        \begin{tabular}{c c c c c}
            \includegraphics[width=0.19\linewidth]{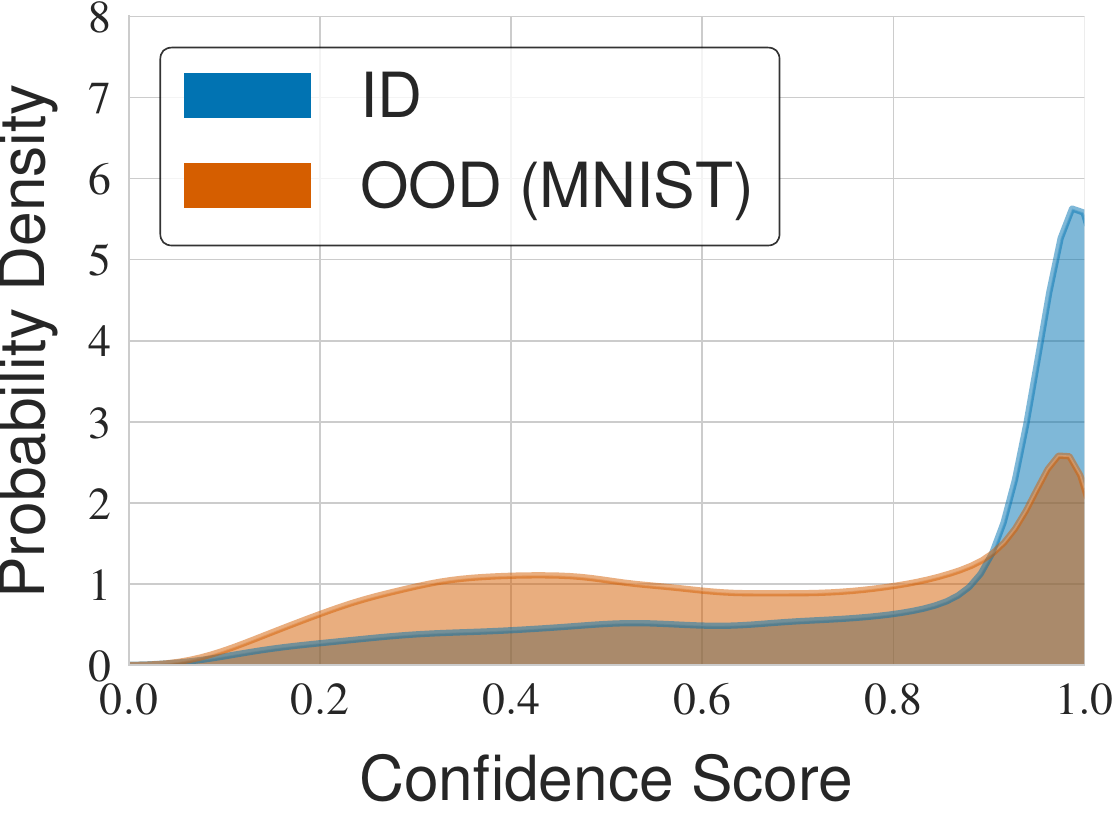} &
            \includegraphics[width=0.19\linewidth]{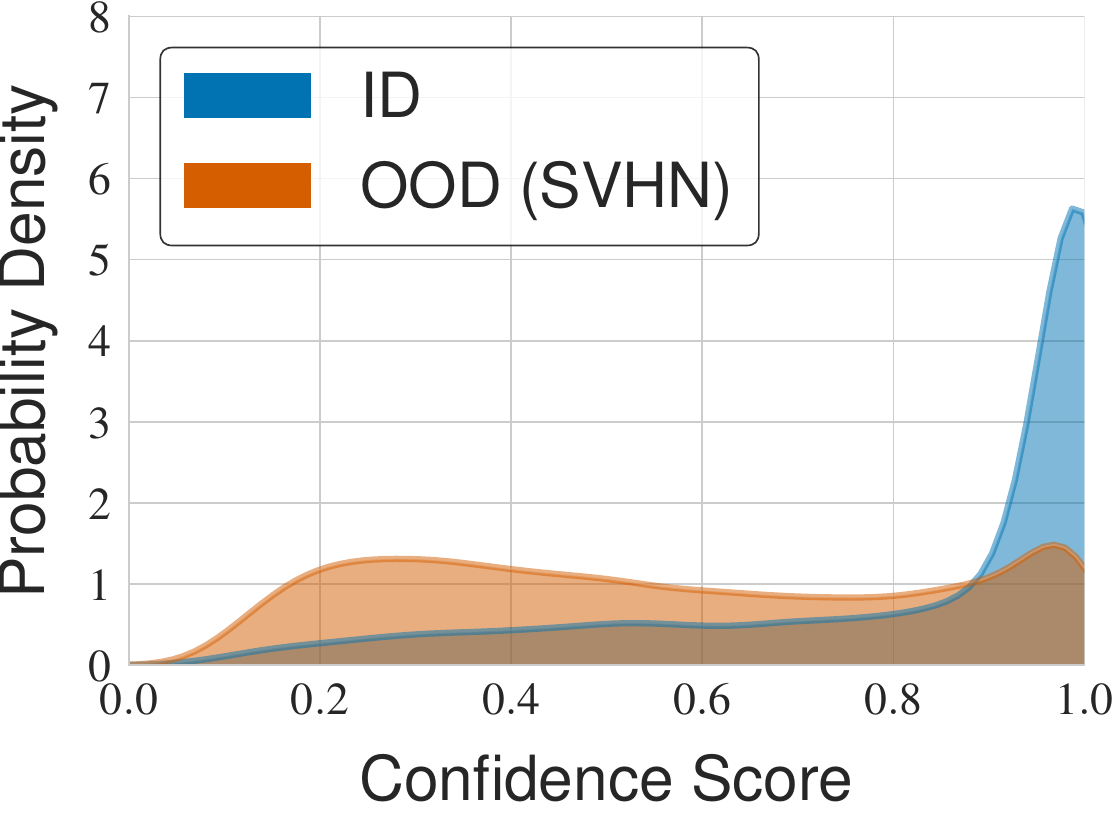} &
            \includegraphics[width=0.19\linewidth]{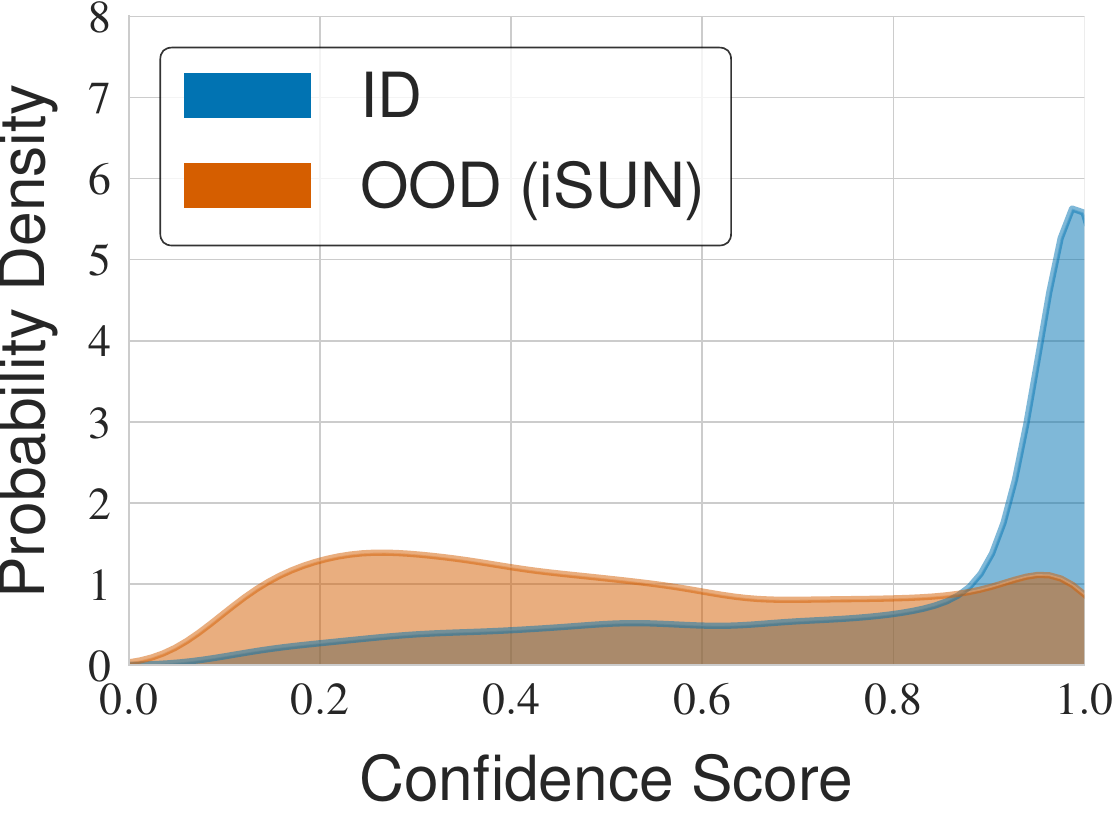} &
            \includegraphics[width=0.19\linewidth]{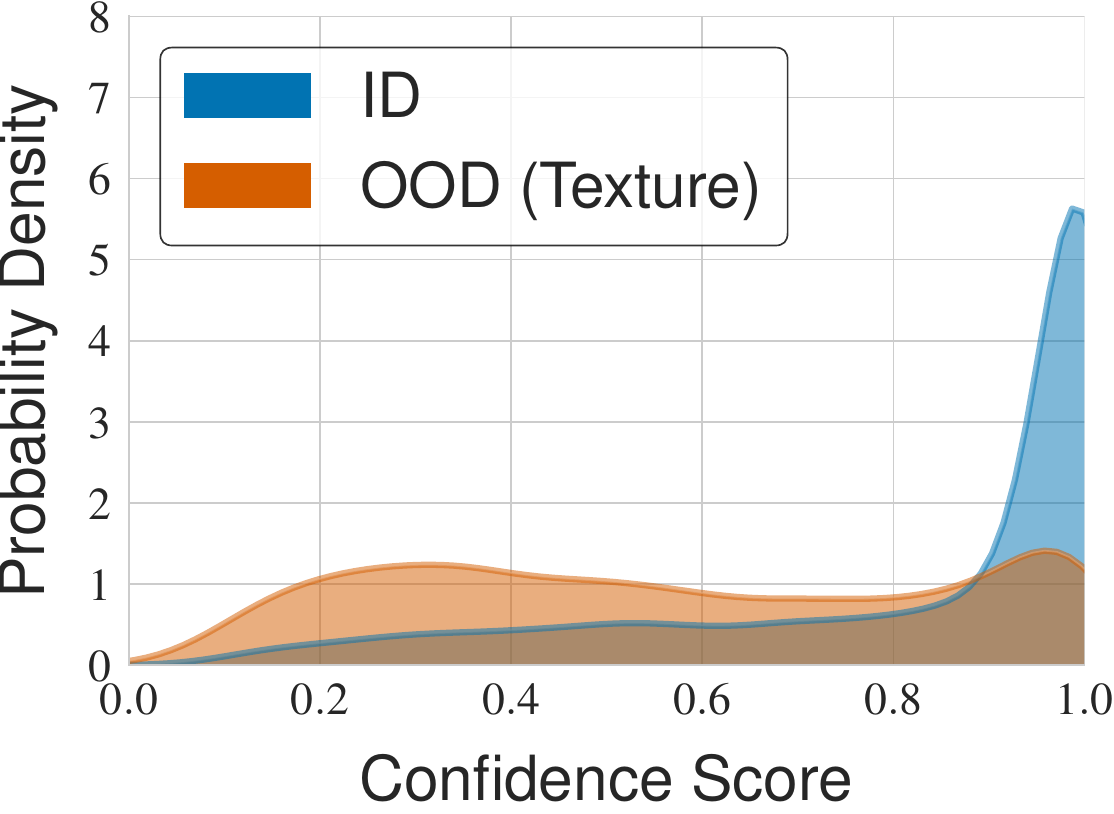} &
            \includegraphics[width=0.19\linewidth]{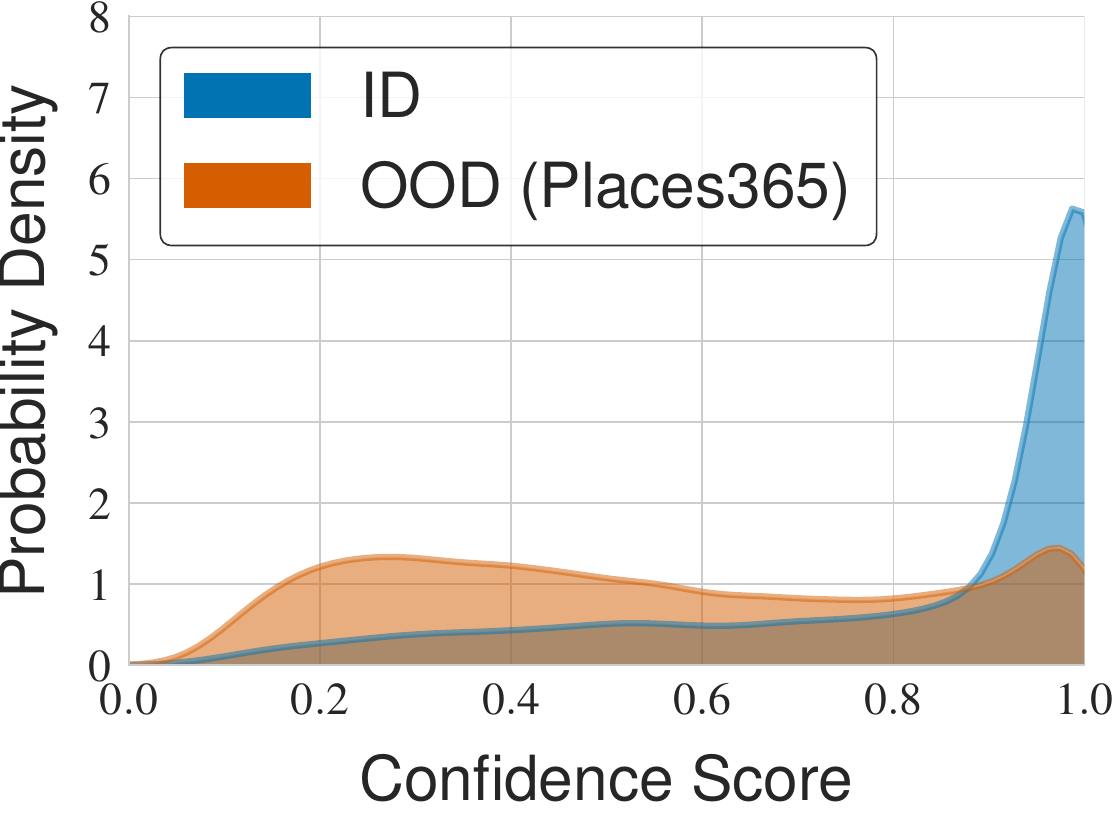} \\

            \includegraphics[width=0.19\linewidth]{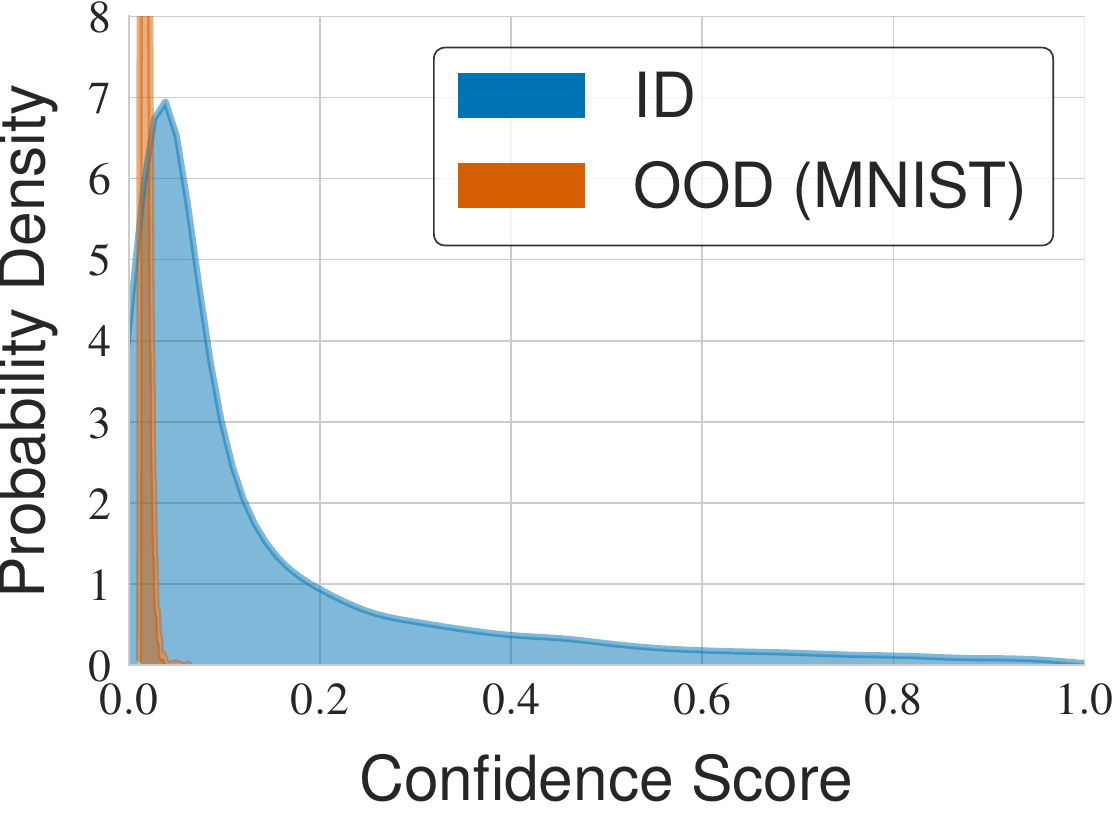} &
            \includegraphics[width=0.19\linewidth]{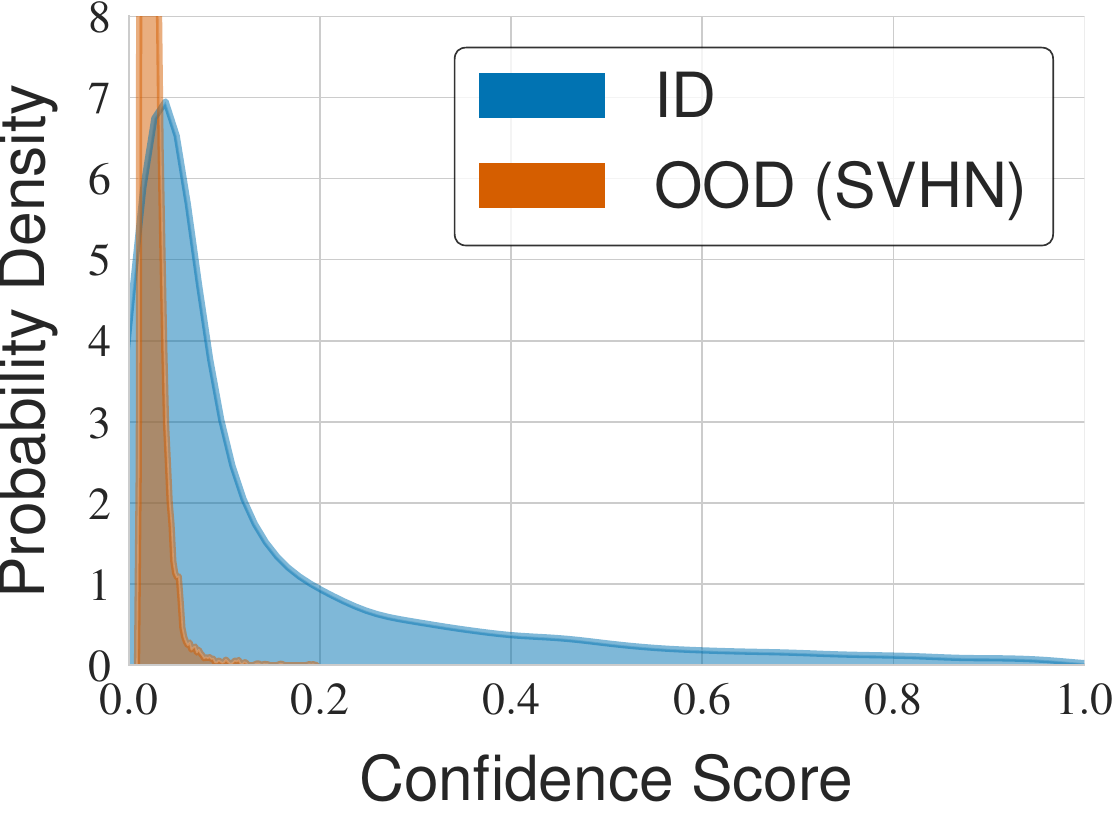} &
            \includegraphics[width=0.19\linewidth]{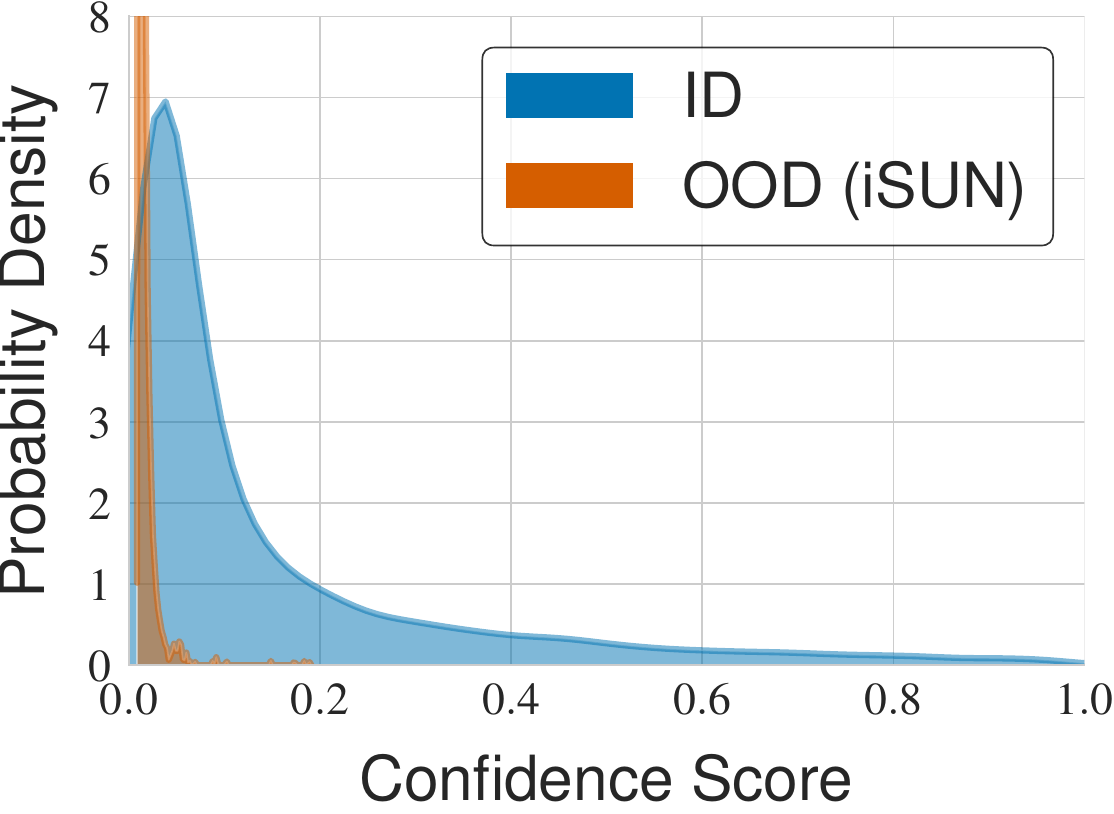} &
            \includegraphics[width=0.19\linewidth]{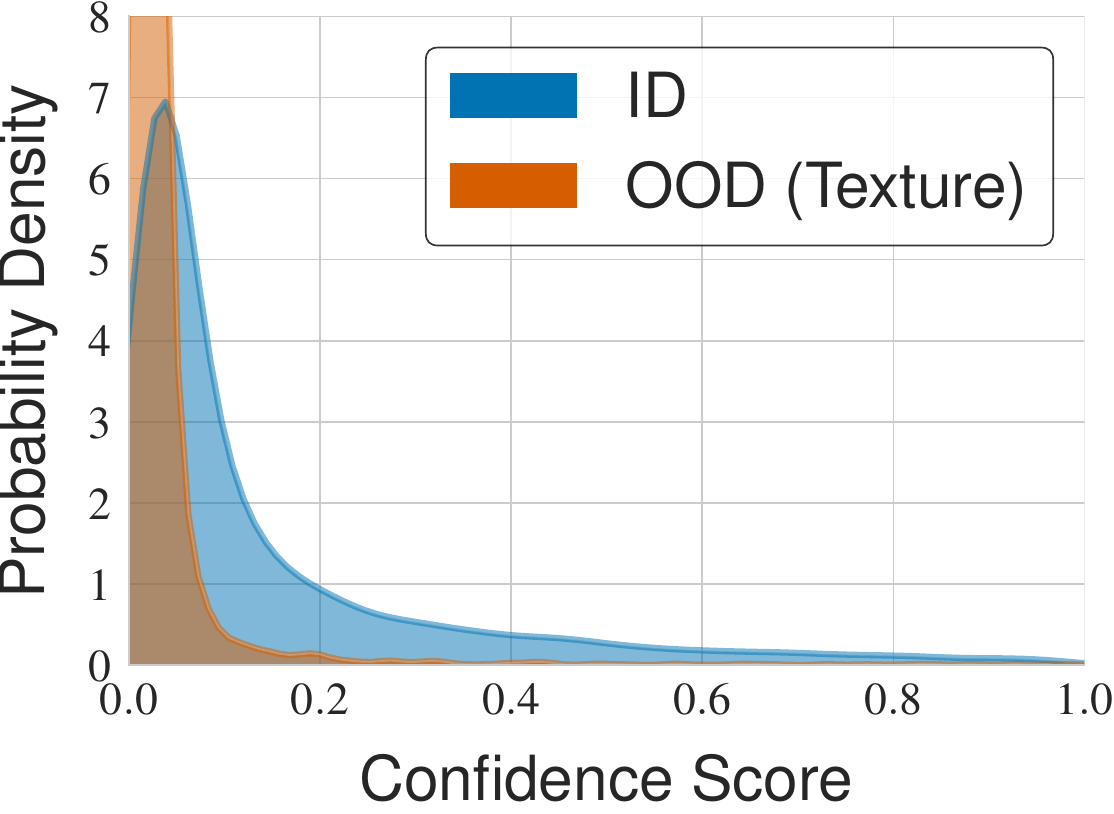} &
            \includegraphics[width=0.19\linewidth]{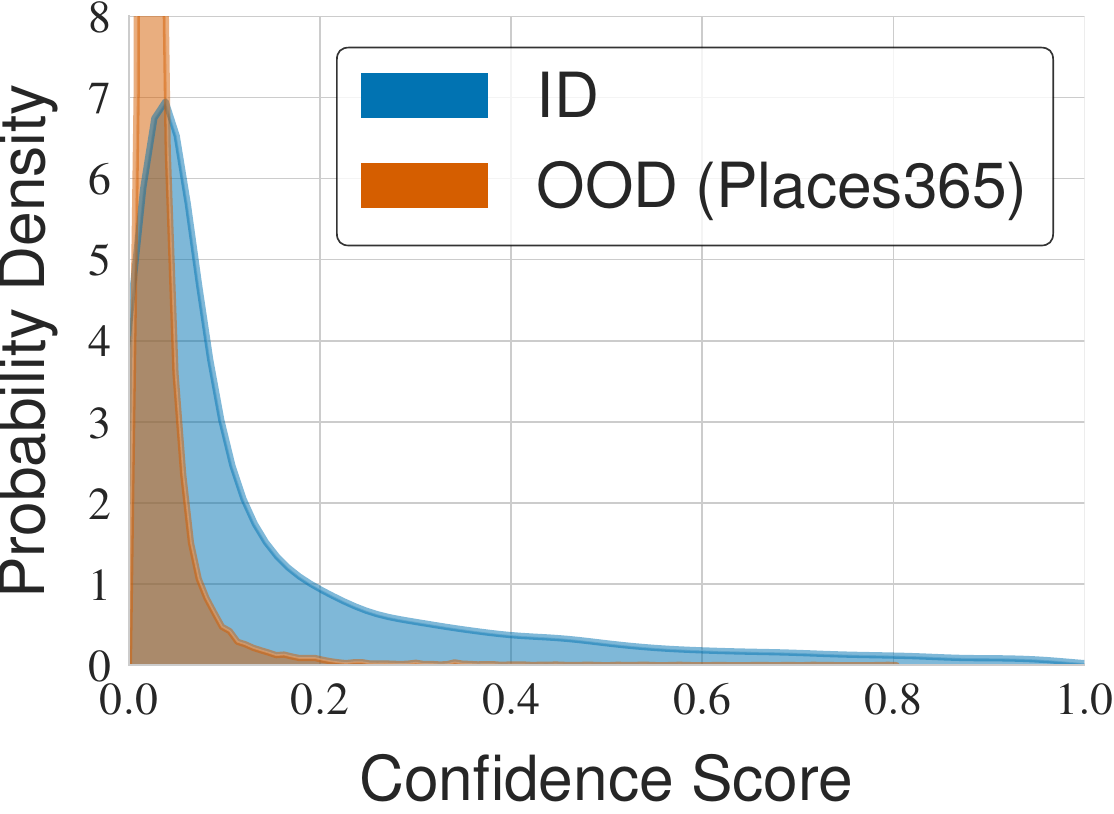}
        \end{tabular}
    }
    \caption{\textbf{Top row:} Visualization of confidence scores of cross-entropy baseline for CIFAR-100 ID and various conventional OOD datasets. \textbf{Bottom row:} Visualization of confidence scores of \ours~for CIFAR-100 ID and various conventional OOD datasets.}
\end{figure}

\subsection{Aircraft datasets}
\begin{figure}[htbp]
    \centering
    \adjustbox{width=0.95\linewidth}{
        \begin{tabular}{c c c c c}
            \includegraphics[width=0.19\linewidth]{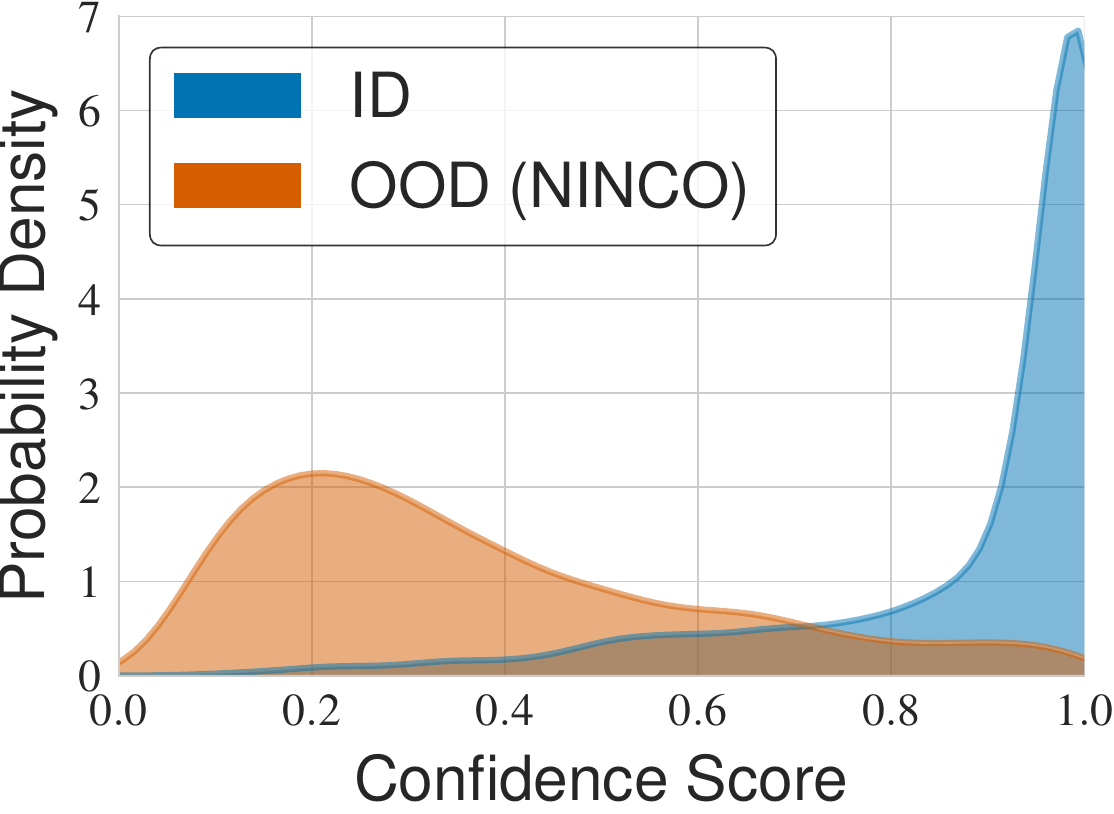} &
            \includegraphics[width=0.19\linewidth]{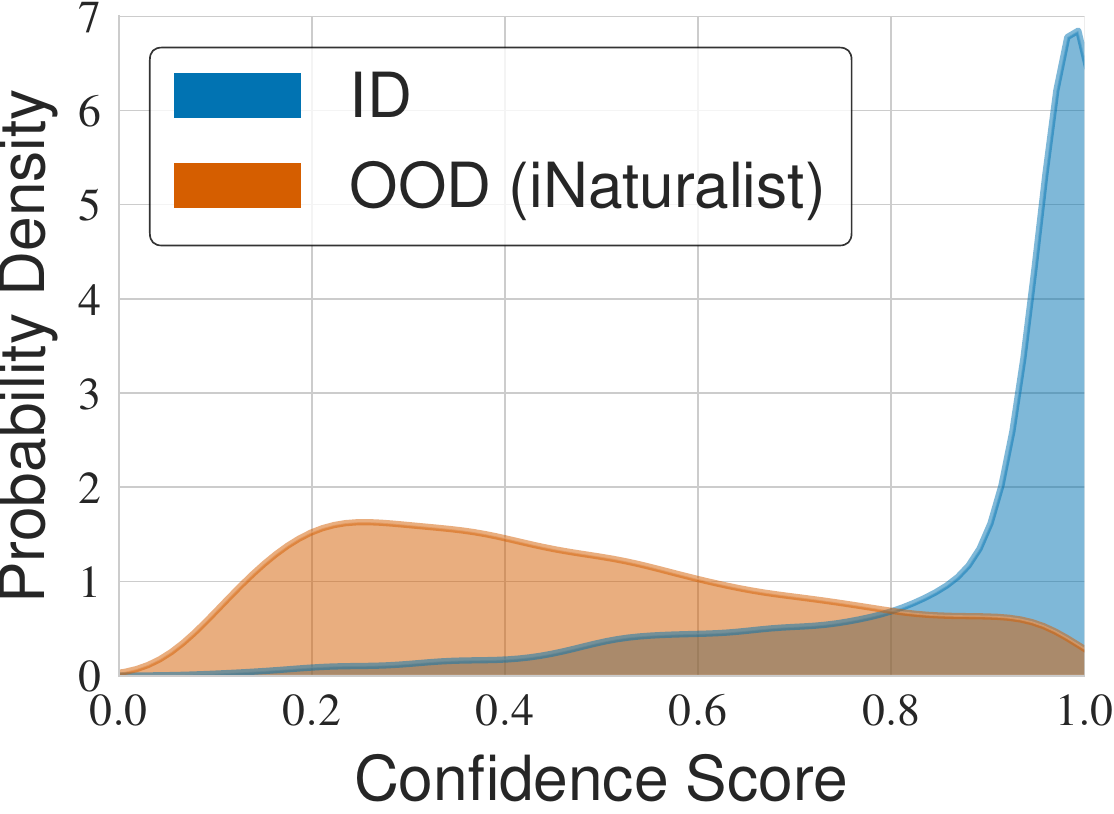} &
            \includegraphics[width=0.19\linewidth]{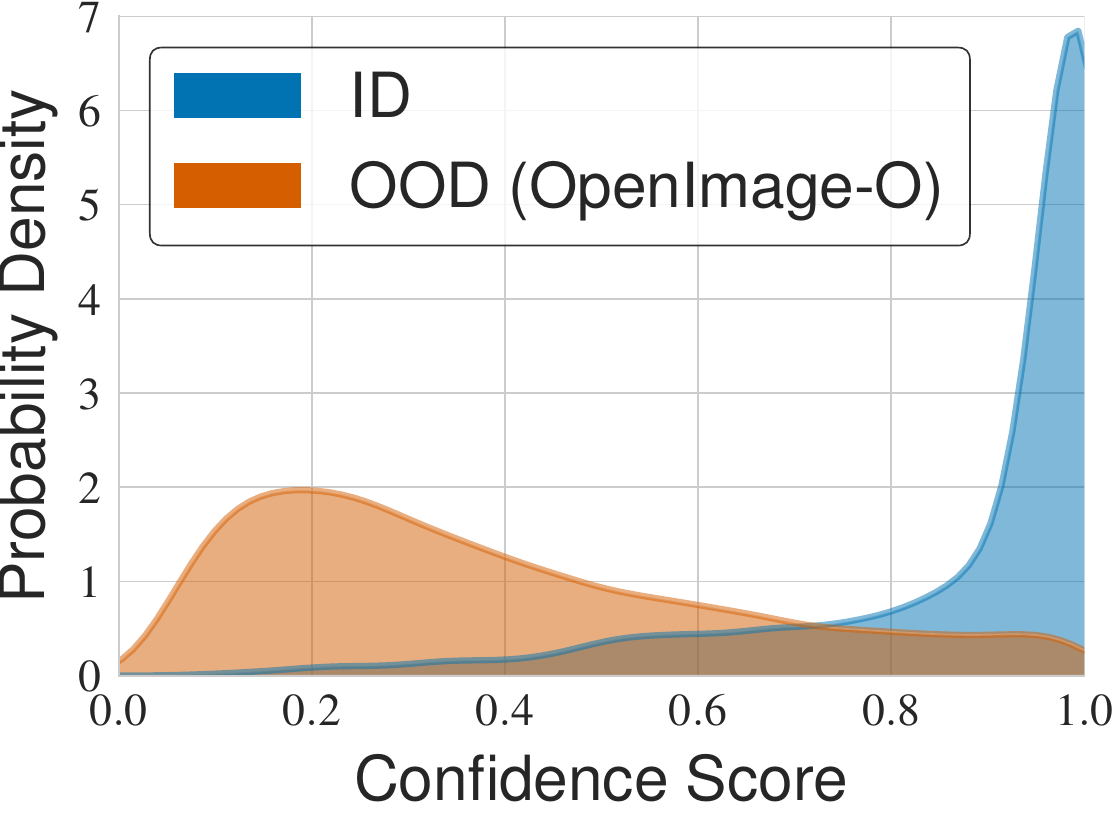} &
            \includegraphics[width=0.19\linewidth]{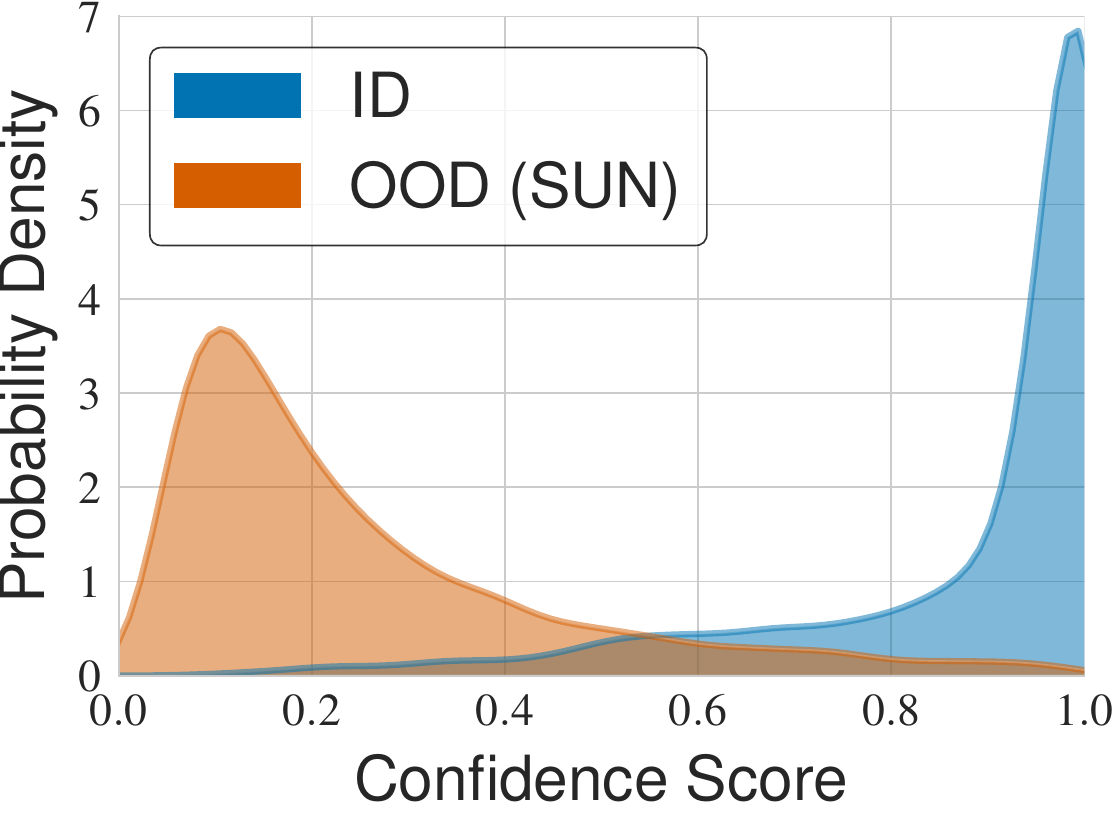} &
            \includegraphics[width=0.19\linewidth]{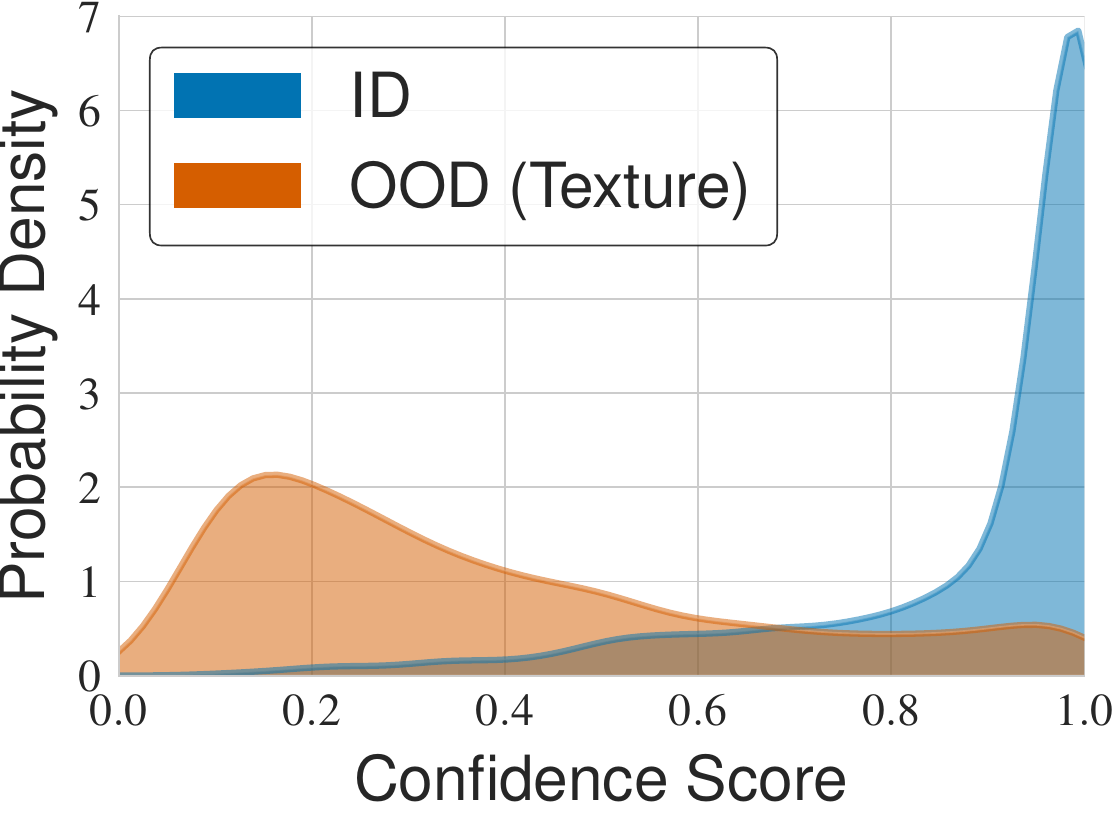} \\

            \includegraphics[width=0.19\linewidth]{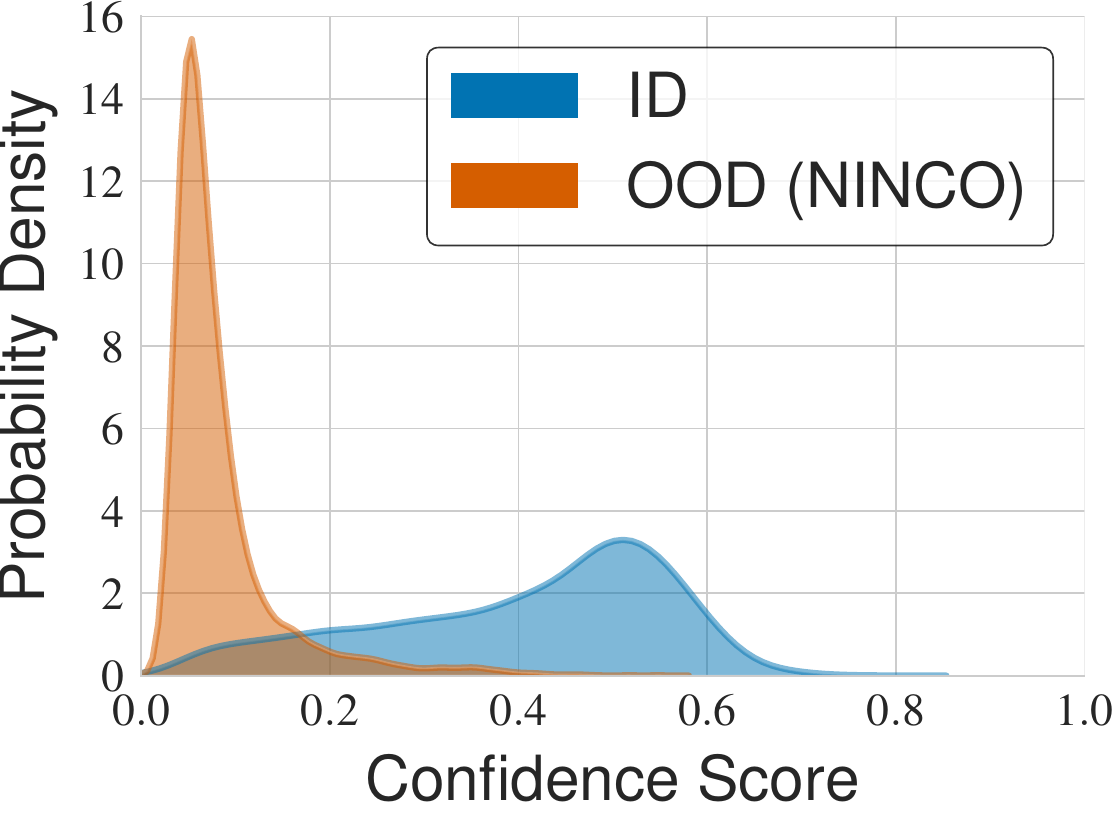} &
            \includegraphics[width=0.19\linewidth]{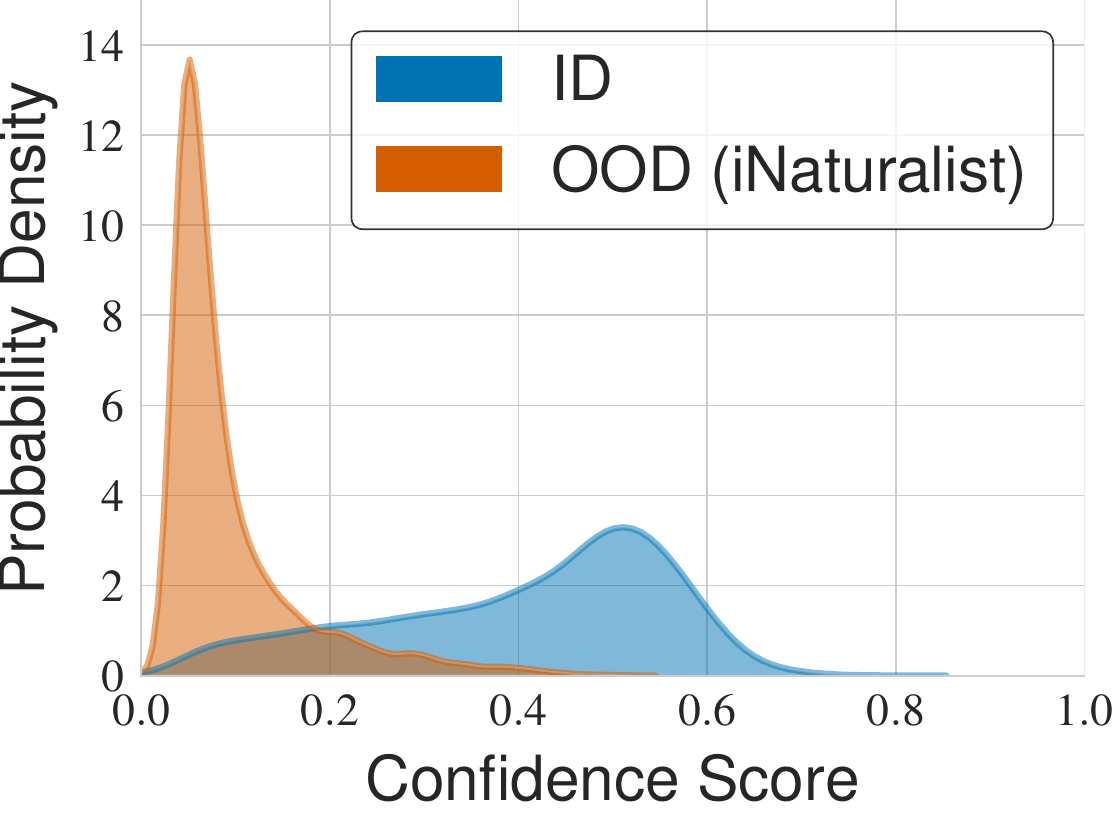} &
            \includegraphics[width=0.19\linewidth]{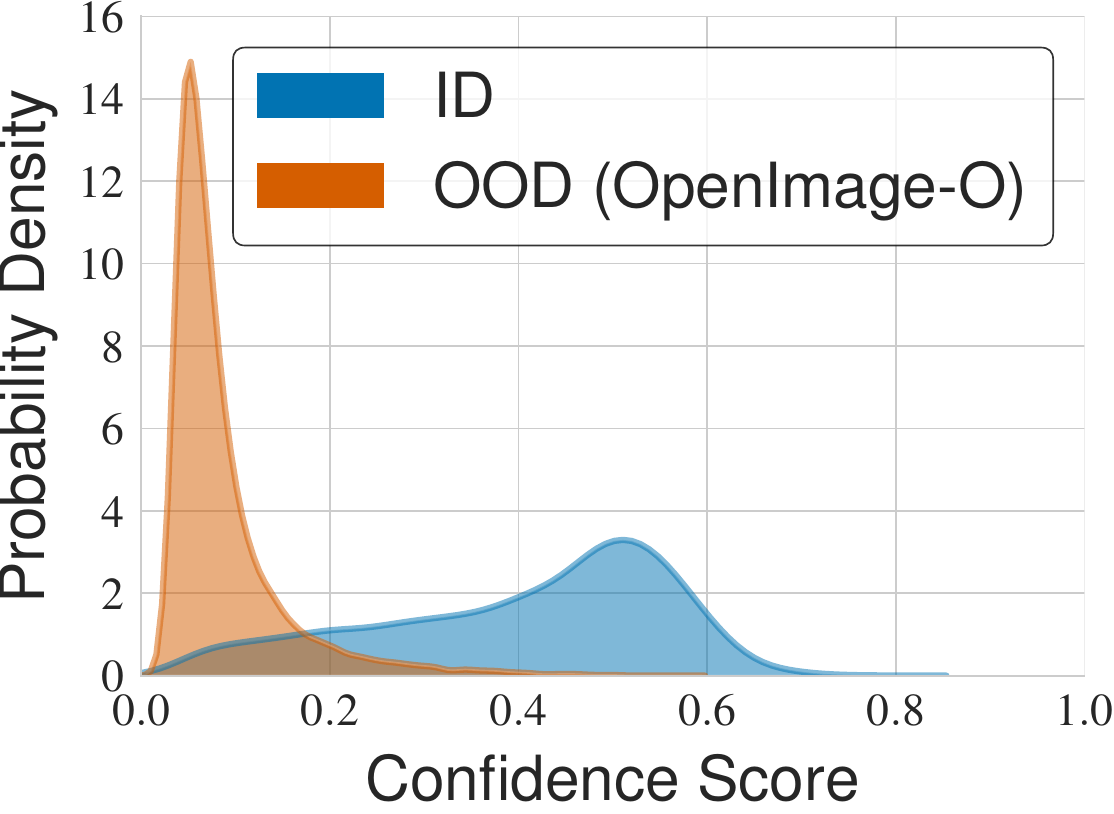} &
            \includegraphics[width=0.19\linewidth]{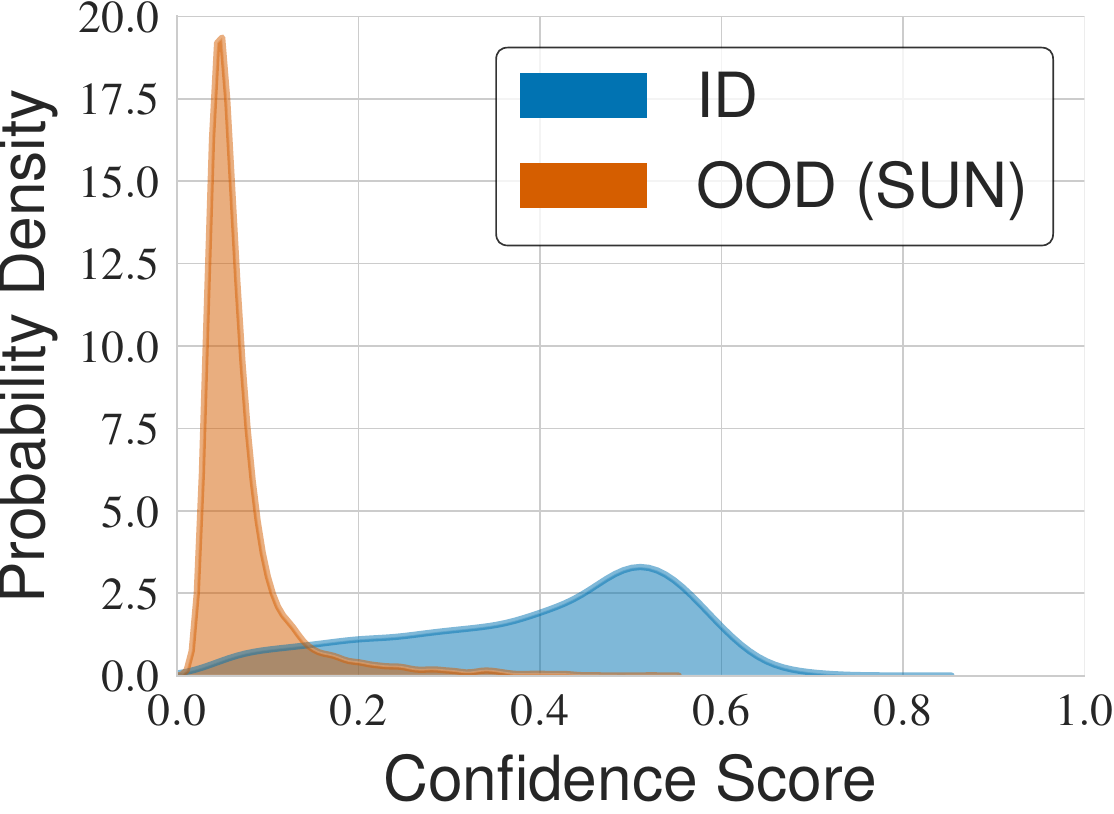} &
            \includegraphics[width=0.19\linewidth]{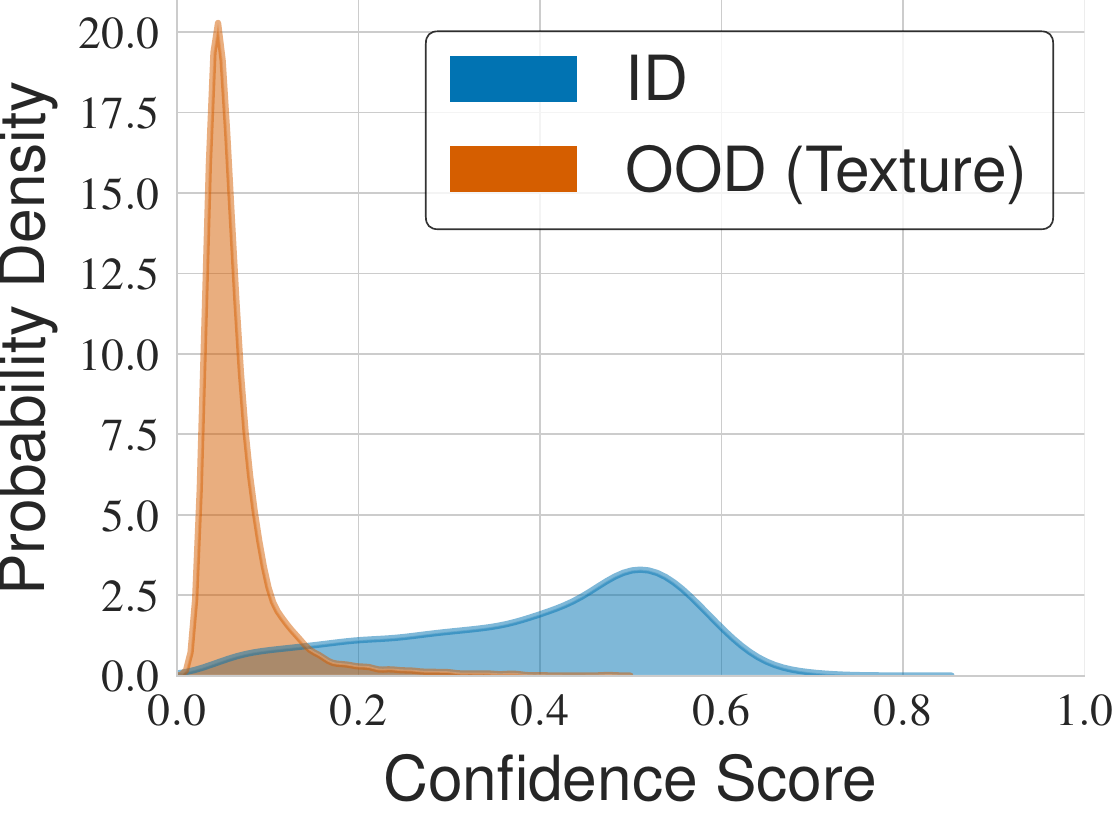}
        \end{tabular}
    }
    \caption{\textbf{Top row:} Visualization of confidence scores of cross-entropy baseline for Aircraft ID and various conventional OOD datasets. \textbf{Bottom row:} Visualization of confidence scores of \ours~for Aircraft ID and various conventional OOD datasets.}
\end{figure}

\subsection{Car datasets}
\begin{figure}[htbp]
    \centering
    \adjustbox{width=0.95\linewidth}{
        \begin{tabular}{c c c c c}
            \includegraphics[width=0.19\linewidth]{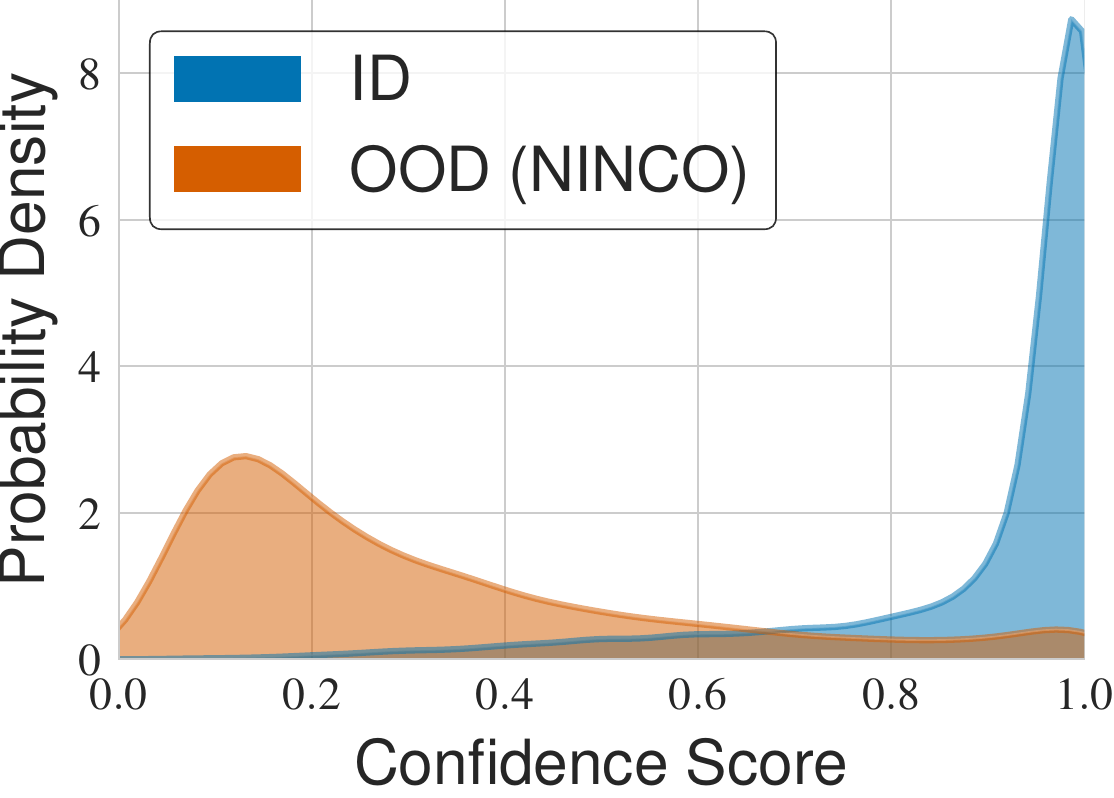} &
            \includegraphics[width=0.19\linewidth]{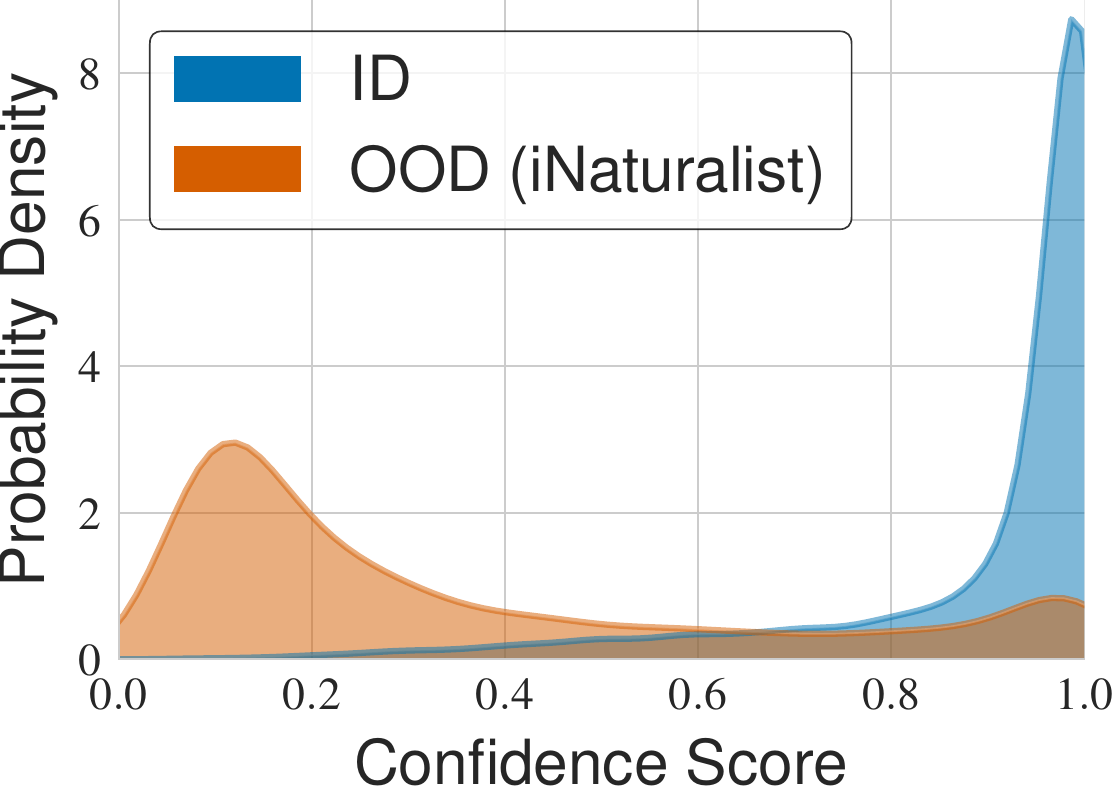} &
            \includegraphics[width=0.19\linewidth]{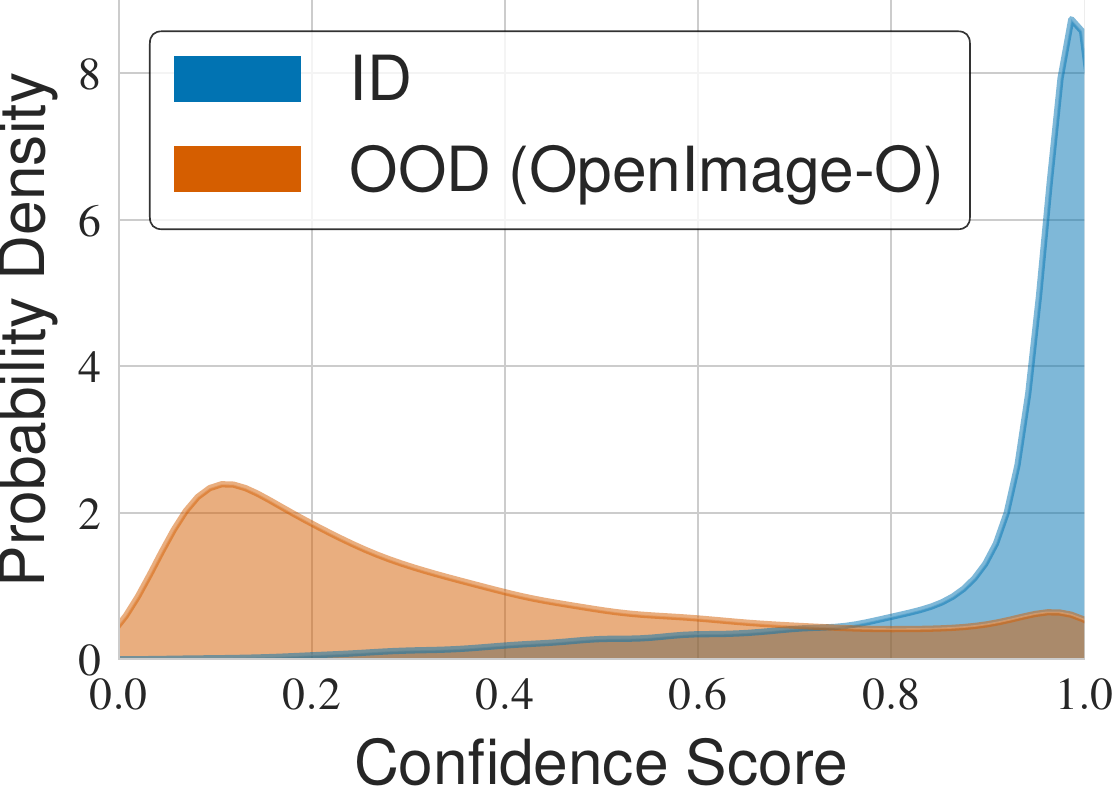} &
            \includegraphics[width=0.19\linewidth]{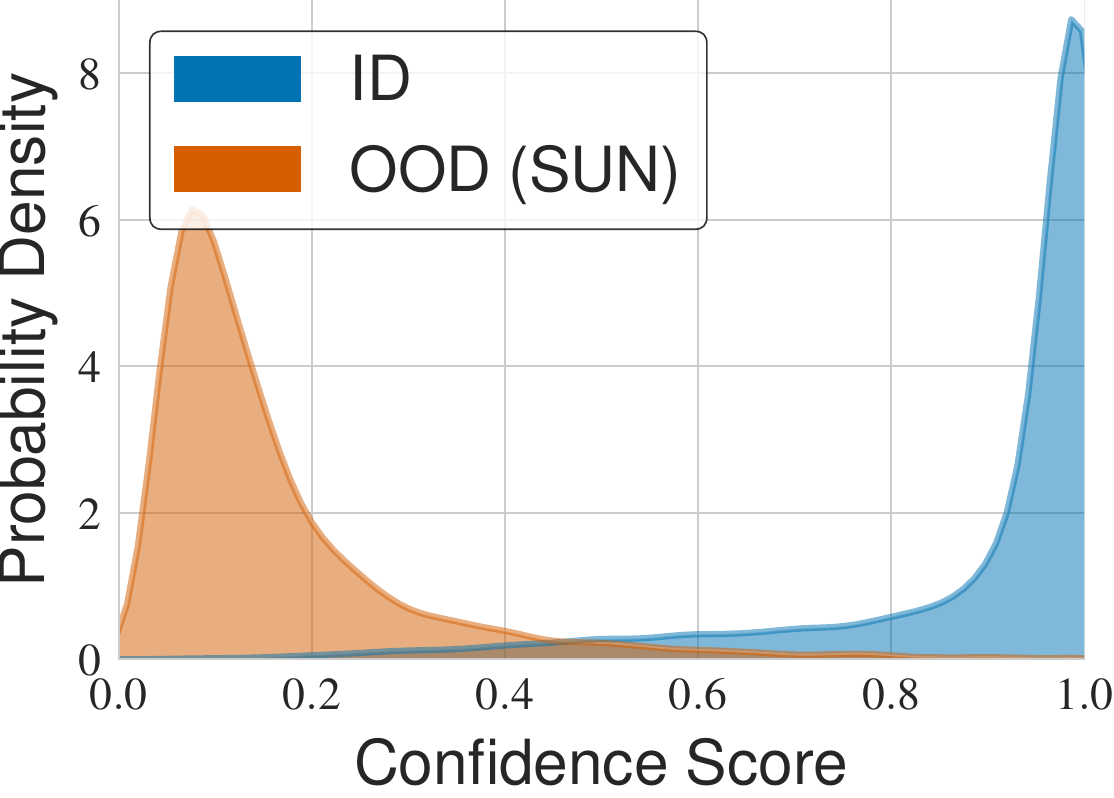} &
            \includegraphics[width=0.19\linewidth]{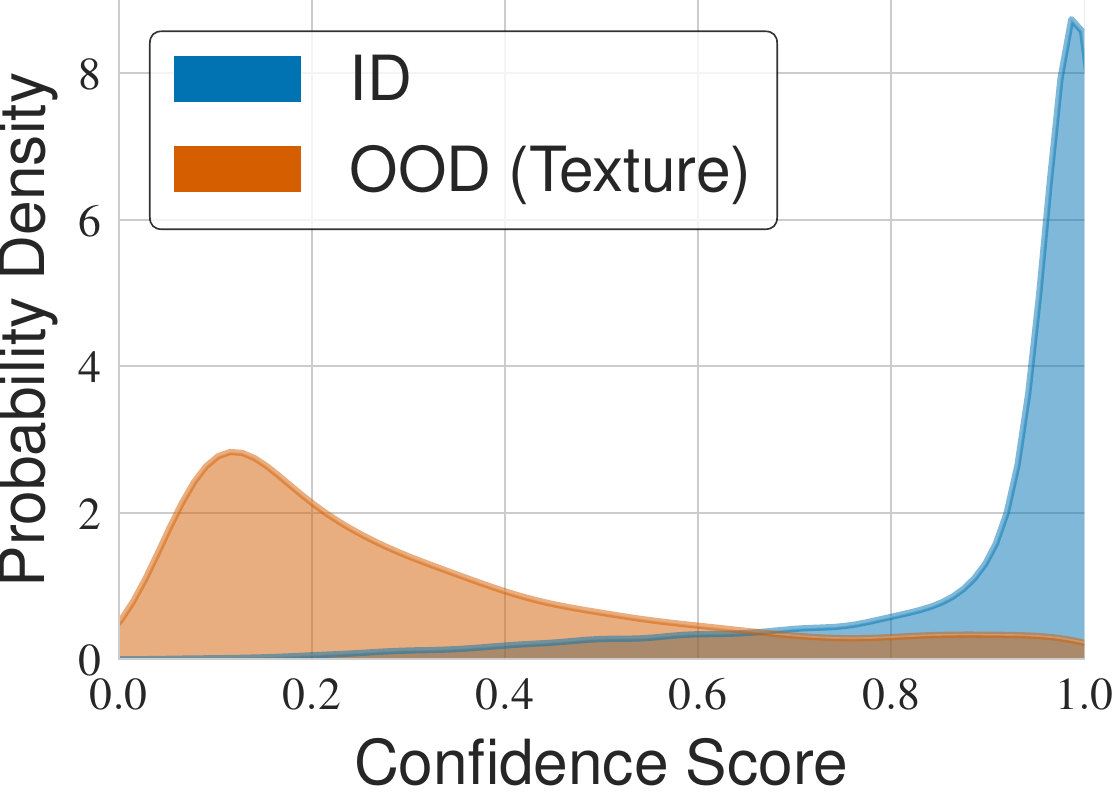} \\

            \includegraphics[width=0.19\linewidth]{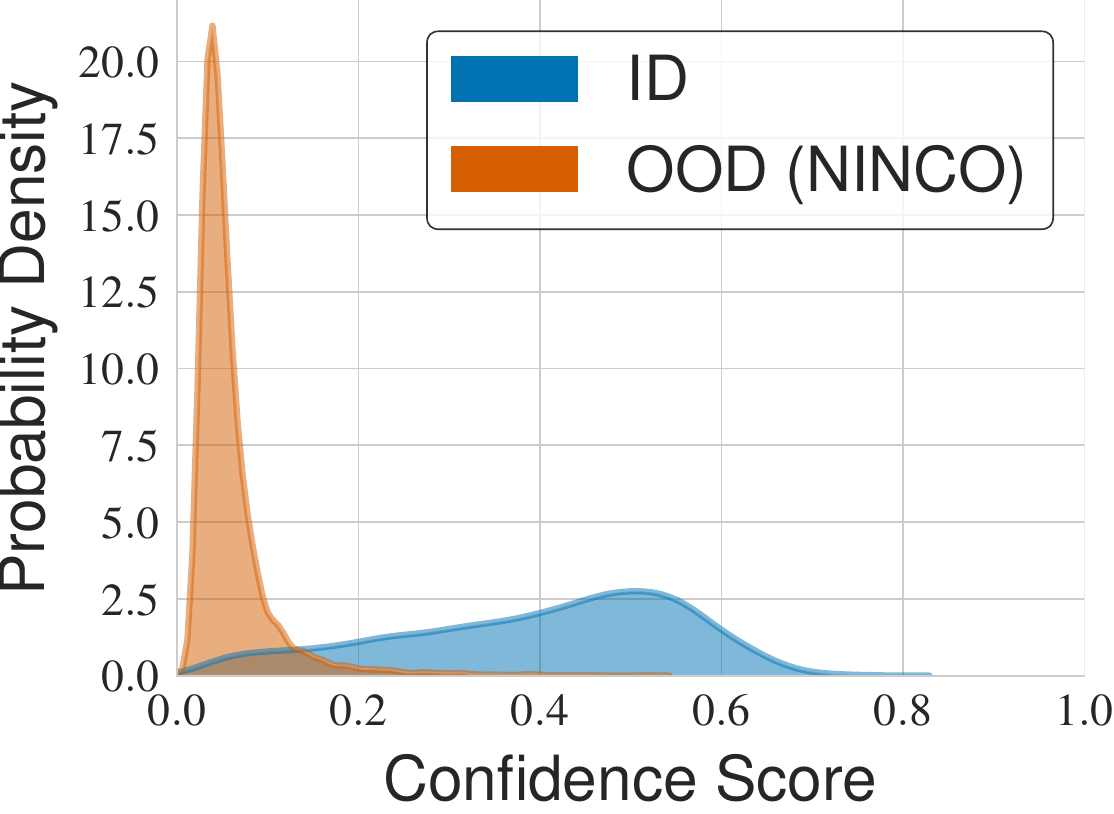} &
            \includegraphics[width=0.19\linewidth]{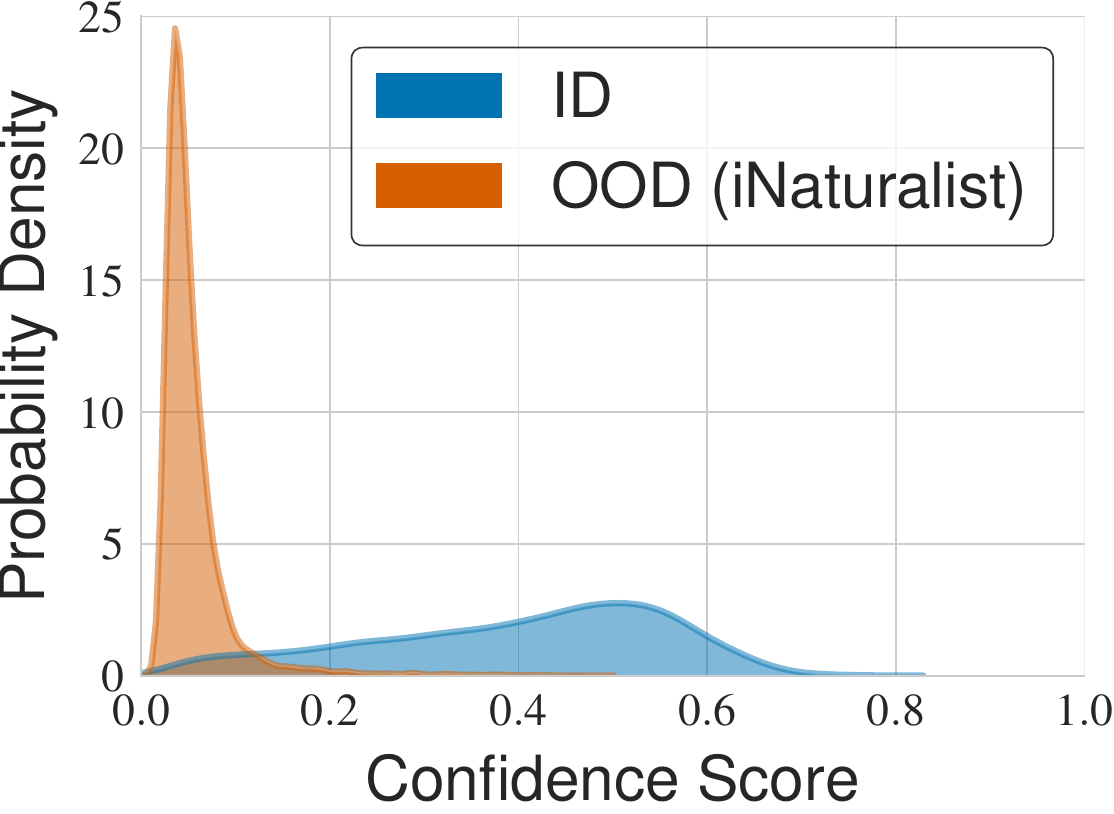} &
            \includegraphics[width=0.19\linewidth]{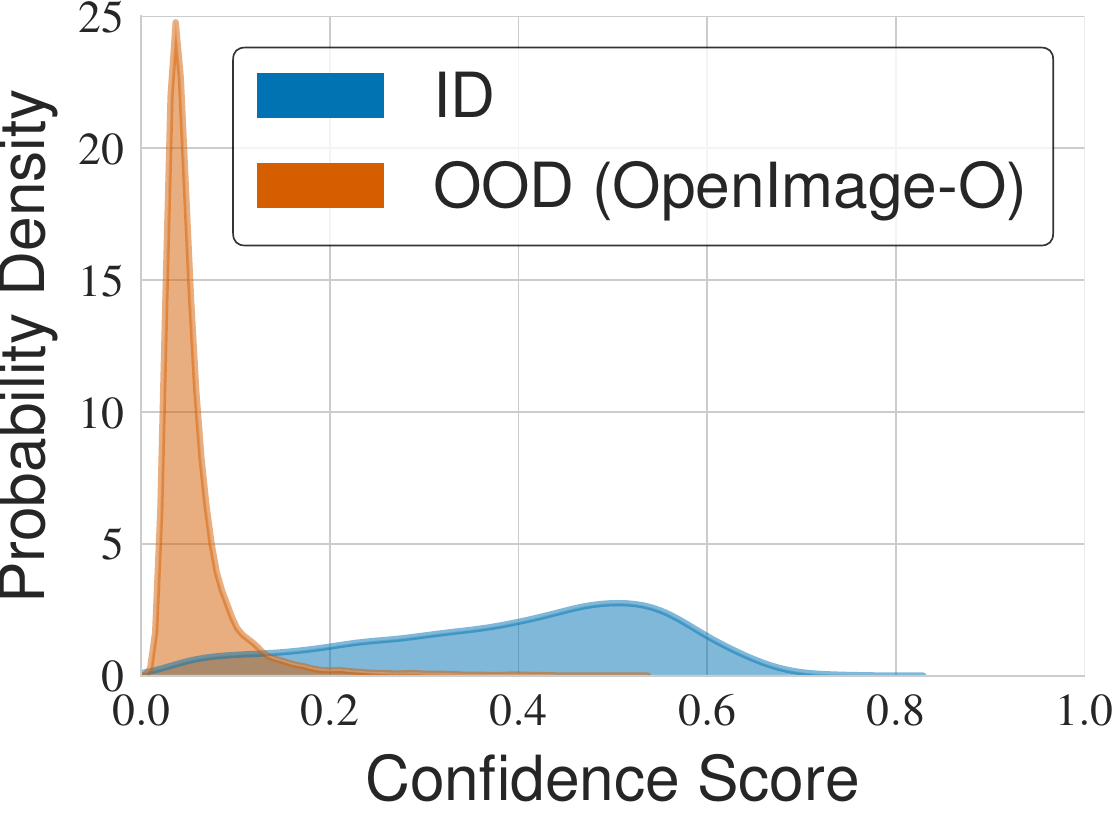} &
            \includegraphics[width=0.19\linewidth]{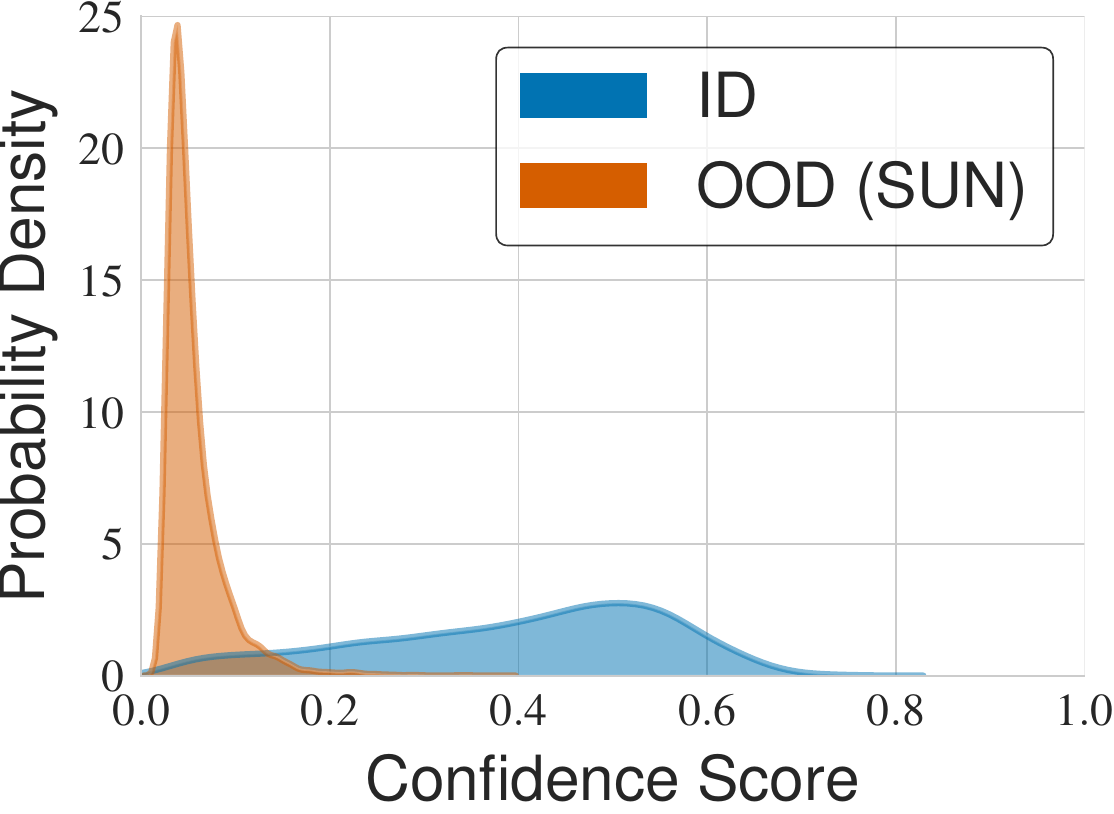} &
            \includegraphics[width=0.19\linewidth]{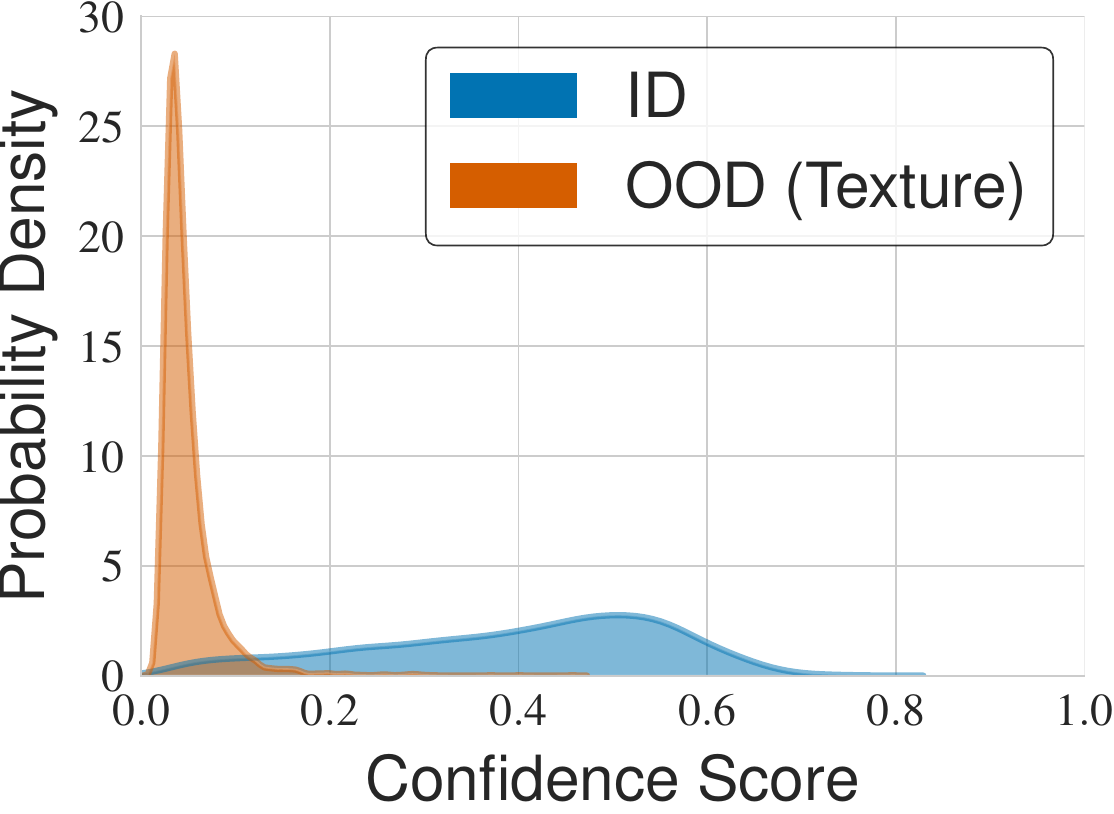}
        \end{tabular}
    }
    \caption{\textbf{Top row:} Visualization of confidence scores of cross-entropy baseline for Car ID and various conventional OOD datasets. \textbf{Bottom row:} Visualization of confidence scores of \ours~for Car ID and various conventional OOD datasets.}
\end{figure}

\newpage

\section{Confidence score plots with spurious and fine-grained OOD datasets:}
\null
\vfill
\label{sec:confidence_plot_spurious}
\subsection{Waterbirds datasets}
\begin{figure}[htbp]
    \centering
    \adjustbox{width=0.95\linewidth}{
    \begin{subfigure}{0.45\linewidth}
        \adjustbox{width=\linewidth,clip}{\includegraphics{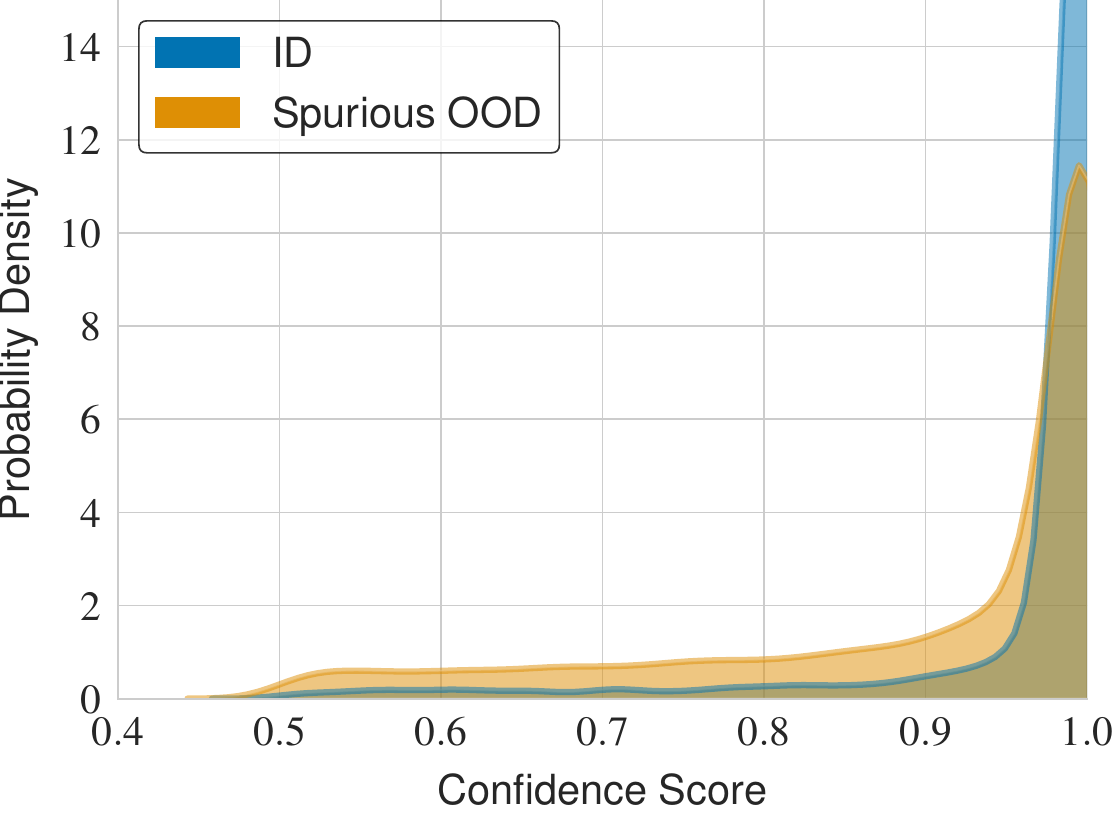}}
        \caption{Cross-entropy baseline}
        \label{fig:waterbirds_baseline_spurious}
    \end{subfigure}
    \hfill
    \begin{subfigure}{0.45\linewidth}
        \adjustbox{width=\linewidth,clip}{\includegraphics{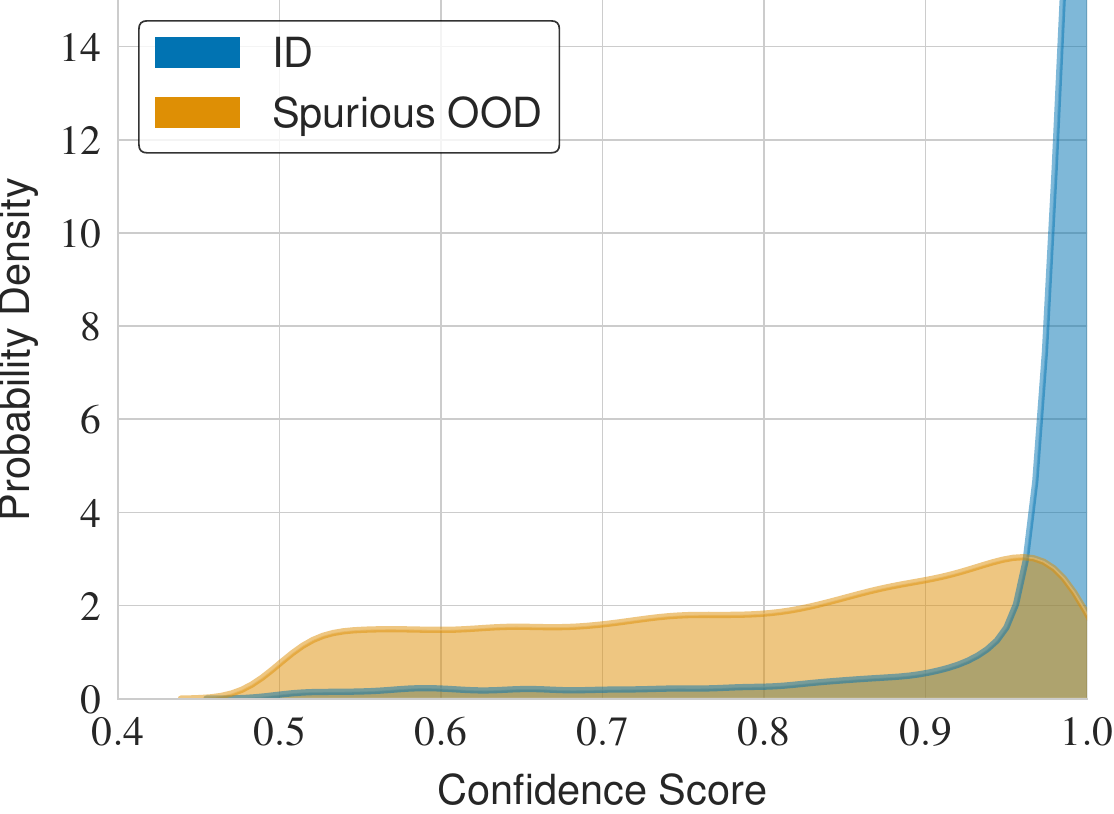}}
        \caption{\ours}
        \label{fig:waterbirds_ios_spurious}
    \end{subfigure}
    }
    \caption{Visualization of confidence scores for Waterbirds ID datasets and corresponding spurious OOD datasets.}
\end{figure}
\null
\vfill
\subsection{CelebA datasets}
\begin{figure}[htbp]
    \centering
    \adjustbox{width=0.95\linewidth}{
    \begin{subfigure}{0.45\linewidth}
        \adjustbox{width=\linewidth,clip}{\includegraphics{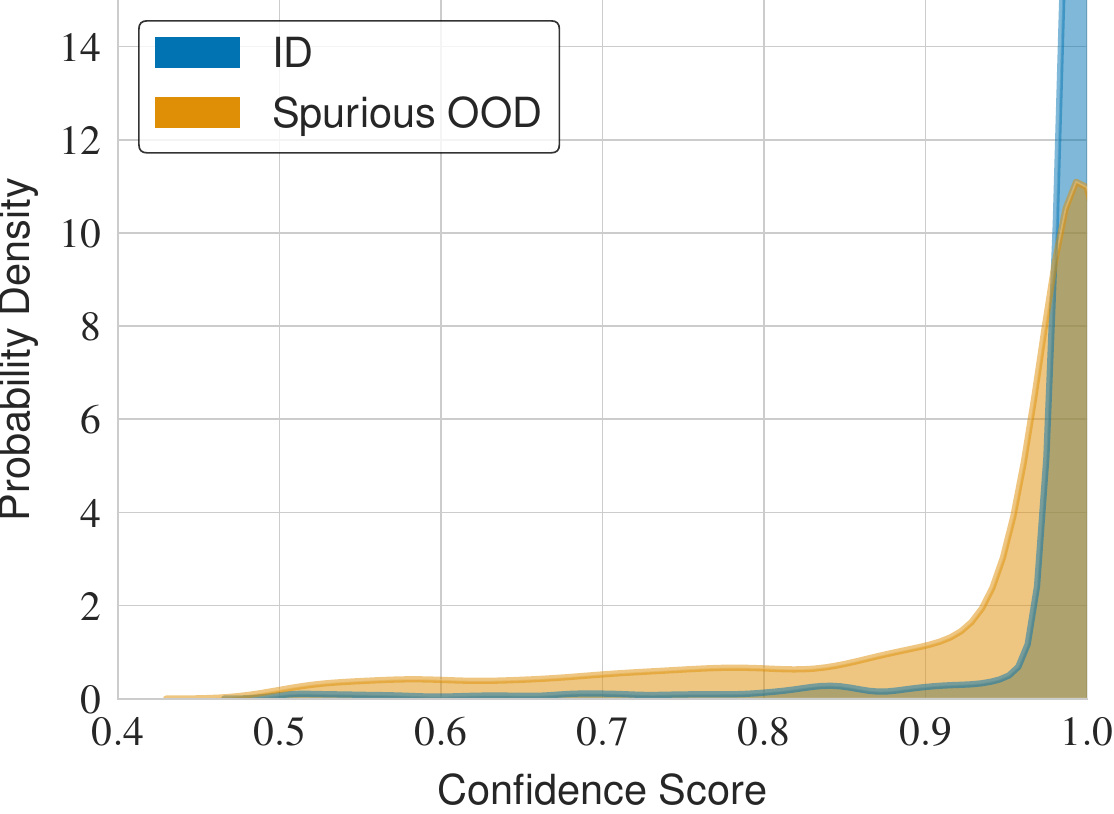}}
        \caption{Cross-entropy baseline}
        \label{fig:celeba_baseline_spurious}
    \end{subfigure}
    \hfill
    \begin{subfigure}{0.45\linewidth}
        \adjustbox{width=\linewidth,clip}{\includegraphics{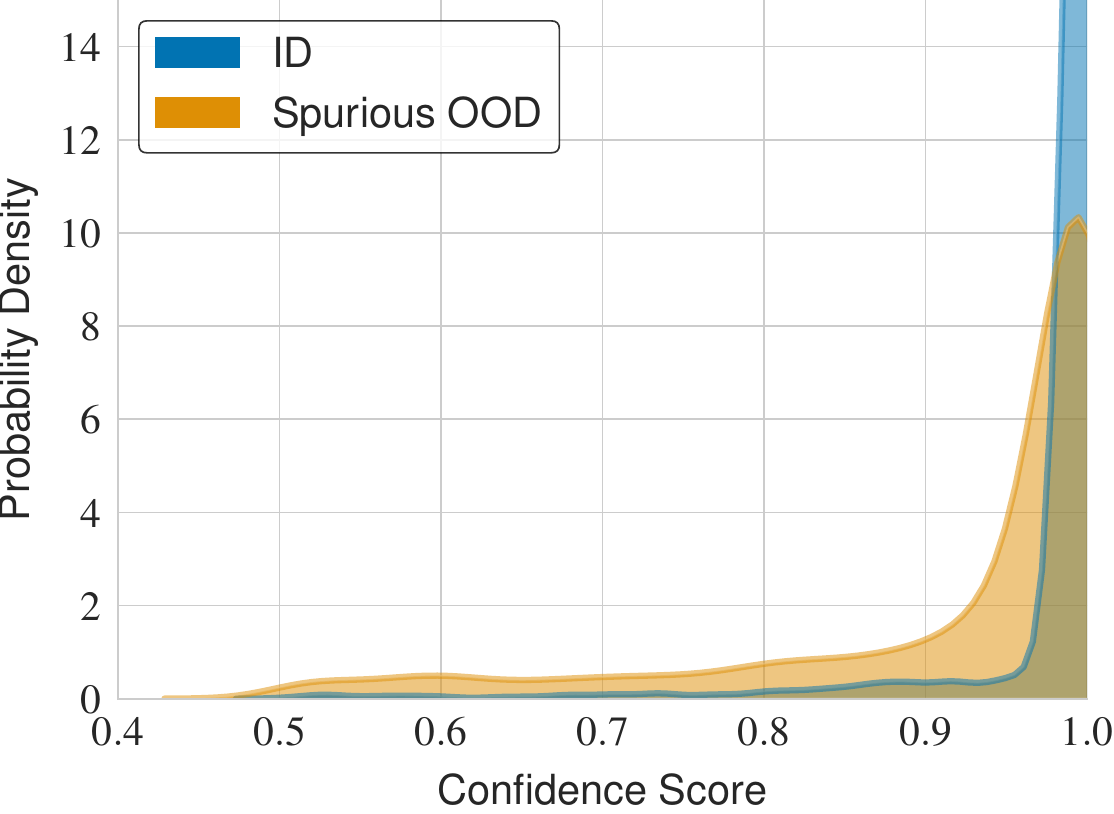}}
        \caption{\ours}
        \label{fig:celeba_ios_spurious}
    \end{subfigure}
    }
    \caption{Visualization of confidence scores for CelebA ID datasets and corresponding spurious OOD datasets.}
\end{figure}
\vfill
\null

\newpage
\vfill
\null
\subsection{Aircraft datasets}
\begin{figure}[htbp]
    \centering
    \adjustbox{width=0.95\linewidth}{
    \begin{subfigure}{0.45\linewidth}
        \adjustbox{width=\linewidth,clip}{\includegraphics{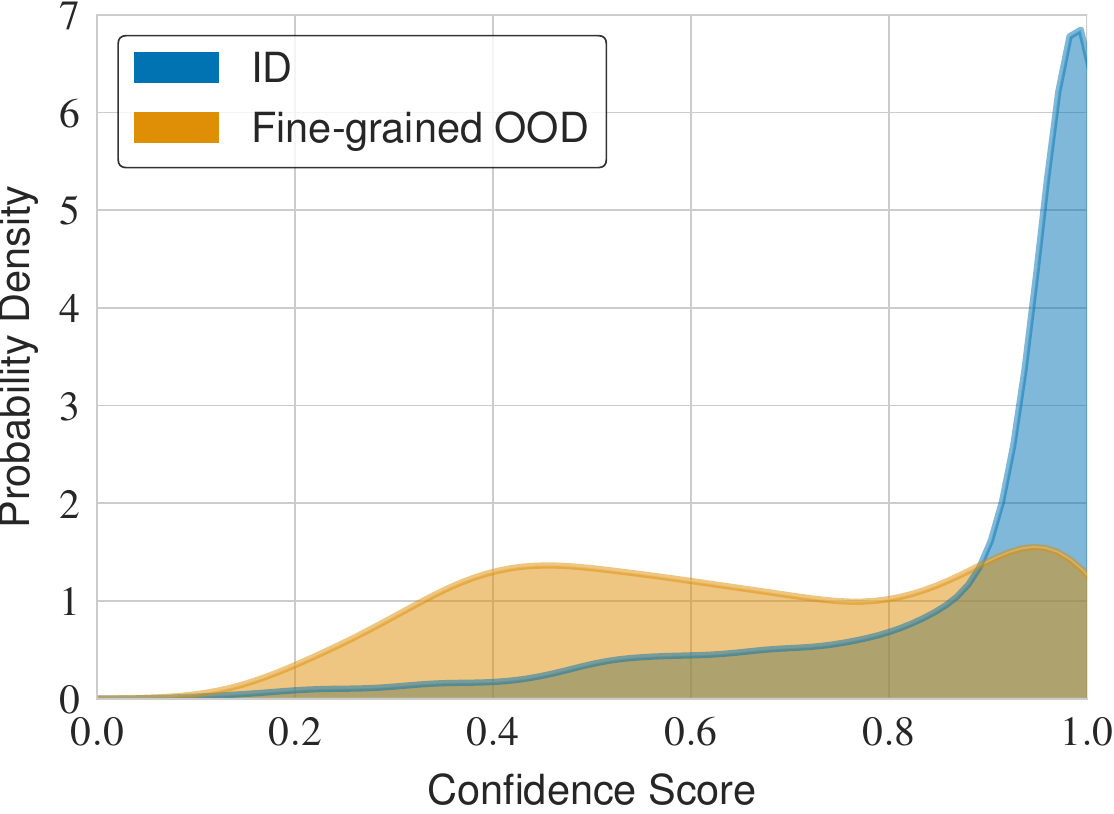}}
        \caption{Cross-entropy baseline}
        \label{fig:aircraft_baseline_spurious}
    \end{subfigure}
    \hfill
    \begin{subfigure}{0.45\linewidth}
        \adjustbox{width=\linewidth,clip}{\includegraphics{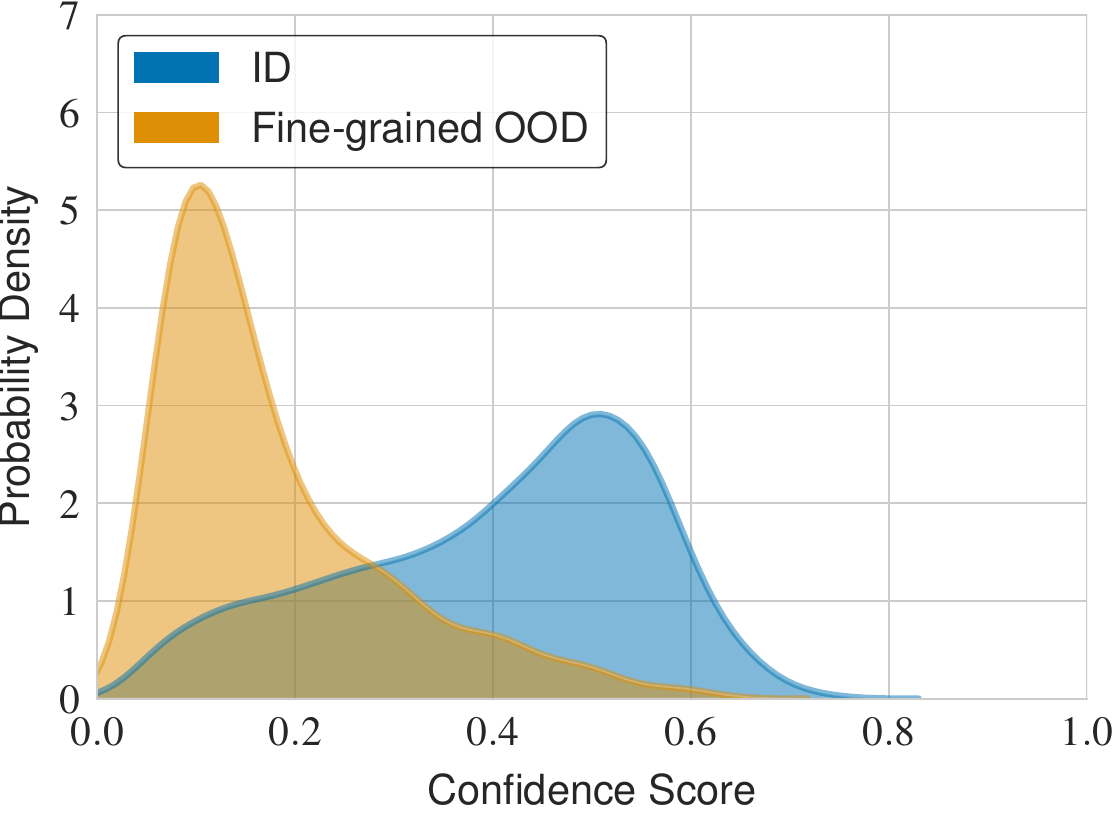}}
        \caption{\ours}
        \label{fig:aircraft_ios_spurious}
    \end{subfigure}
    }
    \caption{Visualization of confidence scores for Aircraft ID datasets and corresponding fine-grained OOD datasets.}
\end{figure}
\vfill
\null
\subsection{Car datasets}
\begin{figure}[htbp]
    \centering
    \adjustbox{width=0.95\linewidth}{
    \begin{subfigure}{0.45\linewidth}
        \adjustbox{width=\linewidth,clip}{\includegraphics{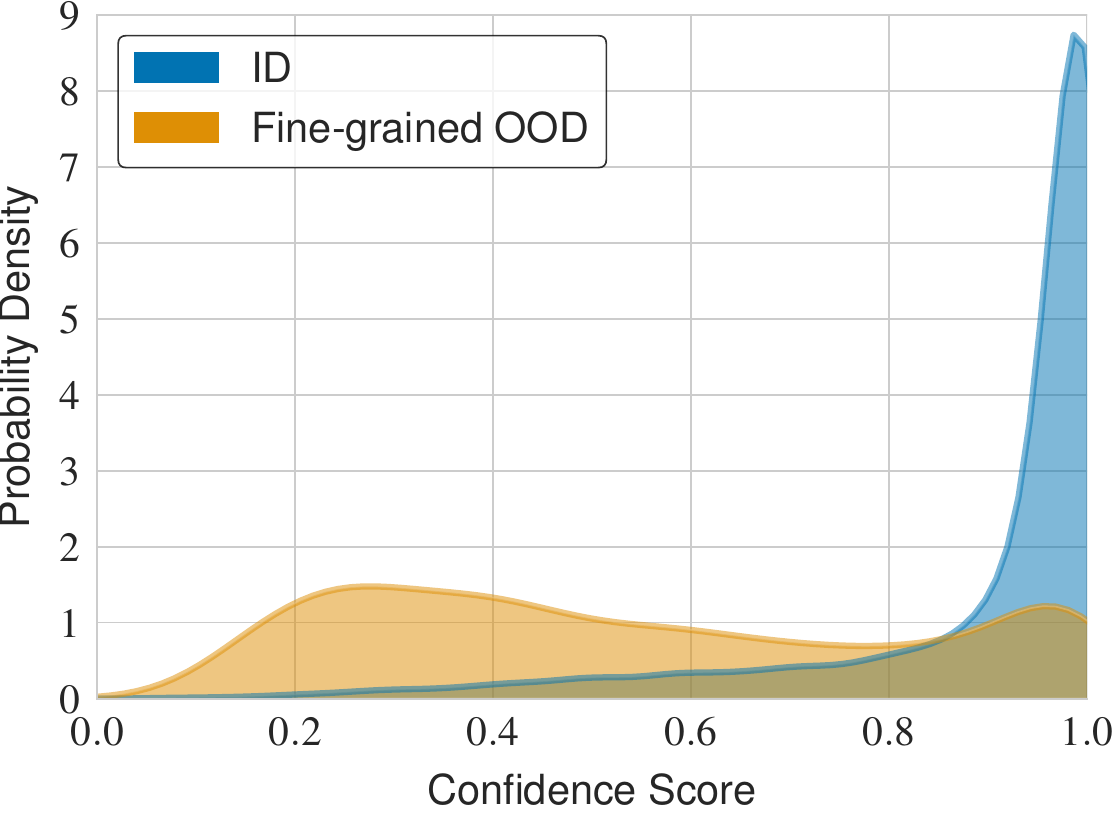}}
        \caption{Cross-entropy baseline}
        \label{fig:cars_baseline_spurious}
    \end{subfigure}
    \hfill
    \begin{subfigure}{0.45\linewidth}
        \adjustbox{width=\linewidth,clip}{\includegraphics{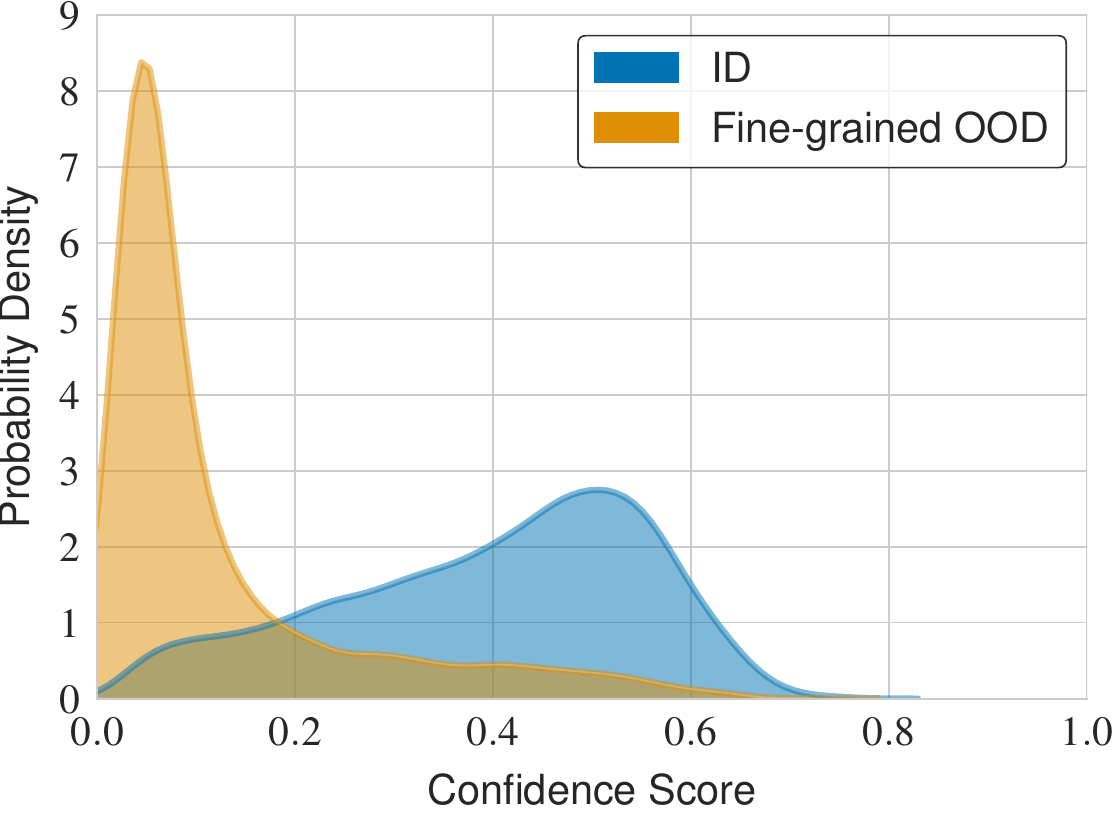}}
        \caption{\ours}
        \label{fig:cars_ios_spurious}
    \end{subfigure}
    }
    \caption{Visualization of confidence scores for Car ID datasets and corresponding fine-grained OOD datasets.}
\end{figure}
\vfill
\null

\newpage
\section{Complete results}
\label{sec:complete_results}

\null
\vfill
\subsection{ODIN vs. I-ODIN} 
\label{appendix:sec:complete_results:odin_compare}
We empirically find that perturbing only the most significant (one) color channel value, rather than all channel values (entire pixels of image), yields the best results across all benchmarks.
\begin{table*}[h]
	\centering
	\adjustbox{max width=1\textwidth}{%
\begin{tabular}{l cc cc cc cc cc cc}
			\toprule
			\multirow{3}{*}{Method} & \multicolumn{6}{c}{\textbf{CIFAR-10 Benchmark}} & \multicolumn{6}{c}{\textbf{CIFAR-100 Benchmark}} \\
			\cmidrule(lr){2-7} \cmidrule(lr){8-13}
			& \multicolumn{2}{c}{CIFAR-100} & \multicolumn{2}{c}{TIN} & \multicolumn{2}{c}{Average} & \multicolumn{2}{c}{CIFAR-10} & \multicolumn{2}{c}{TIN} & \multicolumn{2}{c}{Average} \\
			\cmidrule(lr){2-3} \cmidrule(lr){4-5} \cmidrule(lr){6-7} \cmidrule(lr){8-9} \cmidrule(lr){10-11} \cmidrule(lr){12-13}
			& FPR95$\downarrow$ & AUROC$\uparrow$ & FPR95$\downarrow$ & AUROC$\uparrow$ & FPR95$\downarrow$ & AUROC$\uparrow$ & FPR95$\downarrow$ & AUROC$\uparrow$ & FPR95$\downarrow$ & AUROC$\uparrow$ & FPR95$\downarrow$ & AUROC$\uparrow$ \\
			\midrule
			ODIN & 77.00{\tiny$\pm$5.74} &  82.18{\tiny$\pm$1.87} &  75.38{\tiny$\pm$6.42} &  83.55{\tiny$\pm$1.84} &  76.19{\tiny$\pm$6.08} &  82.87{\tiny$\pm$1.85} & 60.64{\tiny$\pm$0.56} & 78.18{\tiny$\pm$0.14} & 55.19{\tiny$\pm$0.57} & 81.63{\tiny$\pm$0.08} & 57.91{\tiny$\pm$0.51} & 79.90{\tiny$\pm$0.11} \\ 
			\rowcolor{LightGray} I-ODIN & \textbf{61.33{\tiny$\pm$1.18}} & \textbf{86.87{\tiny$\pm$0.14}} & \textbf{50.20{\tiny$\pm$2.39}} & \textbf{88.96{\tiny$\pm$0.25}} & \textbf{55.76{\tiny$\pm$1.78}} & \textbf{87.92{\tiny$\pm$0.17}} & \textbf{59.09{\tiny$\pm$0.62}} & \textbf{79.24{\tiny$\pm$0.10}} & \textbf{51.57{\tiny$\pm$0.66}} & \textbf{82.96{\tiny$\pm$0.05}} & \textbf{55.33{\tiny$\pm$0.64}} & \textbf{81.10{\tiny$\pm$0.07}} \\
			\bottomrule
		\end{tabular}
	}
	\caption{ODIN vs. I-ODIN in near-OOD detection across CIFAR benchmarks of OpenOOD.}
	\label{tab:odin_vs_iodin_cifar_near}
\end{table*}
		
\begin{table*}[h]
	\centering
	\adjustbox{max width=1\textwidth}{%
		\begin{tabular}{l cc cc cc cc cc cc}
			\toprule
			\multirow{3}{*}{Method} & \multicolumn{6}{c}{\textbf{ImageNet-200 Benchmark}} & \multicolumn{6}{c}{\textbf{ImageNet-1k Benchmark}} \\
			\cmidrule(lr){2-7} \cmidrule(lr){8-13}
			& \multicolumn{2}{c}{SSB-Hard} & \multicolumn{2}{c}{NINCO} & \multicolumn{2}{c}{Average} & \multicolumn{2}{c}{SSB-Hard} & \multicolumn{2}{c}{NINCO} & \multicolumn{2}{c}{Average} \\
			\cmidrule(lr){2-3} \cmidrule(lr){4-5} \cmidrule(lr){6-7} \cmidrule(lr){8-9} \cmidrule(lr){10-11} \cmidrule(lr){12-13}
			& FPR95$\downarrow$ & AUROC$\uparrow$ & FPR95$\downarrow$ & AUROC$\uparrow$ & FPR95$\downarrow$ & AUROC$\uparrow$ & FPR95$\downarrow$ & AUROC$\uparrow$ & FPR95$\downarrow$ & AUROC$\uparrow$ & FPR95$\downarrow$ & AUROC$\uparrow$ \\
			\midrule
			ODIN & 73.51{\tiny$\pm$0.38} & 77.19{\tiny$\pm$0.06} &  60.00{\tiny$\pm$0.80} & 83.34{\tiny$\pm$0.12} & 66.76{\tiny$\pm$0.26} & 80.27{\tiny$\pm$0.08} & 
            76.83{\tiny$\pm$0.00} & 71.74{\tiny$\pm$0.00} & 68.16{\tiny$\pm$0.00} & 77.77{\tiny$\pm$0.00} & 72.50{\tiny$\pm$0.00} & 74.75{\tiny$\pm$0.11} \\ 
			\rowcolor{LightGray} I-ODIN & \textbf{69.47{\tiny$\pm$0.37}} & \textbf{80.22{\tiny$\pm$0.01}} & \textbf{49.37{\tiny$\pm$0.75}} & \textbf{85.76{\tiny$\pm$0.08}} & \textbf{59.42{\tiny$\pm$0.48}} & \textbf{82.99{\tiny$\pm$0.04}} & \textbf{76.15{\tiny$\pm$0.00}} & \textbf{72.49{\tiny$\pm$0.00}} & \textbf{58.86{\tiny$\pm$0.00}} & \textbf{80.57{\tiny$\pm$0.00}} & \textbf{67.51{\tiny$\pm$0.00}} & \textbf{76.93{\tiny$\pm$0.00}} \\
			\bottomrule
		\end{tabular}
	}
	\caption{ODIN vs. I-ODIN in near-OOD detection across ImageNet benchmarks of OpenOOD.}
	\label{tab:odin_vs_iodin_imagenet_near}
\end{table*}

\begin{table*}[h]
	\centering
	\adjustbox{max width=1\textwidth}{%

		\begin{tabular}{l cc cc cc cc cc}
			\toprule
			\multirow{3}{*}{Method}				&\multicolumn{2}{c}{\textbf{MINIST}}		&\multicolumn{2}{c}{\textbf{SVHN}}		&\multicolumn{2}{c}{\textbf{Textures}}		&\multicolumn{2}{c}{\textbf{Places365}}		&\multicolumn{2}{c}{\textbf{Average}}		\\
			\cmidrule(lr){2-3} \cmidrule(lr){4-5}\cmidrule(lr){6-7} \cmidrule(lr){8-9} \cmidrule(lr){10-11}
			&FPR95$\downarrow$	&AUROC$\uparrow$	&FPR95$\downarrow$	&AUROC$\uparrow$		&FPR95$\downarrow$		&AUROC$\uparrow$				&FPR95$\downarrow$		&AUROC$\uparrow$					&FPR95$\downarrow$		&AUROC$\uparrow$					\\
			\midrule
			\multicolumn{11}{c}{\emph{CIFAR-10 Benchmark}} \\
			\midrule
			
			ODIN & \textbf{23.83{\tiny$\pm$12.34}} & \textbf{95.24{\tiny$\pm$1.96}} & 68.61{\tiny$\pm$0.52} & 84.58{\tiny$\pm$0.77} & 67.70{\tiny$\pm$11.06} & 86.94{\tiny$\pm$2.26} & 70.36{\tiny$\pm$6.96} & 85.07{\tiny$\pm$1.24} & 57.62{\tiny$\pm$4.24} & 87.96{\tiny$\pm$0.61} \\ 

            \rowcolor{LightGray} I-ODIN & {25.86}{\tiny$\pm$11.23} & {93.45}{\tiny$\pm$1.99} & \textbf{31.80{\tiny$\pm$7.09}} & \textbf{91.55{\tiny$\pm$0.81}} & \textbf{43.36{\tiny$\pm$2.71}} & \textbf{89.78{\tiny$\pm$0.60}} & \textbf{49.20{\tiny$\pm$8.05}} & \textbf{89.16{\tiny$\pm$0.69}} & \textbf{37.55{\tiny$\pm$6.46}} & \textbf{90.99{\tiny$\pm$0.76}} \\
                \midrule
			\multicolumn{11}{c}{\emph{CIFAR-100 Benchmark}} \\
			\midrule
			ODIN & \textbf{45.94{\tiny$\pm$3.29}} &  \textbf{83.79{\tiny$\pm$1.31}} &  67.41{\tiny$\pm$3.88} &  74.54{\tiny$\pm$0.76} &  62.37{\tiny$\pm$2.96} &  \textbf{79.33{\tiny$\pm$1.08}} & 59.71{\tiny$\pm$0.92} &  79.45{\tiny$\pm$0.26} &  58.86{\tiny$\pm$0.79} &  79.28{\tiny$\pm$0.21 } \\
\bottomrule
		\end{tabular}
		
	}
	\caption{ODIN vs. I-ODIN in far-OOD detection across CIFAR benchmarks of OpenOOD.}
	\label{appendix:tab:odin_vs_iodin_cifar_far}
\end{table*}

\begin{table*}[h]
	\centering
	\adjustbox{max width=1\textwidth}{%
		\begin{tabular}{l cc cc cc cc}
			\toprule
			\multirow{3}{*}{Method}	& \multicolumn{2}{c}{\textbf{iNaturalist}} & \multicolumn{2}{c}{\textbf{Textures}} &\multicolumn{2}{c}{\textbf{OpenImage-O}} & \multicolumn{2}{c}{\textbf{Average}}	\\
			\cmidrule(lr){2-3} \cmidrule(lr){4-5}\cmidrule(lr){6-7} \cmidrule(lr){8-9}
			&FPR95$\downarrow$	&AUROC$\uparrow$	&FPR95$\downarrow$	&AUROC$\uparrow$		&FPR95$\downarrow$		&AUROC$\uparrow$				&FPR95$\downarrow$		&AUROC$\uparrow$			\\
			\midrule
			\multicolumn{9}{c}{\emph{ImageNet-200 Benchmark}} \\
			\midrule
			
			ODIN & \textbf{22.39{\tiny$\pm$1.87}} & \textbf{94.37{\tiny$\pm$0.41}} & 42.99{\tiny$\pm$1.56} & \textbf{90.65{\tiny$\pm$0.20}} & 37.30{\tiny$\pm$0.59} & \textbf{90.11{\tiny$\pm$0.15}} & 34.23{\tiny$\pm$1.05} & \textbf{91.71{\tiny$\pm$0.19}} \\ 

            \rowcolor{LightGray} I-ODIN & {24.77{\tiny$\pm$1.91}} & {93.22}{\tiny$\pm$0.43} & 41.27{\tiny$\pm$1.98} & 90.57{\tiny$\pm$0.15} & \textbf{35.24{\tiny$\pm$0.60}} & 89.71{\tiny$\pm$0.19} & \textbf{33.76{\tiny$\pm$1.20}} & 91.16{\tiny$\pm$0.18} \\
                \midrule
			\multicolumn{9}{c}{\emph{ImageNet-1k Benchmark}} \\
			\midrule

            ODIN & {35.98}{\tiny$\pm$0.00} & {91.17}{\tiny$\pm$0.00} & 49.24{\tiny$\pm$0.00} & \textbf{89.00{\tiny$\pm$0.00}} & 46.67{\tiny$\pm$0.00} & 88.23{\tiny$\pm$0.00} & 43.96{\tiny$\pm$0.00} & 89.47{\tiny$\pm$0.00} \\
            
                \rowcolor{LightGray} I-ODIN & \textbf{30.04{\tiny$\pm$0.00}} & \textbf{91.33{\tiny$\pm$0.00}} & \textbf{46.27{\tiny$\pm$0.00}} & 88.37{\tiny$\pm$0.00} & \textbf{37.44{\tiny$\pm$0.00}} & \textbf{89.23{\tiny$\pm$0.00}} & \textbf{37.92{\tiny$\pm$0.00}} & \textbf{89.64{\tiny$\pm$0.00}} \\
			\bottomrule
		\end{tabular}
	}
	\caption{ODIN vs. I-ODIN in far-OOD detection across ImageNet benchmarks of OpenOOD.}
	\label{appendix:tab:odin_vs_iodin_imagenet_far}
\end{table*}
\vfill
\null

\newpage

\null
\vfill
\subsection{CIFAR-10 results}
\begin{table*}[h]
\centering
\renewcommand{\arraystretch}{1.3}
\adjustbox{max width=1\textwidth}{%
\begin{tabular}{lcccccc}
\toprule
Method & \textbf{MNIST} & \textbf{SVHN} & \textbf{iSUN} & \textbf{Texture} & \textbf{Places365} & \textbf{Average} \\
\midrule
\multicolumn{7}{c}{\emph{Posthoc methods}} \\
\midrule
MSP & 23.63{\tiny$\pm$5.81} / 92.63{\tiny$\pm$1.57} & 25.80{\tiny$\pm$1.64} / 91.46{\tiny$\pm$0.40} & 22.86{\tiny$\pm$1.09} / 92.39{\tiny$\pm$0.51} & 34.96{\tiny$\pm$4.66} / 89.89{\tiny$\pm$0.71} & 42.46{\tiny$\pm$3.81} / 88.92{\tiny$\pm$0.47} & 29.94{\tiny$\pm$1.32} / 91.06{\tiny$\pm$0.35} \\
TempScale & 23.53{\tiny$\pm$7.05} / 93.11{\tiny$\pm$1.77} & 26.97{\tiny$\pm$2.65} / 91.66{\tiny$\pm$0.52} & 22.95{\tiny$\pm$1.07} / 92.75{\tiny$\pm$0.50} & 38.16{\tiny$\pm$5.87} / 90.01{\tiny$\pm$0.74} & 45.28{\tiny$\pm$4.52} / 89.11{\tiny$\pm$0.52} & 31.38{\tiny$\pm$1.78} / 91.33{\tiny$\pm$0.42} \\
ODIN & 23.81{\tiny$\pm$12.35} / 95.24{\tiny$\pm$1.96} & 68.61{\tiny$\pm$0.45} / 84.58{\tiny$\pm$0.77} & 27.26{\tiny$\pm$3.81} / 94.63{\tiny$\pm$0.24} & 67.74{\tiny$\pm$11.09} / 86.94{\tiny$\pm$2.26} & 70.35{\tiny$\pm$6.99} / 85.07{\tiny$\pm$1.24} & 51.55{\tiny$\pm$4.03} / 89.29{\tiny$\pm$0.53} \\
MDS & 27.30{\tiny$\pm$3.54} / 90.10{\tiny$\pm$2.41} & 25.95{\tiny$\pm$2.52} / 91.19{\tiny$\pm$0.47} & 32.56{\tiny$\pm$6.52} / 89.86{\tiny$\pm$2.72} & 27.92{\tiny$\pm$4.18} / 92.69{\tiny$\pm$1.06} & 47.66{\tiny$\pm$4.54} / 84.90{\tiny$\pm$2.54} & 32.28{\tiny$\pm$3.97} / 89.75{\tiny$\pm$1.60} \\
MDSEns & \textbf{1.30{\tiny$\pm$0.51}} / \underline{99.17{\tiny$\pm$0.41}} & 74.34{\tiny$\pm$1.04} / 66.56{\tiny$\pm$0.58} & 32.09{\tiny$\pm$12.53} / 93.98{\tiny$\pm$3.08} & 76.07{\tiny$\pm$0.17} / 77.40{\tiny$\pm$0.28} & 94.16{\tiny$\pm$0.33} / 52.47{\tiny$\pm$0.15} & 55.59{\tiny$\pm$2.36} / 77.92{\tiny$\pm$0.57} \\
\bottomrule
\end{tabular}}
\caption{OOD detection (FPR@95$\downarrow$/ AUROC$\uparrow$) results in CIFAR-10 benchmark. Refer to \Cref{appendix:tab:oe_comparison_cifar10} for OE and MixOE results.}
\label{appendix:tab:cifar10_results}
\end{table*}
\vfill
\null

\newpage
\null
\vfill
\subsection{Near-OOD results}
\begin{table*}[h!]
	\centering
        \renewcommand{\arraystretch}{1.1}
	\adjustbox{max width=0.8\textwidth}{%
		\begin{tabular}{l cc cc cc}
			\toprule
			\multirow{3}{*}{Method} & \multicolumn{2}{c}{\textbf{CIFAR-100}}		&\multicolumn{2}{c}{\textbf{Tiny ImageNet (TIN)}}			&\multicolumn{2}{c}{\textbf{Average}}	\\
			\cmidrule(lr){2-3} \cmidrule(lr){4-5}\cmidrule(lr){6-7} 
			&FPR@95$\downarrow$	&AUROC$\uparrow$	&FPR@95$\downarrow$	&AUROC$\uparrow$		&FPR@95$\downarrow$		&AUROC$\uparrow$	\\	
			\midrule
			\multicolumn{7}{c}{\emph{Posthoc methods}} \\
            \midrule
			MSP & 53.08{\tiny$\pm$4.86} &  87.19{\tiny$\pm$0.33} &  43.27{\tiny$\pm$3.00} &  88.87{\tiny$\pm$0.19} &  48.17{\tiny$\pm$3.92} &  88.03{\tiny$\pm$0.25} \\ 
			TempScale & 55.81{\tiny$\pm$5.07} &  87.17{\tiny$\pm$0.40} &  46.11{\tiny$\pm$3.63} &  89.00{\tiny$\pm$0.23} &  50.96{\tiny$\pm$4.32} &  88.09{\tiny$\pm$0.31} \\ 
			ODIN & 77.00{\tiny$\pm$5.74} &  82.18{\tiny$\pm$1.87} &  75.38{\tiny$\pm$6.42} &  83.55{\tiny$\pm$1.84} &  76.19{\tiny$\pm$6.08} &  82.87{\tiny$\pm$1.85} \\ 
			MDS & 52.81{\tiny$\pm$3.62} &  83.59{\tiny$\pm$2.27} &  46.99{\tiny$\pm$4.36} &  84.81{\tiny$\pm$2.53} &  49.90{\tiny$\pm$3.98} &  84.20{\tiny$\pm$2.40} \\ 
			MDSEns & 91.87{\tiny$\pm$0.10} &  61.29{\tiny$\pm$0.23} &  92.66{\tiny$\pm$0.42} &  59.57{\tiny$\pm$0.53} &  92.26{\tiny$\pm$0.20} &  60.43{\tiny$\pm$0.26} \\
\bottomrule
		\end{tabular}
	}
	\caption{Near-OOD detection results on CIFAR-100 and Tiny ImageNet datasets in CIFAR-10 benchmark. OOD detection results for OE and MixOE are omitted to avoid unfair comparison, as Tiny ImageNet is used to incentivize predictive uncertainty during training.}
	\label{appendix:tab:near_ood_results}
\end{table*}
\vfill
\null

\newpage
\null
\vfill
\subsection{CIFAR-100 results}
\begin{table*}[h]
\centering
\renewcommand{\arraystretch}{1.3}
\adjustbox{max width=1\textwidth}{%
\begin{tabular}{lcccccc}
\toprule
Method & \textbf{MNIST} & \textbf{SVHN} & \textbf{iSUN} & \textbf{Texture} & \textbf{Places365} & \textbf{Average} \\
\midrule
\multicolumn{7}{c}{\emph{Posthoc methods}} \\
\midrule
MSP & 57.23{\tiny$\pm$4.67} / 76.08{\tiny$\pm$1.86} & 59.07{\tiny$\pm$2.53} / 78.42{\tiny$\pm$0.89} & 52.64{\tiny$\pm$0.50} / 81.37{\tiny$\pm$0.37} & 61.89{\tiny$\pm$1.28} / 77.32{\tiny$\pm$0.70} & 56.63{\tiny$\pm$0.86} / 79.23{\tiny$\pm$0.29} & 57.49{\tiny$\pm$0.85} / 78.48{\tiny$\pm$0.42} \\

TempScale & 56.04{\tiny$\pm$4.62} / 77.27{\tiny$\pm$1.85} & 57.69{\tiny$\pm$2.67} / 79.79{\tiny$\pm$1.05} & 51.19{\tiny$\pm$0.73} / 82.76{\tiny$\pm$0.50} & 61.55{\tiny$\pm$1.43} / 78.11{\tiny$\pm$0.72} & 56.44{\tiny$\pm$0.93} / 79.80{\tiny$\pm$0.25} & 56.58{\tiny$\pm$0.99} / 79.55{\tiny$\pm$0.51} \\

ODIN & 45.94{\tiny$\pm$3.23} / 83.79{\tiny$\pm$1.30} & 67.43{\tiny$\pm$3.86} / 74.54{\tiny$\pm$0.76} & 36.97{\tiny$\pm$1.24} / 90.10{\tiny$\pm$0.41} & 62.37{\tiny$\pm$2.95} / 79.34{\tiny$\pm$1.08} & 59.73{\tiny$\pm$0.93} / 79.45{\tiny$\pm$0.26} & 54.49{\tiny$\pm$0.72} / 81.44{\tiny$\pm$0.20} \\

MDS & 71.71{\tiny$\pm$2.94} / 67.47{\tiny$\pm$0.81} & 67.22{\tiny$\pm$6.07} / 70.67{\tiny$\pm$6.40} & 77.91{\tiny$\pm$4.91} / 64.07{\tiny$\pm$6.27} & 70.49{\tiny$\pm$2.48} / 76.26{\tiny$\pm$0.69} & 79.60{\tiny$\pm$0.35} / 63.15{\tiny$\pm$0.49} & 73.39{\tiny$\pm$1.69} / 68.32{\tiny$\pm$1.24} \\

MDSEns & 2.83{\tiny$\pm$0.86} / 98.21{\tiny$\pm$0.78} & 82.57{\tiny$\pm$2.58} / 53.76{\tiny$\pm$1.63} & 48.26{\tiny$\pm$3.74} / 91.25{\tiny$\pm$1.35} & 84.94{\tiny$\pm$0.83} / 69.75{\tiny$\pm$1.14} & 96.61{\tiny$\pm$0.17} / 42.27{\tiny$\pm$0.73} & 63.04{\tiny$\pm$0.16} / 71.05{\tiny$\pm$0.33} \\

RMDS & 52.04{\tiny$\pm$6.28} / 79.74{\tiny$\pm$2.49} & 51.65{\tiny$\pm$3.68} / 84.89{\tiny$\pm$1.10} & 53.49{\tiny$\pm$2.41} / 80.80{\tiny$\pm$1.81} & 53.99{\tiny$\pm$1.05} / 83.65{\tiny$\pm$0.51} & 53.57{\tiny$\pm$0.42} / 83.40{\tiny$\pm$0.46} & 52.95{\tiny$\pm$0.13} / 82.50{\tiny$\pm$0.10} \\
\bottomrule

\end{tabular}}
\caption{OOD detection results (FPR@95 $\downarrow$/ AUROC $\uparrow$) in CIFAR-100 benchmark. Refer to \Cref{appendix:tab:oe_comparison_cifar100} for OE and MixOE results.}
\label{appendix:tab:cifar100_results}
\end{table*}
\vfill
\null

\newpage
\null
\vfill
\subsection{ImageNet-100 results}
\label{appendix:sec:imagenet100}
\begin{table}[h]
  \centering
  \resizebox{0.95\linewidth}{!}{%
    \begin{tabular}{@{}lcccc}
      \toprule
      Method & \textbf{iNaturalist} & \textbf{Textures} & \textbf{OpenImage-O} & \textbf{Average} \\
        \midrule
        \multicolumn{5}{c}{\emph{Posthoc methods}} \\
        \midrule
        MSP & 21.09 / 94.87 & 62.82 / 83.45 & 39.82 / 88.08 & 41.24 / 88.80 \\
        TempScale & 20.22 / 95.16 & 63.56 / 83.53 & 39.91 / 88.35 & 41.23 / 89.02 \\
        ODIN & 21.02 / 95.43 & 79.42 / 80.92 & 50.42 / 85.60 & 50.29 / 87.32 \\
        MDS & 88.60 / 52.93 & 60.09 / 85.94 & 80.18 / 66.01 & 76.29 / 68.29 \\
        MDSEns & 88.07 / 57.78 & 73.93 / 75.87 & 86.33 / 64.50 & 82.78 / 66.05 \\
        RMDS & 17.29 / 91.38 & 39.16 / 85.80 & 37.40 / 86.95 & 31.28 / 88.04 \\
\bottomrule 
    \end{tabular}
  }
  \caption{Conventional OOD detection results (FPR@95 $\downarrow$/ AUROC $\uparrow$) in large-scale settings using ImageNet-100 datasets.}
  \label{appendix:tab:imagenet100}
\end{table}
\vfill
\null

\newpage
\null
\vfill
\subsection{ImageNet-100 results (SSB-Hard)}
\label{appendix:sec:imagenet100_ssb}
\begin{table}[h]
  \centering
  \resizebox{0.65\linewidth}{!}{%
    \begin{tabular}{@{}lcc}
      \toprule
      \multirow{2}{*}{Method} & \multicolumn{2}{c}{\textbf{SSB-Hard}} \\
      \cmidrule{2-3}
      & FPR@95 $\downarrow$ & AUROC $\uparrow$ \\
        \midrule
        \multicolumn{3}{c}{\emph{Posthoc methods}} \\
        \midrule
        MSP & 58.33 & 80.94 \\
        TempScale & 58.38 & 80.94 \\
        ODIN & 69.11 & 76.11 \\
        MDS & 81.87 & 60.21 \\
        MDSEns & 92.27 & 51.98 \\
        RMDS & 57.64 & 83.65 \\
        Gram & 89.18 & 54.09 \\
        EBO & 61.33 & 79.23 \\
        GradNorm & 80.71 & 70.53 \\
        ReAct & 60.64 & 79.64 \\
        MLS & 61.31 & 79.62 \\
        KLM & 75.84 & 78.90 \\
        VIM & 59.09 & 79.62 \\ 
        DICE & 66.67 & 72.59 \\
        ASH & 65.16 & 78.17 \\
        SHE & 72.49 & 71.42 \\
        GEN & 58.64 & 80.91 \\
        NNGuide & 62.22 & 79.15 \\
        Relation & 72.31 & 77.25 \\
        SCALE & 65.51 & 79.04 \\
        FDBD & 58.87 & 81.95 \\
        \midrule
        \multicolumn{3}{c}{\emph{Train-time regularization methods}} \\
        \midrule
        ConfBranch & 67.67 & 75.23 \\
        RotPred & 67.64 & 75.72 \\
        G-ODIN & 66.16 & 77.39 \\
        MOS & - & - \\
        VOS & 62.47 & 79.62 \\
        LogitNorm & 59.31 & 80.88 \\
        CIDER & 73.07 & 72.89 \\
        NPOS & 58.80 & 80.18 \\
        Dream-OOD (EBO) & \textbf{55.67} & \underline{83.30} \\ \midrule
        \rowcolor{LightGray} ASCOOD with $\mathcal{L}_\text{OOD}=\mathcal{L}_\text{energy}$ (EBO) & \underline{56.09} & \textbf{83.91} \\
        \rowcolor{LightGray} ASCOOD with $\mathcal{L}_\text{OOD}=\mathcal{L}_\text{KL}$ (EBO) & 57.93 & 81.66 \\
      \bottomrule 
    \end{tabular}
  }
  \caption{near-OOD detection~\citep{vaze2022openset} results (FPR@95 $\downarrow$/ AUROC $\uparrow$) in large-scale settings using ImageNet-100 datasets.}
  \label{appendix:tab:imagenet100_ssb}
\end{table}
\vfill
\null
\clearpage

\newpage
\vfill
\subsection{Conventional OOD detection on Waterbirds benchmark}
\begin{table*}[h]
\centering
\renewcommand{\arraystretch}{1.3}
\adjustbox{max width=1\textwidth}{%
\begin{tabular}{lcccccc}
\toprule
Method & \textbf{NINCO} & \textbf{OpenImage-O} & \textbf{SUN} & \textbf{Texture} & \textbf{iNaturalist} & \textbf{Average} \\
\midrule
\multicolumn{7}{c}{\emph{Posthoc methods}} \\
\midrule
MSP & 57.09{\tiny$\pm$1.77} / 80.25{\tiny$\pm$0.75} & 47.13{\tiny$\pm$2.64} / 82.43{\tiny$\pm$1.12} & 48.33{\tiny$\pm$1.06} / 82.51{\tiny$\pm$0.62} & 34.90{\tiny$\pm$2.11} / 87.70{\tiny$\pm$0.45} & 66.92{\tiny$\pm$1.46} / 66.33{\tiny$\pm$1.11} & 50.87{\tiny$\pm$1.24} / 79.84{\tiny$\pm$0.58} \\

TempScale & 57.09{\tiny$\pm$1.78} / 80.26{\tiny$\pm$0.75} & 47.13{\tiny$\pm$2.64} / 82.43{\tiny$\pm$1.11} & 48.31{\tiny$\pm$1.07} / 82.51{\tiny$\pm$0.62} & 34.90{\tiny$\pm$2.11} / 87.70{\tiny$\pm$0.45} & 66.82{\tiny$\pm$1.50} / 66.34{\tiny$\pm$1.11} & 50.85{\tiny$\pm$1.24} / 79.85{\tiny$\pm$0.58} \\

ODIN & 100.00{\tiny$\pm$0.00} / 83.15{\tiny$\pm$0.54} & 60.32{\tiny$\pm$28.15} / 89.32{\tiny$\pm$0.81} & 61.18{\tiny$\pm$27.47} / 86.97{\tiny$\pm$1.47} & 23.26{\tiny$\pm$0.92} / 95.18{\tiny$\pm$0.28} & 100.00{\tiny$\pm$0.00} / 71.14{\tiny$\pm$0.89} & 68.95{\tiny$\pm$4.82} / 85.15{\tiny$\pm$0.11} \\

MDS & 48.86{\tiny$\pm$3.09} / 92.57{\tiny$\pm$0.56} & 46.44{\tiny$\pm$3.35} / 92.81{\tiny$\pm$0.56} & 91.93{\tiny$\pm$1.06} / 76.82{\tiny$\pm$0.35} & 36.38{\tiny$\pm$6.17} / 94.53{\tiny$\pm$1.11} & 18.24{\tiny$\pm$2.48} / 96.66{\tiny$\pm$0.37} & 48.37{\tiny$\pm$2.23} / 90.68{\tiny$\pm$0.37} \\

MDSEns & 56.79{\tiny$\pm$0.57} / 86.10{\tiny$\pm$0.20} & 60.51{\tiny$\pm$0.05} / 85.62{\tiny$\pm$0.10} & 59.42{\tiny$\pm$0.48} / 84.09{\tiny$\pm$0.23} & 33.29{\tiny$\pm$0.52} / 93.76{\tiny$\pm$0.05} & 38.64{\tiny$\pm$1.28} / 91.09{\tiny$\pm$0.41} & 49.73{\tiny$\pm$0.44} / 88.13{\tiny$\pm$0.19} \\

RMDS & 85.76{\tiny$\pm$3.27} / 70.98{\tiny$\pm$2.01} & 81.95{\tiny$\pm$3.18} / 71.13{\tiny$\pm$2.43} & 74.09{\tiny$\pm$2.64} / 75.90{\tiny$\pm$0.67} & 75.93{\tiny$\pm$3.64} / 72.41{\tiny$\pm$2.22} & 83.27{\tiny$\pm$2.28} / 57.45{\tiny$\pm$4.35} & 80.20{\tiny$\pm$2.41} / 69.57{\tiny$\pm$2.08} \\
\bottomrule
\end{tabular}}
\caption{OOD detection results (FPR@95 $\downarrow$/ AUROC $\uparrow$) on various conventional OOD datasets in Waterbirds benchmark.}
\label{appendix:tab:waterbirds_results}
\end{table*}
\subsection{Spurious OOD detection on Waterbirds benchmark with MSP}
\begin{table*}[h]
\centering
\renewcommand{\arraystretch}{1.3}
\adjustbox{max width=1\textwidth}{%
\begin{tabular}{lcccccc}
\toprule
Method & FPR@95$\downarrow$ & AUROC$\uparrow$ \\
\midrule
Cross-entropy & 60.41{\tiny$\pm$1.52} & 77.18{\tiny$\pm$0.66} \\
\rowcolor{LightGray} \textbf{\ours} & \textbf{24.11{\tiny$\pm$0.57}} & \textbf{90.75{\tiny$\pm$0.08}} \\
\bottomrule
\end{tabular}}
\caption{Spurious OOD detection results in Waterbirds benchmark with MSP scoring.}
\label{appendix:tab:waterbirds_spurious_results}
\end{table*}

\clearpage
\vfill
\subsection{Conventional OOD detection on CelebA benchmark}
\begin{table*}[h]
\centering
\renewcommand{\arraystretch}{1.3}
\adjustbox{max width=1\textwidth}{%
\begin{tabular}{lcccccc}
\toprule
Method & \textbf{NINCO} & \textbf{OpenImage-O} & \textbf{SUN} & \textbf{Texture} & \textbf{iNaturalist} & \textbf{Average} \\
\midrule
\multicolumn{7}{c}{\emph{Posthoc methods}} \\
\midrule
MSP & 11.01{\tiny$\pm$0.73} / 96.12{\tiny$\pm$0.27} & 9.11{\tiny$\pm$0.74} / 96.86{\tiny$\pm$0.31} & 9.28{\tiny$\pm$1.50} / 96.42{\tiny$\pm$0.91} & 9.32{\tiny$\pm$1.09} / 97.05{\tiny$\pm$0.26} & 8.38{\tiny$\pm$1.03} / 96.99{\tiny$\pm$0.54} & 9.42{\tiny$\pm$0.97} / 96.69{\tiny$\pm$0.45} \\

TempScale & 11.01{\tiny$\pm$0.73} / 96.12{\tiny$\pm$0.27} & 9.11{\tiny$\pm$0.74} / 96.86{\tiny$\pm$0.31} & 9.28{\tiny$\pm$1.50} / 96.42{\tiny$\pm$0.91} & 9.32{\tiny$\pm$1.09} / 97.05{\tiny$\pm$0.26} & 8.38{\tiny$\pm$1.03} / 96.99{\tiny$\pm$0.54} & 9.42{\tiny$\pm$0.97} / 96.69{\tiny$\pm$0.45} \\

ODIN & 1.83{\tiny$\pm$0.41} / 99.50{\tiny$\pm$0.25} & 0.94{\tiny$\pm$0.20} / 99.71{\tiny$\pm$0.11} & 0.59{\tiny$\pm$0.34} / 99.85{\tiny$\pm$0.08} & 0.73{\tiny$\pm$0.31} / 99.59{\tiny$\pm$0.13} & 0.45{\tiny$\pm$0.20} / 99.81{\tiny$\pm$0.09} & 0.91{\tiny$\pm$0.28} / 99.69{\tiny$\pm$0.13} \\

MDS & 27.69{\tiny$\pm$9.50} / 95.37{\tiny$\pm$1.48} & 27.91{\tiny$\pm$8.93} / 95.37{\tiny$\pm$1.32} & 12.26{\tiny$\pm$12.44} / 97.30{\tiny$\pm$2.49} & 34.11{\tiny$\pm$14.72} / 94.77{\tiny$\pm$1.75} & 57.76{\tiny$\pm$14.71} / 87.45{\tiny$\pm$6.28} & 32.15{\tiny$\pm$10.27} / 94.05{\tiny$\pm$2.39} \\

MDSEns & 5.38{\tiny$\pm$0.21} / 98.87{\tiny$\pm$0.03} & 6.71{\tiny$\pm$0.12} / 98.65{\tiny$\pm$0.02} & 1.56{\tiny$\pm$0.10} / 99.64{\tiny$\pm$0.02} & 3.22{\tiny$\pm$0.07} / 99.28{\tiny$\pm$0.02} & 2.45{\tiny$\pm$0.03} / 99.43{\tiny$\pm$0.01} & 3.87{\tiny$\pm$0.07} / 99.17{\tiny$\pm$0.02} \\

RMDS & 30.52{\tiny$\pm$9.83} / 91.57{\tiny$\pm$1.90} & 27.59{\tiny$\pm$7.92} / 92.00{\tiny$\pm$1.67} & 43.20{\tiny$\pm$24.73} / 87.92{\tiny$\pm$6.43} & 28.15{\tiny$\pm$5.65} / 92.28{\tiny$\pm$0.77} & 23.92{\tiny$\pm$12.72} / 92.78{\tiny$\pm$3.20} & 30.68{\tiny$\pm$11.44} / 91.31{\tiny$\pm$2.68} \\
\bottomrule
\end{tabular}}
\caption{OOD detection results (FPR@95 $\downarrow$/ AUROC $\uparrow$) on various conventional OOD datasets in CelebA benchmark.}
\label{appendix:tab:celeba_results}
\end{table*}
\subsection{Spurious OOD detection on CelebA benchmark with MSP}
\begin{table*}[h]
\centering
\renewcommand{\arraystretch}{1.3}
\adjustbox{max width=1\textwidth}{%
\begin{tabular}{lcccccc}
\toprule
Method & FPR@95$\downarrow$ & AUROC$\uparrow$ \\
\midrule
Cross-entropy & 56.00{\tiny$\pm$2.73} & 82.53{\tiny$\pm$1.19} \\
\rowcolor{LightGray} \textbf{\ours} & \textbf{49.56{\tiny$\pm$0.09}} & \textbf{83.30{\tiny$\pm$0.37}} \\
\bottomrule
\end{tabular}}
\caption{Spurious OOD detection results in CelebA benchmark with MSP scoring.}
\label{appendix:tab:celeba_spurious_results}
\end{table*}

\newpage
\subsection{Conventional OOD detection on Aircraft benchmark}
\label{appendix:sec:complete_aircraft}
\begin{table*}[h]
\centering
\renewcommand{\arraystretch}{1.3}
\adjustbox{max width=1\textwidth}{%
\begin{tabular}{lcccccc}
\toprule
Method & \textbf{NINCO} & \textbf{OpenImage-O} & \textbf{SUN} & \textbf{Texture} & \textbf{iNaturalist} & \textbf{Average} \\
\midrule
\multicolumn{7}{c}{\emph{Posthoc methods}} \\
\midrule
MSP & 27.21{\tiny$\pm$1.68} / 93.92{\tiny$\pm$0.45} & 29.67{\tiny$\pm$1.49} / 93.31{\tiny$\pm$0.35} & 13.86{\tiny$\pm$0.98} / 97.53{\tiny$\pm$0.17} & 34.94{\tiny$\pm$2.37} / 92.50{\tiny$\pm$0.52} & 29.98{\tiny$\pm$1.25} / 92.48{\tiny$\pm$0.69} & 27.13{\tiny$\pm$0.56} / 93.95{\tiny$\pm$0.04} \\

TempScale & 23.43{\tiny$\pm$1.53} / 95.06{\tiny$\pm$0.42} & 26.41{\tiny$\pm$1.70} / 94.44{\tiny$\pm$0.36} & 8.39{\tiny$\pm$1.19} / 98.12{\tiny$\pm$0.15} & 32.62{\tiny$\pm$2.35} / 93.40{\tiny$\pm$0.47} & 25.65{\tiny$\pm$1.12} / 93.94{\tiny$\pm$0.51} & 23.30{\tiny$\pm$0.51} / 94.99{\tiny$\pm$0.08} \\

ODIN & 18.23{\tiny$\pm$5.06} / 96.07{\tiny$\pm$0.84} & 25.63{\tiny$\pm$4.61} / 94.90{\tiny$\pm$0.71} & 4.58{\tiny$\pm$0.83} / 98.84{\tiny$\pm$0.16} & 46.55{\tiny$\pm$6.48} / 93.23{\tiny$\pm$0.48} & 11.85{\tiny$\pm$3.12} / 97.15{\tiny$\pm$0.67} & 21.37{\tiny$\pm$3.66} / 96.04{\tiny$\pm$0.49} \\

MDS & 3.79{\tiny$\pm$0.50} / 98.94{\tiny$\pm$0.11} & 2.61{\tiny$\pm$0.02} / 99.21{\tiny$\pm$0.08} & 20.96{\tiny$\pm$2.13} / 95.16{\tiny$\pm$0.43} & 0.79{\tiny$\pm$0.13} / 99.82{\tiny$\pm$0.03} & 0.92{\tiny$\pm$0.25} / 99.70{\tiny$\pm$0.11} & 5.82{\tiny$\pm$0.42} / 98.57{\tiny$\pm$0.13} \\

MDSEns & 25.16{\tiny$\pm$1.46} / 95.09{\tiny$\pm$0.22} & 25.12{\tiny$\pm$2.04} / 95.51{\tiny$\pm$0.27} & 54.06{\tiny$\pm$2.01} / 90.02{\tiny$\pm$0.23} & 10.16{\tiny$\pm$0.32} / 97.95{\tiny$\pm$0.06} & 5.39{\tiny$\pm$0.30} / 98.66{\tiny$\pm$0.05} & 23.98{\tiny$\pm$1.17} / 95.45{\tiny$\pm$0.17} \\

RMDS & 3.51{\tiny$\pm$0.40} / 98.97{\tiny$\pm$0.08} & 3.11{\tiny$\pm$0.18} / 99.25{\tiny$\pm$0.04} & 4.93{\tiny$\pm$0.47} / 97.76{\tiny$\pm$0.16} & 2.02{\tiny$\pm$0.31} / 99.66{\tiny$\pm$0.03} & 2.01{\tiny$\pm$0.28} / 99.55{\tiny$\pm$0.07} & 3.12{\tiny$\pm$0.30} / 99.04{\tiny$\pm$0.05} \\
\bottomrule
\end{tabular}}
\caption{OOD detection results (FPR@95 $\downarrow$/ AUROC $\uparrow$) on various conventional OOD datasets in Aircraft benchmark.}
\label{appendix:tab:aircraft_results}
\end{table*}
\subsection{Fine-grained OOD detection on Aircraft benchmark with MSP}
\begin{table*}[h]
\centering
\renewcommand{\arraystretch}{1.3}
\adjustbox{max width=1\textwidth}{%
\begin{tabular}{lcccccc}
\toprule
Method & FPR@95$\downarrow$ & AUROC$\uparrow$ \\
\midrule
Cross-entropy & 63.79{\tiny$\pm$5.71} & 80.53{\tiny$\pm$1.61} \\
\rowcolor{LightGray} \textbf{\ours} & \textbf{52.97{\tiny$\pm$0.98}} & \textbf{86.22{\tiny$\pm$0.09}} \\
\bottomrule
\end{tabular}}
\caption{Fine-grained OOD detection results in Aircraft benchmark with MSP scoring.}
\label{appendix:tab:aircraft_finegrained_results}
\end{table*}

\newpage
\subsection{Conventional OOD detection on Car benchmark}
\label{appendix:sec:complete_car}

\begin{table*}[h]
\centering
\renewcommand{\arraystretch}{1.3}
\adjustbox{max width=1\textwidth}{%
\begin{tabular}{lcccccc}
\toprule
Method & \textbf{NINCO} & \textbf{OpenImage-O} & \textbf{SUN} & \textbf{Texture} & \textbf{iNaturalist} & \textbf{Average} \\
\midrule
\multicolumn{7}{c}{\emph{Posthoc methods}} \\
\midrule

MSP & 18.04{\tiny$\pm$5.67} / 96.23{\tiny$\pm$1.12} & 23.95{\tiny$\pm$8.61} / 95.29{\tiny$\pm$1.62} & 4.48{\tiny$\pm$1.73} / 99.15{\tiny$\pm$0.25} & 19.88{\tiny$\pm$0.82} / 95.99{\tiny$\pm$0.10} & 24.39{\tiny$\pm$18.59} / 95.88{\tiny$\pm$3.10} & 18.14{\tiny$\pm$6.55} / 96.51{\tiny$\pm$1.17} \\

TempScaling & 15.31{\tiny$\pm$5.33} / 96.77{\tiny$\pm$1.10} & 21.68{\tiny$\pm$8.65} / 95.85{\tiny$\pm$1.55} & 2.74{\tiny$\pm$0.96} / 99.41{\tiny$\pm$0.16} & 19.10{\tiny$\pm$1.23} / 96.31{\tiny$\pm$0.14} & 21.27{\tiny$\pm$18.31} / 96.51{\tiny$\pm$2.80} & 16.02{\tiny$\pm$6.60} / 96.97{\tiny$\pm$1.10} \\

ODIN & 13.40{\tiny$\pm$6.54} / 97.27{\tiny$\pm$1.17} & 26.22{\tiny$\pm$12.37} / 95.50{\tiny$\pm$1.74} & 0.74{\tiny$\pm$0.18} / 99.78{\tiny$\pm$0.06} & 33.53{\tiny$\pm$6.80} / 94.92{\tiny$\pm$0.61} & 10.19{\tiny$\pm$10.64} / 98.13{\tiny$\pm$1.68} & 16.82{\tiny$\pm$7.28} / 97.12{\tiny$\pm$1.05} \\

MDS & 13.16{\tiny$\pm$1.74} / 97.07{\tiny$\pm$0.39} & 12.02{\tiny$\pm$1.26} / 97.70{\tiny$\pm$0.25} & 32.44{\tiny$\pm$1.72} / 90.61{\tiny$\pm$1.01} & 1.44{\tiny$\pm$0.29} / 99.70{\tiny$\pm$0.05} & 5.66{\tiny$\pm$1.93} / 98.77{\tiny$\pm$0.42} & 12.94{\tiny$\pm$1.27} / 96.77{\tiny$\pm$0.41} \\

MDSEns & 34.59{\tiny$\pm$0.59} / 91.03{\tiny$\pm$0.02} & 41.52{\tiny$\pm$0.55} / 90.43{\tiny$\pm$0.10} & 36.89{\tiny$\pm$1.28} / 90.98{\tiny$\pm$0.25} & 17.97{\tiny$\pm$0.26} / 96.23{\tiny$\pm$0.06} & 15.06{\tiny$\pm$0.54} / 96.31{\tiny$\pm$0.11} & 29.20{\tiny$\pm$0.39} / 93.00{\tiny$\pm$0.10} \\

RMDS & 5.15{\tiny$\pm$0.17} / 98.27{\tiny$\pm$0.04} & 4.78{\tiny$\pm$0.05} / 98.67{\tiny$\pm$0.07} & 6.37{\tiny$\pm$0.23} / 96.81{\tiny$\pm$0.10} & 2.37{\tiny$\pm$0.21} / 99.62{\tiny$\pm$0.03} & 4.60{\tiny$\pm$0.20} / 98.46{\tiny$\pm$0.14} & 4.65{\tiny$\pm$0.13} / 98.37{\tiny$\pm$0.05} \\
\bottomrule
\end{tabular}}
\caption{OOD detection results (FPR@95 $\downarrow$/ AUROC $\uparrow$) on various conventional OOD datasets in Car benchmark.}
\label{appendix:tab:car_results}
\end{table*}

\subsection{Fine-grained OOD detection on Car benchmark with MSP}
\begin{table*}[h]
\centering
\renewcommand{\arraystretch}{1.3}
\adjustbox{max width=1\textwidth}{%
\begin{tabular}{lcccccc}
\toprule
Method & FPR@95$\downarrow$ & AUROC$\uparrow$ \\
\midrule
Cross-entropy & 58.17{\tiny$\pm$0.99} & 87.12{\tiny$\pm$0.16} \\
\rowcolor{LightGray} \textbf{\ours} & \textbf{51.37{\tiny$\pm$2.82}} & \textbf{90.51{\tiny$\pm$0.36}} \\
\bottomrule
\end{tabular}}
\caption{Fine-grained OOD detection results in Car benchmark with MSP scoring.}
\label{appendix:tab:car_finegrained_results}
\end{table*}
\vfill
\null